\documentclass{osudissert96}
\pdfoutput=1
\def\BibTeX{{\rm B\kern-.05em{\sc i\kern-.025em b}\kern-.08em
    T\kern-.1667em\lower.7ex\hbox{E}\kern-.125emX}}

\renewcommand\typesetChapterTitle[1]{\uppercase{#1}}

\usepackage{amsmath}
\usepackage{amssymb}
\usepackage{bbm}
\usepackage{latexsym}
\usepackage{color}
\usepackage{versions}

\usepackage{caption}
\usepackage{subcaption}

\usepackage{algorithm}
\usepackage{algpseudocode}

\usepackage{float}

\floatstyle{ruled}
\newfloat{subroutine}{htbp}{loa}
\floatname{subroutine}{Subroutine}

\usepackage{mathtools}
\usepackage{amsfonts}
\usepackage{hyperref}
\usepackage{amsthm}
\usepackage{thmtools,thm-restate}

\algdef{SE}[myFOR]{myFor}{myEndFor}[1]{\algorithmicfor\ #1\ \algorithmicdo}{\algorithmicend\ \algorithmicfor}%

\def\final{1}  

\ifnum\final=0  
\newcommand{\lnote}[1]{[{\small Luis: \bf #1}]}
\newcommand{\jnote}[1]{[{\small Joe: \bf #1}]}
\newcommand{\vnote}[1]{[{\small Navin: \bf #1}]}
\newcommand{\nnote}[1]{[{\small Navin: \bf #1}]}
\newcommand{\anonnote}[1]{[{\small anon: \bf #1}]}
\newcommand{\sidecomment}[1]{\marginpar{\tiny #1}}
\newcommand{\details}[1]{{\color{blue}\ [[#1]] }}
\else 
\newcommand{\lnote}[1]{}
\newcommand{\jnote}[1]{}
\newcommand{\anonnote}[1]{}
\newcommand{\vnote}[1]{}
\newcommand{\nnote}[1]{}
\newcommand{\sidecomment}[1]{}
\newcommand{\details}[1]{}
\fi  

\newcommand{\myparagraph}[1]{\paragraph{#1}}

\newenvironment{proofidea}{\noindent{\textit{Proof idea.}}}{\hfill$\square$\medskip}

\renewcommand{\dim}{{n}}
\newcommand{\an}{an}

\newcommand{\RR}{\ensuremath{\mathbb{R}}}
\newcommand{\QQ}{\ensuremath{\mathbb{Q}}}

\newcommand{\ones}{\ensuremath{\mathbbm{1}}}
\newcommand{\expectation}{\operatorname{\mathbb{E}}}
\newcommand{\e}{\expectation}
\newcommand{\conv}{\operatorname{conv}}
\newcommand{\suchthat}{\mathrel{:}}
\newcommand{\norm}[1]{{\lVert#1\rVert}}
\newcommand{\diag}{\operatorname{diag}}
\newcommand{\eps}{\epsilon}
\newcommand{\ind}{\mathbbm{1}}

\newcommand{\cov}{\operatorname{Cov}}
\newcommand{\density}{\rho}

\newcommand{\stable}{\operatorname{Stable}}

\newcommand{\optionalbreak}{}

\includeversion{vstd}

\newcommand{\poly}{\operatorname{poly}}
\newcommand{\polylog}{\operatorname{polylog}}
\newcommand{\abs}[1]{\lvert#1\rvert}
\newcommand{\lrabs}[1]{\left \lvert #1 \right \rvert}
\newcommand{\vol}{\operatorname{vol}}
\newcommand{\expdist}[1]{\mathrm{Exp({#1})}}
\newcommand{\noname}[1]{}

\newcommand{\polar}{\circ}

\newcommand{\dimension}{n}

\newcommand{\measure}{\mathbb{P}}
\newcommand{\pr}{\mathbb{P}}
\newcommand{\inner}[2]{\langle{#1},{#2}\rangle}

\newcommand{\E}{\mathbb{E}}
\newcommand{\R}{\mathbb{R}}
\newcommand{\var}{\operatorname{var}}

\newcommand{\Var}{\var}

\newcommand{\cum}{\kappa}

\newcommand{\centroid}{\Gamma}

\newcommand{\hX}{\hat{X}}
\newcommand{\hE}{\hat{E}}
\newcommand{\deq}{\overset d =}
\newcommand{\giventhat}{\mid}
\newcommand{\ud}{\mathop{}\!\mathrm{d}}

\def\reals{\mathbb{R}}
\def\R{\reals}
\def\N{\mathbb{N}}

\def\eps{\epsilon}

\def\poly{\mathrm{poly}}

\def\Z{\mathbb{Z}}

\def\argmax{\mathrm{argmax}}
\def\argmin{\mathrm{argmin}}
\def\Var{\mathrm{Var}}

\def\Poisson{\mathsf{Poisson}}
\def\Pois{\Poisson}
\def\Bin{\mathrm{Bin}}
\def\Multinom{\mathrm{Multinom}}
\def\Normal{\mathcal{N}}
\def\Norm{\Normal}

\def \deq{\overset d =}
\def\m{\mathrm{m}}
\def\ds{\rule{0pt}{1.5ex}}

\DeclareMathOperator{\sign}{sgn}

\newcommand{\spn}[1]{\mathrm{span}\left(#1\right)}

\newcommand{\sgn}[1]{\mathrm{sgn}\left( #1 \right)}

\renewcommand{\det}[1]{\mathrm{det}\left( #1 \right)}
\renewcommand{\vec}[1]{\mathrm{vec}\left( #1 \right)}

\newcommand{\angles}[1]{\left\langle #1 \right\rangle}
\newcommand{\dist}{\mathrm{dist}}

\newcommand{\Vr}[1]{\mathrm{Var}\left(#1\right)}
\newcommand{\Conjug}[1]{}

\newtheorem*{claim*}{Claim}
\renewcommand{\d}{\operatorname{d}}
\newcommand{\means}{m}
\DeclareRobustCommand{\stirling}{\genfrac\{\}{0pt}{}}

\newcommand{\prob}[1]{{P}\left(#1\right)}
\newcommand{\cumtns}[2]{\ensuremath{\kappa_{#1}^{#2}}}

\usepackage{verbatim}
\usepackage{multirow}

\declaretheorem[name=Theorem,numberwithin=section]{theorem}

\newtheorem{lemma}{Lemma}

\newtheorem{proposition}{Proposition}

\newtheorem{corollary}{Corollary}
\newtheorem{definition}{Definition}

\newtheorem{problem}{Problem}

\newif\iflong
\longtrue
\newif\ifshort
\shorttrue
\iflong
\shortfalse
\fi
\newif\ifextras
\extrastrue

\begin{document}

%
%

\author{Joseph Timothy Anderson}
\title{Geometric Methods for Robust Data Analysis in High Dimension}
\authordegrees{B.S, M.S}  
\unit{Department of Computer Science and Engineering}

\coadvisorname{Luis Rademacher}
\coadvisorname{Anastasios Sidiropolous}
\member{Mikhail Belkin}
\member{Facundo M\'emoli}

%
%

\maketitle

%
%

\disscopyright

%
%

\begin{abstract}


Data-driven applications are growing.
Machine learning and data analysis now finds both scientific and industrial application in biology, chemistry, geology, medicine, and physics.
These applications rely on large quantities of data gathered from automated sensors and user input.
Furthermore, the \emph{dimensionality} of many datasets is extreme: more details are being gathered about single user interactions or sensor readings.
All of these applications encounter problems with a common theme: \emph{use observed data to make inferences about the world}.
Our work obtains the first provably efficient algorithms for Independent Component Analysis (ICA) in the presence of heavy-tailed data.
The main tool in this result is the centroid body (a well-known topic in convex geometry), along with optimization and random walks for sampling from a convex body.
This is the first algorithmic use of the centroid body and it is of independent theoretical interest, since it effectively replaces the estimation of covariance from samples, and is more generally accessible.
We demonstrate that ICA is itself a powerful geometric primitive.
That is, having access to an efficient algorithm for ICA enables us to efficiently solve other important problems in machine learning.
The first such reduction is a solution to the open problem of efficiently learning the intersection of $\dim+1$ halfspaces in $\RR^\dim$, posed in \cite{FJK}.
This reduction relies on a non-linear transformation of samples from such an intersection of halfspaces (i.e. a \emph{simplex}) to samples which are approximately from a linearly transformed product distribution.
Through this transformation of samples, which can be done efficiently, one can then use an ICA algorithm to recover the vertices of the intersection of halfspaces.

Finally, we again use ICA as an algorithmic primitive to construct an efficient solution to the widely-studied problem of learning the parameters of a Gaussian mixture model.
Our algorithm again transforms samples from a Gaussian mixture model into samples which fit into the ICA model and, when processed by an ICA algorithm, result in recovery of the mixture parameters.
Our algorithm is effective even when the number of Gaussians in the mixture grows with the ambient dimension, even polynomially in the dimension.
In addition to the efficient parameter estimation, we also obtain a complexity lower bound for a low-dimension Gaussian mixture model.

\end{abstract}

%
%


%
%

\dedication{\textit{For my father, \\ who always believed that we can do better.}}

%
%

\begin{acknowledgements}
  There are many who make this work possible.

  I thank my excellent faculty committee: Misha Belkin, Facundo M\'emoli, Tasos Sidiropoulos, and Luis Rademacher.
  Much of this thesis was in collaboration with them, and I am in their debt.
  Special thanks go to Luis Rademacher for his years of advice, support, and teaching.

  The departments at Saint Vincent College and The Ohio State University have provided me with excellent environments for research and learning.

  Throughout my education I've had many helpful friends and collaborators;
  they each provided interesting and spirited discussion and insight.

  I am particularly grateful to Br. David Carlson for his instruction, advice, support, and friendship over the years.
  A large portion of my early adulthood was spent in his tutelage, and I am convinced nothing could have made it a better experience.

  I thank my family, especially my parents and sister, for their unfathomable support and encouragement.

  Finally, I thank my wife Sarah for her unfailing love and being an inspiring example of a smart, kind, and wonderful person.

  This work was supported by NSF grants CCF 1350870, and CCF 1422830.
  
\end{acknowledgements}

\begin{vita}

\dateitem{23 February 1990}{Born - Jeannette, PA}

\dateitem{2012}{B.S. Computing \& Information Science and Mathematics \\ Saint Vincent College}

\dateitem{2016}{M.S. Computer Science \& Engineering\\ The Ohio State University}

\dateitem{2012-present}{Graduate Teaching Associate\\
			 The Ohio State University.}

\begin{publist}

\researchpubs

\pubitem{J. Anderson, N. Goyal, A. Nandi, L. Rademacher.
\newblock ``Heavy-Tailed Analogues of the Covariance Matrix for ICA''.
\newblock {\em Thirty-First AAAI Conference on Artificial Intelligence (AAAI-17)}, Feb. 2017.}

\pubitem{J. Anderson, N. Goyal, A. Nandi, L. Rademacher.
\newblock ``Heavy-Tailed Independent Component Analysis''.
\newblock {\em 56th Annual IEEE Symposium on Foundations of Computer Science (FOCS)}, IEEE 290-209, 2015.}

\pubitem{J. Anderson, M. Belkin, N. Goyal, L. Rademacher, J. Voss.
\newblock ``The More The Merrier: The Blessing of Dimensionality for Learning Large Gaussian Mixtures''.
\newblock {\em Conference on Learning Theory (COLT)}, JMLR W\&CP 35:1135-1164, 2014.}

\pubitem{J. Anderson, N. Goyal, L. Rademacher.
\newblock ``Efficiently Learning Simplices''.
\newblock {\em Conference on Learning Theory (COLT)}, JMLR W\&CP 30:1020-1045, 2013.}

\pubitem{J. Anderson, M. Gundam, A. Joginipelly, D. Charalampidis.
\newblock ``FPGA implementation of graph cut based image thresholding''.
\newblock {\em Southeastern Symposium on System Theory (SSST)}, March. 2012.}

\end{publist}

\begin{fieldsstudy}

\majorfield{Computer Science and Engineering}

\begin{studieslist}
\studyitem{Theory of Computer Science}{Luis Rademacher}
\studyitem{Machine Learning}{Mikhail Belkin}
\studyitem{Mathematics}{Facundo M\'emoli}
\end{studieslist}


\end{fieldsstudy}

\end{vita}

%
%

\tableofcontents
\listoffigures

%
%

\chapter{Introduction}
\label{intro.ch}

The ``curse of dimensionality'' is a well-known problem encountered in the study of algorithms.
This curse describes a phenomenon found in many fields of applied mathematics, for instance in numerical analysis, combinatorics, and computational geometry.
In machine learning, many questions arise regarding how well existing algorithms scale as the dimension of the input data increases.
These questions are well motivated in practice, as data includes more details about observed objects or events.
The primary contribution of this research is to demonstrate that several important problems in machine learning are, in fact, efficiently solvable in high-dimension.
This work outlines two main contributions toward this goal.
First, we use the \emph{centroid body}, from convex geometry, as a new algorithmic tool, giving the first algorithm for Independent Component Analysis which can tolerate heavy-tailed data.
Second, we show that ICA itself can be used as an effective algorithmic primitive and that, with access to an efficient ICA algorithm, one can efficiently solve two other important open problems in machine learning: learning an intersection of $\dim+1$ halfspaces in $\RR^\dim$, and learning the parameters of a Gaussian Mixture Model.

The fact that ICA can be used as a new algorithmic primitive, as we demonstrate, introduces new understanding of statistical data analysis.
The second contribution in particular, where we show an efficient algorithm to learn the vertices of a simplex, has a key step in the reduction where one pre-processes the input in a carefully chosen -- but simple -- manner which results in data that will fit the standard ICA model.
This pre-processing step is a random non-linear scaling inspired by the study of $\ell_p$ balls \cite{barthe2005probabilistic}.
This scaling procedure is non-trivial and serves as an example of when one can exploit the \emph{geometry} of a problem to gain statistical insight, and to frame the problem in a new light which has been thoroughly studied.

\section{Organization of this Thesis}

The rest of this thesis is organized as follows.

Chapter~\ref{ch:htica} will introduce the signal separation framework that will be of importance throughout this work: Independent Component Analysis (ICA).
We will develop the general theory behind ICA and our method for improving the state-of-the art ICA method, extending the framework to be approachable when the input data has ``heavy-tailed'' properties.
Our algorithm for ICA is theoretically sound, and comes with provable guarantees for polynomial time and sample complexity.
We then present a more practical variation on our new ICA algorithm, and demonstrate its effectiveness on both synthetic and real-world heavy-tailed data.
This chapter is based on \cite{anon_htica} and \cite{experimental_htica}.

Chapter~\ref{ch:simplex} details an important first step in demonstrating that ICA can be used an effective algorithmic primitive.
This chapter shows that ICA can be used to recover an arbitrary simplex with polynomial sample size.
This chapter is based on work published in \cite{anderson2013efficient}.

Chapter~\ref{ch:gmm} presents a second reduction to ICA which recovers the parameters of a Gaussian Mixture Model efficiently in high dimension.
Furthermore, we give a complexity lower bound for learning a Gaussian Mixture in low dimension which also yields a complexity lower bound for ICA itself in certain situations.
This chapter is based on \cite{anderson2014more}.

\chapter{Robust signal separation via convex geometry}\label{ch:htica}
\details{
	noise robustness, outliers, polytope learning, robust statistics (refer to the Huber book?), mention why we don't use the floating body(?), talk about prewhitening approaches and why they fail here, talk about mixtures of heavy-tailed distributions(?),
	compare with isotropic PCA (which uses a reweighting similar to our Gaussian damping), preprocessing by symmetrization. Clarify what kind of r.v.s are the $s_i$'s: do we assume that they have density? Does the 
	density have to be non-zero everywhere? 

	Confusingly, in some ICA literature the word ``heavy-tailed'' is used to mean distribution with
	positive Kurtosis. This should be clarified right in the beginning and in the abstract too.
	If you do a search for ``heavy-tailed'' the second result is the Wikipedia page on kurtosis.

	Examples and utility of heavy-tailed distributions.
}

The blind source separation problem is the general problem of recovering 
underlying ``source signals'' that have been mixed in some unknown way and are presented to an observer. 
Independent component analysis (ICA) is a popular model for blind source separation where the mixing is performed linearly. 
Formally, if $S$ is an $\dim$-dimensional random vector from an unknown product distribution and $A$ is an invertible linear transformation, one is tasked with recovering the matrix $A$ and the signal $S$, using only access to i.i.d. samples of the transformed signal, namely $X = AS$.
Due to natural ambiguities, the recovery of $A$ is possible only up to 
the signs and permutations of the columns. Moreover, for the recovery to be possible the distributions of 
the random variables $S_i$ must not be a Gaussian distribution. 
ICA has applications in diverse areas such as neuroscience, signal processing, statistics, machine learning.
There is vast literature on ICA; see, e.g., \cite{Comon94, ICA01, ComonJutten}.

Since the formulation of the ICA model, a large number of algorithms have been devised employing a diverse set of 
techniques.
Many of these existing algorithms break the problem into two phases: first, find a transformation which, when applied to the observed samples, gives a new distribution which is \emph{isotropic}, i.e. a rotation of a (centered) product distribution; second, one typically uses an optimization procedure for a functional applied to the samples, such as the fourth directional moment, to recover the axes (or basis) of this product distribution.

To our knowledge, all known efficient algorithms for ICA with provable guarantees require higher moment assumptions such as finiteness of the fourth or higher moments for each component $S_i$.
Some of the most relevant works, e.g. algorithms of \cite{dl95, FJK}, explicitly require the fourth moment to be finite.
Algorithms in \cite{Yeredor, GVX}, which make use of the characteristic function also seem to require at least the fourth moment to be finite: while the characteristic function exists for distributions without moments, the algorithms in these papers use the second or higher derivatives of the (second) characteristic function, and for this to be well-defined one needs the moments of that order to exist.
Furthermore, certain anticoncentration properties of these derivatives are  needed which require that fourth or higher moments exist.

Thus the following question arises: is ICA provably efficiently solvable when the moment condition is weakened so that, say, only the second moment exists, or even when no moments exist? 
By \emph{heavy-tailed ICA} we mean the ICA problem with weak or no moment conditions (the precise
moment conditions will be specified when needed). 

Our focus in this chapter will be efficient algorithms for heavy-tailed ICA with provable guarantees and finite sample analysis.
While we consider this problem to be interesting in its own right, it is also of interest in practice
in a range of applications, e.g. \cite{Kidmose01, KidmoseThesis, shereshevski2001super, ChenBickel04, chen2005consistent, sahmoudi2005blind, wang2009ica}. \nnote{perhaps mention the areas in these papers?}
The problem could also be interesting from the perspective of robust statistics because of the following
 informal connection: 
algorithms solving heavy-tailed ICA might work by focusing on samples in a small (but not low-probability) region in order to get reliable statistics
about the data and ignore the long tail. Thus 
if the data for ICA is corrupted by outliers, the outliers are less likely to affect such an algorithm. 
\nnote{revisit}

In this work, \emph{heavy-tailed} distributions on the real line are 
those for which low order moments are not finite. Specifically, we will be interested in the case when the fourth
or lower order moments are not finite as this is the case that is not covered by previous algorithms. 
We hasten to clarify that in some ICA literature the word heavy-tailed is used with a different and less standard meaning, namely distributions with positive kurtosis; this meaning will not be used in the present work.

\nnote{It seems to me that this sentence does not fit anywhere: perhaps even more troublesome in the ICA context, however, is that if the first moment of the distribution diverges, one cannot directly put the distribution into isotropic position because the ``center'' is no longer defined.}

Heavy-tailed distributions arise in a wide variety of contexts including signal processing and finance;
see \cite{nolan:2015, Rachev2003} for an extensive bibliography.
Some of the prominent examples of heavy-tailed distributions are the Pareto distribution with shape parameter
$\alpha$ which has moments of order less than $\alpha$, the Cauchy distributions, which has moments of order
less than $1$; many more examples can be found on the Wikipedia page for heavy-tailed distributions.
An abundant (and important in applications) supply of heavy-tailed distributions comes from stable distributions;
see, e.g., \cite{nolan:2015}. There is also some theoretical work on learning mixtures of heavy-tailed 
distributions, e.g., \cite{DasguptaHKS05, ChaudhuriR08a}. 

In several applied ICA models with heavy tails it is reasonable to assume that the distributions have finite first moment.
In applications to finance (e.g., \cite{Chen2007594}), heavy tailed distributions are commonly used to model catastrophic but somewhat unlikely scenarios. A standard measures of risk in that literature, the so called conditional value at risk \cite{RockafellarUryasev}, is only finite when the first moment is finite. Therefore, it is reasonable to assume for some of our results that the distributions have finite first moment.

\section{Main result}
Our main result is an efficient algorithm that can recover the mixing matrix $A$ in the model $X=AS$ when each $S_i$ has $1+\gamma$ moments for a constant $\gamma > 0$.
The following theorem states more precisely the guarantees of our algorithm.
The theorem below refers to the algorithm \emph{Fourier PCA} \cite{GVX} which solves ICA under the 
fourth moment assumption. The main reason to use this algorithm is that finite sample guarantees have been proved
for it; we could have plugged in any other algorithm with such guarantee. The theorem below also refers to 
\emph{Gaussian damping}, which is an algorithmic technique we introduce in this chapter and will be explained 
shortly. \nnote{We will assume, for 
simplicity, that our probability distributions have density, although we believe this assumption is not essential.}
\begin{restatable}[Heavy-tailed ICA]{theorem}{main}\label{thm:putting_together}
Let $X=AS$ be an ICA model such that the distribution of $S$ is absolutely continuous, for all $i$ we have $\e (\abs{S_i}^{1+\gamma}) \leq M < \infty$ and normalized so that $\e \abs{S_i} = 1$, and the columns of $A$ have unit norm. 
Let $\Delta > 0$ be such that for each $i \in [n]$ if $S_i$ has finite fourth moment then its fourth cumulant satisfies 
$\abs{\cum_4(S_i)} \geq \Delta$.
Then, given $0<\eps \leq \dim^2$, $\delta > 0$, $s_M \geq \sigma_{\max}(A)$, $s_m \leq \sigma_{\min}(A)$, Algorithm~\ref{alg:orthogonalization_uniform} combined with Gaussian damping and Fourier PCA\lnote{add formal algorithm environment} outputs
$b_1, \ldots, b_n \in \R^n$ such that there are signs $\alpha_i \in \{-1,1\}$ and a permutation
$\pi:[n] \to [n]$ satisfying
\(
    \norm{A_i - \alpha_i b_{\pi(i)}} \le  \epsilon,
\)
with $\poly_\gamma(n, M, 1/s_m, s_M, 1/\Delta, R, 1/R,1/\epsilon, 1/\delta)$ time and sample complexity and 
with probability at least
$1-\delta$. Here $R$ is a parameter of the distributions of the $S_i$ as described below. The degree of the 
polynomial is $O(1/\gamma)$. 
\end{restatable}

We note here that the assumption that $S$ has an absolutely continuous distribution is mostly for convenience in the analysis of Gaussian damping and not essential. In particular, it is not used in Algorithm~\ref{alg:orthogonalization_uniform}.

Intuitively, $R$ in the theorem statement above measures how large a ball we need to restrict the 
distribution to, which has at least a constant (actually $1/\poly(n)$ suffices) probability mass 
and, moreover, each $S_i$ when restricted to the interval $[-R, R]$ has fourth cumulant at least 
$\Omega(\Delta)$. We show that all sufficiently large $R$ satisfy the above conditions and we can efficiently
compute such an $R$; see the discussion 
after Theorem~\ref{thm:ICA-orthogonal-damping} (the restatement in Sec.~\ref{sec:gaussian_damping}). 
For standard heavy-tailed distributions, such as the Pareto distribution, 
$R$ behaves nicely. For example, consider the 
Pareto distribution with shape parameter $= 2$ and scale parameter $=1$, i.e. the distribution with density 
$2/t^3$ for $t \geq 1$ and $0$ otherwise. For this distribution it's easily seen that 
$R = \Omega(\Delta^{1/2})$ suffices for the cumulant condition to be satisfied.

Theorem~\ref{thm:putting_together} requires that the $(1+\gamma)$-moment of the components $S_i$ be finite.
However, if the matrix $A$ in the ICA model is unitary (i.e. $A^TA = I$, or in other words, $A$ is a rotation matrix) then we do not need any moment assumptions:

\begin{restatable}{theorem}{gaussiandamping}
  \label{thm:ICA-orthogonal-damping}%
Let $X=AS$ be an ICA model such that $A \in \R^{n\times n}$ is unitary (i.e., $A^TA = I$) and 
the distribution of $S$ is absolutely continuous.
Let $\Delta > 0$ be such that for each $i \in [n]$ if $S_i$ has finite fourth moment then $\abs{\cum_4(S_i)} \geq \Delta$.
Then, given $\eps, \delta > 0$,
Gaussian damping combined with Fourier PCA outputs
$b_1, \ldots, b_n \in \R^n$ such that there are signs $\alpha_i \in \{-1,1\}$ and a permutation
$\pi:[n] \to [n]$ satisfying \(
    \norm{A_i - \alpha_i b_{\pi(i)}} \le  \epsilon,
\)
in $\poly(n, R, 1/\Delta, 1/\epsilon, 1/\delta)$ time and sample complexity and 
with probability at least
$1-\delta$. Here $R$ is a parameter of the distributions of the $S_i$ as described above. 
\end{restatable}

\myparagraph{Idea of the algorithm.} Like many ICA algorithms, our algorithm has two phases: first orthogonalize the independent components (reduce to the pure rotation case), and then determine the rotation.
In our heavy-tailed setting, each of these phases requires a novel approach and analysis in the heavy-tailed setting.

A standard orthogonalization algorithm is to put $X$ in isotropic position using the covariance matrix 
$\cov(X) := \E(XX^T)$. 
This approach requires finite second moment of $X$, which, in our setting, is not necessarily finite.
Our orthogonalization algorithm (Section \ref{sec:orthogonalization_uniform}) only needs finite 
$(1+\gamma)$-absolute moment and that each $S_i$ is symmetrically distributed. 
(The symmetry condition is not needed for our ICA algorithm, as one can reduce the general case to the symmetric case, see Section \ref{sec:symmetrization}). In order to understand the first absolute moment, it is helpful to look at certain convex bodies induced by the first and second moment. The directional second moment 
$\E_X \bigl((u^TX)^2\bigr)$ is a quadratic form in $u$ and its square root is the support function of a convex body, Legendre's inertia ellipsoid, up to some scaling factor (see \cite{MilmanPajor} for example). 
Similarly, one can show that the directional absolute first moment is the support function of a convex body, the centroid body. 
When the signals $S_i$ are symmetrically distributed, the centroid body of $X$ inherits these symmetries making it absolutely symmetric (see Section \ref{sec:preliminaries} for definitions) up to an affine transformation. 
In this case, a linear transformation that puts the centroid body in isotropic position also orthogonalizes the independent components (Lemma \ref{lemma:uniform-orthogonalizer}).
In summary, the orthogonalization algorithm is the following: find a linear transformation that puts the centroid body of $X$ in isotropic position. One such matrix is given by the inverse of the square root of the covariance matrix of the uniform distribution in the centroid body. Then apply that transformation to $X$ to orthogonalize the independent components.

\details{Moreover, we do not need to use the uniform distribution in the centroid body for this purpose. 
We can use any distribution that has the same symmetries restricted to the centroid body. In particular, we could use $X$ itself restricted to the centroid body. 
There is one important caveat: the mass of $X$ inside its centroid body could be very small. 
We fix this issue by scaling up the centroid body appropriately. In summary, the orthogonalization algorithm is the following: find a linear transformation that puts $X$ restricted to a scaling of the centroid body in isotropic position. Then apply that transformation to $X$ to orthogonalize the independent components.}

We now discuss how to determine the rotation (the second phase of our algorithm). The main idea is to reduce heavy-tailed case to a case where all moments exist and to use an existing ICA algorithm (from \cite{GVX} in our case) to handle the resulting ICA instance. We use \emph{Gaussian damping} to achieve such a reduction. 
By Gaussian damping we mean to multiply the density of the orthogonalized ICA model by a spherical Gaussian density. 

We elaborate now on our contributions that make the algorithm possible.

\myparagraph{Centroid body and orthogonalization.}
The centroid body of a compact set was first defined in \cite{petty1961}. It is defined as the convex set whose support function equals the directional absolute first moment of the given compact set. We generalize the notion of centroid body to any probability measure having finite first moment (see Section \ref{sec:preliminaries} for the background on convexity and Section \ref{sec:centroidbody} for our formal definition of the centroid body for probability measures). 
\nnote{The following text can be included (I started writing it but didn't finish) For putting the centroid body 
in isotropic position we need uniformly random samples from the centroid body. There are well-known algorithms to 
do this if we have membership access to the centroid body, that is to say, there
is an efficient algorithm that given a point answers whether the point is in the body. We give such an algorithm using the ellipsoid algorithm and results from \cite{GLS}.}
In order to put the centroid body in approximate isotropic position, we estimate its covariance matrix. For this, we use uniformly random samples from the centroid body.
There are known methods to generate approximately random points from a convex body given by a membership oracle. We implement an efficient membership oracle for the centroid body of a probability measure with $1+\gamma$ moments. The implementation works by first implementing a membership oracle for the polar of the centroid body via sampling and then using it via the ellipsoid method (see \cite{GLS}) to construct a membership oracle for the centroid body.
\details{Is there an argument without polarity? Would need uniform convergence.}%
As far as we know this is the first use of the centroid body as an algorithmic tool.

An alternative approach to orthogonalization in ICA one might consider is to use the empirical covariance matrix of $X$ even when the distribution is heavy-tailed. A specific problem with this approach is that when the second moment does not exist, the diagonal entries would be very different and grow without bound. This problem gets worse when one collects more samples. This wide range of diagonal values makes the second phase of an ICA algorithm very unstable.

\myparagraph{Linear equivariance and high symmetry.} 
A fundamental property of the centroid body, for our analysis, is that the centroid body is \emph{linearly equivariant}, that is, if one applies an invertible linear transformation to a probability measure then the corresponding centroid body transforms in the same way (already observed in \cite{petty1961}). In a sense that we make precise (Lemma \ref{lem:orthogonalizer}), high symmetry and linear equivariance of an object defined from a given probability measure are sufficient conditions to construct from such object a matrix that orthogonalizes the independent components of a given ICA model. This is another way to see the connection between the centroid body and Legendre's ellipsoid of inertia for our purposes: Legendre's ellipsoid of inertia of a distribution is linearly equivariant and has the required symmetries.

\myparagraph{Gaussian damping.} 
Here we confine ourselves to the special case of ICA when the ICA matrix is unitary, that is $A^TA = I$.  
A natural idea to deal with heavy-tailed distributions is to truncate the 
distribution in far away regions and hope that the truncated distribution still gives us a way to extract 
information. In our setting, this could mean, for example, that we consider the random variable obtained from $X$ 
conditioned on the even that $X$ lies in the ball of radius $R$ centered at the origin. Instead of the ball we 
could restrict to other sets. Unfortunately, in general the resulting random variable does not come from an 
ICA model (i.e., does not have independent components in any basis). Nevertheless one may still be able to use this 
random variable for recovering $A$. We do not know how to get an algorithmic handle on it even in the case of 
unitary $A$. Intuitively, 
restricting to a set breaks the product structure of the distribution that is crucial for recovering the 
independent components. 

We give a novel technique to solve heavy-tailed ICA for unitary $A$. No moment assumptions on the components
are needed for our technique.
We call this technique \emph{Gaussian damping}. Gaussian damping can also
be thought of as restriction, but instead of being restriction to a set it is a ``restriction to a
spherical Gaussian distribution.'' Let us explain. Suppose we have a distribution on $\R^n$ with density 
$\rho_X(\cdot)$. If we restrict this distribution
to a set $A$ (which we assume to be nice: full-dimensional and without any measure theoretic issues) then the
density of the restricted distribution is $0$ outside $A$ and is proportional to $\rho_X(x)$ for $x \in A$. One
can also think of the density of the restriction as being proportional to the product of $\rho_X(x)$ and the
density of the uniform distribution on $A$. 
In the same vein, by restriction to the Gaussian distribution with density proportional to 
$e^{-\norm{x}^2/R^2}$ we simply mean the 
distribution with density proportional to $\rho_X(x) \, e^{-\norm{x}^2/R^2}$. In other words, the density of the 
restriction is obtained by multiplying the two densities. By Gaussian damping of a distribution we mean 
the distribution obtained by this operation.

Gaussian damping provides a tool to solve the ICA problem for unitary $A$ by virtue of the following properties: 
(1) \emph{The damped distribution has finite moments of all orders.} 
This is an easy consequence of the fact that Gaussian density decreases super-polynomially. More precisely, 
one dimensional moment of order $d$ given by the integral $\int_{t \in \R} t^d \rho(t) e^{-t^2/R^2} \, dt$ is 
finite for all $d \geq 0$ for any distribution. 
(2) \emph{Gaussian damping retains the product structure.} Here we use the property of spherical Gaussians 
that it's the (unique) class of spherically symmetric distributions with independent components,
i.e., the density factors: $e^{-\norm{x}^2/R^2} = e^{-x_1^2/R^2}\dotsm e^{-x_n^2/R^2}$ (we are hiding a normalizing
constant factor). Hence the damped density also factors when expressed in terms of the components of 
$s = A^{-1}x$ (again ignoring normalizing constant factors):
\begin{align*}
\rho_X(x) e^{-\norm{x}^2/R^2} 
= \rho_S(s) e^{-\norm{s}^2/R^2} = \rho_{S_1}(s_1) e^{-x_1^2/R^2} \dotsm \rho_{S_n}(S_n) e^{-x_n^2/R^2}.
\end{align*}
Thus we have converted our heavy-tailed ICA model $X=AS$ into another ICA model $X_R = A S_R$ where $X_R$ 
and $S_R$ are obtained by Gaussian damping of $X$ and $S$, resp. To this new model we can apply the existing
ICA algorithms which require at least the fourth moment to exist. This allows us to estimate matrix $A$ (up
to signs and permutations of the columns). 

It remains to explain how we get access to the damped random variable $X_R$. This is done by a simple rejection
sampling procedure. Damping does not come free and one has to pay for it in terms of higher sample 
and computational complexity, but this increase in complexity is mild in the sense that the dependence on 
various parameters of the problem is still of similar nature as for the non-heavy-tailed case. 
A new parameter $R$ is introduced here which 
parameterizes the Gaussian distribution used for damping. 
We explained the intuitive meaning of $R$ after the statement of Theorem~\ref{thm:putting_together}
and that it's a well-behaved quantity for standard distributions. 


\myparagraph{Gaussian damping as contrast function.}
Another way to view Gaussian damping is in terms of \emph{contrast functions}, a general idea that in particular
has been used fruitfully in the ICA literature. Briefly, given a function $f: \R \to \R$, for $u$ on the unit 
sphere in $\R^n$, we compute $g(u) := \E f(u^T X)$. Now the properties of the function $g(u)$ as $u$ varies 
over the unit sphere, such as its local extrema, can help us infer properties of the underlying distribution.
In particular, one can solve the ICA problem for appropriately chosen \emph{contrast function} $f$. In ICA, 
algorithms with provable guarantees use contrast functions such as moments or cumulants (e.g., \cite{dl95, FJK}). 
Many other contrast functions are also used. 
\nnote{References? This could be a good place to refer to the Belkin et al. papers.}
Gaussian damping furnishes a novel class of contrast functions that also leads to provable guarantees. 
E.g., the function given by $f(t) = t^4 e^{-t^2/R^2}$ is in this class. We do not use the contrast function view 
in this work.

\myparagraph{Previous work related to damping.} To our knowledge the damping technique, and more generally
the idea of reweighting the data, is new in the context of ICA. But the general idea of reweighting is not new 
in other somewhat related contexts as we now discuss.
In robust statistics (see, e.g.,~\cite{Huber}), reweighting idea is used for outlier removal by giving less 
weight to far away data points. Apart from this high-level similarity we are not aware of any closer 
connections to our setting; in particular, the weights used are generally different. 

Another related work is \cite{BrubakerVempala}, on isotropic PCA, affine invariant clustering, and learning mixtures of Gaussians. This work uses Gaussian reweighting. However, again we are unaware of any more specific connection to our problem. 

Finally, in \cite{Yeredor, GVX} a different reweighting, using a ``Fourier weight'' $e^{i u^T x}$ (here $u \in \R^n$
is a fixed vector and $x \in \R^n$ is a data point) is used in the computation of the covariance matrix. This 
covariance matrix is useful for solving the ICA problem. But as discussed before, results here do not seem to be
amenable to our heavy-tailed setting.

\nnote{Discussion of how to get a handle on the centroid body: dual body and ellipsoid.
Discussion of floating body versus centroid.
Discussion of why dual characterization of cvar is not used.}

\section{Preliminaries}
\label{sec:preliminaries}

In this section, we review technical definitions and results which will be used in subsequent sections.

For a random vector $X \in \RR^\dim$, we denote the distribution function of $X$ as $F_X$ and, if it exists, the density is denoted $f_X$.

For a real-value random variable $X$, \emph{cumulants} of
$X$ are polynomials in the moments of $X$. For $j \geq 1$, the $j$th
cumulant is denoted $\cum_j(X)$. Denoting $\m_j := \E{X^j}$. 
examples: $\cum_1(X) = \m_1, \cum_2(X) = \m_2 - \m_1^2, \cum_3(X) =
\m_3 - 3 \m_2\m_1 + 2\m_1^3$. In general, cumulants can be defined 
as the coefficients in the logarithm of the moment generating function of $X$:
\begin{align*}
\log(\mathbb{E}_X(e^{tX})) = \sum_{i=1}^{\infty}\cum_i(X) \frac{t^i}{i!}. 
\end{align*}
The first two cumulants are the same as the expectation and the
variance, resp. Cumulants have the property that for two independent
r.v.s $X, Y$ we have $\cum_j(X+Y) = \cum_j(X) + \cum_j(Y)$ (assuming
that the first $j$ moments exist for both $X$ and $Y$).  Cumulants are order-$j$ homogeneous, i.e.~if $\alpha \in \R$ and $X$ is a r.v., then
$\cum_j(\alpha X) = \alpha^j\cum_j(X)$.  The first two
cumulants of the standard Gaussian distribution are the mean $0$ and the 
variance
$1$, and all subsequent Guassian cumulants have value $0$. 

An $\dim$-dimensional \emph{convex body} is a compact convex subset of $\RR^\dim$ with non-empty interior.
We say a convex body $K \subseteq \RR^\dim$ is \emph{absolutely symmetric} if $(x_1, \dots, x_{\dim}) \in K \Leftrightarrow (\pm x_1, \dots, \pm x_{n}) \in K$.
Similarly, we say random variable $X$ (and its distribution) is absolutely symmetric if, for any choice of signs $\alpha_i \in \{-1, 1\}$, $(x_1, \dots, x_n)$ has the same distribution as $(\alpha_1 x_1, \dots, \alpha_n x_n)$.
We say that \an\ $\dim$-dimensional random vector $X$ is \emph{symmetric} if 
$X$ has the same distribution as $-X$.
Note that if $X$ is symmetric with independent components (mutually independent coordinates) then its components are also symmetric.

For a convex body $K$, the Minkowski functional of $K$ is $p_{K}(x) = \inf\{ \alpha \geq 0 \suchthat \alpha x \in K \}$.
The Minkowski sum of two sets $A, B \subseteq \RR^{\dim}$ is the set $A + B = \{ a + b \suchthat a \in A, \; b \in B\}$.

The singular values of a 
matrix $A \in \R^{m \times n}$ will be ordered in the decreasing order: $\sigma_1 \geq
\sigma_2 \geq \ldots \geq \sigma_{\min(m, n)}$. By $\sigma_{\min}(A)$ we mean $\sigma_{\min(m,n)}$. 

We say that a matrix $A \in \R^{n\times n}$ is \emph{unitary} if $A^TA = I$, or in other words $A$ is a 
rotation matrix. (Normally for matrices
with real-valued entries such matrices are called orthogonal matrices and the word unitary is reserved for  
their complex counterparts, however the word orthogonal matrix can lead to confusion in the present work.)
For matrix $C \in \R^{n \times n}$, denote by $\norm{C}_2$ the spectral norm and
by $\norm{C}_F$ the Frobenius norm.
\iflong
We will need the following inequality about the stability of matrix inversion (see for example 
\cite[Chapter III, Theorem 2.5]{stewart1990matrix}).

\begin{lemma}\label{lem:inversion}
Let $\norm{\cdot}$ be a matrix norm such that $\norm{AB} \leq \norm{A} \norm{B}$. Let matrices $C, E \in \R^{n\times n}$ be such that $\norm{C^{-1} E}_2 \leq 1$, and
let $\tilde{C} = C + E$. Then
\begin{equation}
\frac{\norm{\tilde{C}^{-1} - C^{-1}}} {\norm{C^{-1}}} \leq \frac{\norm{C^{-1} E}}{1 - \norm{C^{-1} E}}.
\end{equation}
\end{lemma}

This implies that if $\norm{E}_2 = \norm{\tilde{C} - C}_2 \leq 1/(2 \norm{C^{-1}}_2)$, then
\begin{equation}\label{eq:inverse-stability}
\norm{\tilde{C}^{-1} - C^{-1}}_2 \leq 2 \norm{C^{-1}}_{2}^{2} \norm{E}_{2}.
\end{equation}

\subsection{Heavy-Tailed distributions}

In general, heavy-tailed distributions are those whose tail probabilities go to zero more slowly than an inverse exponential.
Formally, one may write that the distribution of $X$ is heavy-tailed if $e^{\lambda x} P(X > x)$ diverges as $x$ approaches infinity.
However, for our applications, we need only care about the existence of higher moments of the distribution.
The primary statistical assumption made in the following work assumes only that $\e X^{1+\gamma} < \infty$ for some $\gamma > 0$.
The addition of 1 in the above is primarily so that means are well-defined.
Naturally, from this assumption, one can construct heavy-tailed distributions, e.g. with density proportional to $1/x^{2+\gamma}$, which will have first moment, up to the $1+\gamma$ moment, but no higher.

\subsection{Stable Distributions}\label{sec:stable-distributions}

The \emph{Stable Distributions} are an important class of probability distributions characterized by one key property: the family itself is closed under addition and scaler multiplication.
That is, if you have two random variables $X$ and $Y$ with the same stable distribution, $X+Y$ will also be a stable distribution with slightly different parameters.
Two well-known members of this family are the Gaussian and Cauchy distribution, both of which will be important throughout this work.

Stable distributions are fully characterized by four parameters and can be written as $\stable(\alpha, \beta, c, \mu)$ where $\alpha \in (0,2]$ is the stable parameter ($\alpha = 2$ for Gaussian and $\alpha = 1$ for Cauchy), $\beta \in [-1,1]$ is a skewness parameter, $c \in (0,\infty)$ is scale, and $\mu \in \RR$ is location.
The PDF and CDF are not expressible analytically in general, but stable distributions are those for which the characteristic function can be written as \[ \phi(x; \alpha, \beta, c, \mu) = \exp\left( ixu - \abs{ct}^{\alpha} (1 - i\beta \sgn{x} \Phi \right) \] where $\sgn{}$ is the sign function and
\[
  \Phi = 
  \begin{cases}
    \tan(\pi \alpha / 2) \; \text{ if } \alpha \neq 1 \\
    -\frac{2}{\pi} \log \abs{t} \; \text{ otherwise}
  \end{cases}.
\]

A useful property which will be used later, and gives a formalization of the closure property mentioned above is the following: if $X_1, \dots X_n$ are iid with stable density $\density(x; \alpha, \beta, c, \mu)$ and we let \[ Y = \sum_{i=1}^{n} k_i (X_i - \mu) \] where $k_i \in \RR$, then $Y$ will have density \[ s^{-1} \density(x/s; \alpha, \beta, c, 0) \] where $$s = \left( \sum_{i=1}^{n} \abs{k_i}^{\alpha} \right).$$

\subsection{Convex optimization}


We need the following result: Given a membership oracle for a convex body $K$ one can implement efficiently a membership oracle for $K^\polar$, the polar of $K$. 
This follows from applications of the ellipsoid method from \cite{GLS}.
Specifically, we use the following facts: (1) a validity oracle for $K$ can be constructed from a membership oracle for $K$ \cite[Theorem 4.3.2]{GLS}; (2) a membership oracle for $K^\polar$ can be constructed from a validity oracle for $K$ \cite[Theorem 4.4.1]{GLS}.

The definitions and theorems in this section all come (occasionally with slight rephrasing) from \cite{GLS} except for the notion of $(\epsilon, \delta)$-weak oracle. 
[The definitions below use rational numbers instead of real numbers. 
This is done in \cite{GLS} as they work out in detail the important low level issues of how the numbers 
in the algorithm are represented as general real numbers cannot be directly handled by computers. 
These low-level details can also be worked out for the arguments here, but as is customary, we will 
not describe these and use real numbers for the sake of exposition.]
In this section $K \subseteq \R^n$ is a convex body. For $y \in \R^n$, the distance of $y$ to $K$ is given by
$d(y, K) := \min_{z \in K}\,\norm{y-z}$. 
Define $S(K,\eps) := \{y \in \RR^n \suchthat d(y,K) \leq \eps \}$
and
$S(K,-\eps) := \{y \in \RR^n \suchthat S(y,\eps) \subseteq K \}$. Let $\QQ$ denote the set of rational numbers.

\begin{definition}[\cite{GLS}]\label{def:eps-weak-oracle}
The \emph{$\epsilon$-weak membership problem} for $K$ is the following:
Given a point $y \in \QQ^n$ and a rational number $\eps > 0$, 
either (i) assert that $y \in S(K,\eps)$, or (ii) assert that $y \not \in S(K,-\eps)$.
An \emph{$\epsilon$-weak membership oracle} for $K$ is an oracle that solves the weak membership problem for $K$.
For $\delta \in [0,1]$, an \emph{$(\eps, \delta)$-weak membership oracle} for $K$ acts as 
follows: Given a point $y \in \QQ^n$, with probability at least $1-\delta$ it solves the  
$\eps$-weak membership problem for $y, K$, and otherwise its output can be arbitrary.  
\end{definition}



\begin{definition}[\cite{GLS}]\label{def:wval}
The \emph{$\eps$-weak validity problem} for $K$ is the following: 
Given a vector $c \in \QQ^n$, a rational number $\gamma$, and a rational number $\eps > 0$, either
(i) assert that $c^T x \leq \gamma + \eps$ for all $x \in S(K, -\eps)$, or
(ii) assert that $c^T x \geq \gamma - \eps$ for some $x \in S(K, \eps)$.
The notion of $\epsilon$-weak validity oracle and $(\epsilon, \delta)$-weak validity oracle can be defined
similarly to Def.~\ref{def:eps-weak-oracle}. 
\end{definition}



\begin{definition}[{\cite[Section 2.1]{GLS}}]
We say that an oracle algorithm is an \emph{oracle-polynomial time algorithm} for a certain problem defined on a class of convex sets if the running time of the algorithm is bounded by a polynomial in the encoding length of $K$ and in the encoding length of the possibly existing further input, for every convex set $K$ in the given class.
\end{definition}
The \emph{encoding length} of a convex set will be specified below, depending on how the convex set is presented.
\nnote{check that this has been done}

\begin{theorem}[Theorem 4.3.2 in \cite{GLS}]\label{thm:wmem-to-wviol}
Let $R > r > 0$ and $a_0 \in \R^n$. 
There exists an oracle-polynomial time algorithm that solves the weak validity problem for every convex body $K \subseteq \RR^n$ contained in the ball of radius $R$ and containing a ball of radius $r$ centered at $a_0$ given by a weak membership oracle. The encoding length of $K$ is $n$ plus the length of the binary encoding of $R, r$, and $a_0$.
\end{theorem}
We remark that the Theorem~4.3.2 as stated in \cite{GLS} is stronger than the above statement 
in that it constructs a weak violation oracle (not defined here) 
which gives a weak validity oracle which suffices for us.
The algorithm given by Theorem 4.3.2 makes a polynomial (in the encoding length of $K$) number of queries to
the weak membership oracle. 

\begin{lemma}[Lemma 4.4.1 in \cite{GLS}]\label{lem:wviol-polar-wmem}
There exists an oracle-polynomial time algorithm that solves the weak membership problem for $K^\polar$, where $K$ is a convex body contained in the ball of radius $R$ and containing a ball of radius $r$ centered at $0$ given by a weak validity oracle.
The encoding length of $K$ is $n$ plus the length of the binary encoding of $R$ and $r$.
\end{lemma}
Our algorithms and proofs will need more quantitative details from the proofs of the 
above theorem and lemma. These will be mentioned when we need them.

\iflong
\subsection{Algorithmic convexity}
We state here a special case of a standard result in algorithmic convexity: There is an efficient algorithm to estimate the covariance matrix of the uniform distribution in a centrally symmetric convex body given by a weak membership oracle.
\nnote{Does the body need to be centrally symmetric? Also need to say that by the covariance matrix of a convex body we mean the 
covariance matrix of the uniform distribution on the body.}
The result follows from the random walk-based algorithms to generate approximately uniformly random points from a convex body \cite{DBLP:journals/jacm/DyerFK91, MR1608200, DBLP:journals/jcss/LovaszV06}. 
Most papers use access to a membership oracle for the given convex body. In this work we only have access to an $\eps$-weak membership oracle. 
As discussed in \cite[Section 6, Remark 2]{DBLP:journals/jacm/DyerFK91}, essentially the same algorithm implements efficient sampling when given an $\eps$-weak membership oracle.
The problem of estimating the covariance matrix of a convex body was introduced in \cite[Section 5.2]{MR1608200}. 
That paper analyzes the estimation from random points. 
There has been a sequence of papers studying the sample complexity of this problem \cite{MR1665576, MR1694526, MR2190337, MR2276533, MR2276637, MR2956207, adamczak, MR3127875}.
\nnote{I'm not sure we need all these references. Maybe just refer to one of the latest ones and
say see references therein.}

\begin{theorem}\label{thm:covariance_estimation}
Let $K \subseteq \RR^\dim$ be a centrally symmetric convex body given by a weak membership oracle so that $r B_2^\dim \subseteq K \subseteq R B_2^\dim$. Let $\Sigma = \cov(K)$. Then there exists a randomized algorithm that, when given access to the weak membership of $K$ and inputs $r, R, \delta, \eps_c >0$, it outputs a matrix $\tilde \Sigma$ such that with probability at least $1-\delta$ over the randomness of the algorithm\details{and for any response by the weak membership oracle to queries in the ambiguous region}, we have 
\begin{equation}\label{equ:covariance}
(\forall u \in \RR^\dim) \qquad (1-\eps_c) u^T \Sigma u \leq u^T \tilde \Sigma u \leq (1+\eps_c) u^T \Sigma u.
\end{equation}
The running time of the algorithm is $\poly(\dim, \log(R/r), 1/\eps_c, \log(1/\delta))$.
\end{theorem}
Note that \eqref{equ:covariance} implies $\norm{\tilde \Sigma - \Sigma}_2 \leq \eps_c \norm{\Sigma}_2$. \details{\eqref{equ:covariance} $\implies$ $\abs{u^T \tilde \Sigma u - u^T \Sigma u} \leq \eps_c u^T \Sigma u \implies \abs{u^T (\tilde \Sigma -  \Sigma) u} \leq \eps_c u^T u \norm{\Sigma}$. The claim follows from the Rayleigh quotient characterization of eigenvalues and the spectral norm.}
The guarantee in \eqref{equ:covariance} has the advantage of being invariant under linear transformations in the following sense: 
if one applies an invertible linear transformation $C$ to the underlying convex body, the covariance matrix and its estimate become $C \Sigma C^T$ and $C \tilde \Sigma C^T$, respectively. These matrices satisfy 
\begin{equation*}
(\forall u \in \RR^\dim) \qquad (1-\eps_c) u^T C \Sigma C^T u \leq u^T C \tilde \Sigma C^T u \leq (1+\eps_c) u^T C \Sigma C^T u.
\end{equation*}
This fact will be used later. \nnote{Where is this used? As far as I can see, only the sentence right after
the theorem statement gets used.}
\fi

\subsection{The centroid body}\label{sec:centroidbody}

The main tool in our orthogonalization algorithm is the centroid body of a distribution, which we use as a first moment analogue of the covariance matrix.
In convex geometry, the centroid body is a standard (see, e.g., \cite{petty1961,MilmanPajor,gardner1995geometric})  convex body associated to (the uniform distribution on) a given convex body.
Here we use a generalization of the definition from the case of the uniform distribution on a convex body to more general probability measures. \nnote{why switch the language to probability measures?}
Let $X \in \RR^\dim$ be a random vector with finite first moment, that is, for all $u \in \RR^\dim$ we have $\e (\abs{\inner{u}{X}}) < \infty$. Following \cite{petty1961}, consider the function
\(
h(u) = \e (\abs{\inner{u}{X}})
\).
Then it is easy to see that $h(0) = 0$, $h$ is positively homogeneous, and $h$ is subadditive. Therefore, it is the support function of a compact convex set \cite[Section 3]{petty1961}, \cite[Theorem 1.7.1]{MR1216521}, \cite[Section 0.6]{gardner1995geometric}. \nnote{I think the fact that the centroid body of a distribution is convex is a key fact and did not appear in a previous paper even though the proof is basically the same as that for convex bodies (correct? I actually haven't checked it); so it should be emphasized by making it a lemma/prop. and the proof included although it can go to the appendix. Also perhaps explain why it works for the first moment and what happens for higher and lower ones.}
This justifies the following definition:
\begin{definition}[Centroid body
]\label{def:centroid-body2}
Let $X \in \RR^\dim$ be a random vector with finite first moment, that is, for all $u \in \RR^\dim$ we have $\e (\abs{\inner{u}{X}}) < \infty$. 
The \emph{centroid body} of $X$ is the compact convex set, denoted $\Gamma X$, whose support function is $h_{\Gamma X}(u) = \e (\abs{\inner{u}{X}})$.
For a probability measure $\measure$, we define $\Gamma \measure$, the centroid body of $\measure$, as the centroid body of any random vector distributed according to $\measure$.
\end{definition}

The following lemma says that the centroid body is equivariant under linear transformations. 
It is a slight generalization of statements in \cite{petty1961} and \cite[Theorem 9.1.3]{gardner1995geometric}.

\begin{lemma}\label{lem:equivariance}
Let $X$ be a random vector on $\RR^\dim$.
Let $A: \RR^\dim \to \RR^\dim$ be an invertible linear transformation.
Then $\Gamma (AX) = A (\Gamma X)$.
\end{lemma}
\begin{proof}
$
x \in \Gamma(AX) \Leftrightarrow
\forall u \inner{x}{u} \leq \e (\abs{\inner{u}{AX}}) \Leftrightarrow
\forall u \inner{x}{u} \leq \e (\abs{\inner{A^T u}{X}}) \Leftrightarrow
\forall v \inner{x}{A^{-T}v} \leq \e (\abs{\inner{v}{X}}) \Leftrightarrow
\forall v \inner{A^{-1}x}{v} \leq \e (\abs{\inner{v}{X}}) \Leftrightarrow
A^{-1}x \in \Gamma(X) \Leftrightarrow
x \in A\Gamma(X)
$.
\end{proof}
\details{it seems to me (Luis) that linear equivariance is true even in the non-invertible case, but not through that proof} 
\begin{lemma}[{\cite[Section 0.8]{gardner1995geometric}, \cite[Remark 1.7.7]{MR1216521}}]\label{lem:polar-radial}
Let $K \subseteq \RR^\dim$ be a convex body with support function $h_K$ and such that the origin is in the interior of $K$.
Then $K^\polar$ is a convex body and has radial function $r_{K^\polar}(u) = 1/h_{K}(u)$ for $u \in S^{n-1}$.
\end{lemma}
Lemma~\ref{lem:polar-radial} implies that testing membership in $K^{\polar}$ 
is a one-dimensional problem if we have access to the support function $h_K(\cdot)$:
We can decide if a point $z$ is in  $K^{\polar}$ 
by testing if $\norm{z} \leq 1/h_K({z}/{\norm{z}})$ instead of needing to check $\inner{z}{u} \leq 1$ for all $u \in K$.
In our application $K = \Gamma X$. We can estimate $h_{\Gamma X}({z}/{\norm{z}})$ by taking the empirical average of  
$\abs{\inner{x^{(i)}}{{z}/{\norm{z}}}}$ where the $x^{(i)}$ are samples of $X$. This leads to an approximate oracle for $(\Gamma X)^\polar$
which will suffice for our application. The details are in Sec.~\ref{subsec:membership-centroid-body}. 


\iflong
\section{Membership oracle for the centroid body} \label{subsec:membership-centroid-body}
In this section we provide an efficient weak membership oracle (Subroutine~\ref{sub:cvar-oracle}) for the
centroid body $\Gamma X$ of the r.v. $X$.
This is done by first providing a weak membership oracle (Subroutine~\ref{sub:polar-cvar-oracle}) for the polar body
$(\Gamma X)^\polar$.
We begin with a lemma that shows that under certain general conditions the centroid body is 
``well-rounded.'' This property will prove useful in the membership tests.

\begin{lemma}\label{lem:centroid-scaling}
Let $S = (S_1, \dots, S_n) \in \RR^n$ be an absolutely symmetrically distributed random vector such that
$\e (\abs{S_i}) = 1$ for all $i$.
Then $B_{1}^{\dim} \subseteq \Gamma S \subseteq [-1,1]^\dim$.
Moreover, $\dim^{-1/2} B_2^\dim \subseteq (\Gamma S)^\polar \subseteq \sqrt{\dim}B_2^\dim$.
\end{lemma}

\begin{proof}
The support function of $\Gamma S$ is $h_{\Gamma S}(\theta) = \e \abs{\inner{S}{\theta}}$ (Def.~\ref{def:centroid-body2}).
Then, for each canonical vector $e_i$, $h_{\Gamma S}(e_i) = h_{\Gamma S}(-e_i) =  \E\abs{S_i} = 1$.
Thus, $\Gamma S$ is contained in $[-1,1]^\dim$. 
Moreover, since $[-1,1]^\dim \subseteq \sqrt{\dim} B_2^\dim$, we get $\Gamma S \subseteq \sqrt{\dim} B_2^\dim$.

We claim now that each canonical vector $e_i$ is contained in $\Gamma S$.
To see why, first note that since the support function is $1$ along each canonical direction, $\Gamma S$ will touch the facets of the unit hypercube.
Say, there is a point $(1, x_2, x_3, \dots, x_n)$ that touches facet associated to canonical vector $e_1$.
But the symmetry of the $S_i$s implies that $\Gamma S$ is absolutely symmetric, so that $(1,\pm x_2, \pm x_3, ..., \pm x_n)$ is also in the centroid body.
Convexity implies that $(1,0,\dots,0) = e_1$ is in the centroid body.
The same argument applied to all $\pm$ canonical vectors implies that they are all contained in the centroid body, and this with convexity implies that the centroid body contains $B_1^\dim$.
In particular, it contains $n^{-1/2} B_2^\dim$.
\end{proof}


\subsection{Mean estimation using \texorpdfstring{$1 + \gamma$}{1+g} moments}
We will need to estimate the support function of the centroid body in various directions. To this end
we need to estimate the first absolute moment of the projection to a direction. Our assumption that 
each component $S_i$ has finite $(1+\gamma)$-moment will allow us to do this with a reasonable small 
probability of error. This is done via the following Chebyshev-type inequality.

\nnote{I think this section, except for the statement of the lemma, should go to the appendix. Not sure yet where the lemma statement
should go.}

Let $X$ be a real-valued symmetric random variable such that $\E \abs{X}^{1+\gamma} \leq M$ for some $M > 1$ and 
$0 < \gamma < 1$. Then we will prove that the empirical average of the expectation of $X$ converges to the expectation of $X$.

Let $\tilde{\E}_N[\abs{X}]$ be the empirical average obtained from $N$ independent samples 
$X^{(1)}, \ldots, X^{(N)}$,
i.e., $(\abs{X^{(1)}}+\dotsb+\abs{X^{(N)}})/N$.

\begin{lemma}\label{lem:1-plus-eps-chebyshev}
Let $\epsilon \in (0,1)$. With the notation above, for 
$N \geq \left(\frac{8M}{\epsilon}\right)^{\frac{1}{2}+\frac{1}{\gamma}}$,
we have
\begin{align*}
\Pr[\abs{\tilde{\E}_N[\abs{X}]-\E[\abs{X}]} > \epsilon] \leq \frac{8M}{\epsilon^2 N^{\gamma/3}}.
\end{align*}
\end{lemma}

\begin{proof}
Let $T > 1$ be a threshold whose precise value we will choose later. We have
\[ \Pr[\abs{X} \geq T] \leq \frac{\E[\abs{X}^{1+\gamma}]}{T^{1+\gamma}} = \frac{M}{T^{1+\gamma}}. \]
By the union bound,
\begin{align} \label{eqn:union}
\Pr[\exists i \in [N] \mbox{ such that } \abs{X^{(i)}} > T] \leq \frac{NM}{T^{1+\gamma}}.
\end{align}
Define a new random variable $X_T$ by 
\begin{align*}
X_T = \begin{cases} X & \text{if $\abs{X} \leq T$,} \\
                    0 & \text{otherwise.}
\end{cases}
\end{align*}
Using the symmetry of $X_T$ we have
\begin{align}\label{eqn:X_T_upper}
\Var[\abs{X_T}] \leq \E[X_T^2] \leq T \, \E[\abs{X_T}] \leq T\, \E[\abs{X}] \leq T M^{1/(1+\gamma)}. 
\end{align}
By the Chebyshev inequality and \eqref{eqn:X_T_upper} we get 
\begin{align} \label{eqn:chebyshev}
\Pr[\abs{\tilde{\E}_N[\abs{X_T}] - \E[\abs{X_T}]} > \epsilon'] 
\leq \frac{\Var[\abs{X_T}]}{N \epsilon'^2} 
\leq \frac{T M^{1/(1+\gamma)}}{N \epsilon'^2}.
\end{align}
Putting \eqref{eqn:union} and \eqref{eqn:chebyshev} together for $0 < \epsilon' < 1/2$ we get
\begin{align} \label{eqn:union+chebyshev}
  \Pr[\abs{\tilde{E}_N[\abs{X}]-\E[\abs{X_T}]} > \epsilon'] \leq \frac{NM}{T^{1+\gamma}} 
+ \frac{T M^{1/(1+\gamma)}}{N \epsilon'^2}.
\end{align}
Choosing 
\begin{align}\label{eqn:N_T}
N := \frac{T^{1+\gamma/2}}{\epsilon' M^{\gamma/(2(1+\gamma))}}, 
\end{align}
the RHS of the previous equation becomes $\frac{2 M^{1-\gamma/(2(1+\gamma))}}{\epsilon' T^{\gamma/2}}$. The choice of $N$ is made to minimize the RHS; we ignore integrality issues.
Pick $T_0>0$ so that $\abs{\E[\abs{X}]-\E[\abs{X_{T_0}}]} < \epsilon'$. 
To estimate $T_0$, note that
\begin{align*}
M = \E[\abs{X}^{1+\gamma}] \geq T_0^\gamma \, \E[\abs{\abs{X}-\abs{X_{T_0}}}],
\end{align*}
Hence
\begin{align} \label{eqn:mean_X_X_T}
\abs{\E[\abs{X}]-\E[\abs{X_{T_0}}]} \leq \E[\abs{\abs{X} - \abs{X_{T_0}}}] \leq M/T_0^\gamma.
\end{align}
We want $M/T_0^\gamma \leq \epsilon'$ which is equivalent to $T_0 \geq (\frac{M}{\epsilon'})^{1/\gamma}$. 
We set $T_0 := (\frac{M}{\epsilon'})^{1/\gamma}$. 
Then, for $T \geq T_0$ putting together \eqref{eqn:union+chebyshev} and \eqref{eqn:mean_X_X_T} gives
\begin{align*}
\Pr[\abs{\tilde{\E}_N[\abs{X}]-\E[\abs{X}]} > 2\epsilon'] \leq \frac{2 M^{1-\gamma/(2(1+\gamma))}}{\epsilon' T^{\gamma/2}}.
\end{align*}
Setting $\epsilon = 2\epsilon'$ (so that $\epsilon \in (0,1)$) and expressing the RHS of the last 
equation in terms of $N$ via 
\eqref{eqn:N_T}
(and eliminating $T$), 
and using our assumptions $\gamma, \epsilon \in (0,1)$, $M>1$ to get a simpler upper bound, we get
\begin{align*}
\Pr[\abs{\tilde{\E}_N[\abs{X}]-\E[\abs{X}]} > \epsilon] &\leq
\frac{2^{2+\frac{\gamma}{(2+\gamma)}} M^{1-\frac{\gamma}{2(1+\gamma)} + \frac{\gamma^2}{2(2+\gamma)(1+\gamma)}}}
{\epsilon^{1+\frac{\gamma}{2+\gamma}} N^{\frac{\gamma}{2+\gamma}}} \leq \frac{8M}{\epsilon^2 N^{\gamma/3}}.
\end{align*}
Condition $T \geq T_0$, when expressed in terms of $N$ via \eqref{eqn:N_T}, becomes
\begin{align*}
N \geq \frac{2^{\frac{3}{2}+\frac{1}{\gamma}}}{\epsilon^{\frac{1}{2}+\frac{1}{\gamma}}} 
M^{\frac{1}{2} + \frac{1}{\gamma} - \frac{\gamma}{2(1+\gamma)}} \geq \left(\frac{8M}{\epsilon}\right)^{\frac{1}{2}+\frac{1}{\gamma}}.
\end{align*}
\end{proof}


\fi


\subsection{Membership oracle for the polar of the centroid body}
\nnote{This section doesn't seem to need symmetry. Or does it?}
As mentioned before, our membership oracle for $(\Gamma X)^\polar$ (Subroutine~\ref{sub:polar-cvar-oracle}) 
is based on the fact that $1/h_{\Gamma X}$ is the radial function of $(\Gamma X)^\polar$, and that $h_{\Gamma X}$ is the directional absolute first moment of $X$, which can be efficiently estimated by sampling. 
\begin{subroutine}[ht]
\caption{Weak Membership Oracle for $(\Gamma X)^{\polar}$}
\label{sub:polar-cvar-oracle}
\begin{algorithmic}[1]
\Require
Query point $y \in \QQ^\dim$, samples from symmetric ICA model $X = AS$,
bounds $s_M \geq \sigma_{\max}(A)$, $s_m \leq \sigma_{\min}(A)$,
closeness parameter $\eps$, failure probability $\delta$.
\Ensure Weak membership decision for $y \in (\Gamma X)^{\polar}$.
\State Generate iid samples $x^{(1)}, x^{(2)}, \dots, x^{(N)}$ of $X$ for $N = \poly_{\gamma}(n, M, 1/s_m, s_M, 1/\eps, 1/\delta)$.
\State Compute 
\[ 
\tilde{h} = \frac{1}{N} \sum_{i=1}^N \abs{\inner{x^{(i)}}{\frac{y}{\norm{y}}}}.
\]
\State If $\norm{y} \leq 1/\tilde{h}$, report $y$ as feasible. Otherwise, report $y$ as infeasible.
\end{algorithmic}
\end{subroutine}

\begin{lemma}[Correctness of Subroutine~\ref{sub:polar-cvar-oracle}]
Let $\gamma > 0$ be a constant and $X=AS$ be given by a symmetric ICA model such that for all $i$ we have $\e (\abs{S_i}^{1+\gamma}) \leq M < \infty$ and normalized so that $\e \abs{S_i} = 1$.
Let $\eps, \delta > 0$.
Given $s_M \geq \sigma_{\max}(A)$ , $s_m \leq \sigma_{\min}(A)$,
Subroutine~\ref{sub:polar-cvar-oracle} is an $(\eps, \delta)$-weak membership oracle for $(\Gamma X)^{\polar}$ with
using time and sample complexity
$\poly_{\gamma}(n, M, 1/s_m, s_M, 1/\eps, 1/\delta)$. The degree of the polynomial is $O(1/\gamma)$.
\nnote{Revisit the dependence on gamma. Propagate this dependence to other things.}
\end{lemma}

\begin{proof}
Recall from Def.~\ref{def:eps-weak-oracle} that we need to show that, with probability at least $1-\delta$, Subroutine~\ref{sub:polar-cvar-oracle} outputs TRUE when $y \in S((\Gamma X)^{\polar}, \eps)$ and FALSE when $y \not \in S((\Gamma X)^{\polar}, -\eps)$; otherwise, the output can be either TRUE or FALSE arbitrarily.

Fix a point $y \in \QQ^\dim$ and let $\theta := y/\norm{y}$ denote the direction of $y$.
The algorithm estimates the radial function of $(\Gamma X)^{\polar}$ along $\theta$, which is $1/h_{\Gamma X}(\theta)$ (see Lemma~\ref{lem:polar-radial}).
In the following computation, we simplify the notation by using $h = h_{\Gamma X}(\theta)$.
It is enough to show that with probability at least $1-\delta$ the algorithm's estimate, $1/\tilde h$, of the radial function is within $\eps$ of the true value, $1/h$.
\details{Suppose $y$ is in the inner parallel body. Then $1/h = r_{\centroid X}(\theta) \geq \norm{y} + \eps$. Suppose $y$ is not in the outer parallel body. Then $r_{\centroid X}(\theta) +\eps < \norm{y}$.}

For $X^{(1)}, \dots, X^{(N)}$, i.i.d. copies of $X$, the empirical estimator for $h$ is
$\tilde{h} := \frac{1}{N} \sum_{i=1}^{N} \abs{\inner{X^{(i)}}{\theta}}$.

We want to apply Lemma~\ref{lem:1-plus-eps-chebyshev} to $\inner{X}{\theta}$. For this we need a bound on its 
$(1+\gamma)$-moment. 
The following simple bound is sufficient for our purposes: Let $u = A^T \theta$. Then $\abs{u_i} \leq \sigma_{\max}(A)$ for all $i$. Then
\begin{equation}\label{eq:momentX}
\begin{aligned}
\e \lrabs{\inner{X}{\theta}}^{1+\gamma} 
&= \e \lrabs{\theta^T A S}^{1+\gamma} 
= \e \lrabs{\inner{S}{u}}^{1+\gamma} 
= \e \abs{\sum_{i=1}^\dim S_i u_i}^{1+\gamma} 
\leq \sum_{i=1}^\dim \e \lrabs{S_i u_i}^{1+\gamma} \\
&= \sum_{i=1}^\dim \e \lrabs{S_i}^{1+\gamma} \lrabs{u_i}^{1+\gamma} 
\leq M \sum_{i=1}^\dim \lrabs{u_i}^{1+\gamma} 
\leq M \dim \sigma_{\max}(A)^{1+\gamma} \leq M \dim s_M^{1+\gamma}.
\end{aligned}
\end{equation}
Lemma~\ref{lem:1-plus-eps-chebyshev} implies that, for $\eps_1>0$ to be fixed later, and for
\begin{align} \label{eqn:N-lower-bound}
N > \left(\frac{8 M \dim s_M^{1+\gamma}}{\eps_1^2 \delta}\right)^{3/\gamma}, 
\end{align}
we have $P(|\tilde{h} - h| > \eps_1) \leq \delta$.

\nnote{I think some more details should be given above.}

Lemmas~\ref{lem:centroid-scaling} and \ref{lem:equivariance} give that $r B_2^\dim \subseteq \Gamma X$ for $r := s_m/\sqrt{\dim} \leq \sigma_{\min}(A)/\sqrt{\dim}$.
It follows that $h \geq r$. If $|\tilde{h} - h| \leq \eps_1$ and $\eps_1 \leq r/2$, then we have
\[
\left\lvert\frac{1}{h} - \frac{1}{\tilde h} \right\rvert= \frac{\abs{h - \tilde h}}{h \tilde h} \leq \frac{\eps_1}{r (r-\eps_1)} \leq \frac{2 \eps_1}{r^2},
\]
which in turn gives, when $\eps_1 = \min \{r^2 \eps/2, r/2\}$,
\[
P\left(\left|\frac{1}{\tilde{h}} - \frac{1}{h}\right| \leq \eps\right)
\geq P\left(\left|\frac{1}{\tilde{h}} - \frac{1}{h}\right| \leq \frac{2\eps_1}{r^2}\right)
\geq P(|\tilde{h} - h| \leq \eps_1) \geq 1 - \delta.
\]
\nnote{I don't follow ``details'' below and also some of the main text which I moved to ``details'' and I don't see the need for it. If you agree then ``details'' below can be removed.}
\details{To conclude we can assume that $\eps$ is at most half of the circumradius of $(\Gamma X)^\polar$. The circumradius is at most $1/r$. 
Thus, $r^2 \eps/2 \leq r/2$ and it is enough to take $\eps_1 = r^2 \eps /2 \leq \sigma_{\min}(A)^2\eps/(4\dim)$}
\details{Thus, it is enough to take $\eps_1 \leq \min\{ s_m^2\eps/(4\dim) , s_m/(2\sqrt{\dim}) \} \leq \min\{ \sigma_{\min}(A)^2\eps/(4\dim) , \sigma_{\min}(A)/(2\sqrt{\dim}) \}$} Plugging in the value of 
$\epsilon_1$ and, in turn of $r$, into \eqref{eqn:N-lower-bound} gives that it suffices to take 
\[N > \poly_\gamma(n, M, 1/s_m, s_M, 1/\eps, 1/\delta). 
\]
%
\end{proof}

\subsection{Membership oracle for the centroid body}  
We now describe how the weak membership oracle for the centroid body $\Gamma X$ is constructed using 
the weak membership oracle for $(\Gamma X)^\polar$, provided by Subroutine~\ref{sub:polar-cvar-oracle}. 

We will use the following notation: For a convex body $K \in \R^n$, $\eps, \delta >0$, $R \geq r > 0$ such that 
$r B_2^n \subseteq K \subseteq R B_2^n$, oracle 
$\mathsf{WMEM}_K(\epsilon, \delta, R, r)$ is an $(\epsilon, \delta)$-weak membership oracle for $K$. 
Similarly, oracle $\mathsf{WVAL}_K(\epsilon, \delta, R, r)$ is an $(\epsilon, \delta)$-weak validity oracle.
Lemma~\ref{lem:centroid-scaling} along with the equivariance of $\Gamma$ (Lemma~\ref{lem:equivariance}) 
gives $(s_m/\sqrt{n}) B_2^\dim \subseteq \Gamma X \subseteq (s_M \sqrt{n}) B_2^\dim$.
Then $1/(\sqrt{\dim}s_M) B_2^\dim \subseteq (\Gamma X)^{\polar} \subseteq (\sqrt{\dim}/s_m) B_2^\dim$.
Set $r := 1/(\sqrt{\dim}s_M)$ and $R := \sqrt{\dim}/s_m)$.

\myparagraph{Detailed description of Subroutine~\ref{sub:cvar-oracle}.}
There are two main steps:
\begin{enumerate}
\item \label{step1} Use Subroutine~\ref{sub:polar-cvar-oracle} to get an $(\epsilon_2, \delta)$-weak membership oracle $\mathsf{WMEM}_{(\Gamma X)^\polar}(\epsilon_2, \delta, R, r)$ for $(\Gamma X)^\polar$. 
Theorem~4.3.2 of \cite{GLS} (stated as Theorem~\ref{thm:wmem-to-wviol} here) is used in Lemma~\ref{lem:subroutine-2-correctness} to get an algorithm to implement an $(\epsilon_1, \delta)$-weak validity oracle $\mathsf{WVAL}_{(\Gamma X)^\polar}(\epsilon_1, \delta, R, r)$ running in
oracle polynomial time; $\mathsf{WVAL}_{(\Gamma X)^\polar} \allowbreak (\epsilon_1, \delta, R, r)$ invokes
$\mathsf{WMEM}_{(\Gamma X)^\polar}(\epsilon_2, \delta/Q, R, r)$ a polynomial number of times, specifically $Q=\poly(n, \log R)$ (see proof of Lemma~\ref{lem:subroutine-2-correctness}). 
The proof of Theorem~4.3.2 can be modified so that $\epsilon_2 \geq 1/\poly(1/\epsilon_1, R, 1/r)$. 

\item \label{step2} Lemma~4.4.1 of \cite{GLS} (Lemma~\ref{lem:wviol-polar-wmem} here) gives an algorithm to construct an $(\epsilon, \delta)$-weak membership oracle
$\mathsf{WMEM}_{\Gamma X}(\epsilon, \delta, 1/r, 1/R)$ from $\mathsf{WVAL}_{(\Gamma X)^\polar}(\epsilon_1, \delta, R, r)$.
Lemma~4.4.1 in \cite{GLS} shows
$\mathsf{WMEM}_{\Gamma X}(\epsilon, \delta, 1/r, 1/R)$ calls $\mathsf{WVAL}_{(\Gamma X)^\polar}(\epsilon_1, \delta, R, r)$ once, with
$\epsilon_1 \geq 1/ \poly(1/\epsilon, \norm{y}, 1/r)$ (where $y$ is the query point).
\details{Specifically, given $\eps$ and candidate point $y$, it is sufficient to call WVAL with $\eps_1 = \frac{1}{2} \frac{r \eps}{\norm{y} + r \norm{y} + r \eps}$}
\end{enumerate}

\begin{subroutine}[ht]
\caption{Weak Membership Oracle for $\Gamma X$}\label{sub:cvar-oracle}
\begin{algorithmic}[1]
\Require Query point $x \in \RR^\dim$,
samples from symmetric ICA model $X = AS$,
bounds $s_M \geq \sigma_{\max}(A)$, $s_m \leq \sigma_{\min}(A)$,
closeness parameter $\eps$,
failure probability $\delta$, 
access to a weak membership oracle for $(\Gamma X)^{\polar}$.
\Ensure $(\epsilon, \delta)$-weak membership decision for $x \in \Gamma X$.

\State Construct $\mathsf{WVAL}_{(\Gamma X)^\polar}(\epsilon_1, \delta, R, r)$ by invoking
$\mathsf{WMEM}_{(\Gamma X)^\polar}(\epsilon_2, \delta/Q, R, r)$. (See Step~\ref{step1} in the detailed description.)

\State Construct $\mathsf{WMEM}_{\Gamma X}(\epsilon, \delta, 1/r, 1/R)$ by invoking 
$\mathsf{WVAL}_{(\Gamma X)^\polar}(\epsilon_1, \delta, R, r)$. (See Step~\ref{step2} in the detailed description.)

\State Return the output of running $\mathsf{WMEM}_{\Gamma X}(\epsilon, \delta, 1/r, 1/R)$ on $x$. 
\end{algorithmic}
\end{subroutine}

\begin{lemma}[Correctness of Subroutine~\ref{sub:cvar-oracle}]\label{lem:subroutine-2-correctness}
Let $X=AS$ be given by a symmetric ICA model such that for all $i$ we have $\e (\abs{S_i}^{1+\gamma}) \leq M < \infty$ and normalized so that $\e \abs{S_i} = 1$.
Then, given a query point $x \in \RR^\dim$, $0<\eps \leq n^2$, $\delta > 0$, $s_M \geq \sigma_{\max}(A)$, and $s_m \leq \sigma_{\min}(A)$, Subroutine~\ref{sub:cvar-oracle} is an $\eps$-weak membership oracle for $x$ and $\Gamma X$ with probability $1-\delta$ using time and sample complexity
\(
\poly(n, M, 1/s_m, s_M, 1/\eps, 1/\delta).
\)\lnote{query time should also depend on query}
\end{lemma}


\begin{proof}
We first prove that $\mathsf{WVAL}_{(\Gamma X)^\polar}(\epsilon_1, \delta, R, r)$ 
(abbreviated to $\mathsf{WVAL}_{(\Gamma X)^\polar}$ hereafter) works correctly. To this end
we need to show that for any given input, $\mathsf{WVAL}_{(\Gamma X)^\polar}$ acts as an $\epsilon_1$-weak validity
oracle with probability at least $1-\delta$. Oracle $\mathsf{WVAL}_{(\Gamma X)^\polar}$ makes $Q$ queries 
to $\mathsf{WMEM}_{(\Gamma X)^\polar}(\epsilon_2, \delta/Q, R, r)$. If the answer to all these queries were correct then Theorem~4.3.2 from \cite{GLS} would apply
and would give that $\mathsf{WVAL}_{(\Gamma X)^\polar}$ outputs an answer as expected. Since these $Q$ queries are adaptive we
cannot directly apply the union bound to say that the probability of all of them being correct is at least
$1 - Q (\delta/Q) = 1-\delta$. However, a more careful bound allows us to do essentially that.


Let $q_1, \dotsc, q_k$ be the sequence of queries, where $q_i$ depends on the result of the previous queries.
For $i=1, \dotsc, k$, let $B_i$ be the event that the answer to query $q_i$ by Subroutine~\ref{sub:polar-cvar-oracle} is not correct according to the definition of the oracle it implements. 
These events are over the randomness of Subroutine~\ref{sub:polar-cvar-oracle} and event $B_i$ involves the randomness of $q_1, \dotsc, q_i$, as the queries could be adaptively chosen.
By the union bound, the probability that all answers are correct is at least $1-\sum_{i=1}^k \Pr(B_i)$. 
It is enough to show that $\Pr(B_i) \leq \delta/Q$.
To see this, we can condition on the randomness associated to $q_1, \dotsc q_{i-1}$. That makes $q_i$ deterministic, and the probability of failure is now just the probability that Subroutine~\ref{sub:polar-cvar-oracle} fails.
More precisely, $\Pr(B_i \giventhat q_1, \dotsc, q_{i-1}) \leq \delta/Q$, so that
\begin{align*}
\Pr(B_i) &= \int \Pr(B_i \giventhat q_1, \dotsc, q_{i-1}) \Pr(q_1, \dotsc, q_{i-1}) \, dq_1, \dotsc, dq_{i-1} \leq \delta/Q.
\end{align*}

This proves that the first step works correctly. Correctness of the second step follows directly because the
algorithm for construction of the oracle involves a single call to the input oracle as mentioned in 
Step~\ref{step2} of the detailed description. 

Finally, to prove that the running time of Subroutine~\ref{sub:cvar-oracle} is as claimed the main thing to note is that, as mentioned in Step~\ref{step1} of the detailed description, $\epsilon_2$ is polynomially small in $\epsilon_1$ and $\epsilon_1$ is polynomially small in $\epsilon$ and so $\epsilon_2$ is polynomially small in $\epsilon$.
\details{
	To see the parameters passed to the different oracles, let $y \in \QQ^\dim$ be the candidate point, $\eps > 0$ be the (rational) error given to the subroutine.
	We make a single call to the weak validity oracle with candidate point $c := y$, $\gamma := 1$, and error parameter \[\eps_{val} = \frac{1}{2} \frac{r \eps}{\norm{y} + r \norm{y} + r \eps}.\]

	We first define a weak separation oracle which, for a point $y \in \QQ^\dim$ two parameters $\eps_{sep}, \beta \in (0,1) \cap \QQ$, and convex body $(K; n, R, r, a_0)$, first calls the weak-membership with error $\eps_{mem} = \frac{r \eps_{sep}}{4R}$.
	Depending on the response, the separation algorithm either terminates or calls the weak-membership $n$ more times with error parameter \[ \eps_{mem} = C \frac{\beta^2 r^3 \eps_{sep}^2}{n^5 R^2 (R + r)} \]
	for an absolute constant $C$.

	The weak violation oracle works as follows (for parameters $c \in \QQ^n, \gamma, \eps_{val} \in \QQ$): in the \cite{GLS} proof of their Theorem~4.3.2, one defines a shallow-cut oracle which, for an ellipsoid $E(A,a)$ and matrix $Q$ such that $(n+1)^2 A^{-1} = QQ$, uses a single call to the separation oracle just described with parameter $y = Qa$, $\beta = (n+1)^{-1}$, and $\eps_{sep} = \min\{ (n+2)^{-2}, \eps_{val}/\norm{Q^{-1}} \}$.
	Using this shallow-cut oracle, one invokes the shallow-cut ellipsoid algorithm for the body $K' = K \cap \{x \suchthat c^T x \leq \gamma \}$ with radius $R$ and target volume $\eps_{vol} = ((r \eps_{val})/(2Rn))^2$.
	This call to the shallow-ellipsoid method takes time $n + \langle R \rangle + \langle \eps_{vol} \rangle$ where $\langle \cdot \rangle$ is the encoding length, unfortunately overloading our current use of inner product notation.
}
\end{proof}

\section{Orthogonalization via the uniform distribution in the centroid body}\label{sec:orthogonalization_uniform}

The following lemma says that linear equivariance allows orthogonalization: \nnote{something more 
meaningful needs to be said for introducing the following lemma.} \nnote{I moved the following lemma here which I think is the right place for it as it's here and only here that it gets used.}
\nnote{In the following lemma is that two different types of notations are being used for distributions, one using calligraphic letters and the other normal letters. This can be confusing.} \nnote{I suppose by calligraphic P sub S is meant the distribution of S. But this notation has not been
defined I think.}

\begin{lemma}\label{lem:orthogonalizer}
Let $U$ be a family of $n$-dimensional product distributions. Let $\bar U$ be the closure of $U$ under invertible linear transformations.
Let $Q(\measure)$ be an $n$-dimensional distribution defined as a function of $\measure \in \bar U$. Assume that $U$ and $Q$ satisfy:
\begin{enumerate}
\item\label{item:sym}
For all $\measure \in U$, $Q(\measure)$ is absolutely symmetric.
\item\label{item:equivariant} $Q$ is linear equivariant (that is, for any invertible linear transformation $T$ we have $Q(T\measure) = T Q(\measure)$).
\item\label{item:positive}
For any $\measure \in \bar U$, $\cov(Q(\measure))$ is positive definite.
\end{enumerate}
Then for any symmetric ICA model $X=AS$ with $\measure_S \in U$ we have $\cov(Q(\measure_X))^{-1/2}$ is an orthogonalizer of $X$.
\end{lemma}
\begin{proof}
Consider a symmetric ICA model $X=AS$ with $\measure_S \in U$.
Assumptions \ref{item:sym} and \ref{item:positive} imply $ D := \cov(Q(\measure_S))$ is diagonal and positive definite.
This with Assumption \ref{item:equivariant} gives $\cov(Q(\measure_X)) = \cov(A Q(\measure_S)) = A D A^T = A D^{1/2} (A D^{1/2})^T$.
Let $B = \cov(Q(\measure_X))^{-1/2}$ (the unique symmetric positive definite square root).
We have $B = R D^{-1/2} A^{-1}$ for some unitary matrix $R$ (see \cite[pg 406]{horn2012matrix}).
\details{To see this: $I = B \cov(Q(\measure_X)) B = B A D^{1/2} (B A D^{1/2})^T$, which implies $B A D^{1/2}$ is an unitary matrix, say, $R$. The claim follows.}
Thus, $B A = R D^{-1/2}$ has orthogonal columns, that is, it is an orthogonalizer for $X$.
\end{proof}

The following lemma applies the previous lemma to the special case when the distribution 
$Q(\measure)$ is the uniform distribution on $\Gamma \measure$. 

\begin{lemma}\label{lemma:uniform-orthogonalizer}
Let $X$ be a random vector drawn from a symmetric ICA model $X = AS$ such that for all $i$ we have $0 < \e \abs{S_i} < \infty $. Let Y be uniformly random in $\Gamma X$. Then $\cov(Y)^{-1/2}$ is an orthogonalizer of $X$.
\end{lemma}

\begin{proof}
We will use Lemma~\ref{lem:orthogonalizer}. 
After a scaling of each $S_i$, we can assume without loss of generality that $\e(S_i) = 1$. This will allow us to use Lemma~\ref{lem:centroid-scaling}.
Let $U$ = \{$\measure_W \suchthat \measure_W$ is an absolutely symmetric product distribution and $\e \abs{W_i} = 1 $, for all $i$\}. For $\measure \in \bar U$, let $Q(\measure)$ be the uniform distribution on the centroid body of $\measure$. \nnote{I changed ``centroid body'' to ``the uniform distribution on the centroid body.''}
For all $\measure_W \in U$, the symmetry of the $W_{i}$'s implies that $\Gamma \measure_W$, that is, $Q(\measure_W)$, is absolutely symmetric.
By the equivariance of $\Gamma$ (from Lemma~\ref{lem:equivariance}) and Lemma~\ref{lem:centroid-scaling} it follows that $Q$ is linear equivariant. Let $\measure \in \bar U$. Then there exist $A$ and $\measure_W \in U$ such that $\measure = A \measure_W$.  
So we get $\cov(Q(\measure)) = \cov(AQ(\measure_W)) = A \cov(\Gamma \measure_W) A^T$. \nnote{I am not sure what the purpose of
the last two sentences is.}
From Lemma~\ref{lem:centroid-scaling} we know $B_1^\dim \subseteq \Gamma \measure_W$
so that $\cov(\Gamma \measure_W)$ is a diagonal matrix with positive diagonal entries. 
This implies that $\cov(Q(\measure))$ is positive definite and thus by Lemma~\ref{lem:orthogonalizer}, $\cov(Y)^{-1/2}$ is an orthogonalizer of $X$. 
\end{proof}

\begin{algorithm}[ht]
\caption{Orthogonalization via the uniform distribution in the centroid body}\label{alg:orthogonalization_uniform}
\begin{algorithmic}[1]
\Require Samples from symmetric ICA model $X = AS$, bounds $s_M \geq \sigma_{\max}(A)$, $s_m \leq \sigma_{\min}(A)$, error parameters $\eps$ and $\delta$, access to an $(\epsilon, \delta)$-weak membership oracle for $\Gamma X$ provided by Subroutine~\ref{sub:cvar-oracle}.
\Ensure A matrix $B$ which orthogonalizes the independent components of $X$.
\State Let $\tilde \Sigma$ be an estimate of $\cov(\Gamma X)$ obtained
\iflong via Theorem~\ref{thm:covariance_estimation} \fi
\ifshort by generating uniform samples in $\Gamma X$ using Subroutine~\ref{sub:cvar-oracle} and \fi
sampling algorithm such as the one in \cite{DBLP:journals/jacm/DyerFK91}
with $\eps_c = \eps/(2 (\dim+1)^4)$, $r= s_m /\sqrt{\dim}$, $R = s_M \sqrt{\dim}$, and same $\delta$.
\State Return $B = \tilde{\Sigma}^{-1/2}$.
\end{algorithmic}
\end{algorithm}

\begin{theorem}[Correctness of algorithm \ref{alg:orthogonalization_uniform}] \label{thm:correctness_uniform_orth}
Let $X=AS$ be given by a symmetric ICA model such that for all $i$ we have $\e (\abs{S_i}^{1+\gamma}) \leq M < \infty$ and normalized so that $\e \abs{S_i} = 1$.
Then, given $0<\eps \leq \dim^2$, $\delta > 0$, $s_M \geq \sigma_{\max}(A)$ , $s_m \leq \sigma_{\min}(A)$
, Algorithm~\ref{alg:orthogonalization_uniform} outputs
a matrix $B$ so that
$ \norm{A^T B^T B A - D}_{2} \leq \eps$,
for a diagonal matrix $D$ with diagonal entries $d_1, \dotsc, d_\dim$ satisfying $1/(\dim+1)^2 \leq d_i \leq 1$.
with probability at least $1-\delta$ using $\poly_\gamma(n, M, 1/s_m, s_M,1/\eps, 1/\delta)$ time and sample complexity.\lnote{philosophical issue}
\details{current use of $1+\gamma$ Chebyshev imposes $1/\delta$ dependence. Might be improvable to $\log 1/\delta$ via something more involved such as median of means}
\end{theorem}
\begin{proof}
From Lemma~\ref{lem:centroid-scaling} we know $B_1^\dim \subseteq \Gamma S \subseteq [-1,1]^\dim$. 
Using the equivariance of $\Gamma$ (Lemma~\ref{lem:equivariance}), we get $\sigma_{\min}(A)/\sqrt{\dim} B_2^\dim \subseteq \Gamma X \subseteq \sigma_{\max}(A) \sqrt{\dim}B_2^\dim$.
Thus, to satisfy the roundness condition of  Theorem~\ref{thm:covariance_estimation}
we can take $r := s_m/\sqrt{\dim} \leq \sigma_{\min}(A)/\sqrt{\dim}$, $R := s_M \sqrt{\dim} \geq \sigma_{\max}(A) \sqrt{\dim}$.

Let $\tilde \Sigma$ be the estimate of $\Sigma := \cov(\Gamma X)$ computed by the algorithm.
Let $\tilde \Delta := A^{-1} \tilde \Sigma A^{-T}$ be the estimate of $\Delta := \cov{\Gamma S}$ obtained from $\tilde \Sigma$ according to how covariance matrices transform under invertible linear transformations of the underlying random vector.
As in the proof of Lemma~\ref{lem:orthogonalizer}, we have $\tilde \Sigma = A \tilde \Delta A^T$ and $B = R \tilde \Delta^{-1/2} A^{-1}$ for some unitary matrix $R$.
Thus, we have $A^T B^T B A = \tilde \Delta^{-1}$. 
It is natural then to set $D := \Delta^{-1} = \cov(\Gamma S)^{-1}$. Let $d_1, \dotsc, d_\dim$ be the diagonal entries of $D$.
We have, using Lemma~\ref{lem:inversion},
\begin{align}
\norm{A^T B^T B A - D}_2 
&= \norm{\tilde \Delta^{-1} - \Delta^{-1}}_2 \notag\\
&= \norm{\Delta^{-1}}_2 \frac{\norm{\Delta^{-1} (\tilde \Delta - \Delta)}_2}{1-\norm{\Delta^{-1} (\tilde \Delta - \Delta)}_2}.\label{equ:inversion}
\end{align}
As in \eqref{eq:inverse-stability}, we show that $\norm{\Delta^{-1} (\tilde \Delta - \Delta)}_2$ is small:

We first bound $(d_i)$, the diagonal entries of $D = \Delta^{-1}$.
Let $d_{\max} := \max_i d_i$ and $d_{\min} := \min_i d_i$. 
We find simple estimates of these quantities: 
We have $d_{\min} = 1/\norm{\Delta}_2$ and $\norm{\Delta}_2$ is the maximum variance of $\Gamma S$ along coordinate axes. 
From Lemma~\ref{lem:centroid-scaling} we know $\Gamma S \subseteq [-1,1]^\dim$, so that $\norm{\Delta}_2 \leq 1$ and $d_{\min} \geq 1$. 
Similarly, $d_{\max} = 1/\sigma_{\min}(\Delta)$, where $\sigma_{\min}(\Delta)$ is the smallest diagonal entry of $\Delta$. In other words, it is the minimum variance of $\Gamma S$ along coordinate axes. 
From Lemma~\ref{lem:centroid-scaling} we know $\Gamma S \supseteq B_1^\dim$, so that $\Gamma S \supseteq [-e_i, e_i]$ for all $i$ and Lemma~\ref{lem:minvariance} below implies $\sigma_{\min}(\Delta) \geq 1/(\dim+1)^2$. That is, $d_{\max} \leq (\dim+1)^2$.

From the bounds on $d_i$, Theorem~\ref{thm:covariance_estimation} and the fact discussed after it, we have
\begin{align*}
\norm{\Delta^{-1} (\tilde \Delta - \Delta)} 
\leq d_{\max} \norm{\Delta} \eps_c 
\leq (\dim+1)^2 \eps_c 
\leq 1/2,
\end{align*}
when $\eps_c \leq 1/(2(n+1)^2)$.\details{This together with the choice of $\eps_c$ below gives the assumption that $\eps <n^2$.}

This in \eqref{equ:inversion} with Theorem~\ref{thm:covariance_estimation} again gives
\begin{align*}
\norm{A^T B^T B A - D} &
\leq \\ 2 \norm{\Delta^{-1}} & \norm{\Delta^{-1} (\tilde \Delta - \Delta)} 
\leq 2 \norm{\Delta^{-1}}^2 \norm{\tilde \Delta - \Delta} 
\leq 2 d_{\max}^2 \eps_c \norm{\Delta} 
\leq 2 (\dim+1)^4 \eps_c.
\end{align*}
The claim follows by setting $\eps_c = \eps/(2 (\dim+1)^4)$.
The sample and time complexity of the algorithm comes from the calls to Subroutine~\ref{sub:cvar-oracle}. The 
number of calls is given by Theorem~\ref{thm:covariance_estimation}. This leads to the complexity as claimed.
\end{proof}

\iflong
\begin{lemma}\label{lem:minvariance}
Let $K \subseteq \RR^\dim$ be an absolutely symmetric convex body such that $K$ contains the segment $[-e_1, e_1] = \conv \{ e_1, -e_1\}$ (where $e_1$ is the first canonical vector). Let $X = (X_1, \dotsc, X_\dim)$ be uniformly random in $K$. Then $\var(X_1) \geq 1/(n+1)^2$.
\end{lemma}
\fi
\begin{proof}
Let $D$ be a diagonal linear transformation so that $DK$ isotropic. Let $d_{11}$ be the first entry of $D$.
It is known that any $\dim$-dimensional isotropic convex body is contained in the ball of radius $\dim+1$ \cite{MilmanPajor, sonnevend1989applications},\cite[Theorem 4.1]{KLS}.
Note that $D K$ contains the segment $[- d_{11} e_1, d_{11} e_1]$. This implies $d_{11} \leq (\dim+1)$. Also, by isotropy we have, $1 = \var(d_{11} X_1) = d_{11}^2 \var(X_1)$. The claim follows.
\end{proof}

\iflong
\section{Gaussian damping} \label{sec:gaussian_damping}
In this section we give an efficient algorithm for the heavy-tailed ICA problem when the ICA matrix is a
unitary matrix; 
no assumptions on the existence of moments of the $S_i$ will be required. 


The basic idea behind our algorithm is simple and intuitive: using $X$ we construct another ICA model $X_R = A S_R$, where $R > 0$ is a parameter which will be chosen later.
The components of $S_R$ have light-tailed distributions; in particular, all moments exist.
We show how to generate samples of $X_R$ efficiently using samples of $X$.
Using the new ICA model, the matrix $A$ can be estimated by applying existing ICA algorithms.

For a random variable $Z$ we will denote its the probability density function by $\rho_Z(\cdot)$.
The density of $X_R$ is obtained by multiplying the density of $X$ by a Gaussian damping factor.
More precisely,
\begin{align*}
\rho_{X_R}(x) \propto \rho_X(x) e^{-\norm{x}^2/R^2}.
\end{align*}
Define
\begin{align*}
K_{X_R} := \int_{\R^n} \rho_X(x) e^{-\norm{x}^2/R^2} \ud x,
\end{align*}
then
\begin{align*}
\rho_{X_R}(x) = \frac{1}{K_{X_R}} \rho_X(x) e^{-\norm{x}^2/R^2}.
\end{align*}

We will now find the density of $S_R$. Note that if $x$ is a value of $X_R$, and $s = A^{-1}x$ is the corresponding value of $S_R$,
then we have
\begin{align*}
\rho_{X_R}(x)
= \frac{1}{K_{X_R}} \rho_X(x) e^{-\norm{x}^2/R^2} 
= \frac{1}{K_{X_R}} \rho_S(s) e^{-\norm{As}^2/R^2}
= \frac{1}{K_{X_R}} \rho_S(s) e^{-\norm{s}^2/R^2} 
=: \rho_{S_R}(s),
\end{align*}
where we used that $A$ is a unitary matrix so that $\norm{As}=\norm{s}$. 
Also, $\rho_X(x) = \rho_S(s)$ follows from the change of variable formula and the fact that $\lrabs{\det A}=1$. 
We also used crucially the fact that the Gaussian distribution is spherically-symmetric.
We have now specified the new ICA model $X_R = A S_R$, and what remains is to show how to generate samples of $X_R$. 

\myparagraph{Rejection sampling.} Given access to samples from $\rho_X$ we will use rejection
sampling (see e.g. [Robert--Casella] \lnote{xxx}) to generate samples from $\rho_{X_R}$.

\begin{enumerate}
\item Generate $x \sim \rho_X$.
\item Generate $z \sim U[0,1]$.
\item If $z \in [0, e^{-\norm{x}^2/R^2}]$, output $x$; else, go to the first step.
\end{enumerate}

The probability of outputting a sample with a single trial in the above algorithm is $K_{X_R}$. Thus, the expected number of trials in the above algorithm for generating a sample is $1/K_{X_R}$.

We now choose $R$. 
There are two properties that we want $R$ sufficiently large so as to satisfy: (1) $K_{X_R} \geq C_1$ and
$\abs{\cum_4 S_{j, R}} \geq 1/n^{C_2}$ \jnote{Changed from $C$ to $C_2$} where $C_1 \in (0, 1/2)$ and $C_2 > 0$ are constants. Such a choice of $R$ exists and can be made efficiently; 
we outline this after the statement of Theorem~\ref{thm:ICA-orthogonal-damping}. \nnote{Actually, the statement about cumulants seems to require more care if some of the Si are not heavy-tailed}
Thus, the expected number of trials in rejection sampling before generating a sample is bounded above by
$1/C_1$. The lower bound on $K_{X_R}$ will also be useful in bounding the moments of the $S_{j,R}$, where 
$S_{j,R}$ is the random variable obtained by Gaussian damping of $S_j$ with parameter $R$, that is to say
\begin{align*}
\rho_{S_{j,R}}(x) \propto \rho_{S_j}(x) e^{-\norm{x}^2/R^2}`
\end{align*}


Define
\begin{align*}
K_{S_{j,R}} := \int_{\R}\rho_{s_j}(s_j) e^{-s_j^2/R^2} ds_j
\end{align*}
and let $K_{S_{{-j},R}}$ be the product of $K_{S_{k, R}}$ over $k \in [n]\setminus \{j\}$.
By $s_{-j} \in  \R^{n-1}$ we denote the vector $s \in \R^n$ with its $j$th element removed, then notice that
\begin{align} \label{eqn:KR-product}
K_{X_R} &= K_{S_R} = K_{S_{1, R}} K_{S_{2, R}} \ldots K_{S_{n, R}}, \\
K_{X_R} &= K_{S_{j, R}} K_{S_{{-j},S}}.
\end{align}

We can express the densities of individual components of $S_R$ as follows:
\begin{align*}
\rho_{S_{j,R}}(s_j) = \int_{\R^{n-1}} \rho_{S_R}(s) \, d s_{{-j}}
= \frac{1}{K_{X_R}} \int_{\R^{n-1}} \rho_S(s) e^{-\norm{s}^2/R^2} \, d s_{{-j}}
= \frac{K_{S_{{-j},R}} }{K_{X_R}} \rho_{S_{j}}(s_j) \, e^{-s_j^2/R^2}.
\end{align*}

This allows us to derive bounds on the moments of $S_{j,R}$:
\begin{align}
\E[S_{j,R}^4]
&= \frac{K_{S_{{-j},R}}}{K_{X_R}} \int_{\R} s_j^4 \, \rho_{S_j}(s_j) e^{-s_j^2/R^2} d s_j \nonumber \\
&\leq \frac{K_{S_{{-j},R}}}{K_{X_R}} \left(\max_{z \in \R} z^4 e^{-z^2/R^2}\right) \int_{\R} \rho_{S_j}(s_j) \, d s_j 
< \frac{K_{S_{{-j},R}}}{K_{X_R}} R^4 
\leq \frac{1}{K_{X_R}} R^4 \nonumber \\
&\leq \frac{1}{C_1} R^4. \label{eqn:damping-moment-bound}
\end{align}

We now state Theorem~4.2 from \cite{GVX} in a special case by setting parameters $k$ and $k_i$ in that theorem
to $4$ for $i \in [n]$. The algorithm analyzed in Theorem~4.2 of \cite{GVX} is called Fourier PCA. 
 \begin{theorem}\cite{GVX}\label{thm:ICA}
  Let $X \in \R^n$ be given by an ICA model $X=AS$ where $A \in \R^{ n
    \times n}$ is unitary 
  and the $S_i$ are mutually independent, $\E[S_i^4] \le M_4$ for some positive constant $M_4$,
and $\abs{\cum_{4}(S_i)} \ge \Delta$. For
  any $\epsilon > 0$ 
  with probability at least $1-\delta$, Fourier PCA will recover vectors $\{b_1, \ldots, b_n\}$
  such that there exist signs $\alpha_i = \pm 1$ and a permutation $\pi:[n]\to [n]$ satisfying
  \begin{align*}
    \norm{A_i - \alpha_i b_{\pi(i)}} \le \epsilon,
  \end{align*}
  using $\poly(n, M_4, 1/\Delta, 1/\epsilon, 1/\delta)$ samples. 
The running time of the algorithm is also of the same form. 
\end{theorem}

\nnote{To check: Is there an upper bound needed on epsilon in the hypothesis of the previous theorem?}

Combining the above theorem with Gaussian damping gives the following theorem. As previously noted, 
since we are doing rejection 
sampling in Gaussian damping, the expected number of trials to generate $N$ samples of $X_R$ is $N/K_{X_R}$. 
One can similarly prove high probability guarantees for the number of trials needed to generate $N$ samples.

We remark that the choice of $R$ in the above theorem can be made algorithmically in an efficient way.
Theorem~\ref{thm:cumulant-damping} below shows that as we increase $R$ the cumulant $\cum_4(S_{j,R})$ goes to infinty.
This shows that for any $\Delta>0$ there exists $R$ so as to satisfy the condition of the above theorem, namely
$\abs{\cum_{4}(S_{i,R})} \ge \Delta$. 
We now briefly
indicate how such an $R$ can be found efficiently (same sample and computational costs as in 
Theorem~\ref{thm:ICA-orthogonal-damping} above):
For a given $R$, we can certainly estimate $K_{X_R} = \int_{\R^n} \rho_X(x) e^{-\norm{x}^2/R^2} dx$ from samples, 
i.e. by the empirical mean
$\frac{1}{N}\sum_{i \in [N]} e^{-\norm{x^{(i)}}^2/R^2}$ of samples $x^{(1)}, \ldots, x^{(N)}$.\lnote{what's the distribution of the samples? is this really efficient?} 
This allows us to search for $R$ so that $K_{X_R}$ is as large as we want. 
This also gives us an upper bound on the fourth moment via Eq. \eqref{eqn:damping-moment-bound}.
To ensure that the fourth cumulants of all $S_{i,R}$ are large, note that for $a \in \R^n$ we have 
$\cum_4(a_1S_{1,R} + \dotsb + a_n S_{n,R}) = \sum_{i\in [n]} a_i^4 \cum_4(S_{i,R})$. 
We can estimate this quantity empirically, and minimize 
over $a$ on the unit sphere (the minimization can be done, e.g., using the algorithm in \cite{FJK}). 
This would give an estimate of $\min_{i \in [n]}\cum_4(S_{i,R})$
and allows us to search for an appropriate $R$.
\jnote{As seen in our experiments, it may not be this simple, since the optimization suggested would require access to samples of $S$, or some self-reinforcing procedure that first found approximate $S$, then estimated $R$ from these samples, finding then even more accurate samples of $S$, and so on.}

For the algorithm to be efficient, we also need $K_{X_R} \geq C_1$. This is easily achieved as we can empirically
estimate $K_{X_R}$ using the number of trials required in rejection sampling, and search for sufficiently large $R$ that makes the estimate sufficiently larger than $C_1$.

\nnote{
TODO: (1) Show for some example heavy-tailed distributions (e.g., Pareto) what the value of a good $R$ is in the 
above theorem. 
(2) \cite{GVX} assumes that the $S_j$ are all centered. This is not true in our application. We can center our r.v. with
respect to the empirical mean (this would require recomputing of the fourth moment etc. for the shifted $S_j$), but that's not
exact centering because of the use of the empirical mean instead of the actual mean. Better way is to use symmetrization.}

\subsection{The fourth cumulant of Gaussian damping of heavy-tailed distributions}
It is clear
that if r.v. $X$ is such that $\E (X^4) = \infty$ and $\E (X^2) < \infty$, then
$\cum_4(X_R) = \E (X_R^4) - 3 (\E (X_R^2))^2 \to \infty$ as $R \to \infty$. 
However, it does not seem clear when we have $\E (X^2) = \infty$ as well. We will show that in this case we also
get $\cum_4(X_R) \to \infty$ as $R \to \infty$.

We will confine our discussion to symmetric random variables for simplicity of exposition; for the purpose of
our application of the theorem this is w.l.o.g. by the argument in Sec.~\ref{sec:symmetrization}.

\begin{theorem}\label{thm:cumulant-damping}
Let $X$ be a symmetric real-valued random variable with $\E (X^4) = \infty$. 
Then $\cum_4(X_R) \to \infty$ as $R \to \infty$.
\end{theorem}
\begin{proof}
Fix a symmetric r.v. $X$ with $\E X^2 = \infty$; as previously noted, if $\E X^2 < \infty$ then the theorem is
easily seen to be true. 
Since $X$ is symmetric and we will be interested in the fourth cumulant, we can restrict our attention
to the positive part of $X$. So in the following we will actually assume that $X$ is a positive random variable. 
Fix $C > 100$ to be any large positive constant.
Fix a small positive constant $\eps_1 < \frac{1}{100 \, C}$. Also fix another small positive constant $\eps_2 < 1/10$.
Then there exists 
$a > 0$ such that
\begin{align} \label{eqn:a_eps1}
\Pr[X \geq a] =  \int_{a}^\infty \rho_X(x) dx \leq \eps_1.
\end{align}
\jnote{Equals and inequality were interchanged. Is this what was really intended?}

Let $\tilde{m}_2(R) := \int_0^\infty x^2\rho_X(x) e^{-x^2/R^2} dx$. 
Recall that $K_{X_R} = \int_{0}^{\infty} e^{-x^2/R^2} \rho(x) dx = \e e^{-X^2/R^2}$.
Note that if $R \geq a$ (which we assume in the sequel), then
\begin{align} \label{eqn:K_{X_R}}
1 > K_{X_R} > \frac{1-\eps_1}{e}. 
\end{align}
\details{For the second inequality, in the integral for $K_{X_R}$ restrict it to $x \in [0, a]$, now note that in this range the integrand $e^{-x^2/R^2}$ is $\geq 1/e$ (as in the previous like we assumed $R \geq a$), and the probability mass in this range is $1-\eps_1$ (by (16)).}
Since $\tilde{m}_2(R) \to \infty$ as $R \to \infty$, by choosing $R$ sufficiently large we can ensure that
\begin{align} \label{eqn:b_eps2damping}
\int_a^\infty x^2 \rho_X(x) e^{-x^2/R^2} dx \geq (1-\eps_2) \int_0^\infty x^2 \rho_X(x) e^{-x^2/R^2} dx.
\end{align}
 \details{ to see clearly, split the right side and combine non-constant terms on the left, they will go to infinity with $R$}
Moreover, we choose $R$ to be sufficiently large so that $\sqrt{C\tilde{m}_2(R)} > a$.
Then
\begin{align*}
\tilde{m}_2(R)
&= \int_0^\infty x^2 \rho_X(x) e^{-x^2/R^2} dx \\
&= \int_0^a x^2 \rho_X(x) e^{-x^2/R^2} dx \\ & \quad + \int_a^{\sqrt{C\tilde{m}_2(R)}} x^2 \rho_X(x) e^{-x^2/R^2} dx + \int_{\sqrt{C\tilde{m}_2(R)}}^\infty x^2 \rho_X(x) e^{-x^2/R^2} dx \\
&\leq \eps_2 \int_0^\infty x^2 \rho_X(x) e^{-x^2/R^2} dx \\ & \quad + \int_a^{\sqrt{C\tilde{m}_2(R)}} x^2 \rho_X(x) e^{-x^2/R^2} dx + \int_{\sqrt{C\tilde{m}_2(R)}}^\infty x^2 \rho_X(x) e^{-x^2/R^2} dx
\;\; \\ 
&\leq \eps_2\, \tilde{m}_2(R) + \eps_1 C \, \tilde{m}_2(R) + \int_{\sqrt{C\tilde{m}_2(R)}}^\infty x^2 \rho_X(x) e^{-x^2/R^2} dx \;\; \text{(by \eqref{eqn:a_eps1})} \\
&= (C\eps_1+\eps_2) \, \tilde{m}_2(R) + \int_{\sqrt{C \tilde{m}_2(R)}}^\infty x^2 \rho_X(x) e^{-x^2/R^2} dx.
\end{align*}
Summarizing the previous sequence of inequalities:
\begin{align} \label{eqn:estimate_second2}
\int_{\sqrt{C \tilde{m}_2(R)}}^\infty x^2 \rho_X(x) e^{-x^2/R^2} dx \geq (1- C\eps_1 - \eps_2) \, \tilde{m}_2(R).
\end{align}

Now
\begin{align*}
\E X_R^4 > K_{X_R} \, \E X_R^4 &= \int_0^\infty x^4 \rho_X(x) e^{-x^2/R^2} dx \;\;\text{(the inequality uses \eqref{eqn:K_{X_R}})}\\
&\geq \int_{\sqrt{C \tilde{m}_2(R)}}^\infty x^4 \rho_X(x) e^{-x^2/R^2} dx \\
&\geq C \tilde{m}_2(R) \int_{\sqrt{C \tilde{m}_2(R)}}^\infty x^2 \rho_X(x) e^{-x^2/R^2} dx \\
&\geq C \tilde{m}_2(R) (1- C\eps_1 - \eps_2) \, \tilde{m}_2(R) \;\;\text{(by \eqref{eqn:estimate_second2})} \\
&= C (1- C\eps_1 - \eps_2) \, \tilde{m}_2(R)^2 \\
&= C (1- C\eps_1 - \eps_2) \, (K_{X_R} \, \E X_R^2)^2 \\
&\geq  C (1- C\eps_1 - \eps_2) \left(\frac{1-\eps_1}{e}\right)^2 (\E X_R^2)^2 \;\;\text{(by \eqref{eqn:K_{X_R}})}.
\end{align*}

Now note that $C (1- C\eps_1 - \eps_2) \left(\frac{1-\eps_1}{e}\right)^2 > 10$ for our choice of the parameters. Thus
$\cum_4(X_R) = \E X_R^4 - 3 (\E X_R^2)^2 > 7 (\E X_R^2)^2$, and by our assumption
$\E X_R^2 \to \infty$ with $R \to \infty$.
\end{proof} 
\fi

\subsection{Symmetrization} \label{sec:symmetrization}
As usual we work with the ICA model $X = AS$. 
Suppose that we have an ICA algorithm that works when each of the component random variable $S_i$ 
is symmetric, i.e. its probability density function satisfies $\phi_i(y) = \phi_i(-y)$ for all $y$, 
with a polynomial dependence on the upper bound $M_4$ on the fourth moment of the $S_i$ and inverse
polynomial dependence on the lower bound $\Delta$ on the fourth cumulants of the $S_i$. Then we show
that we also have an algorithm without the symmetry assumption and with a similar dependence on $M_4$
and $\Delta$. 
We show that without loss of generality we may restrict our attention to symmetric densities, 
i.e. we can assume that each
of the $S_i$ has density function satisfying $\phi_i(y) = \phi_i(-y)$. To this end, let $S'_i$ be an 
independent copy of $S_i$ and set $\bar{S}_i := S_i - S'_i$. Similarly, let $\bar{X}_i := X_i - X'_i$. 
Clearly, the $S_i$ and $X_i$ have symmetric densities.
The new random variables still satisfy the ICA model: $\bar{X}_i = A \bar{S}_i$. 
Moreover, the moments and cumulants of the $\bar{S}_i$ behave similarly to those of the $S_i$: 
For the fourth moment, assuming it exists, we have
\( \E[\bar{S}_i^4] = \E[(S_i-S'_i)^4] \leq 2^4 \E[S_i^4]\).
The inequality above can easily be proved using the binomial expansion and H\"older's inequality:
\begin{align*}
\E[(S_i-S'_i)^4] &= \\ \E[S_i^4] + & 4 \E[S_i^3]\,\E[S'_i] + 6 \E[S_i^2]\,\E[(S'_i)^2] + 4 \E[S_i]\,\E[(S'_i)^3] + 
\E[(S'_i)^4] \leq 16 \, \E[S_i^4].
\end{align*}
The final inequality follows from the fact that for each term in the LHS, e.g. $\E[S_i^3]\,\E[S'_i]$ we have 
$\E[S_i^3]\,\E[S'_i] \leq \E[S_i^4]^{3/4} \E[(S'_i)^4]^{1/4} = \E[S_i^4]$.

For the fourth cumulant, again assuming its existence, we have 
\( \cum_4(\bar{S}_i) = \cum_4(S_i - S'_i) = \cum_4(S_i)+\cum_4(-S_i) = 2\, \cum_4(S_i) \).

Thus if the fourth cumulant of $S_i$ is away from $0$ then so is the fourth cumulant of $\bar{S}_i$.

\iflong
\section{Putting things together}

In this section we combine the orthogonalization procedure (Algorithm~\ref{alg:orthogonalization_uniform}
with performance guarantees in Theorem~\ref{thm:correctness_uniform_orth}) 
with ICA for unitary $A$ via Gaussian damping to prove our main theorem, Theorem \ref{thm:putting_together}.


As noted in the introduction, intuitively, $R$ in the theorem statement above  measures how large a ball we need to restrict the 
distribution to so that there is at least a constant (or $1/\poly(n)$ if needed) probability mass in it 
and moreover each $S_i$ when restricted to the interval $[-R, R]$ has the fourth cumulant at least 
$\Omega(\Delta)$. 
Formally, $R > 0$ is such that $\int_{\R^n} \rho_{\hX}(x) e^{-\norm{x}^2/R^2} \ud x \geq p(n) > 0$, where 
$1/\poly(n) < p(n) < 1$ can be chosen, and for simplicity, we will fix to $1/2$. Moreover, $R$ satisfies that 
$\cum_4(S_{i,R}) \geq \Omega(\Delta/n^4)$ for all $u \in S^{n-1}$ and $i \in [n]$, where $S_{i,R}$ is the Gaussian damping
with parameter $R$ of $S_i$.

In Sec.~\ref{sec:symmetrization} we saw that the moments and cumulants of the non-symmetric random variable behave similarly to those of the symmetric random variable.
So by the argument of Sec.~\ref{sec:symmetrization} we assume that our ICA model is symmetric.
Theorem~\ref{thm:correctness_uniform_orth} shows that Algorithm~\ref{alg:orthogonalization_uniform} gives 
us a new ICA model with the ICA matrix having approximately orthogonal columns. 
We will apply Gaussian damping to this new ICA model. 
In Theorem~\ref{thm:correctness_uniform_orth}, it was convenient to use the normalization $\E\abs{S_i} = 1$ for all $i$. 
But for the next step of Gaussian damping we will use a different normalization, namely, the columns of the ICA matrix have unit length. 
This will require us to rescale (in the analysis, not algorithmically) the components of $S$ appropriately as we now describe. 

Algorithm~\ref{alg:orthogonalization_uniform} provides us with a matrix $B$ such that the columns of 
$C = BA$ are approximately orthogonal: $C^T C \approx D$ where $D$ is a diagonal matrix. 
Thus, we can rewrite our ICA model as $Y = C S$, where $Y = BX$. 
We rescale $C_i$, the $i$th column of $C$, 
by multiplying it by $1/\norm{C_i}$. Denoting by $L$ the diagonal matrix with the $i$th diagonal entry 
$1/\norm{C_i}$, the matrix obtained after the above rescaling of $C$ is $CL$ and we have 
$(CL)^T CL \approx I$. We can again rewrite our ICA model as $Y = (CL) (L^{-1}S)$. Setting $\hE := CL$ and $T := L^{-1}S$ we can rewrite our ICA model as $\hat{Y} = \hE T$. This is the model we will plug into the Gaussian damping procedure. Had Algorithm~\ref{alg:orthogonalization_uniform} provided us with 
perfect orthogonalizer $B$ (so that $C^T C = D$) we would obtain a model $Y = E T$ where $E$ is 
unitary. We do get however that $\hE \approx E$. 
To continue with a more standard ICA notation, from here on we will write $\hX = \hE S$ for $\hat{Y} = \hE T$
and $X = ES$ for $Y = E T$. 

Applying Gaussian damping to $X = ES$ gives us a new ICA model $X_R = E S_R$ as we saw in 
Sec.~\ref{sec:gaussian_damping}. 
But the model we have access to is $\hX = \hE S$. 
We will apply Gaussian damping to it to get the r.v. $\hX_R$. 
 Formally, ${\hX}_R$ is defined starting with the model $\hX = \hE S$
just as we defined $X_R$ starting with the model $X = ES$ (recall that for a random variable $Z$, we denote
its probability density function by $\rho_Z(\cdot)$):
\begin{align*}
\rho_{{\hX}_R}(x) := \frac{1}{K_{{\hX}_R}} \rho_{\hX}(x) e^{-\norm{x}^2/R^2},
\end{align*}
where $K_{{\hX}_R} := \int_{\R^n} \rho_{\hX}(x) e^{-\norm{x}^2/R^2} \ud x$. The parameter $R$ has been chosen so that
$K_{{\hX}_R} > C_1 := 1/2$ and $\cum_4\,\inner{\hX_R}{u} \geq \Delta$ for all $u \in \S^{n-1}$. By the discussion 
after Theorem~\ref{thm:ICA-orthogonal-damping} (the restatement in Sec.~\ref{sec:gaussian_damping}), 
this choice of $R$ can be made efficiently. (The discussion there is in terms of the directional moments of
$S$, but note that the directional moments of $X$ also give us directional moments of $S$. We omit further 
details.)
But now since the matrix $\hE$ in our ICA model is only approximately unitary, after applying Gaussian damping the obtained random variable ${\hX}_R$ is not given by an ICA model (in particular, it may not have independent coordinates in any basis), although it is close to $X_R$ in a sense to be made precise soon.

Because of this, Theorem~\ref{thm:ICA} is not directly usable for plugging in the samples of ${\hX}_R$. 
To address this discrepancy we will need a robust version of Theorem~\ref{thm:ICA} which also requires us
to specify in a precise sense that ${\hX}_R$ and $X_R$ are close. To this end,
we need some standard terminology from probability theory.
The characteristic function of r.v. $X \in \R^n$ is defined to be $\phi_X(u) = \E (e^{i u^TX})$, 
where $u \in \R^n$.
The cumulant generating function, also known as the second characteristic function,
is defined by $\psi_X(u) = \log \phi_X(u)$. 
The algorithm in \cite{GVX} estimates the second derivative of $\psi_X(u)$ and computes its eigendecomposition. 
(In \cite{Yeredor} and \cite{GVX}, 
this second derivative is interpreted as a kind of covariance matrix of $X$ but with the twist that 
a certain ``Fourier'' weight is used in the expectation computation for the covariance matrix. We will 
not use this interpretation here.\lnote{is this comment relevant? remove?}) 
Set $\Psi_{X}(u) := D^2 \psi_X(u)$, the Hessian matrix of $\psi_X(u)$. We can now state the robust 
version of Theorem~\ref{thm:ICA}.
\nnote{maybe mention somewhere here that we do not have to use GVX and could also use other papers}

\begin{theorem}\label{thm:ICA-robust}
Let $X$ be an $\dim$-dimensional random vector given by an ICA model $X=AS$ where $A \in \R^{ n \times n}$ is unitary 
  and the $S_i$ are mutually independent, $\E[S_i^4] \le M_4$
and $\abs{\cum_{4}(S_i)} \ge \Delta$ for positive constants $M_4$ and $\Delta$. 
Also let $\eps_{\ref{thm:ICA-robust}} \in [0,1]$.  
Suppose that we have another random variable $\hX$ that is close to $X$ in the following sense:
\begin{align*}
\lrabs{\Psi_{\hX}(u)-\Psi_X(u)} \leq \eps_{\ref{thm:ICA-robust}},
\end{align*}
for any $u \in \R^n$ with $\norm{u} \leq 1$.
Moreover, $\E[\inner{X}{u}^4] \le M_4$ for $\norm{u} \leq 1$.
When Fourier PCA is given samples of $\hX$ it will recover 
vectors $\{b_1, \ldots, b_n\}$
  such that there exist signs $\alpha_i \in \{-1, 1\}$ and a permutation $\pi:[n]\to [n]$ satisfying
  \begin{align*}
    \norm{A_i - \alpha_i b_{\pi(i)}} \le \epsilon_{\ref{thm:ICA-robust}} \left(\frac{M_4}{\delta \Delta}\right)^5,
  \end{align*}
  in $\poly(n, M_4, 1/\Delta, 1/\epsilon_{\ref{thm:ICA-robust}}, 1/\delta_{\ref{thm:ICA-robust}})$ samples and time complexity and 
 with probability at least $1-\delta_{\ref{thm:ICA-robust}}$.
\end{theorem}
While this theorem is not stated in \cite{GVX}, it is easy to derive from their proof of 
Theorem~\ref{thm:ICA}; we now briefly sketch the proof of Theorem~\ref{thm:ICA-robust} indicating the changes
one needs to make to the proof of Theorem~\ref{thm:ICA} in \cite{GVX}. 

\begin{proof}
Ideally, for input model $X=AS$ with $A$ unitary, algorithm Fourier PCA  would proceed by diagonalizing 
$\Psi_X(u)$. But it can only compute an approximation 
$\tilde{\Psi}_X(u)$ which is the empirical estimate
for $\Psi_X(u)$. For all $u$ with $\norm{u} \leq 1$, it is shown that with high probability we have
\begin{align} \label{eqn:empirical-sigma}
\norm{\tilde{\Psi}_X(u)-\Psi_X(u)}_F<\epsilon.
\end{align}
Then, a matrix perturbation 
argument is invoked to show that if the diagonalization procedure used in Fourier PCA is applied to 
$\tilde{\Psi}_X(u)$ instead of $\Psi_X(u)$, 
one still recovers a good approximation of $A$. This previous step uses a random $u$ chosen from a Gaussian 
distribution so that the eigenvalues of $\Psi_X(u)$ are sufficiently spaced apart for the eigenvectors to be 
recoverable (the assumptions on the distribution ensure that the requirement of $\norm{u} \leq 1$ 
is satisfied with high probability).
The only property of $\tilde{\Psi}_X(u)$ used in this argument is
\eqref{eqn:empirical-sigma}. To prove Theorem~\ref{thm:ICA-robust}, we show that the estimate 
$\tilde{\Psi}_{\hX}$ is also good: 
\begin{align*}
\norm{\tilde{\Psi}_{\hX}(u) - \Psi_X(u)}_F 
< \norm{\tilde{\Psi}_{\hX}(u) - \Psi_{\hX}(u)}_F + \norm{\Psi_{\hX}(u) - \Psi_X(u)}_F 
< 2\epsilon,
\end{align*}
where we ensured that $\norm{\tilde{\Psi}_{\hX}(u) - \Psi_{\hX}(u)}_F < \epsilon$ by taking sufficiently many samples of $\hat{X}$ to get a good estimate with probability at least $\delta$; as in \cite{GVX}, 
a standard concentration argument shows that 
$\poly(n, M_4, 1/\Delta, 1/\epsilon, 1/\delta)$ samples suffice for this purpose. 
Thus the diagonalization procedure can be applied to $\tilde{\Psi}_{\hX}(u)$.
The upper bound of $2\epsilon$ above
translates into error 
$< \epsilon \left(\frac{M_4}{\delta \Delta}\right)^5$ in the final recovery guarantee, with the extra factor coming from the eigenvalue gaps of $ \Psi_{X}(u)$.\lnote{Too sketchy} 
\end{proof}

To apply Theorem~\ref{thm:ICA-robust} to our situation, we need
\begin{align} \label{eqn:Sigma_estimate}
\tilde{\Psi}_{\hat{X}_R}(u) \approx \Psi_{X_R}(u).
\end{align}
This will follow from the next lemma. 

Note that
\begin{align} \label{eqn:expansion-Sigma}
\Psi_{X}(u) = D^2 \psi_X(u) = \frac{D^2 \phi_X(u)}{\phi_X(u)} - \frac{(D \phi_X(u))^T(D \phi_X(u))}{\phi_X(u)^2}.
\end{align}
(The gradient $D \phi_X(u)$ is a row vector.)
Thus, to show \eqref{eqn:Sigma_estimate} it suffices to show that each expression on the RHS of the previous 
equation is appropriately close: 

\begin{lemma} \label{lem:Sigma-estimates}
Let $\lambda \in [0,1/2]$, and let $E, \hE \in \R^{n \times n}$ such that $E$ is unitary and 
$\norm{E-\hE}_2 \leq \lambda^2/3$. Let $X_R$ and ${\hX}_R$ be the random variables obtained by 
applying Gaussian damping to the ICA models $X=ES$ and $\hX = \hE S$, resp.  
Then, for $\norm{u} \leq 1$, we have 
\begin{align*}
\lrabs{\phi_{X_R}(u) - \phi_{\hX_R}(u)} &\leq R \lambda^2/3 + 4\lambda + \frac{4\lambda}{K_{X_R}}, \\
\norm{D \phi_{X_R}(u) - D \phi_{\hX_R}(u)} &\leq O(n\lambda R^2), \\
\norm{D^2\phi_{X_R}(u) - D^2\phi_{\hX_R}(u)}_F &\leq O(n^2\lambda R^3).
\end{align*}
\end{lemma}

\begin{proof}
We will only prove the first inequality; proofs of the other two are very similar and will be omitted. 
In the second equality in the displayed equations below we use that 
$\int_{\R^n}e^{iu^Tx} e^{-\norm{x}^2/R^2}\rho_{\hX}(x)\, dx = \int_{\R^n}e^{iu^T\hE s} e^{-\norm{\hE s}^2/R^2}\rho_{S}(s) \, ds$.
One way to see this is to think of the two integrals as 
expectations: $\E\left(e^{iu^T\hX} e^{-\norm{\hX}^2/R^2}\right) = \E\left( e^{iu^T\hE S} e^{-\norm{\hE S}^2/R^2}\right)$.
\begin{align}
|\phi_{X_R}(u) & - \phi_{\hX_R}(u)| \nonumber \\
&= \lrabs{ \frac{1}{K_{X_R}}\int_{\R^n}e^{iu^Tx} e^{-\norm{x}^2/R^2}\rho_X(x) dx
- \frac{1}{K_{\hX_R}}\int_{\R^n}e^{iu^Tx} e^{-\norm{x}^2/R^2}\rho_{\hX}(x) dx} \nonumber \\
&= \lrabs{ \frac{1}{K_{X_R}}\int_{\R^n}e^{iu^TEs} e^{-\norm{Es}^2/R^2}\rho_S(s) ds
- \frac{1}{K_{\hX_R}}\int_{\R^n}e^{iu^T\hE s} e^{-\norm{\hE s}^2/R^2}\rho_{S}(s) ds} \nonumber \\
&\leq \frac{1}{K_{X_R}} \lrabs{ \int_{\R^n} e^{iu^TEs} e^{-\norm{Es}^2/R^2}\rho_S(s) ds 
- \int_{\R^n}e^{iu^T\hE s} e^{-\norm{\hE s}^2/R^2}\rho_{S}(s) ds} \nonumber \\
& + \lrabs{\frac{1}{K_{X_R}} - \frac{1}{K_{\hX_R}}} \cdot \lrabs{\int_{\R^n}e^{iu^T\hE s} e^{-\norm{\hE s}^2/R^2}\rho_{S}(s) ds} \nonumber \\
&\leq \frac{1}{K_{X_R}} \underbrace{\int_{\R^n} \lrabs{e^{iu^TEs} e^{-\norm{Es}^2/R^2} - e^{iu^T\hE s} e^{-\norm{\hE s}^2/R^2}}\rho_{S}(s) ds}_{G} \nonumber
\\
&+  \underbrace{\lrabs{\frac{1}{K_{X_R}} - \frac{1}{K_{\hX_R}}} K_{\hX_R}}_{H}. 
\end{align}

Now 
\begin{align*}
G = \int_{\R^n} \lrabs{ e^{iu^T(\hE-E) s} e^{(\norm{E s}^2 - \norm{\hE s}^2)/R^2} - 1} e^{-\norm{Es}^2/R^2} \rho_S(s) ds.
\end{align*}

We have 
\begin{align*}
\lrabs{ e^{iu^T(\hE-E) s} e^{(\norm{E s}^2 - \norm{\hE s}^2)/R^2} - 1}
\leq \lrabs{ e^{iu^T(\hE-E) s} -1} + \lrabs{e^{(\norm{E s}^2 - \norm{\hE s}^2)/R^2} - 1},
\end{align*}

and so 
\begin{align} \label{eqn:G-decomposition}
G \leq & \int_{\R^n} \lrabs{ e^{iu^T(\hE-E) s} -1}\, e^{-\norm{Es}^2/R^2} \rho_S(s) ds \\ & + 
\int_{\R^n} \lrabs{e^{(\norm{E s}^2 - \norm{\hE s}^2)/R^2} - 1}\, e^{-\norm{Es}^2/R^2} \rho_S(s) ds.
\end{align}

For the first summand in \eqref{eqn:G-decomposition} note that
\begin{align*}
\lrabs{ e^{iu^T(\hE-E) s} -1} \leq \lrabs{ u^T (\hat{E}-E) s}.
\end{align*}
This follows from the fact that for real $\theta$ we have 
\begin{align*}
\lrabs{e^{i \theta}-1}^2 = (\cos\theta -1)^2 + \sin^2\theta = 2 - 2\cos\theta= 4\sin^2(\theta/2) \leq \theta^2.
\end{align*}

So, 
\begin{align*}
&\int_{\R^n} \lrabs{u^T (\hat{E}-E) s} e^{-\norm{Es}^2/R^2} \rho_S(s) ds \\
& \leq \norm{u} \norm{\hat{E}-E}_2 \int_{\R^n} \norm{s} e^{-\norm{s}^2/R^2} \rho_S(s) ds  
\;\;\;\text{(using $\norm{Es}=\norm{s}$ as $E$ is unitary)} \\
& \leq \norm{u} \norm{\hat{E}-E}_2 \left(\max_{z \in \R} z e^{-z^2/R^2}  \right) \int_{\R^n} \rho_S(s) ds  \\
& \leq \norm{u} \norm{\hat{E}-E}_2 R  \\
&\leq R \lambda^2/3,
\end{align*}
where the last inequality used our assumption that $\norm{u} \leq 1$. 

We will next bound second summand in \eqref{eqn:G-decomposition}.
Note that $\norm{Es} = \norm{s}$ and 
$\norm{Es}^2 - \norm{\hE s}^2 \leq (\norm{E}+\norm{\hE})(\norm{E-\hat{E}})\norm{s}^2$. 
Since $\norm{E-\hat{E}} \leq \lambda^2/3 << 1$, we get 
$\norm{Es}^2 - \norm{\hE s}^2 \leq (\norm{E}+\norm{\hE})(\norm{E-\hat{E}}) \norm{s}^2 \leq \lambda^2 \norm{s}^2$. We will use that $e^\lambda - 1 < \lambda + \lambda^2$
for $\lambda \in [0,1/2]$ which is satisfied by our assumption. 
Now the second summand in \eqref{eqn:G-decomposition} can be bounded as follows. 

\begin{align*}
& \int_{\R^n} \lrabs{e^{(\norm{E s}^2 - \norm{\hE s}^2)/R^2} - 1} e^{-\norm{Es}^2/R^2} \rho_S(s) ds \\
&\leq  \int_{\R^n} \lrabs{e^{\lambda^2 \norm{s}^2/R^2}-1} e^{-\norm{s}^2/R^2} \rho_S(s) ds \\
&\leq  \int_{\R^n} \lrabs{e^{\lambda^2 \norm{s}^2/R^2}-1} e^{-\norm{s}^2/R^2} \rho_S(s) ds \\
&=  \int_{\norm{s} \leq R/\sqrt{\lambda}} \lrabs{e^{\lambda^2 \norm{s}^2/R^2}-1} e^{-\norm{s}^2/R^2} \rho_S(s) ds 
\\ & \quad +  \int_{\norm{s} > R/\sqrt{\lambda}} \lrabs{e^{\lambda^2 \norm{s}^2/R^2}-1} e^{-\norm{s}^2/R^2} \rho_S(s) ds \\
&\leq  \int_{\norm{s} \leq R/\sqrt{\lambda}} (\lambda+\lambda^2) e^{-\norm{s}^2/R^2} \rho_S(s) ds 
+  \int_{\norm{s} > R/\sqrt{\lambda}} e^{-(1-\lambda^2)\norm{s}^2/R^2} \rho_S(s) ds \\
&\leq (\lambda+\lambda^2) K_{X_R} +  e^{-(1-\lambda^2)/\lambda} \\
&\leq 2\lambda K_{X_R} + 2\lambda \;\;\;\text{(using $\lambda < 1/2$)}. 
\end{align*}

Combining our estimates gives
\begin{align*}
G \leq \frac{R \lambda^2}{3 K_{X_R}} + 2\lambda + \frac{2\lambda}{K_{X_R}}.
\end{align*}

Finally, to bound $H$, note that 
\begin{align*}
\frac{1}{K_{X_R}}\lrabs{K_{X_R}-K_{\hX_R}} &= \frac{1}{K_{X_R}} \lrabs{ \int_{\R^n} e^{-\norm{x}^2/R^2}\rho_X(x) dx
- \int_{\R^n} e^{-\norm{x}^2/R^2}\rho_{\hX}(x) dx} \\
&= \frac{1}{K_{X_R}} \lrabs{ \int_{\R^n} e^{-\norm{Es}^2/R^2}\rho_S(s) ds - \int_{\R^n} e^{-\norm{\hE s}^2/R^2}\rho_{S}(s) ds}.
\end{align*} 
This we just upper-bounded above by $2\lambda + \frac{2\lambda}{K_{X_R}}$. 

Thus we have the final estimate
\begin{align*}
&\lrabs{\phi_{X_R}(u) - \phi_{\hX_R}(u)} \leq G + H \leq R \lambda^2/3 + 4\lambda + \frac{4\lambda}{K_{X_R}}.
\end{align*}

The proofs of the other two upper bounds in the lemma follow the same general pattern with slight changes.
\end{proof}

We are now ready to prove Theorem~\ref{thm:putting_together}.
\begin{proof}[Proof of Theorem~\ref{thm:putting_together}]
We continue with the context set after the statement of Theorem~\ref{thm:putting_together}. The plan 
is to apply Theorem~\ref{thm:ICA-robust} to ${\hX}_R$ and $X_R$. To this end we begin by showing that the premise
of Theorem~\ref{thm:ICA-robust} is satisfied.

Theorem~\ref{thm:correctness_uniform_orth}, with $\delta_{\ref{thm:correctness_uniform_orth}}=\delta/2$ and
$\epsilon_{\ref{thm:correctness_uniform_orth}}$ to be specified,
provides us with a matrix $B$ such that the columns of $BA$ are approximately orthogonal:
$\norm{(BA)^TBA - D}_2 \leq \epsilon_{\ref{thm:correctness_uniform_orth}}$ for some diagonal matrix $D$. 
Now set $\hE := B A L$, where $L = \diag(L_1, \dots, L_\dim) = \diag(1/\sqrt{d_1}, \ldots, 1/\sqrt{d_n})$. Theorem~\ref{thm:correctness_uniform_orth}
implies that
\begin{align} \label{eqn:L_i}
1 \leq L_i \leq (n+1).
\end{align}
Then
\begin{align*}
\norm{\hE^T\hE - I}_2 &= \norm{(BAL)^T(BAL) - I}_2 \\
&= \norm{L^T(BA)^TBA L - L^TDL}_2
\leq \norm{L}_2^2 \norm{(BA)^TBA - D}_2 \leq (n+1)^2\epsilon_{\ref{thm:correctness_uniform_orth}},
\end{align*}
because $L_i \leq (n+1)$ by \eqref{eqn:L_i}.
For $\hE$ as above, there exists a unitary $E$ such that 
\begin{align} \label{eqn:EhE}
\norm{E-\hE}_2 \leq (n+1)^2 \epsilon_{\ref{thm:correctness_uniform_orth}},
\end{align}
by
Lemma~\ref{lem:procrust} below.

By our choice of $R$, the components of $S_R$ satisfy 
$\cum_4(S_{i,R}) \geq \Delta$ and 
$\E\, S_{i,R}^4 \leq  2 R^4$ (via \eqref{eqn:damping-moment-bound} and our choice $C_1=1/2$). 
Hence $M_4 \leq 2 R^4$. 
The latter
bound via \eqref{eqn:EhE} gives $\E[\inner{X}{u}^4] \le 2 (1 + (n+1)^2 \epsilon_{\ref{thm:correctness_uniform_orth}}) R^4$ for all $u \in \S^{n-1}$.

Finally, Lemma~\ref{lem:Sigma-estimates} 
with \eqref{eqn:expansion-Sigma} \
and simple estimates give 
$
\norm{\Psi_{\hX_R} - \Psi_{X_R}}_F \leq O(n^4 R^4 \epsilon_{\ref{thm:correctness_uniform_orth}}^{1/2}).
$

We are now ready to apply Theorem~\ref{thm:ICA-robust} with 
$\epsilon_{\ref{thm:ICA-robust}} = O(n^4 R^4 \epsilon_{\ref{thm:correctness_uniform_orth}}^{1/2})$ and $\delta_{\ref{thm:ICA-robust}}=\delta/2$. 
This gives that
Fourier PCA produces output $b_1, \ldots, b_n$ such that there are signs $\alpha_i = \pm 1$ and permutation $\pi:[n] \to [n]$ such that 

\begin{align} \label{eqn:application-ICA-robust}
    \norm{A_i - \alpha_i b_{\pi(i)}} \le  O(n^4 R^4 \epsilon_{\ref{thm:correctness_uniform_orth}}^{1/2}) \left(\frac{R^4}{\delta \Delta}\right)^5,
\end{align}
with $\poly(n, 1/\Delta, R, 1/R, 1/\epsilon_{\ref{thm:correctness_uniform_orth}}, 1/\delta)$ sample and time complexity. 
Choose $\epsilon_{\ref{thm:correctness_uniform_orth}}$ so that the RHS of \eqref{eqn:application-ICA-robust} is $\epsilon$.

The number of samples and
time needed for orthogonalization is 
\\ $\poly_\gamma(n, M, 1/s_m, S_M, 1/\epsilon_{\ref{thm:correctness_uniform_orth}}, 1/\delta)$. 
Substituting the value of $\epsilon_{\ref{thm:correctness_uniform_orth}}$ the previous bound becomes 
$\poly_\gamma(n, M, 1/s_m, S_M, 1/\Delta, R, 1/R,1/\epsilon, 1/\delta)$. 
The probability of error, coming from the applications of Theorem~\ref{thm:correctness_uniform_orth} 
and \ref{thm:ICA-robust} is at most 
$\delta/2 +\delta/2 = \delta$. 
\end{proof}

\begin{lemma}\label{lem:procrust}
Let $\hE \in \R^{n\times n}$ be such that $\norm{\hE^T\hE - I}_2 \leq \epsilon$. Then there exists a unitary matrix $E \in \R^{n\times n}$ such that $\norm{E-\hE}_2 \leq \epsilon$. 
\end{lemma}
\begin{proof}
This is related to a special case of the so-called orthogonal Procrustes problem \cite[Section 12.4.1]{MR1417720}, where one looks for a unitary matrix $E$ that minimizes $\norm{E-\hE}_F$. A formula for an optimal $E$ is $E = UV^T$, where $U \Sigma V^T$ is the singular value decomposition of $\hE$, with singular values $(\sigma_i)$. Although we do not need the fact that this $E$ minimizes $\norm{E-\hE}_F$, it is good for our purpose:
\[
\norm{\hE - E}_2 = \norm{U \Sigma V^T - U V^T}_2 = \norm{\Sigma - I}_2 = \max_i \abs{\sigma_i -1}.
\]
By our assumption
\[
\norm{\hE^T \hE - I}_2 = \norm{V \Sigma^2 V^T - I}_2 = \norm{\Sigma^2 - I}_2 = \max_i \abs{\sigma_i^2 -1} = \max_i (\sigma_i + 1) \abs{\sigma_i - 1} \leq \eps.
\]
This implies $\max_i \abs{\sigma_i -1} \leq \eps$. The claim follows.
\end{proof}
\fi

\section{Improving orthogonalization}\label{sec:orthogonalization}


As noted above, the technique in \cite{anon_htica}, while being provably efficient and correct, suffers from practical implementation issues.
Here we discuss two alternatives: orthogonalization by \emph{centroid body scaling} and orthogonalization by using the empirical covariance.
The former, orthogonalization via centroid body scaling, uses the samples already present in the algorithm rather than relying on a random walk to draw samples which are approximately uniform in the algorithm's approximation of the centroid body (as is done in \cite{anon_htica}).
This removes the dependence on random walks and the ellipsoid algorithm; instead, we use samples that are distributed according to the original heavy-tailed distribution but \emph{non-linearly scaled} to lie inside the centroid body.
We prove in Lemma \ref{lemma:cvar-orthogonalizer} that the covariance of this subset of samples is enough to orthogonalize the mixing matrix $A$.
Secondly, we prove that one can, in fact, ``forget'' that the data is heavy tailed and orthogonalize by using the empirical covariance of the data, even though it diverges, and that this is enough to orthogonalize the mixing matrix $A$. However, as observed in experimental results, 
in general this has a downside compared to orthogonalization via centroid body in that it could cause numerical instability during
the ``second'' phase of ICA as the data obtained is less well-conditioned. 
This is illustrated directly in the table in Figure~\ref{fig:stable-parameter-estimation-and-error} containing the singular value and 
condition number of the mixing matrix $BA$ in the approximately orthogonal ICA model.

\subsection{Orthogonalization via centroid body scaling}\label{sec:centroidbodyscaling}

In \cite{anon_htica}, another orthogonalization procedure, namely \emph{orthogonalization via the uniform distribution in the centroid body} is theoretically proven to work. 
Their procedure does not suffer from the numerical instabilities and composes well with the second phase of ICA algorithms. 
An impractical aspect of that procedure is that it needs samples from the uniform distribution in the centroid body.

We described orthogonalization via centroid body in Section \ref{intro.ch}, except for the estimation of $p(x)$, the Minkowski functional of the centroid body. The complete procedure is stated in Subroutine \ref{sub:centroid}. 

We now explain how to estimate the Minkowski functional. The Minkowski functional was informally described in Section \ref{intro.ch}. 
The Minkowski functional of $\Gamma X$ is formally defined by $p(x) := \inf \{t>0 \suchthat x \in t \Gamma X\}$.
Our estimation of $p(x)$ is based on an explicit linear program (LP) \eqref{eq:lpminkowski} that gives the Minkowski functional of the centroid body of a finite sample of $X$ \emph{exactly} and then arguing that a sample estimate is close to the actual value for $\Gamma X$. For clarity of exposition, we only analyze formally a special case of LP \eqref{eq:lpminkowski} that decides \emph{membership} in the centroid body of a finite sample of $X$ (LP \eqref{eq:lp}) and approximate membership in $\Gamma X$. 
This analysis is in Section \ref{sec:direct_membership}. 
Accuracy guarantees for the approximation of the Minkowski functional follow from this analysis.

\begin{subroutine}[ht]
\caption{Orthogonalization via centroid body scaling}
\label{sub:centroid}
\begin{algorithmic}[1]
\Require
Samples $(X^{(i)})_{i=1}^N$ of ICA model $X = AS$ so each $S_i$ is symmetric with $(1+\gamma)$ moments.
\Ensure 
Matrix $B$ approximate orthogonalizer of $A$ 
\For{$i = 1:N$}, 
\State Let $\lambda^*$ be the optimal value of \eqref{eq:lpminkowski} with $q = X^{(i)}$. Let $d_i = 1/\lambda^*$. Let $Y^{(i)} = \frac{\tanh d_i}{d_i} X^{(i)}$.
\EndFor
\State Let $C = \frac{1}{N} \sum_{i=1}^N Y^{(i)} {Y^{(i)}}^T$.
Output $B = C^{-1/2}$.
\end{algorithmic}
\end{subroutine}

\begin{lemma}\label{lemma:cvar-orthogonalizer}
Let $X$ be a random vector drawn from an ICA model $X=AS$ such that for all $i$ we have $\e \abs{S_i} = 1$ and $S_i$ is symmetrically distributed. 
Let $Y = \frac{\tanh p(X)}{p(X)} X$ where $p(X)$ is the Minkoswki functional of $\Gamma X$.
\details{$\centroid X$ contains $A B_1^\dim$}
Then $\cov(Y)^{-1/2}$ is an orthogonalizer of $X$.
\end{lemma}

\begin{proof}
We will be applying Lemma \ref{lem:orthogonalizer}.
Let $U$ denote the set of absolutely symmetric product distributions $\measure_W$ over $\RR^\dim$ such that $\e \abs{W_i} = 1$ for all $i$.
For $\measure_V \in \bar U$, let $Q(\measure_V)$ be equal to the distribution obtained by scaling $V$ as described earlier, that is, distribution of $\alpha V$,  where $\alpha = \frac{\tanh p(V)}{p(V)} $, $p(V)$ is the Minkoswki functional of $\Gamma \measure_V $.

For all $\measure_W \in U$, $W_i$ is symmetric and $\e \abs{W_i} = 1$ which implies that $\alpha W$, that is, $Q(\measure_W)$ is absolutely symmetric.
Let $\measure_V \in \bar U$. Then $Q(\measure_V)$ is equal to the distribution of $\alpha V$. For any invertible linear transformation $T$ and measurable set  $\mathcal{M}$, we have 
$Q(T\measure_V)(\mathcal{M}) = Q(\measure_{TV})(\mathcal{M})
=\measure_{ \alpha TV}(\mathcal{M})
=\measure_{\alpha V}(T^{-1}\mathcal{M})
=TQ(\measure_{V})(\mathcal{M})$.
\details{$\alpha= \frac{\tanh p(TV)}{p(TV)}, p(TV) =$  Minkoswki functional of $\Gamma TV$, that is, $T \Gamma V$ from Lemma \ref{lem:equivariance} (linear equivariance of $\Gamma$). $p(T^{-1}TV) =$ Minkoswki functional of $T^{-1}T \Gamma V$}
Thus $Q$ is linear equivariant. 
Let $\measure \in \bar U$. Then there exist $A$ and $\measure_W \in U$ such that $\measure = A \measure_W$.  
We get $\cov(Q(\measure)) = \cov(AQ(\measure_W))$. Let $W_\alpha = \alpha W$.
Thus, $\cov(AQ(\measure_W)) = A \e (W_\alpha {W}_\alpha^T ) A^T$ where $\e (W_\alpha {W}_\alpha^T)$ is a diagonal matrix with elements $\e (\alpha^2 W_i^2 )$ which are non-zero because we assume $\e \abs{W_i} = 1$. 
This implies that $\cov(Q(\measure))$ is positive definite and thus by Lemma \ref{lem:orthogonalizer}, $\cov(Y)^{-1/2}$ is an orthogonalizer of $X$. 
\end{proof}

\subsection{Orthogonalization via covariance}\label{sec:covariance}

Here we show the somewhat surprising fact that  orthogonalization of heavy-tailed signals is sometimes possible by using the ``standard'' approach: inverting the empirical covariance matrix.
The advantage here, is that it is computationally very simple, specifically that having heavy-tailed data incurs very little computational penalty on the process of orthogonalization alone.
It's standard to use covariance matrix for \emph{whitening} when the second moments of all independent components 
exist \cite{ICA01}:
Given samples from the ICA model $X = AS$, we compute the empirical covariance matrix $\tilde\Sigma$ which 
tends to the true covariance matrix as we take more samples and set 
$B=\tilde\Sigma^{-1/2}$. Then one can show that $BA$ is a rotation matrix, and thus by pre-multiplying the data by
$B$ we obtain an ICA model $Y = BX = (BA)S$, where the mixing matrix $BA$ is a rotation matrix, and this model is then amenable to various algorithms. 
In the heavy-tailed regime where the second moment does not exist for some of the components, 
there is no true covariance matrix and the empirical covariance diverges as we take more samples. 
However, for any fixed number of samples one can still compute the empirical covariance matrix. 
In previous work (e.g., \cite{ChenBickel04}), the empirical covariance matrix was used for whitening in the heavy-tailed regime with good empirical performance; \cite{ChenBickel04} also provided some theoretical analysis to explain this surprising performance. 
However, their work (both experimental and theoretical) was limited to some very special cases 
(e.g., only one of the components is heavy-tailed, or there are only two components both with stable
distributions without finite second moment).

We will show that the above procedure 
(namely pre-multiplying the data by $B:=\tilde\Sigma^{-1/2}$)  
``works" under considerably more general conditions, namely
if $(1+\gamma)$-moment exists for $\gamma > 0$ for each independent component $S_i$. 
By ``works" we mean that instead of whitening the data (that is $BA$ is rotation matrix) it does something slightly weaker
but still just as good for the purpose of applying ICA algorithms in the next phase. It \emph{orthogonalizes} the 
data, that is now $BA$ is close to a matrix whose columns are orthogonal.
In other words, $(BA)^T(BA)$ is close to a diagonal matrix (in a sense made precise in 
Theorem~\ref{thm:alg-1-correctness}).

\iflong
Let $X$ be a real-valued symmetric random variable such that $\E (\abs{X}^{1+\gamma}) \leq M$ for some $M > 1$ and 
$0 < \gamma < 1$. The following lemma from \cite{anon_htica} says that the empirical average of the absolute value of $X$ converges to the expectation of $|X|$. The proof, which we omit, follows an argument similar to the proof of the Chebyshev's inequality. 
Let $\tilde{\E}_N[\abs{X}]$ be the empirical average obtained from $N$ independent samples 
$X^{(1)}, \ldots, X^{(N)}$,
i.e., $(\abs{X^{(1)}}+\dotsb+\abs{X^{(N)}})/N$.

\fi
\begin{theorem}[Orthogonalization via covariance matrix]\label{thm:alg-1-correctness}
Let $X$ be given by ICA model $X=AS$. 
Assume that there exist $t, p, M > 0$  and $\gamma \in (0,1)$
such that for all $i$ we have

(a) $\e (\abs{S_i}^{1+\gamma}) \leq M < \infty$,

(b) (normalization) $\e \abs{S_i} = 1$, and

(c) $\Pr(\abs{S_i} \geq t) \geq p$.
Let $x^{(1)}, \dotsc, x^{(N)}$ be i.i.d. samples according to $X$. Let $\tilde{\Sigma} = (1/N) \sum_{k=1}^N x^{(k)} {x^{(k)}}^T$ and $B = \tilde{\Sigma}^{-1/2}$.
Then for any $\eps, \delta \in (0,1)$, $\norm{(BA)^T B A - D}_{2} \leq \eps$
for a diagonal matrix $D$ with diagonal entries $d_1, \dotsc, d_\dim$ satisfying $0 < d_i, 1/d_i \leq \max\{ 2/pt^2, N^4 \}$ for all $i$
with probability $1-\delta$ when
$N \geq \poly(n, M, 1/p, 1/t, 1/\eps, 1/\delta)$.
\end{theorem}
\iflong
\begin{proofidea}
For $i\neq j$ we have $\e(S_i S_j) = 0$ (due to our symmetry assumption on $S$) 
and $\e(\abs{S_i S_j}) = \e(\abs{S_i})\e(\abs{S_j}) < \infty$.
We have $(BA)^T BA = L^{-1}$, where $L = (1/N) \sum_{k=1}^N s^{(k)} {s^{(k)}}^T$. The off-diagonal entries of $L$ 
converge to $0$: We have $L_{i,j} = \E S_i S_j = (\E S_i)(\E S_j)$. Now by our assumption that $(1+\gamma)$-moments 
exist, Lemma~\ref{lem:1-plus-eps-chebyshev} is applicable and implies that empirical average $\tilde\E S_i$ tends
to the true average $\E S_i$ as we increase the number of samples. The true average is $0$ because of our assumption
of symmetry (alternatively, we could just assume that the $X_i$ and hence $S_i$ have been centered). 
The diagonal entries of $L$ are bounded away from $0$: This is clear when the second moment is finite,
and follows easily by hypothesis (c) when it is not. \nnote{is the last statement correct}
Finally, one shows that if in $L$ the diagonal entries highly dominate the off-diagonal entries, then the same
is true of $L^{-1}$.
\nnote{is that how the proof goes}
\end{proofidea}
\fi

\iflong

\begin{proof}
\lnote{fix indices}
We have $(BA)^T BA = L^{-1}$, where $L = (1/N) \sum_{k=1}^N s^{(k)} {s^{(k)}}^T$.
By assumption, $\e L_{ij} = 0$ for $i \neq j$.
Note that $\e \abs{s_i s_j}^{1+\gamma} \leq M^2$ and so by Lemma~\ref{lem:1-plus-eps-chebyshev}, for $i \neq j$,
\[
P(\abs{L_{ij}} > \eps_1) \leq \frac{8 M^2}{\eps_1^2 N^{\gamma/3}}
\]
when $N \geq (\frac{8 M^2}{\eps_1})^{\frac{1}{2} + \frac{1}{\gamma}}$.

Now let $D := \diag(L^{-1}_{11}, L^{-1}_{22}, \dots, L^{-1}_{\dim \dim})$.
Then when $\abs{L_{ij}} < \eps_1$ for all $i \neq j$, we have $\norm{L - D^{-1}}_{2} \leq \norm{L - D^{-1}}_{F} \leq n \eps_1$.
The union bound then implies
\begin{equation}\label{eq:off-diagonal-union-bound}
  \begin{aligned}
 P( \norm{L - D^{-1}}_{2} < n\eps_1 ) & \geq P( \norm{L - D^{-1}}_{F} < n\eps_1 ) \\
 & \geq P( \forall i\neq j, \abs{L_{ij}} \leq \eps_1) \\
 & \geq 1 - \frac{8 \dim^2 M^2}{\eps_1^2 N^{\gamma/3}}
  \end{aligned}
\end{equation}
when $N \geq (\frac{8 M^2}{\eps_1})^{\frac{1}{2} + \frac{1}{\gamma}}$.

Next, we aim to bound $\norm{D}_2$ which can be done by writing
\begin{equation}
\norm{D}_2 = \frac{1}{\sigma_{\min}(D^{-1})} = \frac{1}{\min_{i \in [n]} L_{ii}}
\end{equation}
where $L_{ii} = (1/N) \sum_{k=1}^{N} {s_i^{(k)}}^2$.
Consider the random variable $\ind(s_i^2 \geq t^2)$.
We can calculate $\e \sum_j \ind({s_i^{(j)}}^2 \geq t^2) \geq Np$ and use a Chernoff bound to see
\begin{equation}
P\left(\sum_{k \in [N]} \ind({s_i^{(k)}}^2 \geq t^2) \leq \frac{Np}{2} \right) \leq \exp\bigg(-\frac{Np}{8}\bigg)
\end{equation}
and when $\sum_{k \in [N]} \ind({s_i^{(k)}}^2 \geq t^2) \geq \frac{Np}{2}$, we have $L_{ii} \geq t^2 p/2$.
Then with probability at least $1-n\exp(-Np/8)$, all entries of $D^{-1}$ are at least $t^2p/2$.
Using this, if $N \geq N_1 := (8/p)\ln(3n/\delta)$ then $\norm{D}_{2} \leq 2/pt^2$ with probability at least $1 - \delta/3$.

Similarly, suppose that $\norm{D}_{2} \leq 2/pt^2$ and choose $\eps_1 = \min\{ \frac{t^4 p^2}{4n} \cdot \frac{\eps}{2}, \frac{1}{pt^2} \}$ and
\[
N_2 := \max\bigg\{ \bigg( \frac{24n^2 M^2}{\eps_1^2 \delta} \bigg)^{3/\gamma}, \bigg(\frac{8 M^2}{\eps_1}\bigg)^{\frac{1}{2} + \frac{1}{\gamma}} \bigg\}
\]
so that when $N \geq N_2$, we have
$\norm{L - D^{-1}}_{2} \leq 1/(2 \norm{D}_{2})$ and $\norm{L - D^{-1}}_{2} \leq t^4 p^2 \eps / 8$ with probability at least $1-\delta/3$.
Invoking (\ref{eq:inverse-stability}), when $N \geq \max\{N_1, N_2\}$, we have
\begin{equation}\label{eq:inverse-bound}
\norm{L^{-1} - D}_{2} \leq 2 \norm{D}_{2} \norm{L-D^{-1}}_{2} \leq 2 \frac{4}{p^2 t^4} \frac{t^4 p^2 \eps}{8} = \eps
\end{equation}
with probability at least $1 - 2 \delta/3$.

Finally, we upper bound $1/d_i$ for a fixed $i$ by using Markov's inequality:
\begin{equation}
\begin{split}
P\left(\frac{1}{d_i} > N^5 \right) &= P(L_{ii} > N^4) = P\bigg( \sum_{j}^{N} {s_i^{(j)}}^2 > N^5 \bigg) \\
&\leq N P(S_i^2 > N^4 ) \leq N P(\abs{S_i} > N^2) \\ 
&\leq N \frac{\e \abs{S_i}}{N^2} = \frac{1}{N}
\end{split}
\end{equation}
so that $1/d_i \leq N^4$ for all $i$ with probability at least $1-\delta/3$ when $N \geq N_3 := n/3\delta$.
Therefore, when $N \geq \max\{N_1, N_2, N_3 \}$, we have $\norm{L^{-1} - D}_{2} \leq \eps$, $d_i \leq 2/pt^2$, and $1/d_i \leq N^4$ for all $i$ with overall probability at least $1-\delta$.
\end{proof}

We used Lemma~\ref{lem:inversion}. 
\fi

In Theorem~\ref{thm:alg-1-correctness}, the diagonal entries are lower bounded, which avoids some degeneracy, but they could still grow quite large because of the heavy tails.
This is a real drawback of orthogonalization via covariance. HTICA, using the more sophisticated \emph{orthogonalization via centroid body scaling} does not have this problem. We can see this in the right table of Figure \ref{fig:stable-parameter-estimation-and-error}, where the condition number of ``centroid'' is much smaller than the condition number of ``covariance.''

\subsection{New Membership oracle for the centroid body}\label{sec:direct_membership}
We will now describe and theoretically justify a new and practically efficient $\eps$-weak membership oracle for $\Gamma X$, which is a black-box that can answer approximate membership queries in $\Gamma X$.
More precisely:
\begin{definition}
The \emph{$\epsilon$-weak membership problem} for $K\subseteq \RR^n$ is the following:
Given a point $y \in \QQ^n$ and a rational number $\eps > 0$, 
either (i) assert that $y \in K_\eps$, or (ii) assert that $y \not \in K_{-\eps}$.
An \emph{$\epsilon$-weak membership oracle} for $K$ is an oracle that solves the weak membership problem for $K$.
For $\delta \in [0,1]$, an \emph{$(\eps, \delta)$-weak membership oracle} for $K$ acts as 
follows: Given a point $y \in \QQ^n$, with probability at least $1-\delta$ it solves the  
$\eps$-weak membership problem for $y, K$, and otherwise its output can be arbitrary.  
\end{definition}
We start with an informal description of the algorithm and its correctness.

The algorithm implementing the oracle (Subroutine \ref{sub:centroid-oracle}) is the following: 
Let $q \in \RR^\dim$ be a query point. 
Let $X_1, \dotsc, X_N$ be a sample of random vector $X$.
Given the sample, let $Y$ be uniformly distributed in $\{X_1, \dotsc, X_N\}$. 
Output YES if $q \in \Gamma Y$, else output NO.

Idea of the correctness of the algorithm: If $q$ is not in $(\Gamma X)_\eps$, then there is a hyperplane separating $q$ from $(\Gamma X)_\eps$. 
Let $\{x \suchthat a^T x = b \}$ be the hyperplane, satisfying $\norm{a} = 1$, $a^T q > b$ and $a^T x \leq b$ for every $x \in (\Gamma X)_\eps$.
Thus, we have $h_{(\Gamma X)_\eps}(a) \leq b$ and $h_{\Gamma X}(a) \leq b - \eps$.
We have \[ h_{\Gamma Y}(a) = \e(\abs{a^T Y}) = (1/N) \sum_{i=1}^N \abs{a^T X_i}. \]
\iflong
By Lemma \ref{lem:1-plus-eps-chebyshev}, 
\else
By \cite[Lemma 14]{anon_htica},
\fi
$(1/N) \sum_{i=1}^N \abs{a^T X_i}$ is within $\eps$ of $\e \abs{a^T X} = h_{\Gamma X}(a) \leq b - \eps$ when $N$ is large enough with probability at least $1-\delta$ over the sample $X_1, \dotsc, X_N$. 
In particular, $h_{\Gamma Y}(a) \leq b$, which implies $q \notin \Gamma Y$ and the algorithm outputs NO, with probability at least $1-\delta$. 

If $q$ is in $(\Gamma X)_{-\eps}$, let $y = q + \eps \hat q \in \Gamma X$.  
\iflong
We will prove the following claim: 
\fi

\iflong
Informal claim (Lemma \ref{lem:centroid_approximation}): 
\else
Claim: 
\fi
For $p \in \Gamma X$, for large enough $N$ and with probability at least $1-\delta$ there is $z \in \Gamma Y$ so that $\norm{z - p} \leq \eps/10$. 

This claim applied to $p=y$ to get $z$, convexity of $\Gamma Y$ and the fact that $\Gamma Y$ contains $B \simeq \sigma_{\min}(A) B_2^\dim$ 
\iflong
(Lemma \ref{lem:innerball}) 
\fi
imply that $q \in \conv (B \cup \{ z \}) \subseteq \Gamma Y$ and the algorithm  outputs YES.

\iflong
We will prove the claim now. Let $p \in \Gamma X$. 
By the dual characterization of the centroid body (Proposition \ref{prop:dualcharacterization}), 
there exists a function $\lambda:\RR^\dim \to \RR$ such that $p = \e (\lambda(X) X)$ with $-1 \leq \lambda \leq 1$.
Let 
\(
z = \frac{1}{N} \sum_{i=1}^N \lambda(X_i) X_i.
\)
We have $\e_{X_i} (\lambda(X_i) X_i) = p$ and $\e_{X_i} ( \abs{\lambda(X_i) X_i}^{1+\gamma} ) \leq \e_{X_i} ( \abs{X_i}^{1+\gamma} ) \leq M$.
\iflong
By Lemma \ref{lem:1-plus-eps-chebyshev}
\else
By \cite[Lemma 14]{anon_htica}
\fi
and a union bound over every coordinate we get $\pr( \norm{p - z} \geq \eps) \leq \delta$ for $N$ large enough. 
\fi 

\iflong
\subsubsection{Formal Argument}
\else
We conclude with the main formal claims of the argument and a precise description of the oracle below:
\fi

\begin{lemma}\label{lem:innerball}
Let $S = (S_1, \dots, S_n) \in \RR^n$ be an absolutely symmetrically distributed random vector such that
$\e (\abs{S_i}) = 1$ and $\e (\abs{S_i}^{1+\gamma}) \leq M < \infty$ for all $i$. 
Let $S^{(i)}, i=1, \dotsc, N$ be a sample of i.i.d. copies of $S$. Let $Y$ be a random vector, uniformly distributed  in $S^{(1)}, \dotsc, S^{(N)}$. 
Then $ (1-\eps') B_{1}^{\dim} \subseteq \Gamma Y$ whenever
\[
N \geq \left(\frac{16 M \dim^4}{(\eps')^2} \delta'\right)^{\frac{1}{2} + \frac{3}{\gamma}}.
\]
\end{lemma}

\begin{proof}
From Lemma \ref{lem:centroid-scaling} we know $\pm e_i \in \Gamma S$. 
It is enough to apply Lemma \ref{lem:centroid_approximation} to $\pm e_i$ with $\eps = \eps'/\sqrt{\dim}$ and $\delta = \delta'/(2\dim)$.
This gives, for any $\theta \in S^{\dim-1}$,  $h_{\Gamma Y} (\theta) \geq h_{\Gamma S} (\theta) -\eps \geq h_{B_1^\dim} (\theta) -\eps \geq (1-\sqrt{\dim} \eps) h_{B_1^\dim}(\theta) = (1-\eps')h_{B_1^\dim}(\theta)$. 
In particular, $\Gamma Y \supseteq (1-\eps') B_1^\dim$.
\end{proof} 

\begin{proposition}[Dual characterization of centroid body]\label{prop:dualcharacterization}
Let $X$ be a $\dim$-dimensional random vector with finite first moment, that is, for all $u \in \RR^\dim$ we have $\e (\abs{\inner{u}{X}}) < \infty$. Then 
\iflong
\begin{equation}\label{equ:dualcentroid}
    \Gamma X = \{ \e \bigl(\lambda(X) X \bigr) \suchthat \text{$\lambda:\RR^n \to [-1,1]$ is measurable}\}.
\end{equation}
\else
    $\Gamma X = \{ \e \bigl(\lambda(X) X \bigr) \suchthat \text{$\lambda:\RR^n \to [-1,1]$ is measurable}\}$.
\fi
\end{proposition}
\begin{proof}
Let $K$ denote the rhs of the conclusion.
We will show that $K$ is a non-empty, closed convex set and show that $h_K = h_{\Gamma X}$, which implies \eqref{equ:dualcentroid}.

By definition, $K$ is a non-empty bounded convex set. To see that it is closed, let $(y_k)_k$ be a sequence in $K$ such that $y_k \to y \in \RR^\dim$. 
Let $\lambda_k$ be the function associated to $y_k \in K$  according to the definition of $K$.
Let $\measure_X$ be the distribution of $X$.
We have $\norm{\lambda_k}_{L^\infty(\measure_X)} \leq 1$ and, passing to a subsequence $k_j$, $(\lambda_{k_j})$ converges to $\lambda \in L^\infty(\measure_X)$ in the weak-$*$ topology $\sigma(L^\infty(\measure_X), L^1(\measure_X))$, where $-1 \leq \lambda \leq 1$.
\footnote{This is a standard argument, see \cite{MR2759829} for the background. Map $x \mapsto x_i$ is in $L^1(\measure_X)$. \cite[Theorem 4.13]{MR2759829} gives that $L^1(\measure_X)$ is a separable Banach space. \cite[Theorem 3.16]{MR2759829} (Banach-Alaoglu-Bourbaki) gives that the unit ball in $L^\infty(\measure_X)$ is compact in the weak-* topology. \cite[Theorem 3.28]{MR2759829} gives that the unit ball in $L^\infty(\measure_X)$ is metrizable and therefore sequentially compact in the weak-* topology. Therefore, any bounded sequence in $L^\infty(\measure_X)$ has a convergent subsequence in the weak-* topology.}
This implies $\lim_{j} \e (\lambda_{k_j} (X) X_i) = \lim_j \int_{\RR^\dim} \lambda_{k_j} (x) x_i \ud \measure_X(x) = \int_{\RR^\dim} \lambda(x) x_i \ud \measure_X(x) = \e (\lambda (X) X_i) $.
Thus, we have $y = \lim_j y_{k_j} = \lim_j \e ((\lambda_{k_j}(X) X) = \e (\lambda(X) X)$ and $K$ is closed.

To conclude, we compute $h_K$ and see that it is the same as the definition of $h_{\Gamma X}$. In the following equations $\lambda$ ranges over functions such that $\lambda:\RR^n \to \RR$ is Borel-measurable and $-1 \leq \lambda \leq 1$.
\begin{align*}
h_K(\theta)
&= \sup_{y \in K} \inner{y}{\theta} \\
&= \sup_{\lambda} \e( \lambda(X) \inner{X}{\theta}) \\
\intertext{and setting $\lambda^*(x) = \sign \inner{x}{\theta}$,}
&= \e( \lambda^*(X) \inner{X}{\theta}) \\
&= \e (\abs{\inner{X}{\theta}}).
\end{align*}
\end{proof}

\begin{lemma}[LP]\label{lemma:centroid-lp}
Let $X$ be a random vector uniformly distributed in $\{x^{(i)}\}_{i=1}^{N} \subseteq \RR^\dim$. Let $q \in \RR^\dim$.
Then: 
\begin{enumerate}
\item $\Gamma X = \frac{1}{N} \sum_{i=1}^N [-x^{(i)}, x^{(i)}]$.

\item 
Point $q~\in~\Gamma X$ iff there is a solution $\lambda~\in~\RR^N$ to the following linear feasibility problem:
\begin{equation}\label{eq:lp}
\iflong
\begin{aligned}
&  \frac{1}{N} \sum_{i=1}^N \lambda_i x^{(i)} = q \\
& -1 \leq \lambda_i \leq 1 \quad \forall i.
\end{aligned}
\else
\frac{1}{N} \sum_{i=1}^N \lambda_i x^{(i)} = q,
-1 \leq \lambda_i \leq 1 \quad \forall i.
\fi
\end{equation}
\item 
Let $\lambda^*$ be the optimal value of (always feasible) linear program
\begin{equation}\label{eq:lpminkowski}
\begin{aligned}
\iflong
&\lambda^* = \max \lambda \\
\text{s.t. } & \frac{1}{N} \sum_{i=1}^N \lambda_i x^{(i)} = \lambda q \\
& -1 \leq \lambda_i \leq 1 \quad \forall i
\else
\lambda^* = \max \lambda, 
\text{s.t. }  \frac{1}{N} \sum_{i=1}^N \lambda_i x^{(i)} = \lambda q, 
 \lambda \in [-1,1]^N
\fi
\end{aligned}
\end{equation}
with $\lambda^* = \infty$ if the linear program is unbounded. Then the Minkowski functional of $\Gamma X$ at $q$ is $1/\lambda^*$.
\end{enumerate}
\end{lemma}
%
\begin{proof}
\begin{enumerate}
\item
This is proven in \cite{MR0279689}.
It is also a special case of Proposition \ref{prop:dualcharacterization}. 
We include an argument here for completeness. 
Let $K := \frac{1}{N} \sum_{i=1}^N [-x^{(i)}, x^{(i)}]$. 
We compute $h_K$ to see it is the same as $h_{\Gamma X}$ in the definition of $\Gamma X$ (Definition \ref{def:centroid-body2}). 
As $K$ and $\Gamma X$ are non-empty compact convex sets, this implies $K = \Gamma X$.
We have 
\begin{align*}
h_K(y) 
&= \sup_{\lambda_i \in [-1,1]} \frac{1}{N} \sum_{i=1}^N \lambda_i x^{(i)} \cdot y \\
&= \max_{\lambda_i \in \{-1,1\}} \frac{1}{N} \sum_{i=1}^N \lambda_i x^{(i)} \cdot y \\
&= \frac{1}{N} \sum_{i=1}^N \abs{ x^{(i)} \cdot y }\\
&= \e (\abs{X \cdot y}).
\end{align*}
\item This follows immediately from part 1. 
\item This follows from part 1 and the definition of Minkowski functional.\qedhere
\end{enumerate}
\end{proof}
\begin{subroutine}[ht]
\caption{Weak Membership Oracle for $\Gamma X$}\label{sub:centroid-oracle}
\begin{algorithmic}[1]
\Require Query point $q \in \RR^\dim$,
samples from symmetric ICA model $X = AS$,
bounds $s_M \geq \sigma_{\max}(A)$, $s_m \leq \sigma_{\min}(A)$,
closeness parameter $\eps$,
failure probability $\delta$.
\Ensure $(\epsilon, \delta)$-weak membership decision for $q \in \Gamma X$.
\State Let $N = \poly(n, M, 1/s_m, s_M, 1/\eps, 1/\delta)$.
\State Let $(x^{(i)})_{i=1}^N$ be an i.i.d. sample of $X$. 
\State Check the feasibility of linear program \eqref{eq:lp}. If feasible, output YES, otherwise output NO.
\end{algorithmic}
\end{subroutine}

\begin{proposition}[Correctness of Subroutine \ref{sub:centroid-oracle}]
Let $X=AS$ be given by an ICA model such that for all $i$ we have $\e (\abs{S_i}^{1+\gamma}) \leq M < \infty$, $S_i$ is symmetrically distributed and normalized so that $\e \abs{S_i} = 1$.
Then, given a query point $q \in \RR^\dim$,  $\eps>0$, $\delta > 0$, $s_M \geq \sigma_{\max}(A)$, and $s_m \leq \sigma_{\min}(A)$, Subroutine~\ref{sub:centroid-oracle} is an $\eps$-weak membership oracle for $q$ and $\Gamma X$ with probability $1-\delta$ using time and sample complexity
\(
\poly(n, M, 1/s_m, s_M, 1/\eps, 1/\delta).
\)\lnote{query time should also depend on query, unless we are talking about arithmetic operations}
\end{proposition}
\begin{proof}
Let $Y$ be uniformly random in $(x^{(i)})_{i=1}^N$. There are two cases corresponding to the guarantees of the oracle: 
\begin{itemize}
\item Case $q \notin (\Gamma X)_\eps$. 
Then there is a hyperplane separating $q$ from $(\Gamma X)_\eps$.
Let $\{x \in \RR^\dim \suchthat a^T x = b \}$ be the separating hyperplane, parameterized so that $a \in \RR^\dim$, $b \in \RR$, $\norm{a} = 1$, $a^T q > b$ and $a^T x \leq b$ for every $x \in (\Gamma X)_\eps$. 
In this case $h_{(\Gamma X)_\eps}(a) \leq b$ and $h_{\Gamma X}(a) \leq b - \eps$.
At the same time, $h_{\Gamma Y}(a) = \e (\abs{a^T Y}) = (1/N) \sum_{i=1}^N \abs{a^T x^{(i)}}$.

We want to apply Lemma~\ref{lem:1-plus-eps-chebyshev} to $a^T X$ to get that $h_{\Gamma Y}(a) = (1/N) \sum_{i=1}^N \abs{a^T x^{(i)}}$ is within $\eps$ of $h_{\Gamma X}(a) = \e ( \abs{a^T X} )$. 
For this we need a bound on the $(1+\gamma)$-moment of $a^T X$. 
We use the bound given in \ref{eqn:N-lower-bound}.

\item Case $q \in (\Gamma X)_{-\eps}$.
Let $y = q + \eps \hat q = q(1+\frac{\eps}{\norm{q}})$.
Let $\alpha = 1+\frac{\eps}{\norm{q}}$.
Then $y \in \Gamma X$.
Invoke Lemma \ref{lem:centroid_approximation} for i.i.d. sample $(x^{(i)})_{i=1}^N$ of $X$ with $p = y$ and $\eps$ equal to some $\eps_1 > 0$ to be fixed later to conclude $y \in (\Gamma Y)_{\eps_1}$.
That is, there exist $z \in \Gamma Y$ such that 
\begin{equation}\label{equ:z}
\norm{z - y} \leq \eps_1.
\end{equation}
Let $w = z/\alpha$.
Given \eqref{equ:z} and the relationships $y=\alpha q$ and $z=\alpha w$, we have 
\begin{equation}\label{equ:perturbation}
\norm{w-q} \leq \norm{z - y} \leq \eps_1.
\end{equation}
From Lemma \ref{lem:innerball} with $\eps'= 1/2$ and equivariance of the centroid body 
(Lemma \ref{lem:equivariance}) 
we get $\Gamma Y \supseteq \frac{\sigma_{\min}(A) }{2 \sqrt{\dim}} B_2^\dim$. This and convexity of $\Gamma Y$ imply $\conv\{\frac{\sigma_{\min}(A) }{2 \sqrt{\dim}} B_2^\dim \cup \{z\}\} \subseteq \Gamma Y$. In particular, the ball around $w$ of radius
\[
r := \left(1-\frac{1}{\alpha}\right) \frac{\sigma_{\min}(A) }{2 \sqrt{\dim}}
\]
is contained in $\Gamma Y$. 
The choice $\eps_1 = r \geq $ and \eqref{equ:perturbation} imply $q \in \Gamma Y$ and Subroutine \ref{sub:centroid-oracle} outputs YES whenever
\[
N \geq \left( \frac{8 M \dim^2}{r^2 \delta} \right)^{\frac{1}{2} + \frac{1}{\gamma}}.
\]
To conclude, remember that $q \in (\Gamma X)_{-\eps}$. Therefore $\norm{q} + \eps \leq \sqrt{\dim} \sigma_{\max}(A)$ (from Lemma \ref{lem:centroid-scaling} and equivariance of the centroid body, Lemma \ref{lem:equivariance}). This implies
\begin{align*}
r 
&= \frac{\eps}{\norm{q} + \eps} \frac{\sigma_{\min}(A) }{2 \sqrt{\dim}} \\
&\geq  \frac{\eps \sigma_{\min}(A) }{2 \dim \sigma_{\max}(A)}
\end{align*}
\end{itemize}
The claim follows.
\end{proof}
\iflong
\begin{lemma}\label{lem:centroid_approximation}
Let $X$ be a $\dim$-dimensional random vector such that for all coordinates $i$ we have $\e (\abs{X_i}^{1+\gamma}) \leq M < \infty$.
Let $p \in \Gamma X$. 
Let $(X^{(i)})_{i=1}^N$ be an i.i.d. sample of $X$.
Let $Y$ be uniformly random in $(X^{(i)})_{i=1}^N$. Let $\eps>0$, $\delta>0$. If $N \geq \left( \frac{8 M \dim^2}{\eps^2 \delta} \right)^{\frac{1}{2} + \frac{3}{\gamma}}$, then, with probability at least $1-\delta$, $p \in (\Gamma Y)_\eps$.
\end{lemma}
\fi
\begin{proof}
By Proposition \ref{prop:dualcharacterization}, there exists a measurable function $\lambda:\RR^\dim \to \RR$, $-1 \leq \lambda \leq 1$ such that $p = \e (X \lambda(X))$.
Let 
\[
z = \frac{1}{N} \sum_{i=1}^N X^{(i)} \lambda(X^{(i)}).
\]
By Proposition \ref{prop:dualcharacterization}, $z \in \Gamma Y$.

We have $\e_{X^{(i)}} (X^{(i)} \lambda (X^{(i)})) = p$ and, for every coordinate $j$,
\[
\e_{X^{(i)}} (\abs{X^{(i)}_j \lambda (X^{(i)})}^{1+\gamma}) 
\leq \e_{X^{(i)}}(\abs{X^{(i)}_j}^{1+\gamma}) 
\leq  M.
\]
By Lemma \ref{lem:1-plus-eps-chebyshev} and for any fixed coordinate $j$ we have, over the choice of $(X^{(i)})_{i=1}^N$,
\[
\pr( \abs{p_j - z_j} \geq \eps/\sqrt{\dim} ) 
\leq \frac{8 M}{(\eps/\sqrt{\dim})^2 N^{\gamma/3}} 
= \frac{8 M \dim}{\eps^2 N^{\gamma/3}}
\]
whenever $N \geq (8M\sqrt{n}/\eps)^{\frac{1}{2} + \frac{1}{\gamma}}$.
A union bound over $\dim$ choices of $j$ gives:
\[
\pr( \norm{p - z} \geq \eps) 
\leq \frac{8 M \dim^2}{\eps^2 N^{\gamma/3}}.
\]
So  
\(
\pr( \norm{p - z} \geq \eps) \leq \delta
\)
whenever
\[
N \geq \left( \frac{8 M \dim^2}{\eps^2 \delta} \right)^{3/\gamma}
\]
and $N \geq (8M\sqrt{n}/\eps)^{\frac{1}{2} + \frac{1}{\gamma}}$.
The claim follows.
\end{proof}

\section{Empirical Study}\label{sec:experiments}

In this section, we show experimentally that heavy-tailed data poses a significant challenge for current ICA algorithms, and compare them with HTICA in different settings.
We observe some clear situations where heavy-tails seriously affect the standard ICA algorithms, and that these problems are frequently avoided by using the heavy-tailed ICA framework.
In some cases, HTICA does not help much, but maintains the same performance of plain FastICA.

To generate the synthetic data, we create a simple heavy-tailed density function
$f_{\eta}(x)$ proportional to ${(\abs{x}+1.5)^{-\eta}}$,
which is symmetric, and for $\eta > 1$, $f_\eta$ is the density of a distribution which has finite $k<\eta-1$ moment.
The signal $S$ is generated with each $S_i$ independently distributed from $f_{\eta_i}$.
The mixing matrix 
$A \in \RR^{\dim \times \dim}$ is generated with each coordinate i.i.d. $\mathcal{N}(0,1)$, columns normalized to unit length.
To compare the quality of recovery, the columns of the estimated mixing matrix, $\tilde{A}$ are permuted to align with the closest matching column of $A$, via the Hungarian algorithm.
We use the Frobenius norm to measure the error, but all experiments were also performed using the well-known Amari index \cite{amari1996new}; the results have similar behavior and are not presented here.


\subsection{Heavy-tailed ICA when $A$ is orthogonal: Gaussian damping and experiments}
\label{sec:ica-orthogonal}

Focusing on the third step above, where the mixing matrix already has orthogonal columns, ICA algorithms already suffer dramatically from the presence of heavy-tailed data.
As proposed in \cite{anon_htica}, Gaussian damping is a preprocessing technique that converts data from an ICA model $X=AS$, where 
$A$ is unitary (columns are orthogonal with unit $l_2$-norm)
to data from a related ICA model $X_R = A S_R$, where $R > 0$ is a parameter to be chosen.
The independent components of $S_R$ have finite moments of all orders and so the existing algorithms can 
estimate $A$.


Using samples of $X$, we construct the damped random variable $X_R$, with pdf $\rho_{X_R}(x) \propto \rho_{X}(x) \exp({-\norm{x}^2/R^2})$.
To normalize the right hand side, we can estimate
\[K_{X_R} = \e \exp({-\norm{X}^2/R^2})\]
so that
\[ \rho_{X_R}(x) =  \rho_{X}(x) \exp({-\norm{x}^2/R^2})/K_{X_R}. \]
If $x$ is a realization of $X_{R}$, then $s = A^{-1}x$ is a realization of the random variable $S_R$ and we have that $S_R$ has pdf $\rho_{S_R}(s) = \rho_{X_R}(x)$.
To generate samples from this distribution, we use rejection sampling on samples from $\rho_{X}$. 
When performing the damping, we binary search over $R$ so that about 25\% of the samples are rejected.
For more details about the technical requirements for choosing $R$, see \cite{anon_htica}.

\begin{figure}[t]
  \centering
    \includegraphics[width=0.49\textwidth]{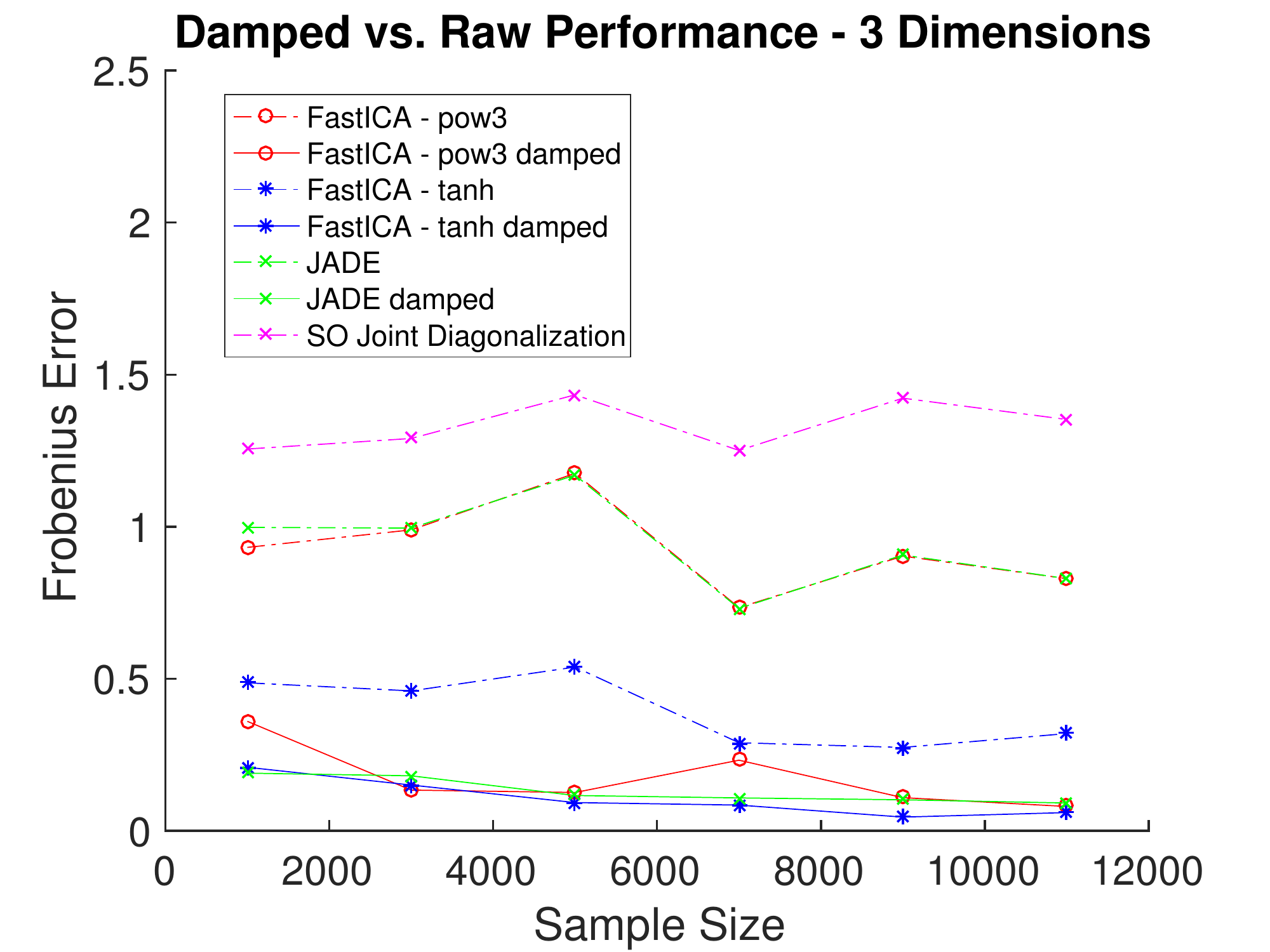}
    \includegraphics[width=0.49\textwidth]{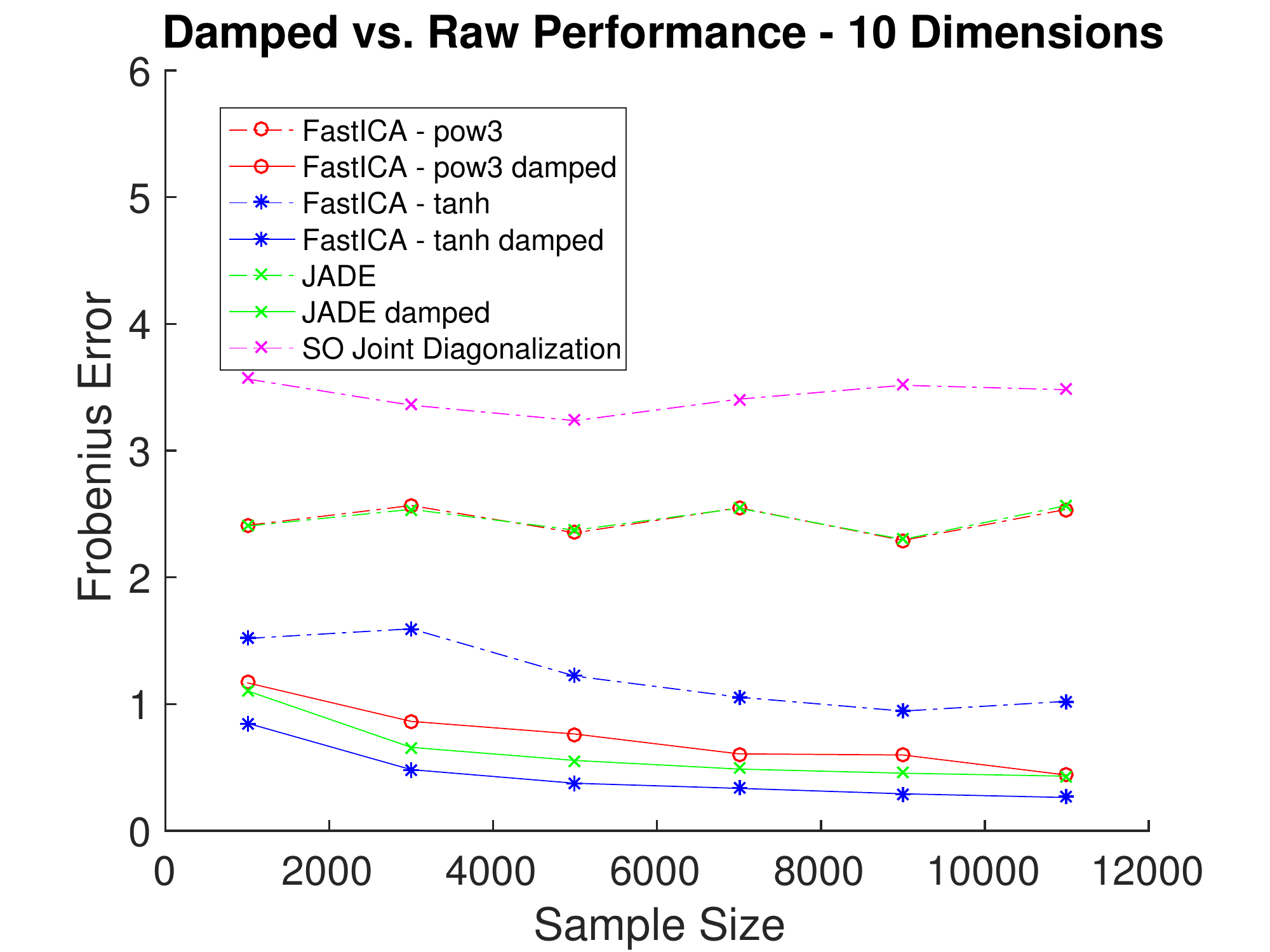}
    \caption{The error of ICA with and without damping (solid lines and dashed lines, resp.), with unitary mixing matrix. The error is averaged over ten trials, in 3 and 10 dimensions where $\eta = (6,6,2.1)$ and $\eta = (6,\dotsc,6,2.1,2.1)$, resp.
    }\label{fig:orthogonal-mixedexp}
\end{figure}

\begin{figure}[t]
  \centering
    \includegraphics[width=0.5\textwidth]{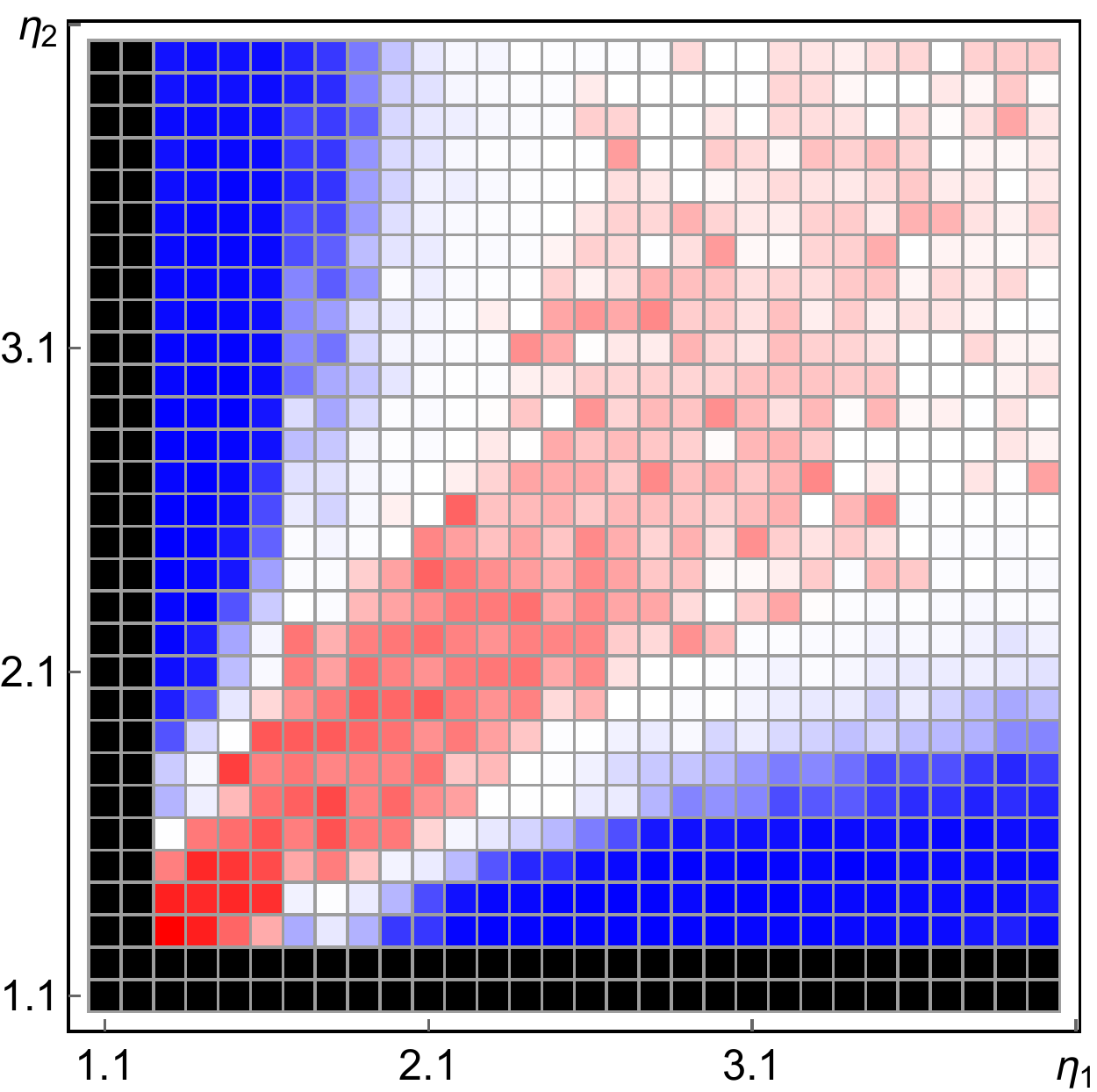}
  \caption{The difference between the errors of FastICA with and without damping in 2 dimensions, averaged over 40 trials. 
    For a single cell, the parameters are given by the coordinates, $\eta = (i,j)$. Red indicates that FastICA without damping does better than FastICA with damping, white indicates that the error difference is 0 and the blue indicates that FastICA with damping performs better than without damping. Black indicates that FastICA without damping failed (did not return two independent components).}
  \label{fig:2-dim-comparison}
\end{figure}

Figure~\ref{fig:orthogonal-mixedexp} shows that, when $A$ is already a perfectly orthogonal matrix, but where $S$ may have heavy-tailed coordinates, several standard ICA algorithms perform better after damping the data. In fact, without damping, some do not appear to converge to a correct solution.
We compare ICA with and without damping in this case: (1) FastICA using the fourth cumulant (``FastICA - pow3''), (2) FastICA using $\log \cosh$ (``FastICA - tanh''), (3) JADE, and (4) Second Order Joint Diagonalization as in, e.g., \cite{cardoso89}
.

\subsection{Experiments on synthetic heavy-tailed data}\label{sec:synthetic-data}\label{sec:synthetic-data-experiments}

We now present the results of HTICA using different orthogonalization techniques: (1)  Orthogonalization via \emph{covariance} (Section \ref{sec:covariance} (2) Orthogonalization via the \emph{centroid} body (Section~\ref{sec:centroidbodyscaling}) (3) the ground truth, directly inverting the mixing matrix (\emph{oracle}), and (4) No orthogonalization, and also no damping (for comparison with plain FastICA) (\emph{identity}).

\begin{figure*}[t]
  \centering
    \includegraphics[width=0.32\textwidth]{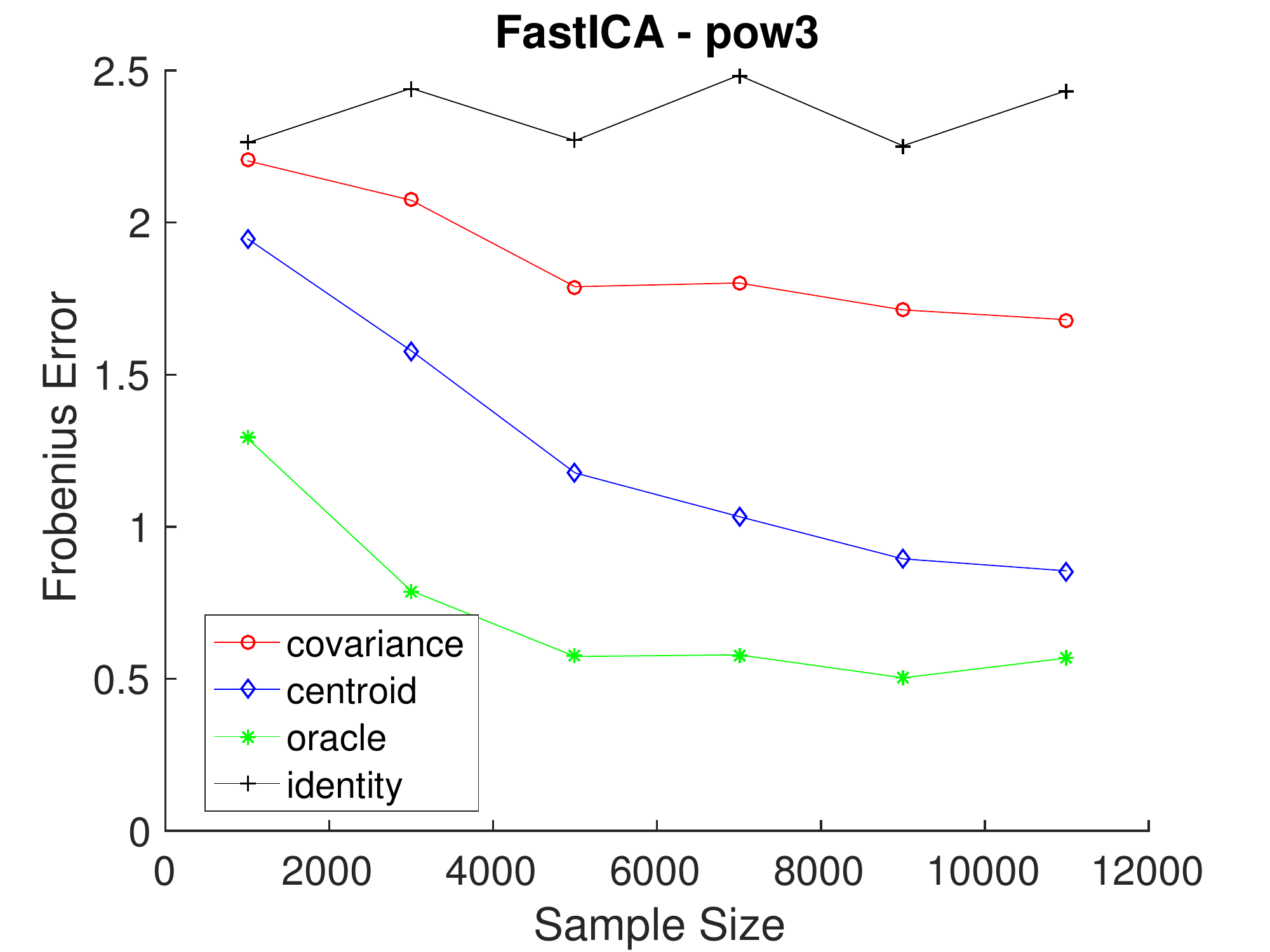}
    \includegraphics[width=0.32\textwidth]{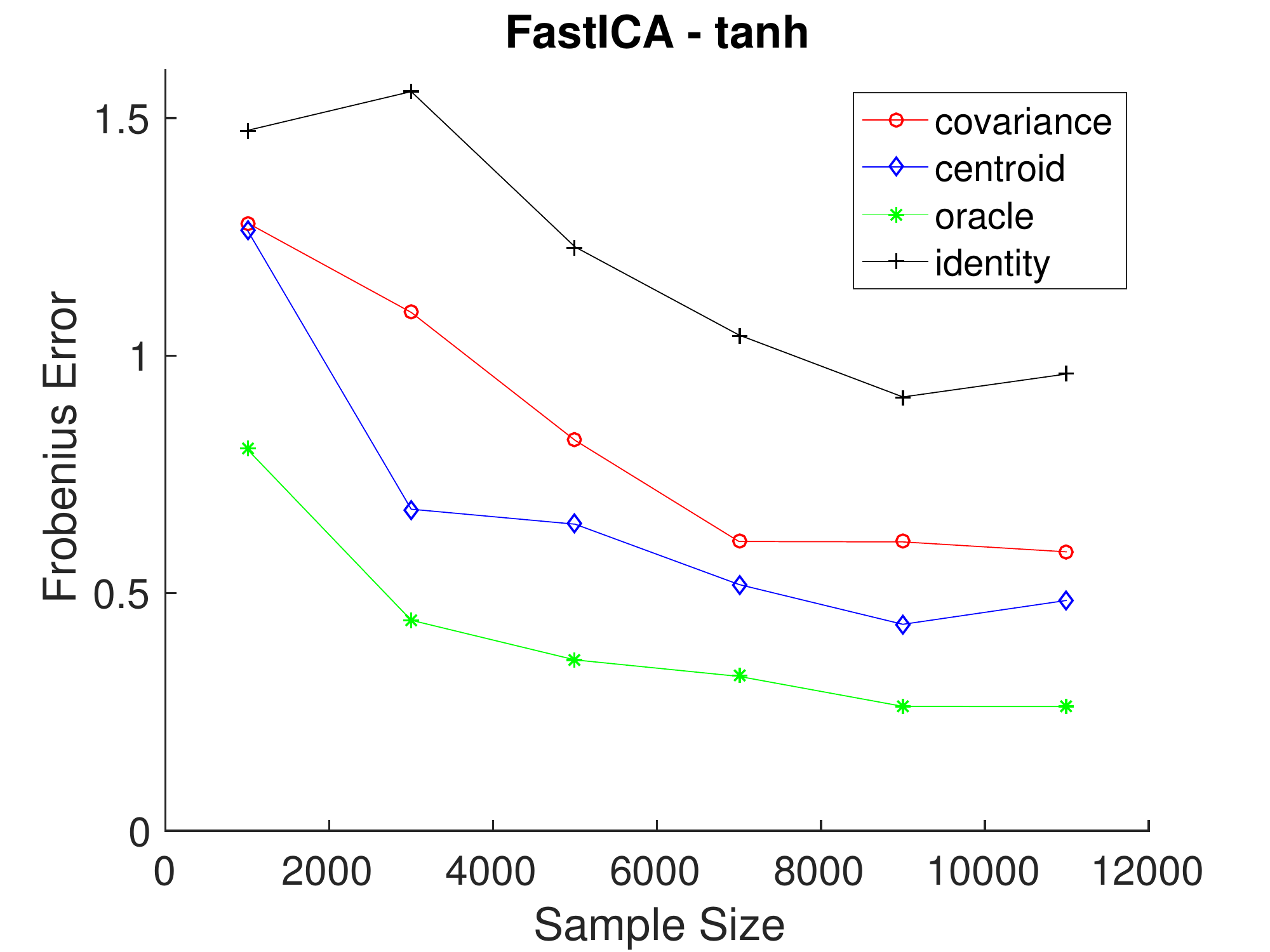}
    \includegraphics[width=0.32\textwidth]{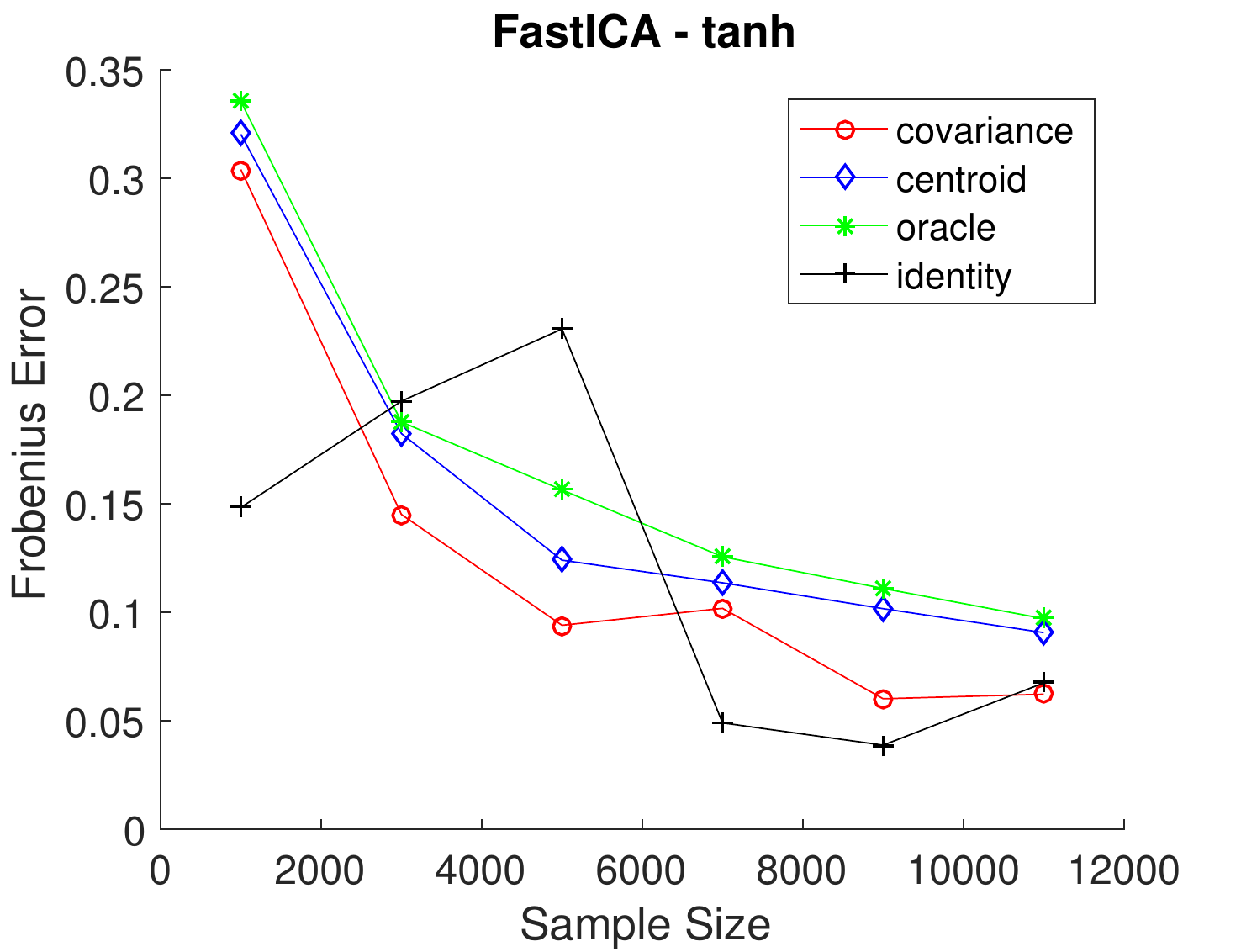}
    \caption{The Frobenius error of the recovered mixing matrix with the `pow3' and `tanh' contrast functions, on 10-dimensional data, averaged over ten trials. The mixing matrix $A$ is random with unit norm columns, not orthogonal.
    In the left and middle figures, the distribution has $\eta = (6,\dotsc,6,2.1,2.1)$ while in the right figure, $\eta = (2.1,\dotsc,2.1)$ (see Section \ref{sec:synthetic-data-experiments} for a discussion).
    }\label{fig:nonorthogonal-mixedexp}
\end{figure*}

The ``mixed'' regime in the left and middle of Figure~\ref{fig:nonorthogonal-mixedexp} (where some signals are \emph{not} heavy-tailed)
demonstrates a very dramatic contrast between different orthogonalization methods, even when only two heavy-tailed signals are present.

In the experiment with different methods of orthogonalization it was observed that when all exponents are the same or very close, orthogonalization via covariance performs better than orthogonalization via centroid and the true mixing matrix as seen in Figure~\ref{fig:nonorthogonal-mixedexp}.
A partial explanation is that, given the results in Figure~\ref{fig:orthogonal-mixedexp}, we know that equal exponents favor FastICA without damping and orthogonalization (\emph{identity} in Figure \ref{fig:nonorthogonal-mixedexp}). 
The line showing the performance with no orthogonalization and no damping (``identity'') behaves somewhat erratically, most likely due the presence of the heavy-tailed samples.
Additionally, damping and the choice of parameter $R$ is sensitive to scaling. A scaled-up distribution will be somewhat hurt because fewer samples will survive damping.

\subsection{ICA on speech data}\label{sec:voice-data}

While the above study on synthetic data provides interesting situations where heavy-tails can cause problems for ICA, we provide some results here which use real-world data, specifically human speech.
To study the performance of HTICA on voice data, we first examine whether the data 
is heavy-tailed.
The motivation to use speech data comes from observations by the signal processing community (e.g. \cite{kidmose2000alpha}) that speech data can be modeled by $\alpha$-stable distributions.
For an $\alpha$-stable distribution, with $\alpha \in (0,2)$, only the moments of order less than $\alpha$ will be finite.
We present here some results on a data set of human speech according to the standard cocktail party model, from \cite{uky_data}.

\iflong
The physical setup of the experiments (the human speakers and microphones) is shown in Figure~\ref{fig:voice-data-layouts}.

\begin{figure}
  \centering
  \includegraphics[width=0.49\textwidth]{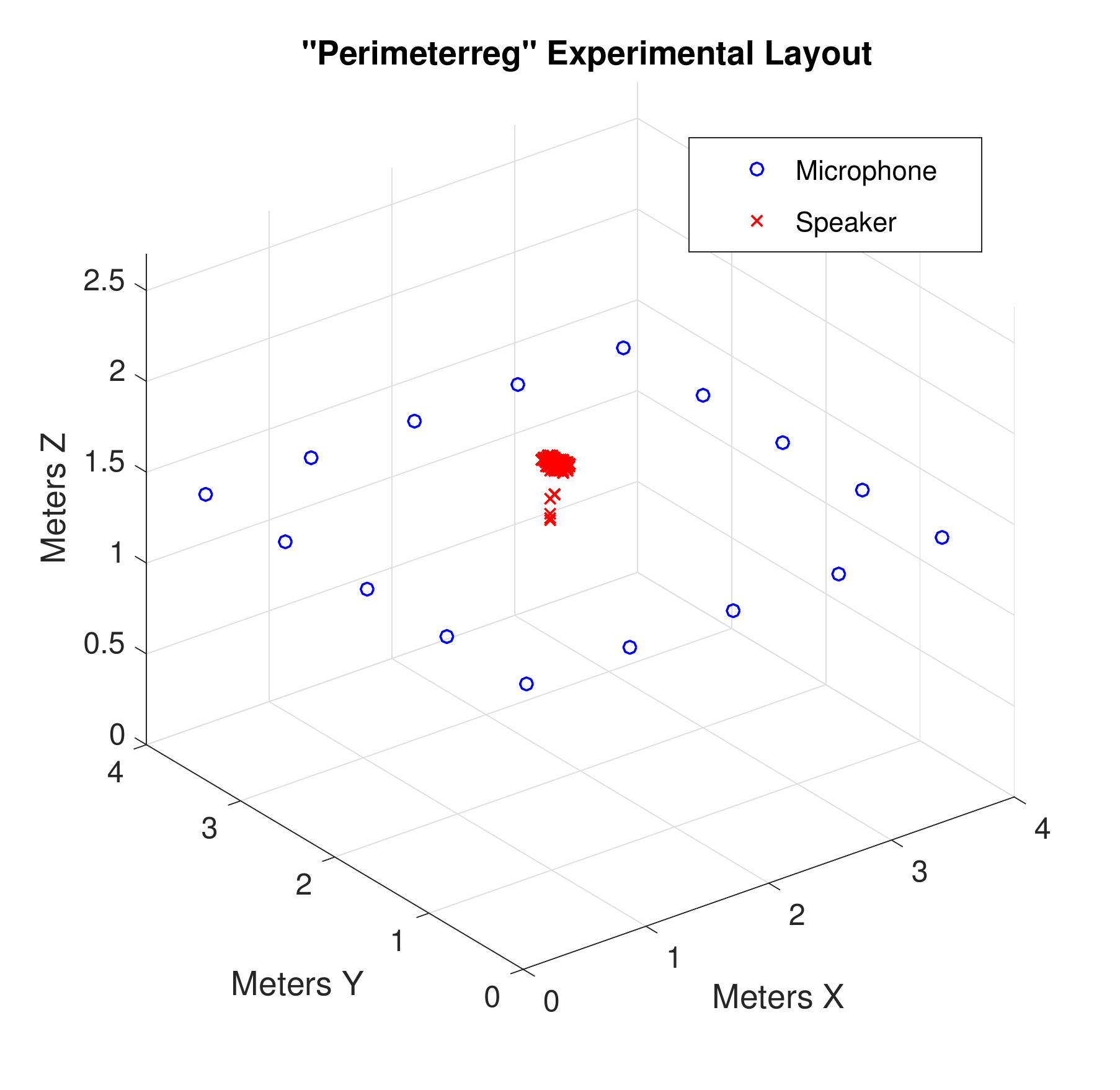}
  \includegraphics[width=0.49\textwidth]{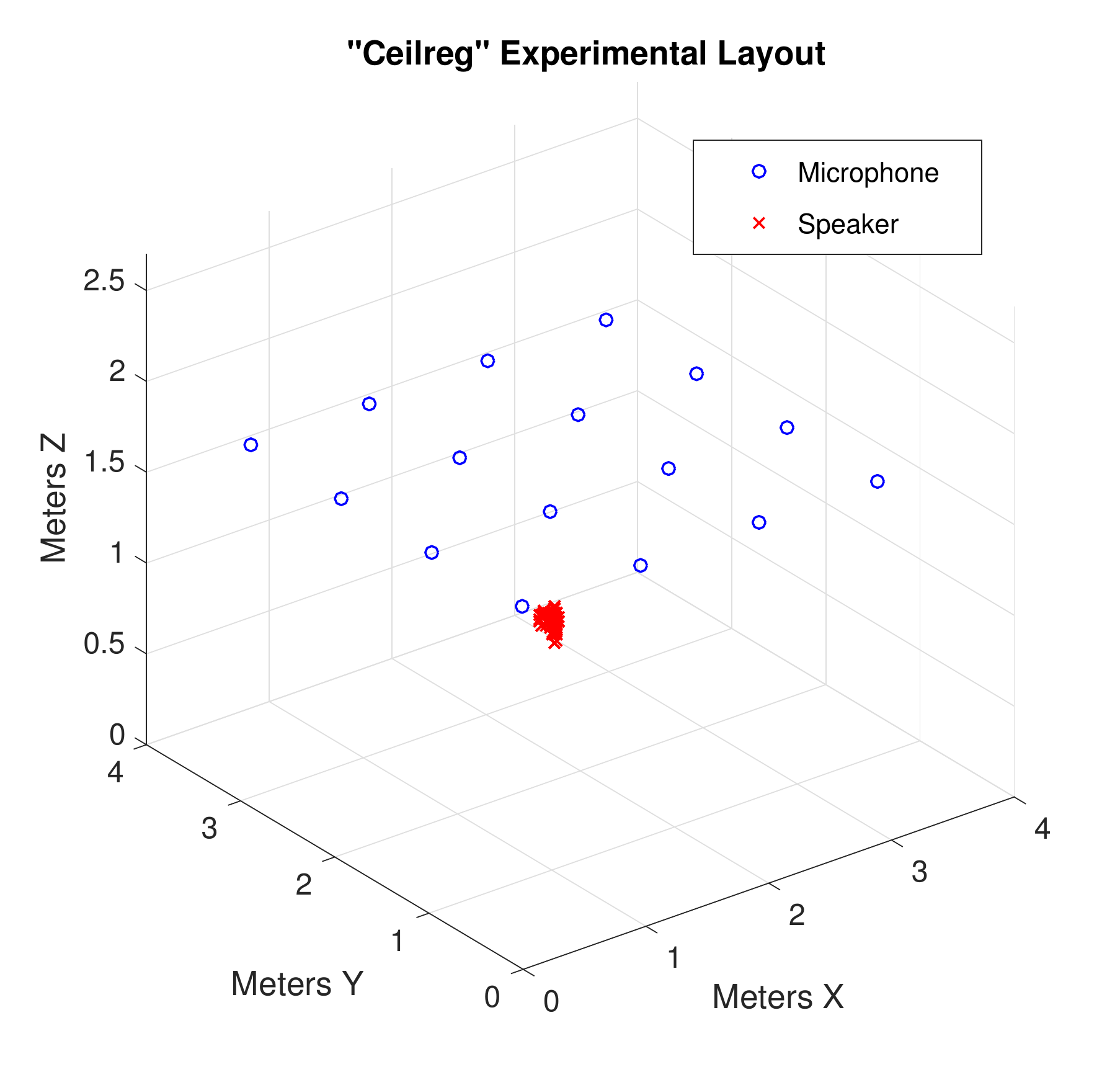}
  \caption{Microphone (blue circles) and human speaker (red ``$\times$'') layouts for the ``perimeterreg'' (left) and ``ceilreg'' (right) voice separation data sets.}
  \label{fig:voice-data-layouts}
\end{figure}
\fi

To estimate whether the data is heavy-tailed, as in \cite{kidmose2000alpha}, we estimate parameter $\alpha$ of a best-fit $\alpha$-stable distribution.
This estimate is in Figure~\ref{fig:stable-parameter-estimation-and-error} for one of the data sets collected.
We can see that the estimated $\alpha$ is  clearly in the heavy-tailed regime for some signals.
\jnote{could use more references here.}
\jnote{include in an appendix the actual estimator formula.}

\begin{figure}[t]
  \centering
  \includegraphics[width=0.49\textwidth]{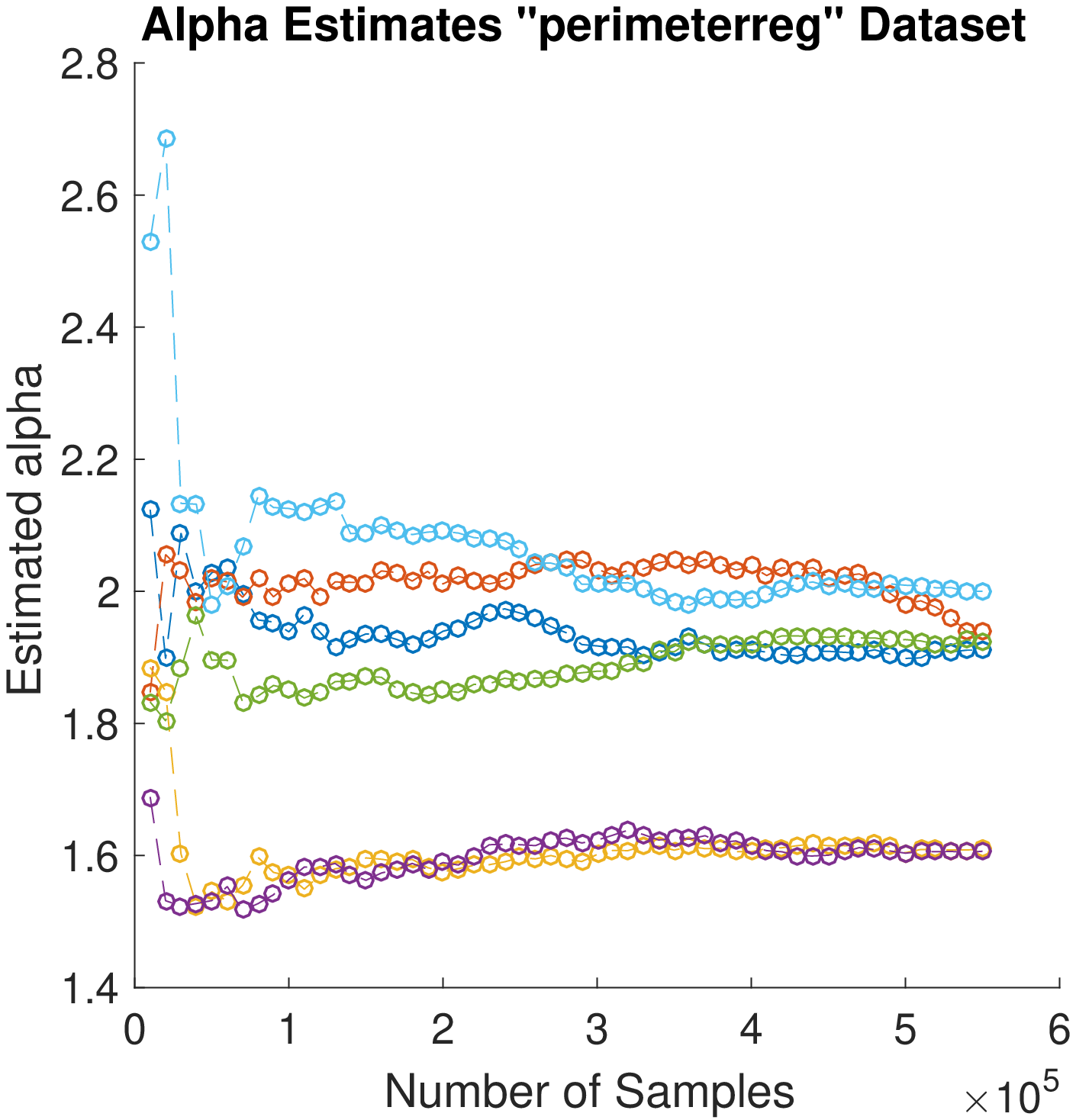}
  \includegraphics[width=0.49\textwidth]{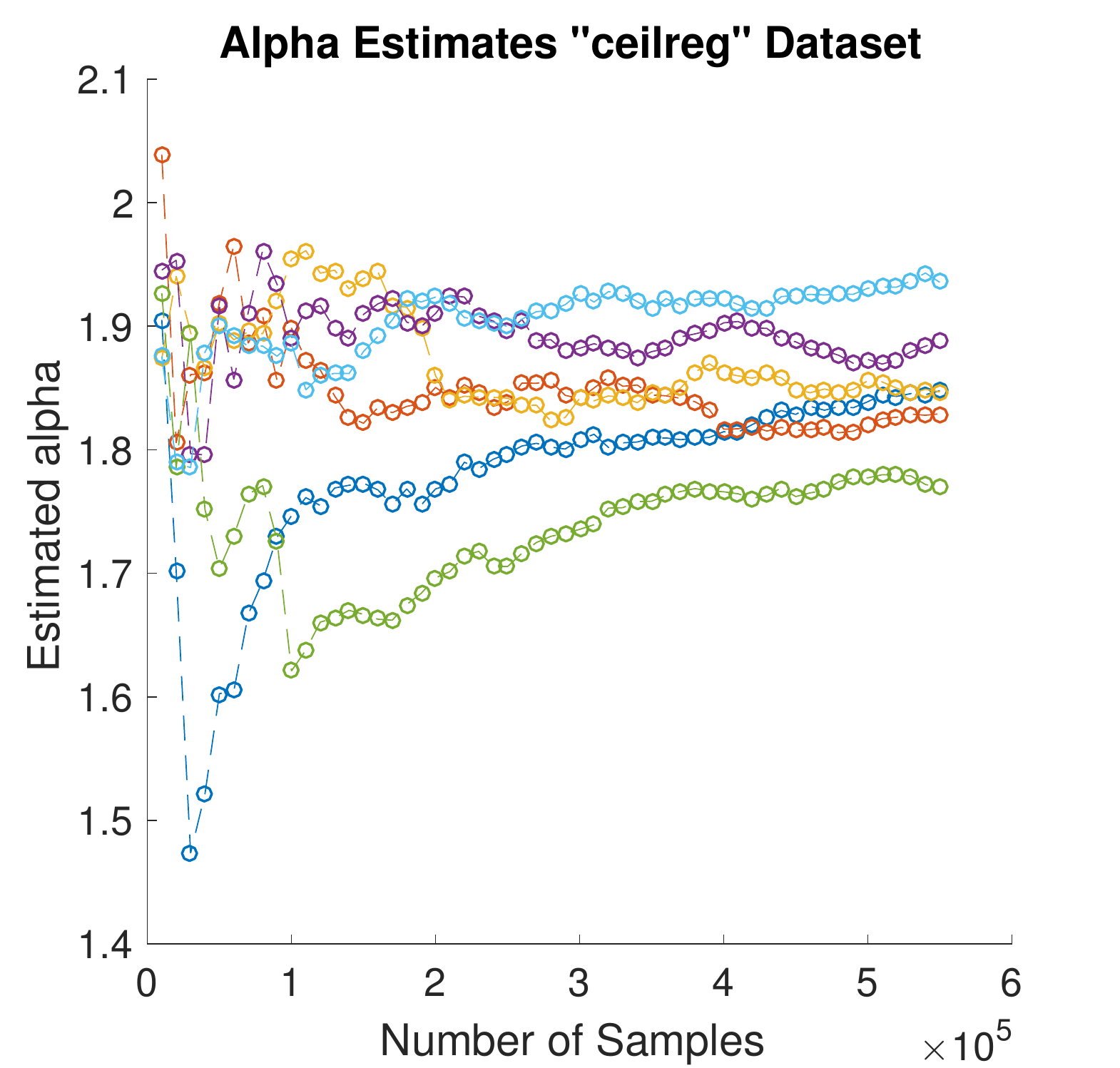}
  \caption{Stability parameter estimates for two voice recording data sets. Each line is one of the speech signals, we can see that the stability parameter is typically less than 2, and in the left data set, most of the signals have parameter close to 2, while in the right, the parameter is more commonly between 1.7 and 2.}
  \label{fig:stable-parameter-estimation-and-error}
\end{figure}

\begin{figure}[ht]
  \centering
  \includegraphics[width=0.5\textwidth]{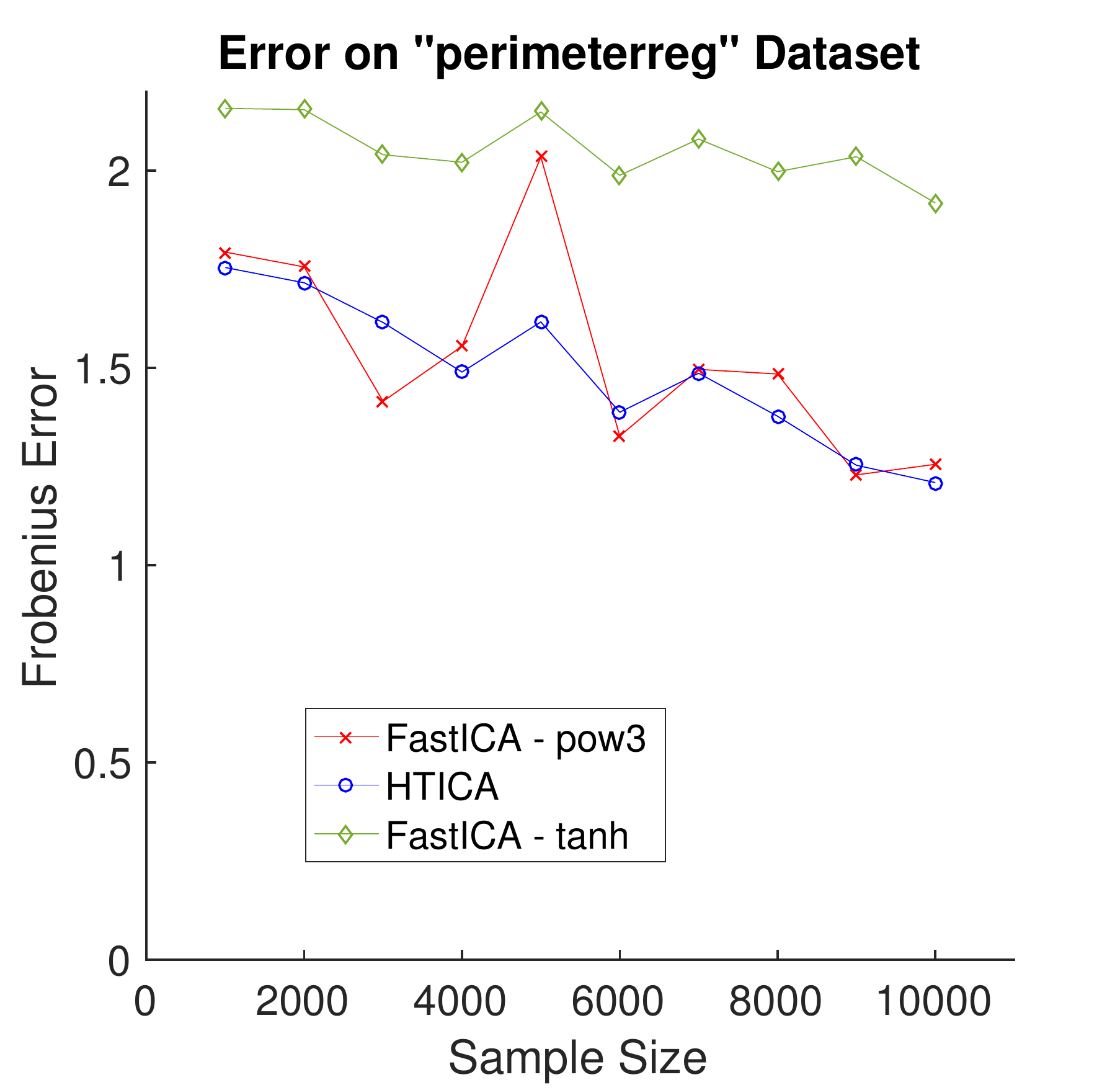}
  \caption{(Left): Error of estimated mixing matrix on the ``perimeterreg'' data, averaged over ten trials.}
  \label{fig:htica-error}
\end{figure}

\begin{figure}[ht]
  \centering
    \hspace{0.5cm}
    \resizebox{0.24\textwidth}{!}{%
      \begin{tabular}{|l|l|}
      \hline
      Signal & $\hat{\alpha}$ \\ \hline
      1 & 1.91 \\ \hline
      2 & 1.93 \\ \hline
      3 & 1.61 \\ \hline
      4 & 1.60 \\ \hline
      5 & 1.92 \\ \hline
      6 & 2.00 \\ \hline
      \end{tabular}
    }
    \hspace{0.7cm}
    \resizebox{0.5\textwidth}{!}{%
 \begin{tabular}{|c|l|l|l|l|l|}
\hline
\multicolumn{2}{|c|}{}               & \multicolumn{2}{l|}{$\sigma_{\min}$} & \multicolumn{2}{l|}{$\sigma_{\max}/\sigma_{\min}$} \\ \hline
\multicolumn{2}{|l|}{Orthogonalizer} & Centroid         & Covariance        & Centroid                & Covariance               \\ \hline
\multirow{6}{*}{Samples} & 1000  & 0.9302           & 0.9263            & 27.95                   & 579.95                   \\ \cline{2-6} 
                             & 3000  & 0.9603           & 0.9567            & 20.44                   & 410.11                   \\ \cline{2-6} 
                             & 5000  & 0.9694           & 0.9673            & 19.25                   & 490.11                   \\ \cline{2-6} 
                             & 7000  & 0.9739           & 0.9715            & 18.90                   & 347.68                   \\ \cline{2-6} 
                             & 9000  & 0.9790           & 0.9708            & 20.12                   & 573.18                   \\ \cline{2-6} 
                             & 11000 & 0.9793           & 0.9771            & 18.27                   & 286.34                   \\ \hline
\end{tabular}
}

  \caption{1) Stability parameter $\alpha$ estimates of each component in the ``perimeterreg'' data. Values below 2 are in the heavy-tailed regime.
    2): Smallest singular value and condition number of the orthogonalization matrix $BA$ computed via the centroid body and the covariance.
  The data was sampled with parameter $\eta = (6,6,6,6,6,6,6,6,2.1,2.1)$.}
  \label{fig:tables}
\end{figure}

Using data from \cite{uky_data}, we perform the same experiment as in Section~\ref{sec:synthetic-data-experiments}: generate a random mixing matrix with unit length columns, mix the data, and try to recover the mixing matrix.
Although the mixing is synthetic, the setting makes the resulting mixed signals same as real.
Specifically, the experiment was conducted in a room with chairs, carpet, plasterboard walls, and windows on one side.
There was natural noise including vents, computers, florescent lights, and traffic noise through the windows.

Figure~\ref{fig:stable-parameter-estimation-and-error} demonstrates that HTICA (orthogonalizing with centroid body scaling, Section~\ref{sec:centroidbodyscaling}) applied to speech data yields some noticeable improvement in the recovery of the mixing matrix, primarily in that it is less susceptible to data that causes FastICA to have large error ``spikes.''
Moreover, in many cases, running only FastICA on the mixed data failed to even recover all of the speech signals, while HTICA succeeded.
In these cases, we had to re-start FastICA until it recovered all the signals.


\chapter{Learning simplices}\label{sec:simplex-learning}\label{ch:simplex}

We are given uniformly random samples from an unknown convex body in $\RR^\dim$, how many samples are needed to approximately reconstruct
the body? It seems intuitively clear, at least for $\dim = 2, 3$, that
if we are given sufficiently many such samples then we can reconstruct (or learn) the body with very little error. For general $n$, it is known to require $2^{\Omega(\sqrt{\dim})}$
samples~\cite{GR09} (see also \cite{KOR} for a similar lower bound in a different but related model of learning).
This is an information-theoretic lower bound and no computational considerations are involved. As mentioned
in \cite{GR09}, it turns out that if the body has few facets (e.g. polynomial in $\dim$), then polynomial in $\dim$ samples are
sufficient for approximate reconstruction.
This is an information-theoretic upper bound and no efficient algorithms (i.e., with running time poly$(n)$) are known.
(We remark that to our knowledge the same
situation holds for polytopes with poly$(n)$ vertices.)  Here, we study the
reconstruction problem for the special case when the input bodies are restricted to be (full-dimensional) simplices. We show that
in this case one can in
fact learn the body efficiently. More precisely, the algorithm  knows that the input body is a simplex but only up to an affine transformation, and the problem is to recover this affine transformation.
This answers a question of \noname{Frieze~et~al.~}\cite[Section 6]{FJK}.

The problem of learning a simplex is also closely related to the well-studied problem of learning intersections of half-spaces.
Suppose that
the intersection of $\dim+1$ half-spaces in $\RR^\dim$ is bounded, and we are given poly$(\dim)$ uniformly random samples from it. Then our
learning simplices result directly implies that we can learn the $\dim+1$ half-spaces. This also has the advantage of being a proper
learning algorithm, meaning that the output of the algorithm is a set of $\dim+1$ half-spaces, unlike many of the previous algorithms.

\paragraph{Previous work.}

Perhaps the first approach to learning simplices that comes to mind is to find a minimum volume simplex containing the samples. This can be shown
to be a good approximation to the original simplex. (Such minimum volume estimators have been studied in machine learning literature, see e.g. \cite{ScholkopfPSSW01} for the problem of estimating the support of a probability distribution. We are not aware of any
technique that applies to our situation and provides theoretical guarantees.)
However, the problem of finding a minimum volume simplex is in general
NP-hard~\cite{Packer}. This hardness is not directly applicable for our problem because our input is a random sample and not a general
point set. Nevertheless, we do not have an algorithm for directly finding a minimum volume simplex; instead we use ideas similar to those used in Independent Component Analysis (ICA).
ICA studies the following problem:
Given a sample from an affine transformation of a random vector with independently distributed coordinates, recover the affine transformation (up to some unavoidable ambiguities).
\noname{Frieze~et~al.~}\cite{FJK} gave an efficient algorithm for this problem
(with some restrictions on the allowed distributions, but also with some weaker requirements than full independence)
along with most of the details of a rigorous analysis (a complete analysis of a special case can be found in \noname{Arora et al.~}\cite{Arora2012}; see also \noname{Vempala and Xiao~}\cite{VempalaX11} for a generalization of ICA to subspaces along with a rigorous analysis). 
The problem of learning parallelepipeds from uniformly random samples is a special case of this problem. 
\cite{FJK} asked if one could learn other convex bodies, and in particular simplices, efficiently from uniformly random samples.
\noname{Nguyen and Regev~}\cite{NguyenR09} gave a simpler and rigorous algorithm and analysis for the case of learning parallelepipeds with similarities to the popular FastICA algorithm of \noname{Hyv\"{a}rinen~}\cite{Hyvarinen99}.
The algorithm in \cite{NguyenR09} is a first order algorithm unlike Frieze~et~al.'s second order algorithm.

The algorithms in both \cite{FJK, NguyenR09} make use of the fourth moment function of the probability distribution. Briefly, the fourth moment in direction $u \in \RR^n$ is $\mathbb{E}(u\cdot X)^4$, where $X \in \RR^n$ is
the random variable distributed according to the input distribution. The moment function can be estimated from
the samples. The independent components of the distribution correspond to local maxima or minima of the
moment function, and can be approximately found by finding the local maxima/minima of the moment function
estimated from the sample.


More information on ICA including historical remarks can be found in \cite{ICAbook, BlindSS}.
Ideas similar to ICA have been used in
statistics in the context of projection pursuit since the mid-seventies.
It is not clear how to apply ICA to the simplex learning problem directly as there is no clear independence among the components. Let us
note that \noname{Frieze et al.~}\cite{FJK} allow certain kinds of dependencies among the components, however this does not appear to be useful
for learning simplices.

Learning intersections of half-spaces is a well-studied problem in learning theory. The problem
of PAC-learning intersections of even two half-spaces is open, and there is evidence that it is
hard at least for sufficiently large number of half-spaces: E.g., \noname{Klivans and Sherstov~}\cite{KlivansS09} prove
that learning intersections of $n^\epsilon$ half-spaces in $\RR^\dim$ (for constant $\epsilon>0$) is
hard under standard cryptographic assumptions (PAC-learning is possible, however, if one also
has access to a membership oracle in addition to random samples \cite{KwekP98}). Because of this, much effort has been expended
on learning when the distribution of random samples is some simple distribution,
see e.g. \cite{KlivansS07, Vempala10, VempalaFocs10} and references therein. This line of work makes substantial progress
towards the goal of learning intersections of $k$ half-spaces efficiently, however it falls short of
being able to do this in time polynomial in $k$ and $n$; in particular, these algorithms do not seem to be
able to learn simplices.  The distribution of samples in
these works is either the Gaussian distribution or the uniform distribution over a ball.
\noname{Frieze et al.~}\cite{FJK} and \noname{Goyal and Rademacher~}\cite{GR09} consider the uniform distribution over
the intersection. Note that this requires that the intersection be bounded. Note also that one
only gets positive samples in this case unlike other work on learning intersections of
half-spaces. The problem of learning convex bodies can also be thought of as learning a distribution or density estimation
problem for a special class of distributions.

\noname{Gravin et al.~}\cite{Gravin12} show how to reconstruct a polytope with $N$ vertices in $\RR^n$, given its first $O(nN)$ moments
in $(n+1)$ random directions. In our setting, where we have access to only a polynomial number of random samples, it's not clear
how to compute moments of such high orders to the accuracy required for the algorithm of \cite{Gravin12} even for simplices.

A recent and parallel work of \noname{Anandkumar et al.~}\cite{AnandTensor} is closely related to ours. They show that tensor decomposition
methods can be applied to low-order moments of various latent variable models to estimate their parameters. The latent variable 
models considered by them include Gaussian mixture models, hidden Markov models and latent Dirichlet allocations. The tensor 
methods used by them and the local optima technique we use seem closely related. One could view our work, as well as theirs, as 
showing that the method of moments along with existing algorithmic techniques can be applied to certain unsupervised learning problems.

In Section~\ref{sec:probabilistic-results}, we give the theorems and proofs which allow the efficient reduction to ICA.

\paragraph{Our results}

For clarity of the presentation, we use the following machine model for the running time: a random access machine that allows the following exact arithmetic operations over real numbers in constant time: addition, subtraction, multiplication, division and square root.

The estimation error is measured using total variation distance, denoted $d_{TV}$ (see Section \ref{sec:simplex-preliminaries}).

\begin{theorem} \label{thm:main}
There is an algorithm (Algorithm~\ref{alg:simplex} below) such that given access to random samples from a simplex
$S_{INPUT} \subseteq \RR^\dim $, with probability at least $1-\delta$ over the sample and the randomness of the algorithm, it outputs $n+1$ vectors that are the vertices of a simplex $S$ so that $d_{TV}(S, S_{INPUT}) \leq \eps$.
The algorithm runs in time polynomial in $\dim$, $1/\eps$ and $1/\delta$.
\end{theorem}

As mentioned earlier, our algorithm uses ideas from ICA. 
Our algorithm uses the third moment instead of the fourth moment used in certain versions of ICA. The third moment is not useful for learning symmetric bodies such as the cube as it is identically $0$. 
It is however useful for learning a simplex where it provides useful information, and is easier to handle than the fourth moment.
One of the main contributions of our work is the understanding of the third moment of a simplex and the structure of local maxima. This is more involved than in previous work as the simplex has no obvious independence structure, and the moment polynomial one gets has no obvious structure unlike for ICA.

The probability of success of the algorithm can be ``boosted'' so that the dependence of the running time on $\delta$ is only linear in $\log(1/\delta)$ as follows: The following discussion uses the space of simplices with total variation distance as the underlying metric space. Let $\eps$ be the target distance. Take an algorithm that succeeds with probability $5/6$ and error parameter $\eps'$ to be fixed later (such as Algorithm \ref{alg:simplex} with $\delta = 1/6$). Run the algorithm $t = O(\log 1/\delta)$ times to get $t$ simplices. By a Chernoff-type argument, at least $2t/3$ simplices are within $\eps'$ of the input simplex with probability at least $1- \delta/2$.

By sampling, we can estimate the distances between all pairs of simplices with additive error less than $\eps'/10$ in time polynomial
in $t, 1/\epsilon'$ and $\log{1/\delta}$ so that all estimates are correct with probability at least $1-\delta/2$.
For every output simplex, compute the number of output simplices within estimated distance $(2+1/10)\eps'$. With probability at least
$1-\delta$ both of the desirable events happen, and then necessarily there is at least one output simplex, call it $S$, that has $2t/3$ output simplices within estimated distance $(2+1/10)\eps'$. Any such $S$ must be within $(3+2/10) \eps'$ of the input simplex. Thus, set $\eps'=\eps/(3+2/10)$.

While our algorithm for learning simplices uses techniques for ICA, we have to do 
substantial work to make those techniques work for the simplex problem. 
We also show a more direct connection between the problem of learning a simplex and ICA: a randomized reduction from the problem of learning a simplex to ICA. The connection is based on a known representation of the uniform measure on a simplex as a normalization of a vector having independent coordinates. Similar representations are known for the uniform measure in an $n$-dimensional $\ell_p$ ball (denoted $\ell_p^n$) \cite{barthe2005probabilistic} and the \emph{cone measure} on the boundary of an $\ell_p^n$ ball \cite{schechtman1990volume, rachev1991approximate, MR1396997} (see Section \ref{sec:simplex-preliminaries}  for the definition of the cone measure). 
These representations lead to a reduction from the problem of learning an affine transformation of an $\ell_p^n$ ball to ICA. These reductions show connections between estimation problems with no obvious independence structure and ICA. 
They also make possible the use of any off-the-shelf implementation of ICA. 
However, the results here do not supersede our result for learning simplices because to our knowledge no rigorous analysis is available for 
the ICA problem when the distributions are the ones in the above reductions. 

\paragraph{Idea of the algorithm.}

 The new idea for the algorithm is that after putting
the samples in a suitable position (see below), the third moment of the sample can be used to recover the simplex using a simple 
FastICA-like algorithm. We outline our algorithm next.

As any full-dimensional simplex can be mapped to any other full-dimensional simplex by an invertible affine transformation, it is enough to determine the translation and linear transformation that would take the given simplex to some canonical simplex.
As is well-known for ICA-like problems (see, e.g., \cite{FJK}), 
this transformation can be determined \emph{up to a rotation} from the mean and the covariance matrix of the uniform distribution on the given simplex.
The mean and the covariance matrix can be estimated efficiently from a sample.
A convenient choice of an $n$-dimensional simplex is the convex hull of the canonical vectors in $\RR^{\dim+1}$.
We denote this simplex $\Delta_n$ and call it the \emph{standard simplex}.
So, the algorithm begins by picking an arbitrary invertible affine transformation $T$ that maps $\RR^\dim$ onto the hyperplane $\{x \in \RR^{\dim+1} \suchthat \ones \cdot x = 1 \}$. We use a $T$ so that $T^{-1}(\Delta_n)$ is an isotropic\footnote{See Section \ref{sec:simplex-preliminaries}.} simplex. In this case, the algorithm brings the sample set into isotropic position and embeds it in $\RR^{\dim+1}$ using $T$.
After applying these transformations we may assume (at the cost of small errors in the final result) that our sample set is obtained by sampling from an unknown rotation of the standard simplex
that leaves the all-ones vector (denoted $\ones$ from now on) invariant (thus this rotation keeps the center of mass of the standard simplex fixed), and the problem is to recover this rotation.


To find the rotation, the algorithm will find the vertices of the rotated simplex approximately. This can be done efficiently because of the following characterization of the vertices: Project the vertices of the simplex onto the hyperplane through the origin orthogonal to $\ones$ and normalize the resulting vectors. Let $V$ denote this set of $n+1$ points. Consider the problem of maximizing the third moment of the uniform distribution in the simplex along unit vectors orthogonal to $\ones$. Then $V$ is the complete set of local maxima and the complete set of global maxima (Theorem~\ref{thm:maxima}). A fixed point-like iteration (inspired by the analysis of FastICA~\cite{Hyvarinen99} and of gradient descent in \cite{NguyenR09}) starting from a random point in the unit sphere finds a local maximum efficiently with high probability. By the analysis of the coupon collector's problem, $O(\dim \log \dim)$ repetitions are highly likely to find all local maxima.

\paragraph{Idea of the analysis.}

In the analysis, we first argue that after putting the sample in isotropic position and mapping it through $T$, it is enough to analyze the algorithm in the case where the sample comes from a simplex $S$ that is close to a simplex $S'$ that is the result of applying a rotation leaving $\ones$ invariant to the standard simplex. The closeness here depends on the accuracy of the sample covariance and mean as an estimate of the input simplex's covariance matrix and mean. A sample of size $O(n)$ guarantees
(\cite[Theorem 4.1]{MR2601042}, \cite[Corollary 1.2]{1106.2775}) that the covariance and mean are close enough so that the uniform distributions on $S$ and $S'$ are close in total variation. We show that the subroutine that finds the vertices (Subroutine \ref{alg:vertex}), succeeds with some probability when given a sample from $S'$. By definition of total variation distance, Subroutine \ref{alg:vertex} succeeds with almost as large probability when given a sample from $S$ (an argument already used in \cite{NguyenR09}). As an additional simplifying assumption, it is enough to analyze the algorithm (Algorithm \ref{alg:simplex}) in the case where the input is isotropic, as the output distribution of the algorithm is equivariant with respect to affine invertible transformations as a function of the input distribution.

\section{Preliminaries}\label{sec:simplex-preliminaries}
An $\dim$-simplex is the convex hull of $\dim+1$ points in $\RR^{\dim}$ that do not lie on an $(\dim-1)$-dimensional affine hyperplane.
It will be convenient to work with the standard $\dim$-simplex
$\Delta^{\dim}$ living in $\RR^{\dim+1}$ defined as the convex hull of the $\dim+1$ canonical unit
vectors $e_1, \ldots, e_{\dim+1}$; that is
\begin{align*}
\Delta^{\dim} = \{(x_0, \ldots, x_{\dim}) \in \RR^{\dim+1} &\suchthat x_0 + \dotsb + x_{\dim} = 1 \optionalbreak 
\text{ and } x_i \geq 0 \text{ for all } i\}.
\end{align*}
The canonical simplex $\Omega^\dim$ living in $\RR^\dim$ is given by
\begin{align*}
\{(x_0, \dotsc, x_{\dim-1}) \in \RR^{\dim} &\suchthat x_0 + \dotsb + x_{\dim-1} \leq 1 \optionalbreak
\text{ and } x_i \geq 0 \text{ for all } i\}.
\end{align*}
Note that $\Delta^\dim$ is the facet of $\Omega^{\dim+1}$ opposite to the origin.

Let $B_n$ denote the $n$-dimensional Euclidean ball.


The complete homogeneous symmetric polynomial of degree $d$ in variables $u_0, \ldots, u_{\dim}$, denoted $h_n(u_0, \ldots, u_{\dim})$,
is the sum of all monomials of degree $d$ in the variables:
\begin{align*}
h_d(u_0, \ldots, u_\dim)
&= \sum_{k_0 + \dotsb + k_\dim = d} u_0^{k_0} \dotsm u_\dim^{k_\dim} \optionalbreak
= \sum_{0 \leq i_0 \leq i_1 \leq \dotsb \leq i_d \leq \dim} u_{i_0} u_{i_1} \dotsm u_{i_d}.
\end{align*}

Also define the $d$-th power sum as
$$
p_d(u_0, \ldots, u_\dim) = u_0^d + \ldots + u_\dim^d.
$$

For a vector $u = (u_0, u_1, \ldots, u_\dim)$, we define
\[
u^{(2)} = (u_0^2, u_1^2, \ldots, u_\dim^2).
\]
Vector $\ones$ denotes the all ones vector (the dimension of the vector will be clear from the context).

A random vector $X \in \RR^\dim$ is \emph{isotropic} if $\e(X)=0$ and $\e(X X^T) = I$. A compact set in $\RR^\dim$ is isotropic if a uniformly distributed random vector in it is isotropic. The inradius of an isotropic $n$-simplex is $\sqrt{(n+2)/n}$, the circumradius is $\sqrt{n(n+2)}$.

The total variation distance between two probability measures is $d_{TV}(\mu, \nu) = \sup_A \abs{\mu(A) - \nu(A)}$ for measurable $A$. For two compact sets $K, L \subseteq \RR^\dim$, we define the total variation distance $d_{TV}(K, L)$ as the total variation distance between the corresponding uniform distributions on each set. It can be expressed as
\[
    d_{TV}(K,L) =
    \begin{cases}
        \frac{\vol{K \setminus L}}{\vol K} & \text{if $\vol K \geq \vol L$}, \\
        \frac{\vol{L \setminus K}}{\vol L} & \text{if $\vol L > \vol K$.}
    \end{cases}
\]
This identity implies the following elementary estimate:
\begin{lemma}\label{lem:bmtv}
Let $K, L$ be two compact sets in $\RR^\dim$. Let $0 < \alpha \leq 1 \leq \beta$ such that
$
\alpha K \subseteq L \subseteq \beta K$.
Then $d_{TV}(K,L) \leq 2\left(1-(\alpha/\beta)^\dim\right)$.
\end{lemma}
\begin{proof}
We have $d_{TV}(\alpha K, \beta K) = 1-(\alpha/\beta)^\dim$. Triangle inequality implies the desired inequality.
\end{proof}

\begin{lemma}\label{lem:coupons}
Consider the coupon collector's problem with $n$ coupons where every coupon occurs with probability at least $\alpha$. Let $\delta >0$. Then with probability at least $1-\delta$ all coupons are collected after $\alpha^{-1} (\log n + \log 1/\delta)$ trials.
\end{lemma}
\begin{proof}
The probability that a particular coupon is not collected after that many trials is at most
\[
(1-\alpha)^{\alpha^{-1} (\log n + \log 1/\delta)} \leq e^{-\log n - \log 1/\delta} = \delta/n.
\]
The union bound over all coupons implies the claim.
\end{proof}

For a point $x \in \RR^{\dimension}$, $\norm{x}_p = (\sum_{i = 1}^{\dimension} |x_i|^p)^{1/p}$ is the standard $\ell_p$ norm.  The unit $\ell_p^n$ ball is defined by 
\[
B_p^{\dimension} = \{x \in \RR^{\dimension}: \norm{x}_p \leq 1 \}.
\]

The Gamma distribution is denoted as $\mathrm{Gamma}(\alpha, \beta)$ and has density $f(x; \alpha, \beta) = \frac{\beta^\alpha}{\Gamma(\alpha)} x^{\alpha - 1} e^{- \beta x } 1_{x \geq 0}$, 
with shape parameters $\alpha, \beta >0$.
$\mathrm{Gamma}(1, \lambda)$ is the exponential distribution, denoted $\expdist{\lambda}$. 
The Gamma distribution also satisfies the following additivity property: If $X$ is distributed as $\mathrm{Gamma}(\alpha, \beta)$ and $Y$ is distributed as $\mathrm{Gamma}(\alpha', \beta)$, then $X+Y$ is distributed as $\mathrm{Gamma}(\alpha + \alpha', \beta)$.

The cone measure on the surface $\partial K$ of centrally symmetric convex body $K$ in $\RR^{\dimension}$ \cite{barthe2005probabilistic, schechtman1990volume, rachev1991approximate, MR1396997} is 
defined by
\[
\mu_{K}(A) = \frac{\vol(ta; a \in A, 0 \leq t \leq 1)}{\vol(K)}.
\] 
It is easy to see that $\mu_{B_p^{\dimension}}$ is uniform on $\partial B_p^{\dimension}$ for $p \in \{1, 2, \infty\}$.

From \cite{schechtman1990volume} and \cite{rachev1991approximate} we have the following representation of the cone measure on $\partial B_p^n$:
\begin{theorem}\label{thm:cone}
Let $G_1, G_2, \dotsc, G_n$ be iid random variables with density proportional to $\exp(-\abs{t}^p)$. Then the random vector $X = G/\norm{G}_p$
is independent of $\norm{G}_p$. Moreover, $X$ is distributed according to $\mu_{B_p^{\dimension}}$.
\end{theorem}

From \cite{barthe2005probabilistic}, we also have the following variation, a representation of the uniform distribution in $B_p^n$:
\begin{theorem} \label{theorem:fullnaor} 
Let $G = (G_1, \dotsc, G_{\dimension})$ be iid random variables with density proportional to $\exp(-|t|^p)$. Let $Z$ be a random variable distributed as $\expdist{1}$, independent of $G$. Then the random vector
\[
V = \frac{G}{\bigl( \sum_{i=1}^{\dimension} \abs{G_i}^p + Z\bigr)^{1/p}}
\]
is uniformly distributed in $B_p^n$.
 \end{theorem} 

 See e.g. \cite[Section 20]{MR1324786} for the change of variable formula in probability.

\section{Computing the moments of a simplex}\label{sec:moment}
The $k$-th moment $m_k(u)$ over $\Delta^\dim$ is the function
\[
u \mapsto \mathbb{E}_{X \in \Delta_\dim}((u\cdot X)^k).
\]
In this section we present a formula for the moment over $\Delta^\dim$.
Similar more general formulas appear in \cite{LasserreA}.
We will use the following result from \cite{GrundmannM78} for $\alpha_i\geq 0$:
\[
\int_{\Omega^{\dim+1}} x_0^{\alpha_0} \dotsm x_\dim^{\alpha_\dim} dx = \frac{\alpha_0! \dotsm \alpha_\dim!}{(\dim+1+ \sum_i \alpha_i)!}.
\]

From the above we can easily derive a formula for integration over $\Delta^\dim$:
$$
\int_{\Delta^\dim} x_0^{\alpha_0} \dotsm x_\dim^{\alpha_\dim} dx = \sqrt{n+1} \cdot \frac{\alpha_0! \dotsm \alpha_\dim!}{(\dim + \sum_i \alpha_i)!}.
$$

Now
\begin{align*}
\int_{\Delta^\dim} &(x_0 u_0 + \ldots + x_\dim u_\dim)^k dx \\
&= \sum_{k_0 + \dotsb + k_\dim = k} \binom{k}{k_0!, \ldots, k_\dim!} u_0^{k_0} \ldots u_\dim^{k_\dim} \int_{\Delta^\dim} x_0^{k_0} \ldots x_\dim^{k_\dim} dx \\
&= \sum_{k_0 + \dotsb + k_\dim = k} \binom{k}{k_0!, \ldots, k_\dim!} u_0^{k_0} u_0^{k_0} \ldots u_\dim^{k_\dim}
\frac{ \sqrt{\dim+1} \cdot k_0! \ldots k_\dim!}{(\dim + \sum_i k_i)!} \\
&= \frac{k! \sqrt{\dim+1}}{(\dim+k)!}  \sum_{k_0 + \dotsb + k_\dim = k} u_0^{k_0} \ldots u_\dim^{k_\dim} \\
&= \frac{k! \sqrt{\dim+1}}{(\dim+k)!} \; h_k(u).
\end{align*}

The variant of Newton's identities for the complete homogeneous symmetric polynomial gives the following relations which
can also be verified easily by direct computation:
$$ 3 h_3(u) = h_2(u) p_1(u) + h_1(u) p_2(u) + p_3(u),$$
$$ 2 h_2(u) = h_1(u) p_1(u) + p_2(u) = p_1(u)^2 + p_2(u).$$

Divide the above integral by the volume of the standard simplex $|\Delta_n|=\sqrt{\dim+1}/{\dim!}$
to get the moment:
\begin{eqnarray*}
m_3(u) & = & \frac{3! \sqrt{\dim+1}}{(\dim+3)!} h_3(u) / |\Delta_n| \\
       & = & \frac{2 (h_2(u) p_1(u) + h_1(u) p_2(u) + p_3(u))}{(\dim+1)(\dim+2)(\dim+3)}  \\
       & = & \frac{(p_1(u)^3 + 3 p_1(u) p_2(u) + 2 p_3(u))}{(\dim+1)(\dim+2)(\dim+3)} .
\end{eqnarray*}

\section{Subroutine for finding the vertices of a rotated standard simplex} \label{sec:standard}
In this section we solve the following simpler problem: Suppose we have $\text{poly}(\dim)$ samples from a rotated copy
$S$ of the standard simplex, where the rotation is such that it leaves $\ones$ invariant. The problem
is to approximately estimate the vertices of the rotated simplex from the samples.

We will analyze our algorithm in the coordinate system in which the input simplex is the standard simplex. This is only
for convenience in the analysis and the algorithm itself does not know this coordinate system.


As we noted in the introduction, our algorithm is inspired by the algorithm of \noname{Nguyen and Regev~}\cite{NguyenR09} for
the related problem of learning hypercubes and also by the FastICA algorithm in \cite{Hyvarinen99}. New ideas are
needed for our algorithm for learning simplices; in particular, our update rule is different. With the right update
rule in hand the analysis turns out to be quite similar to the one in \cite{NguyenR09}.

We want to find local maxima of the sample third moment. A natural approach to do this would be to use gradient
descent or Newton's method (this was done in \cite{FJK}). Our algorithm, which only uses first order information, can
be thought of as a fixed point algorithm leading to a particularly simple analysis and fast convergence. Before stating
our algorithm we describe the update rule we use.

We will use the abbreviation $C_\dim = (\dim+1)(\dim+2)(\dim+3)/6$. Then, from the expression for $m_3(u)$ we get
\[
\nabla m_3(u) =  \frac{1}{6C_\dim} \left(3 p_1(u)^2 \ones + 3 p_2(u) \ones + 6 p_1(u) u + 6 u^{(2)}\right).
\]
Solving for $u^{(2)}$ we get
\begin{align}
u^{(2)} &= C_\dim \nabla m_3(u) - \frac{1}{2} p_1(u)^2 \ones - \frac{1}{2} p_2(u) \ones -  p_1(u) u \nonumber \\
       &=  C_\dim  \nabla m_3(u) - \frac{1}{2} (u \cdot \ones)^2 \ones - \frac{1}{2} (u \cdot u)^2 \ones - (u \cdot \ones) u.\label{equ:squaring}
\end{align}

While the above expressions are in the coordinate system where the input simplex is the canonical simplex, the
important point is that all terms in
the last expression can be computed in any coordinate system that is obtained by a rotation leaving $\ones$ invariant.
Thus, we can compute $u^{(2)}$ as well independently of what coordinate system we are working in. This immediately
gives us the algorithm below.
We denote by $\hat{m}_3(u)$ the sample third moment, i.e., $\hat{m}_3(u)=\frac{1}{t}\sum_{i=1}^{t}(u \cdot r_i)^3$ for $t$ samples.
This is a polynomial in $u$, and the gradient is computed in the obvious way. Moreover, the gradient of the sample moment is clearly an unbiased estimator of the gradient of the moment; a bound on the deviation is given in the analysis (Lemma \ref{lem:gradient}).
For each evaluation of the gradient of the sample moment, we use a fresh sample.

It may seem a bit alarming that the fixed point-like iteration is squaring the coordinates of $u$, leading to an extremely fast growth (see Equation \ref{equ:squaring} and Subroutine 1). But, as in other algorithms having quadratic convergence like certain versions of Newton's method, the convergence is very fast and the number of iterations is small. We show below that it is $O(\log(n/\delta))$, leading to a growth of $u$ that is polynomial in $n$ and $1/\delta$. The boosting argument described in the introduction makes the final overall dependence in $\delta$ to be only linear in
$\log (1/\delta)$.

We state the following subroutine for $\mathbb{R}^n$ instead of $\mathbb{R}^{n+1}$ (thus it is learning a rotated copy of
$\Delta^{n-1}$ instead of $\Delta^n$). This is for notational convenience so that we work with $n$ instead of $n+1$.

\begin{subroutine}
\caption{Find one vertex of a rotation of the standard simplex $\Delta^{n-1}$ via a fixed point iteration-like algorithm}\label{alg:vertex}
\begin{algorithmic}
\State Input: Samples from a rotated copy of the $n$-dimensional standard simplex (for a rotation that leaves $\ones$ invariant).
\State Output: An approximation to a uniformly random vertex of the input simplex.
\vspace{.2em}
\hrule
\vspace{.2em}
\State  Pick $u(1) \in S^{\dim-1}$, uniformly at random.
\myFor{$i=1$ to $r$}
\begin{align*}
u(i+1) := &C_{\dim-1} \nabla \hat{m}_3(u(i)) - \frac{1}{2} (u(i) \cdot \ones)^2 \ones \optionalbreak
- \frac{1}{2} (u(i) \cdot u(i))^2 \ones - (u(i) \cdot \ones) u(i).
\end{align*}
\State Normalize $u(i+1)$ by dividing by $\norm{u(i+1)}_2$.
\myEndFor
\State Output $u(r+1)$.
\end{algorithmic}
\end{subroutine}

\begin{lemma}\label{lem:gradient}
Let $c>0$ be a constant, $\dim > 20$, and $0<\delta<1$. Suppose that Subroutine 1 uses a sample of size
$t = 2^{17}\dim^{2c+22}(\frac{1}{\delta})^2\ln\frac{2\dim^5 r}{\delta}$ for each evaluation of the gradient and runs for
$r =\log \frac{4(c+3) n^2\ln{n}}{\delta}$
iterations. Then with probability at least $1-\delta$
Subroutine 1 outputs a vector within distance $1/n^c$ from a vertex
of the input simplex.  With respect of the process of picking a sample and running the algorithm, each vertex is equally likely to be the nearest.

\end{lemma}

Note that if we condition on the sample, different vertices are not equally likely over the randomness of the algorithm. That is, if we try to find all vertices running the algorithm multiple times on a fixed sample, different vertices will be found with different likelihoods.
\begin{proof}
Our analysis has the same outline as that of \noname{Nguyen and Regev~}\cite{NguyenR09}.
This is because the iteration that we get is the same as that of \cite{NguyenR09} except that cubing is replaced by
squaring (see below); however some details in our proof are different.
In the proof below, several of the inequalities are quite loose and are so chosen to make the computations
simpler.

We first prove the lemma assuming that the gradient computations are exact and then show how to handle samples.
We will carry out the analysis in the coordinate system where the given simplex is the standard simplex. This is
only for the purpose of the analysis, and this coordinate system is not known to the algorithm.
Clearly, $u(i+1) =  (u(i)^2_1, \ldots, u(i)_{\dim}^2)$.
It follows that, $$u(i+1) = (u(1)_1^{2^i}, \ldots, u(1)_{\dim}^{2^i}).$$
Now, since we choose $u(1)$ randomly, with probability at least $(1-(n^2-n)\delta')$
 one of the
coordinates of $u(1)$ is greater than all the other coordinates in absolute value by a factor of at least $(1+ \delta')$,
where $0 < \delta' < 1$.
(A similar argument is made in \cite{NguyenR09} with different parameters.
 We briefly indicate the proof for our case:
The probability that the event in question does not happen is less than the probability that there are two coordinates $u(1)_a$
and $u(1)_b$ such that their absolute values are within factor $1+\delta'$, i.e.
$1/(1+\delta') \leq |u(1)_a|/|u(1)_b| < 1+\delta'$. The probability that for given $a, b$ this event happens can be seen as the
Gaussian area of the four sectors (corresponding to the four choices of signs of $u(1)_a, u(1)_b$) in the plane each with angle
less than $2\delta'$. By symmetry, the Gaussian volume of these sectors is $2\delta'/(\pi/2) < 2\delta'$.
The probability that
such a pair $(a,b)$ exists is less than $2 \binom{n}{2} \delta'$.)
Assuming this happens, then after $r$ iterations, the ratio between the largest coordinate (in absolute value) and the absolute value of any other coordinate is at least
$(1+\delta')^{2^r}$.
Thus, one of the coordinates is very
close to $1$ and others are very close to $0$, and so $u(r+1)$ is very close to a vertex of the input simplex.

Now we drop the assumption that the gradient is known exactly. For each evaluation of the gradient we use
a fresh subset of samples of $t$ points. Here $t$ is chosen so that each evaluation of the gradient is within $\ell_2$-distance
$1/\dim^{c_1}$ from its true value with probability at least $1-\delta''$, where $c_1$ will be set at the end of the proof.
An application of the Chernoff bound yields that we can take
$t = 200\dim^{2c_1+4}\ln\frac{2\dim^3}{\delta''}$; we omit the details. Thus all the $r$ evaluations of the gradient are
within distance $1/\dim^{c_1}$ from their true values with probability at least $1-r\delta''$.

We assumed that our starting vector $u(1)$ has a coordinate greater than every other coordinate by a factor
of
$(1+\delta')$ in absolute value; let us assume without loss of generality that this is the first coordinate.
Hence $|u(1)_1| \geq 1/\sqrt{n}$.
When expressing $u^{(2)}$ in terms of the gradient, the gradient gets multiplied by $C_{n-1} < n^3$ (we are assuming $n>20$),
keeping this in mind and letting $c_2 = c_1-3$ we get for $j \neq 1$
\begin{align*}
\frac{|u(i+1)_1|}{|u(i+1)_j|} \geq \frac{u(i)_1^2-1/n^{c_2}}{u(i)_j^2+1/n^{c_2}} \geq \frac{u(i)_1^2 (1-n^{-(c_2-1)})}{u(i)_j^2+1/n^{c_2}}.
\end{align*}

If $u(i)_j^2 > 1/\dim^{c_2-c_3}$, where $1\leq c_3 \leq c_2 -2$ will be determined later, then we get
\begin{align} \label{eq:if}
|u(i+1)_1|/|u(i+1)_j|
&> \frac{1-1/n^{c_2-1}}{1+1/n^{c_3}} \cdot \left(\frac{u(i)_1}{u(i)_j}\right)^2 \nonumber\\
&>  (1-1/n^{c_3})^2 \left(\frac{u(i)_1}{u(i)_j}\right)^2.
\end{align}

Else,
\begin{align*}
|u(i+1)_1|/|u(i+1)_j|
&> \frac{1/n-1/n^{c_2}}{1/n^{c_2-c_3}+1/n^{c_2}} \\
&> \left(1-\frac{1}{n^{c_3}}\right)^2 \cdot n^{c_2-c_3 -1} \\
&> \frac{1}{2} n^{c_2-c_3-1},
\end{align*}
where we used $c_3 \geq 1$ and $n>20$ in the last inequality.


We choose $c_3$ so that
\begin{align}\label{eq:c3}
\left(1-\frac{1}{n^{c_3}}\right)^2(1+\delta') > (1+\delta'/2).
\end{align}
For this, $\delta' \geq 32/n^{c_3}$ or equivalently $c_3 \geq (\ln{(32/\delta')})/\ln{n}$ suffices.

For $c_3$ satisfying \eqref{eq:c3} we have $(1-\frac{1}{n^{c_3}})^2(1+\delta')^2 > (1+\delta')$.
It then follows from \eqref{eq:if} that
the first coordinate continues to remain the largest in absolute value by a factor of at least $(1+\delta')$ after each iteration.
Also, once we have $|u(i)_1|/|u(i)_j| > \frac{1}{2} n^{c_2-c_3-1}$, we
have  $|u(i')_1|/|u(i')_j| > \frac{1}{2} n^{c_2-c_3-1}$ for all $i'>i$.

\eqref{eq:if} gives that after $r$ iterations we have
\begin{align*}
\frac{|u(r+1)_1|}{|u(r+1)_j|}
&>  (1-1/n^{c_3})^{2+2^2+\ldots+2^r} \left(\frac{u(1)_1}{u(1)_j}\right)^{2^r} \\
&\geq (1-1/n^{c_3})^{2^{r+1}-2} (1+\delta')^{2^r}.
\end{align*}

Now if $r$ is such that $(1-1/n^{c_3})^{2^{r+1}-2} (1+\delta')^{2^r} > \frac{1}{2} n^{c_2-c_3-1}$, we will be guaranteed that
$|u(r+1)_1|/|u(r+1)_j| > \frac{1}{2} n^{c_2-c_3-1}$. This condition is satisfied if we have
$(1-1/n^{c_3})^{2^{r+1}} (1+\delta')^{2^r} > \frac{1}{2} n^{c_2-c_3-1}$, or equivalently
$((1-1/n^{c_3})^2 (1+\delta'))^{2^r} \geq \frac{1}{2} n^{c_2-c_3-1}$.
Now using \eqref{eq:c3} it suffices to choose $r$ so that
$(1+\delta'/2)^{2^r} \geq \frac{1}{2} n^{c_2-c_3-1}$. Thus we can take $r = \log(4(c_2-c_3)(\ln{n})/\delta')$.


Hence we get $|u(r+1)_1|/|u(r+1)_j| > \frac{1}{2} n^{c_2-c_3-1}$.
It follows that for
$u(r+1)$, the $\ell_2$-distance from the vertex $(1, 0, \ldots, 0)$ is at most $8/n^{c_2-c_3-2} < 1/n^{c_2-c_3-3}$ for $\dim > 20$; we omit
easy details.

Now we set our parameters: $c_3 = 1+(\ln(32/\delta')/\ln{n})$ and $c_2 - c_3 - 3 = c$ and
$c_1 = c_2 + 3 = 7 + c + \ln(32/\delta')/\ln{n}$ satisfies all the constraints we imposed on $c_1, c_2, c_3$.
Choosing $\delta'' = \delta'/r$, we get that the procedure succeeds with probability at least
$1-(\dim^2-\dim)\delta' - r\delta'' > 1-\dim^2\delta'$. Now setting $\delta'=\delta/n^2$ gives the overall probability of error $\delta$,
and the number of samples and iterations as claimed in the lemma.
\end{proof}

\section{Learning simplices} \label{sec:general}

In this section we give our algorithm for learning general simplices, which uses the subroutine from the previous section.
The learning algorithm uses an affine map $T:\RR^{\dim} \to \RR^{\dim+1}$ that maps some isotropic simplex to the standard simplex.
We describe now a way of constructing such a map: Let $A$ be a matrix having as columns an orthonormal basis of $\ones^\perp$ in $\RR^{\dim+1}$.
To compute one such $A$, one can start with the $(n+1)$-by-$(n+1)$ matrix $B$ that has ones in the diagonal and first column, everything else is zero.
Let $QR=B$ be a QR-decomposition of $B$. By definition we have that the first column of $Q$ is parallel to $\ones$ and the rest of the columns span $\ones^\perp$.
Given this, let $A$ be the matrix formed by all columns of $Q$ except the first.
We have that the set $\{A^T e_i\}$ is the set of vertices of a regular $n$-simplex.
Each vertex is at distance
\[
\sqrt{\left(1-\frac{1}{n+1}\right)^2 + \frac{n}{(n+1)^2}} = \sqrt{\frac{n}{n+1}}
\]
from the origin, while an isotropic simplex has vertices at distance $\sqrt{n(n+2)}$ from the origin. So an affine transformation that maps an isotropic simplex in $\RR^\dim$ to the standard simplex in $\RR^{\dim+1}$ is $T(x) = \frac{1}{\sqrt{(n+1)(n+2)}} A x + \frac{1}{n+1} \ones_{\dim+1}$.


\begin{algorithm}
\caption{Learning a simplex.}\label{alg:simplex}
\begin{algorithmic}
\State Input: Error parameter $\eps>0$. Probability of failure parameter $\delta>0$. Oracle access to random points from some $n$-dimensional simplex $S_{INPUT}$.

\State Output: $V =\{v(1), \dotsc, v(\dim+1)\} \subseteq \RR^{\dim}$ (approximations to the vertices of the simplex).
\vspace{.2em}
\hrule
\vspace{.2em}
\State Estimate the mean and covariance using $t_1=\poly(n,1/\eps, 1/\delta)$ samples $p(1), \dotsc, p(t_1)$:
\[
\mu = \frac{1}{t_1} \sum_i p(i),
\]
\[
\Sigma = \frac{1}{t_1} \sum_i (p(i)-\mu) (p(i)-\mu)^T.
\]
\State Compute a matrix $B$ so that $\Sigma = BB^T$ (say, Cholesky decomposition).

\State Let $U = \emptyset$.
\myFor{$i=1$ to $m$ (with $m = \poly (n, \log 1/\delta))$}

\State Get $t_3=\poly(n, 1/\eps, \log 1/\delta)$ samples $r(1),\dotsc r(t_3)$ and use $\mu, B$ to map them to samples $s(i)$ from a nearly-isotropic simplex: $s(i) = B^{-1}(r(i)-\mu)$.

\State Embed the resulting samples in $\RR^{\dim+1}$ as a sample from an approximately rotated standard simplex:
Let $l(i) = T(s(i))$.

\State Invoke Subroutine \ref{alg:vertex} with sample $l(1), \dotsc, l(t_3)$ to get $u \in \RR^{\dim+1}$.
\State Let $\tilde u$ be the nearest point to $u$ in the affine hyperplane $\{x \suchthat x \cdot \ones = 1\}$. If $\tilde u$ is not within $1/\sqrt{2}$ of a point in $U$, add $\tilde u$ to $U$. (Here $1/\sqrt{2}$ is half of the edge length of the standard simplex.)
\myEndFor
\State Let
\begin{align*}
V &= B T^{-1}(U) + \mu \optionalbreak
= \sqrt{(n+1)(n+2)} B A^T\left(U-\frac{1}{n+1}\ones \right) + \mu.
\end{align*}
\end{algorithmic}
\end{algorithm}

To simplify the analysis, we pick a new sample $r(1), \dotsc, r(t_3)$ to find every vertex, as this makes every vertex equally likely to be found when given a sample from an isotropic simplex. (The core of the analysis is done for an isotropic simplex; this is enough as the algorithm's first step is to find an affine transformation that puts the input simplex in approximately isotropic position. The fact that this approximation is close in total variation distance implies that it is enough to analyze the algorithm for the case of exact isotropic position, the analysis carries over to the approximate case with a small loss in the probability of success. See the proof below for the details.) A practical implementation may prefer to select one such sample outside of the for loop, and find all the vertices with just that sample---an analysis of this version would involve bounding the probability that each vertex is found (given the sample, over the choice of the starting point of gradient descent) and a variation of the coupon collector's problem with coupons that are not equally likely.

\begin{proof}[Theorem~\ref{thm:main}]
As a function of the input simplex, the distribution of the output of the algorithm is equivariant under invertible affine transformations.
Namely, if we apply an affine transformation to the input simplex, the distribution of the output is equally transformed.\footnote{To see this: the equivariance of the algorithm as a map between distributions is implied by the equivariance of the algorithm on any given input sample. Now, given the input sample, if we apply an affine transformation to it, this transformation is undone except possibly for a rotation by the step $s(i) = B^{-1}(r(i)-\mu)$. A rotation may remain because of the ambiguity in the characterization of $B$. But the steps of the algorithm that follow the definition of $s(i)$ are equivariant under rotation, and the ambiguous rotation will be removed at the end when $B$ is applied again in the last step.}
The notion of error, total variation distance, is also invariant under invertible affine transformations.
Therefore, it is enough to analyze the algorithm when the input simplex is in isotropic position.
In this case $\norm{p(i)} \leq n+1$ (see Section \ref{sec:preliminaries}) and we can set $t_1 \leq \poly(n,1/\eps', \log(1/\delta))$ so that $\norm{\mu} \leq \eps'$ with probability at least $1-\delta/10$ (by an easy application of Chernoff's bound), for some $\eps'$ to be fixed later.
Similarly, using results from 
\cite[Theorem 4.1]{MR2601042}, a choice of $t_1 \leq n {\eps'}^{-2} \polylog(1/\eps') \polylog(1/\delta)$ implies that the empirical second moment matrix \[\bar \Sigma = \frac{1}{t_1} \sum_i p(i) p(i)^T\] satisfies $\norm{\bar \Sigma - I} \leq \eps'$ with probability at least $1-\delta/10$. We have $\Sigma = \bar \Sigma - \mu \mu^T$ and this implies $\norm{\Sigma - I} \leq \norm{\bar \Sigma - I} + \norm{\mu \mu^T} \leq 2\eps'$.
Now, $s(1), \dotsc, s(t_3)$ is an iid sample from a simplex $S'=B^{-1}(S_{INPUT}-\mu)$. Simplex $S'$ is close in total variation distance to some isotropic simplex\footnote{The isotropic simplex $S_{ISO}$ will typically be far from the (isotropic) input simplex, because of the ambiguity up to orthogonal transformations in the characterization of $B$.} $S_{ISO}$. More precisely, Lemma~\ref{lem:tv} below shows that
\begin{equation}\label{equ:tv}
d_{TV}(S', S_{ISO}) \leq 12 \dim \eps',
\end{equation}
with probability at least $1-\delta/5$.

Assume for a moment that $s(1), \dotsc, s(t_3)$ are from $S_{ISO}$. The analysis of Subroutine \ref{alg:vertex} (fixed point-like iteration) given in Lemma~\ref{lem:gradient} would guarantee the following: Successive invocations to Subroutine~\ref{alg:vertex} find approximations to vertices of $T(S_{ISO})$ within Euclidean distance $\eps''$ for some $\eps''$ to be determined later and $t_3 = \poly(\dim,1/\eps'', \log 1/\delta)$. We ask for each invocation to succeed with probability at least $1-\delta/(20m)$ with $m = n (\log n + \log 20/\delta)$. Note that each vertex is equally likely to be found. The choice of $m$ is so that, if all $m$ invocations succeed (which happens with probability at least $1-\delta/20$), then the analysis of the coupon collector's problem, Lemma~\ref{lem:coupons}, implies that we fail to find a vertex with probability at most $\delta/20$. Overall, we find all vertices with probability at least $1-\delta/10$.

But in reality samples $s(1), \dotsc, s(t_3)$ are from $S'$, which is only \emph{close} to $S_{ISO}$. The estimate from \eqref{equ:tv} with appropriate $\eps' = \poly(1/n, \eps'', \delta)$ gives
\[
d_{TV}(S', S_{ISO}) \leq \frac{\delta}{10} \frac{1}{ t_3 m},
\]
which implies that the total variation distance between the joint distribution of all $t_3 m$ samples used in the loop and the joint distribution of actual samples from the isotropic simplex $S_{ISO}$ is at most $\delta/10$, and this implies that the loop finds approximations to all vertices of $T(S_{ISO})$ when given samples from $S'$ with probability at least $1-\delta/5$. The points in $U$ are still within Euclidean distance $
\eps''$ of corresponding vertices of $T(S_{ISO})$.

To conclude, we turn our estimate of distances between estimated and true vertices into a total variation estimate, and map it back to the input simplex. Let $S''=\conv T^{-1} U$. As $T$ maps an isotropic simplex to a standard simplex, we have that $\sqrt{(n+1)(n+2)} T$ is an isometry, and therefore the vertices of $S''$ are within distance $\eps''/\sqrt{(n+1)(n+2)}$ of the corresponding vertices of $S_{ISO}$. Thus, the corresponding support functions are uniformly within \[\eps''' = \eps''/\sqrt{(n+1)(n+2)}\] of each other on the unit sphere. This and the fact that $S_{ISO} \supseteq B_n$ imply
\[
(1 - \eps''') S_{ISO}\subseteq S'' \subseteq (1 + \eps''')S_{ISO}.
\]
Thus, by Lemma \ref{lem:bmtv}, $d_{TV}(S'', S_{ISO}) \leq
1 - (\frac{1-\eps'''}{1+\eps'''})^\dim \leq 1-(1-\eps''')^{2n}\leq 2n \eps''' \leq 2\eps''$
and this implies that the total variation distance between the uniform distributions on $\conv V$ and the input simplex is at most $2 \eps''$. Over all random choices, this happens with probability at least $1-2\delta/5$. We set $\eps'' = \eps/2$.
\end{proof}

\begin{lemma}\label{lem:tv}
Let $S_{INPUT}$ be an $\dim$-dimensional isotropic simplex. Let $\Sigma$ be an $\dim$-by-$\dim$ positive definite matrix such that $\norm{\Sigma - I} \leq \eps < 1/2$. Let $\mu$ be an $n$-dimensional vector such that $\norm{\mu} \leq \eps$. Let $B$ be an $n$-by-$n$ matrix such that $\Sigma = BB^T$. Let $S$ be the simplex $B^{-1}(S_{INPUT}-\mu)$. Then there exists an isotropic simplex $S_{ISO}$ such that $d_{TV}(S, S_{ISO}) \leq 6 \dim \eps$.
\end{lemma}
\begin{proof}
We use an argument along the lines of the orthogonal Procrustes problem (nearest orthogonal matrix to $B^{-1}$, already in \cite[Proof of Theorem 4]{NguyenR09}): Let $U D V^T$ be the singular value decomposition of $B^{-1}$. Let $R = U V^T$ be an orthogonal matrix (that approximates $B^{-1}$). Let $S_{ISO} = R S_{INPUT}$.

We have $S = UDV^T (S_{INPUT} - \mu)$. Let $\sigma_{min}$, $\sigma_{max}$ be the minimum and maximum singular values of $D$, respectively. This implies:
\begin{align}
\sigma_{min} UV^T (S_{INPUT} - \mu) &\subseteq S \subseteq \sigma_{max}UV^T (S_{INPUT} - \mu), \nonumber \\
\sigma_{min} (S_{ISO} - R\mu)  &\subseteq S \subseteq \sigma_{max}( S_{ISO} - R\mu).\label{equ:isotropy}
\end{align}
As $S_{ISO} \supseteq B_n$, $\norm{\mu} \leq 1$, $R$ is orthogonal and $S_{ISO}$ is convex, we have
\[
S_{ISO}-R\mu \supseteq (1-\norm{\mu}) S_{ISO}.
\]
Also,
\begin{align*}
S_{ISO} - R\mu &\subseteq S_{ISO} + \norm{\mu} B_n \\
&\subseteq S_{ISO} (1+ \norm{\mu}).
\end{align*}
This in \eqref{equ:isotropy} gives
\[
\sigma_{min} (1-\norm{\mu}) S_{ISO} \subseteq S \subseteq \sigma_{max}(1+\norm{\mu}) S_{ISO}.
\]
This and Lemma \ref{lem:bmtv} imply
\[
d_{TV}(S, S_{ISO}) \leq 2\left(1- \left(\frac{\sigma_{min} (1-\norm{\mu}) }{ \sigma_{max}(1+\norm{\mu})}\right)^\dim\right).
\]
The estimate on $\Sigma$ gives $\sigma_{min} \geq \sqrt{1-\eps}$, $\sigma_{max} \leq \sqrt{1+\eps}$.  Thus
\begin{align*}
d_{TV}(S, S_{ISO}) &\leq 2\left(1- \left(\frac{1-\eps}{1+\eps} \right)^{3\dim/2}\right) \\
& \leq 2\left(1- \left(1-\eps\right)^{3\dim}\right) \\
& \leq 6\dim \eps.
\end{align*}
\end{proof}

\section{The local and global maxima of the 3rd moment of the standard simplex and the isotropic simplex} \label{sec:maxima}


In this section we study the structure of the set of local maxima of the third moment as a function of the direction (which happens to be essentially $u \mapsto \sum u_i^3$ as discussed in Section \ref{sec:moment}). This is not necessary for our algorithmic result, however it gives insight into the geometry of the third moment (the location of local maxima/minima and stationary points) and suggests that more direct optimization algorithms like gradient descent and Newton's method will also work, although we will not prove that.

\begin{theorem}\label{thm:maxima}
Let $K \subseteq \RR^\dim$ be an isotropic simplex. Let $X$ be random in $K$. Let $V = \{x_i\}_{i=1}^{\dim+1} \subseteq \RR^\dim$ be the set of normalized vertices of $K$. Then $V$ is a complete set of local maxima and a complete set of global maxima of $F:S^{\dim-1} \to \RR$ given by $F(u) = \e ((u \cdot X)^3)$.
\end{theorem}
\noindent\emph{Proof idea:} Embed the simplex in $\RR^{\dim+1}$. Show that the third moment is proportional to the complete homogeneous symmetric polynomial of degree 3, which for the relevant directions is proportional to the sum of cubes. To conclude, use first and second order optimality conditions to characterize the set of local maxima.
\begin{proof}
Consider the standard simplex
\[
\Delta^n = \conv \{e_1, \dotsc, e_{\dim+1} \} \subseteq \RR^{\dim+1}
 \]
and identify it with $V$ via a linear map $A :\RR^{\dim+1} \to \RR^{\dim}$ so that $A(\Delta^n) = V$. Let $Y$ be random in $\Delta^n$. Consider $G: S^{\dim} \to \RR$ given by $G(v) = m_3(v) = \e ((v \cdot Y)^3)$. Let $U = \{ v \in \RR^{\dim+1} \suchthat v \cdot \ones =0, \norm{v}=1 \}$ be the equivalent feasible set for the embedded problem.
We have $G(v) = cF(A v)$ for any $v \in U$ and some constant $c > 0$ independent of $v$.
To get the theorem, it is enough to show that the local maxima of $G$ in $U$ are precisely the normalized versions of the projections of the canonical vectors onto the hyperplane orthogonal to $\ones = (1, \dotsc, 1)$. According to Section \ref{sec:moment}, for $v \in U$ we have
\[
G(v) \propto p_3(v).
\]
Using a more convenient but equivalent constant, we want to enumerate the local maxima of the problem
\begin{equation}\label{equ:opt}
\begin{aligned}
\max \frac{1}{3} p_3(v)& \\
\text{s.t.} \quad v \cdot v &= 1 \\
v \cdot \ones &= 0 \\
v &\in \RR^{\dim+1}.
\end{aligned}
\end{equation}
The Lagrangian function is
\[
L(v, \lambda_1, \lambda_2) = \frac{1}{3} \sum_i v_i^3 - \lambda_1 \sum_i v_i - \lambda_2\frac{1}{2} \left(\biggl(\sum_i v_i^2 \biggr) -1 \right).
\]
The first order condition is $\nabla_v L = 0$, that is,
\begin{equation}\label{equ:foc}
v_i^2 = \lambda_1 + \lambda_2 v_i \quad \text{for $i= 1, \dotsc, \dim+1$.}
\end{equation}
Consider this system of equations on $v$ for any fixed $\lambda_1$, $\lambda_2$. Let $f(x) = x^2$, $g(x)= \lambda_1 + \lambda_2 x$. The first order condition says $f(v_i) = g(v_i)$, where $f$ is convex and $g$ is affine. That is, the $v_i$s can take at most two different values. As our optimization problem \eqref{equ:opt} is symmetric under permutation of the coordinates, we conclude that, after putting the coordinates of a point $v$ in non-increasing order, if $v$ is a local maximum of \eqref{equ:opt}, then $v$ must be of the form
\[
v = (a, \dotsc, a, b, \dotsc, b),
\]
where $a>0>b$ and there are exactly $\alpha$ $a$s and $\beta$ $b$s, for $\alpha, \beta \in \{1, \dotsc, \dim\}$.

We will now study the second order necessary condition (SONC) to eliminate from the list of candidates all vectors with $\alpha >1$. It is easy to see that the surviving vectors are exactly the promised scaled projections of the canonical vectors. This vectors must all be local and global maxima: At least one of them must be a global maximum as we are maximizing a continuous function over a compact set and all of them have the same objective value so all of them are local and global maxima.

The SONC at $v$ asks for the Hessian of the Lagrangian to be negative semidefinite when restricted to the tangent space to the constraint set at $v$ \cite[Section 11.5]{MR2423726}.
We compute the Hessian (recall that $v^{(2)}$ is the vector of the squared coordinates of $v$):
\[
\nabla_v L = v^{(2)} - \lambda_1 \ones - \lambda_2 v
\]
\[
\nabla^2_v L = 2 \diag(v) - \lambda_2 I
\]
where $\diag(v)$ is the $(\dim + 1)$-by-$(\dim+1)$ matrix having the entries of $v$ in the diagonal and 0 elsewhere.

A vector in the tangent space is any $z \in \RR^{\dim+1}$ such that $z \cdot \ones = 0$, $v \cdot z = 0$, and definiteness of the Hessian is determined by the sign of $z^T \nabla^2_v L z$ for any such $z$, where
\[
z^T \nabla^2_v L z = \sum_{i=1}^{\dim+1} z_i^2 (2 v_i - \lambda_2).
\]
Suppose $v$ is a critical point with $\alpha \geq 2$. To see that such a $v$ cannot be a local maximum, it is enough to show $2a > \lambda_2$, as in that case we can take $z = (1, -1, 0,\dotsc, 0)$ to make the second derivative of $L$ positive in the direction $z$.

In terms of $\alpha, \beta, a, b$, the constraints of \eqref{equ:opt} are $\alpha a + \beta b = 0$, $\alpha a^2 + \beta b^2 = 1$, and this implies $a = \sqrt{\frac{\beta}{\alpha (\dim+1) }}$, $b = - \sqrt{\frac{\alpha}{\beta (\dim+1) }}$. The inner product between the first order condition \eqref{equ:foc} and $v$ implies $\lambda_2 = \sum v_i^3 = \alpha a^3 + \beta b^3$. It is convenient to consider the change of variable $\gamma = \alpha/(\dim+1)$, as now candidate critical points are parameterized by certain discrete values of $\gamma$ in $(0,1)$. This gives $\beta = (1-\gamma)(\dim+1)$, $ a = \sqrt{(1-\gamma)/(\gamma(\dim+1))}$ and
\begin{align*}
\lambda_2 &= (\dim+1)\biggl[\gamma \left(\frac{1-\gamma}{\gamma (\dim+1)}\right)^{3/2} \\
&\qquad - (1-\gamma)\left(\frac{\gamma}{(1-\gamma) (\dim+1)}\right)^{3/2}\biggr] \\
&= \frac{1}{\sqrt{(\dim+1) \gamma (1-\gamma)}} \left[(1-\gamma)^2 - \gamma^2\right] \\
&= \frac{1}{\sqrt{(\dim+1) \gamma (1-\gamma)}} [1- 2\gamma].
\end{align*}
This implies
\begin{align*}
2 a - \lambda_2 &= \frac{1}{\sqrt{(\dim+1) \gamma (1-\gamma)}} [2(1-\gamma) -1 + 2\gamma] \\
    &= \frac{1}{\sqrt{(\dim+1) \gamma (1-\gamma)}}.
\end{align*}
In $(0,1)$, the function given by $\gamma \mapsto 2a-\lambda_2 = \frac{1}{\sqrt{(\dim+1) \gamma (1-\gamma)}}$ is convex and symmetric around $1/2$, where it attains its global minimum value, $2/ \sqrt{\dim+1}$, which is positive.
\end{proof}

\section{Probabilistic results used}\label{sec:probabilistic-results}

In this section we show the probabilistic results underlying the reductions from learning simplices and $\ell_p^n$ balls to ICA. The results are Theorems \ref{thm:simplexscaling} and \ref{thm:lpscaling}. They each show a simple non-linear rescaling of the respective uniform distributions that gives a distribution with independent components (Definition \ref{def:ic}).

Theorem \ref{thm:simplexscaling} below gives us, in a sense, a ``reversal'' of the representation of the cone measure on $\partial B_p^n$, seen in Theorem \ref{thm:cone}. Given any random point in the standard simplex, we can apply a simple non-linear scaling and recover a distribution with independent components.  

\begin{definition}\label{def:ic}
We say that a random vector $X$ has \emph{independent components} if it is an affine transformation of a random vector having independent coordinates.
\end{definition}

\begin{theorem}\label{thm:simplexscaling} 
Let $X$ be a uniformly random vector in the $(\dimension-1)$-dimensional standard simplex $\Delta_{n-1}$. Let $T$ be a random scalar distributed as $\mathrm{Gamma}(n, 1)$. Then the coordinates of $T X$ are iid as $\expdist{1}$.

Moreover, if $A:\mathbb{R}^{\dimension} \rightarrow \mathbb{R}^{\dimension}$ is an invertible linear transformation, then the random vector $TA(X)$ has independent components.
\end{theorem}
\begin{proof}
In the case where $p = 1$, Theorem \ref{thm:cone} restricted to the positive orthant implies that for random vector $G = (G_1, \dotsc, G_{\dimension})$, if each $G_i$ is an iid exponential random variable $\expdist{1}$, then $( G/\norm{G}_1, \norm{G}_1)$ has the same (joint) distribution as $(X,T)$. Given the measurable function $f(x,t) = xt$, $f(X,T)$ has the same distribution as $f( G/\norm{G}_1, \norm{G}_1)$. That is, $XT$ and $G$ have the same distribution\footnote{See \cite[Theorem 1.1]{MR1456629} for a similar argument in this context.}.

For the second part, we know $T A(X) = A(TX)$ by linearity. By the previous argument the coordinates of $TX$ are independent. This implies that $A(TX)$ has independent components.
\end{proof}

The next lemma complements the main result in \cite{barthe2005probabilistic}, Theorem 1 (Theorem \ref{theorem:fullnaor} elsewhere here). They show a representation of the uniform distribution in $B_p^n$, but they do not state the independence that we need for our reduction to ICA.
\begin{lemma} \label{lemma:independence} 
Let $p \in [1, \infty)$. Let $G=(G_1, \dotsc, G_n)$ be iid random variables with density proportional to $\exp(-\abs{t}^p)$. Let $W$ be a nonnegative random variable with distribution $\expdist{1}$ and independent of $G$. Then the random vector
\[
\frac{G}{(\norm{G}_p^p + W)^{1/p}}
\]
is independent of $(\norm{G}_p^p + W)^{1/p}$.
\end{lemma}

\begin{proofidea}
We aim to compute the joint density, showing that it is a product of individual densities. To avoid complication, we raise everything to the $p$th power, which eliminates extensive use of the chain rule involved in the change of variables. 
\end{proofidea}

\begin{proof}
It is enough to show the claim conditioning on the orthant in which $G$ falls, and by symmetry it is enough to prove it for the positive orthant. Let random variable $H = (G_1^p, G_2^p, \dotsc, G_n^p)$.
Since raising (strictly) positive numbers to the $p$th power is injective, it suffices to show that the random vector 
\[
X = \frac{H}{\sum_{i=1}^n{H}_i + W}
\]
is independent of the random vector $Y = \sum_{i=1}^n{H}_i + W$. 
 
First, let $U$ be the interior of the support of $(X,Y)$, that is $U = \{ x \in \RR^{n} : x_i > 0, \sum_i x_i < 1 \}\times \{y \in \RR :y>0\}$  and consider $h: U \rightarrow \RR^{\dimension}$ and $w: U \rightarrow \RR$ where
\[
h(x,y) = x y
\] and 
\[
w(x,y) = y - \sum_{i=1}^n h(x,y)_i = y - \sum_{i=1}^n x_i \cdot y = y\left(1 - \sum_{i=1}^n x_i\right).
\]
The random vector $(H,W)$ has a density $\density_{H,W}$ supported on $V  = \operatorname{int} \RR^{n+1}_+$ and 
$$(x,y) \mapsto (h(x,y), w(x,y))$$
is one-to-one from $U$ onto $V$. 
Let $J(x,y)$ be the determinant of its Jacobian. This Jacobian is
\[
\begin{pmatrix}
y & 0 & \cdots & 0 & x_1 \\
0 & y & \cdots & 0 & x_2 \\
\vdots & & & & \vdots \\
0 & 0 & \cdots & y & x_n \\
-y & -y & \cdots & -y & 1 - \sum_{i=1}^n x_i \\
\end{pmatrix}
\]
which, by adding each of the first $n$ rows to the last row, reduces to
\[
\begin{pmatrix}
y & 0 & \cdots & 0 & x_1 \\
0 & y & \cdots & 0 & x_2 \\
\vdots & & & & \vdots \\
0 & 0 & \cdots & y & x_n \\
0 & 0 & \cdots & 0 & 1 \\
\end{pmatrix},
\]
the determinant of which is trivially $J(x,y) = y^n$.

We have that $J(x,y)$ is nonzero in $U$. Thus, $(X,Y)$ has density $\density_{X,Y}$ supported on $U$ given by
\begin{align*}
\density_{X,Y}(x,y) &= f_{H,W}\bigl(h(x,y),w(x,y)\bigr) \cdot \abs{J(x)}.
\end{align*}

It is easy to see\footnote{See for example \cite[proof of Theorem 3]{barthe2005probabilistic}.} that each $H_i = G_i^p$ has density $\mathrm{Gamma}(1/p, 1)$ and thus $\sum_{i=1}^n H_i$ has density $\mathrm{Gamma}(n/p,1)$ by the additivity of the Gamma distribution.  We then compute the joint density
\begin{align*}
\density_{X,Y}(x,y) &= \density_{H,W}\bigl(h(x,y),w(x,y)\bigr) \cdot y^n \\
&= \density_{H,W}\Bigl(xy, y(1 - \sum\limits_{i=1}^n x_i)\Bigr) \cdot y^n.
\end{align*}
Since $W$ is independent of $H$,
\begin{align*}
\density_{X,Y}(x,y) &=  \density_{W}\bigg(y\Big(1 - \sum\limits_{i=1}^n x_i\Big)\bigg) \cdot y^n \prod_{i=1}^n \density_{H_i}(x_iy) 
\end{align*}
where
\begin{align*}
\prod_{i=1}^{\dimension}\Big(\density_{H_i}(x_iy)\Big) \cdot \density_{W}\bigg(y\Big(1 - \sum\limits_{i=1}^n x_i\Big)\bigg) \cdot y^n 
&\propto \prod_{i=1}^n \left[e^{-x_iy}(x_iy)^{\frac{1}{p} - 1}\right] \exp\left(-y(1-\sum\limits_{i=1}^n x_i)\right) y^n \\
&\propto \biggl(\prod_{i=1}^n x_i^{\frac{1}{p} - 1}\biggr) y^{n/p}.
\end{align*}
The result follows.
\end{proof}

With this in mind, we show now our analog of Theorem \ref{thm:simplexscaling} for $B_p^{\dimension}$.

\begin{theorem}\label{thm:lpscaling}
Let $X$ be a uniformly random vector in $B_p^{\dimension}$. Let $T$ be a random scalar distributed as $\mathrm{Gamma}((n/p)+1,1)$. Then the coordinates of $T^{1/p} X$ are iid, each with density proportional to $\exp(-\abs{t}^p)$.
Moreover, if $A: \RR^{\dimension} \rightarrow \RR^{\dimension}$ is an invertible linear transformation, then the random vector given by $T^{1/p}A(X)$ has independent components.
\end{theorem}

\begin{proof}
Let $G = (G_1, \dotsc, G_{\dimension})$ where each $G_i$ is  iid as $\mathrm{Gamma}(1/p, 1)$. Also, let $W$ be an independent random variable distributed as $\expdist{1}$. Let $S = \bigl(\sum_{i=1}^n |G_i|^p + W \bigr)^{1/p}$. 

By Lemma \ref{lemma:independence} and Theorem \ref{theorem:fullnaor} we know $(G/S,S)$ has the same joint distribution as $(X,T^{1/p})$. Then for the measurable function $f(x,t) = xt$, we immediately have $f(X,T^{1/p})$ has the same distribution as $f(G/S, S)$ and thus $XT^{1/p}$ has the same distribution as $G$. 

For the second part, since $T$ is a scalar, we have $T^{1/p}A(X) = A(T^{1/p}X)$. By the previous argument we have that the coordinates of $T^{1/p}X$ are independent. Thus, $A(T^{1/p}X)$ has independent components. 
\end{proof}

%

This result shows that one can obtain a vector with independent components from a sample in a linearly transformed $\ell_p$ ball.
In Section \ref{sec:simplex-reduction} we show that they are related in such as way that one can recover the linear transformation from the independent components via ICA.


\section{Reducing simplex learning to ICA}\label{sec:simplex-reduction}

We now show randomized reductions from the following two natural statistical estimation problems to ICA:

\begin{problem}[simplex] 
Given uniformly random points from an $n$-dimensional simplex, estimate the simplex.
\end{problem} 

This is the same problem of learning a simplex as in the rest of the section, we just restate it here for clarity.

To simplify the presentation for the second problem, we ignore the estimation of the mean of an affinely transformed distribution. That is, we assume that the $\ell_p^n$ ball to be learned has only been \emph{linearly} transformed.

\begin{problem}[linearly transformed $\ell_p^n$ balls]
Given uniformly random points from a linear transformation of the  $\ell_p^n$-ball, estimate the linear transformation.
\end{problem}

These problems do not have an obvious independence structure. Nevertheless, known representations of the uniform measure in an $\ell_p^n$ ball and the \emph{cone measure} (defined in Section \ref{sec:simplex-preliminaries}) on the surface of an $\ell_p^n$ ball can be slightly extended to map a sample from those distributions into a sample with independent components by a non-linear scaling step.
The use of a non-linear scaling step to turn a distribution into one having independent components has been done before \cite{MR2756189, DBLP:conf/nips/SinzB08}, but there it is applied \emph{after} finding a transformation that makes the distribution axis-aligned. This alignment is attempted using ICA (or variations of PCA) on the original distribution \cite{MR2756189, DBLP:conf/nips/SinzB08}, without independent components, and therefore the use of ICA is somewhat heuristic. One of the contributions of our reduction is that the rescaling we apply is ``blind'', namely, it can be applied to the original distribution. In fact, the distribution does not even need to be isotropic (``whitened''). The distribution resulting from the reduction has independent components and therefore the use of ICA on it is well justified.

The reductions are given in Algorithms \ref{alg:simplexreduction} and \ref{alg:lpreduction}. To state the reductions, we denote by $ICA(s(1), s(2), \ldots)$ the invocation of an ICA routine. 
It takes samples $s(1), s(2), \ldots$ of a random vector $Y=AX + \mu$, where the coordinates of $X$ are independent, and returns an approximation to a square matrix $M$ such that $M(Y-\e(Y))$ is isotropic and has independent coordinates. 
The theory of ICA \cite[Theorem 11]{comon1994independent} implies that if $X$ is isotropic and at most one coordinate is distributed as a Gaussian, then such an $M$ exists and it satisfies $M A = DP$, where $P$ is a permutation matrix and $D$ is a diagonal matrix with diagonal entries in $\{-1, 1\}$. We thus need the following definition to state our reduction: Let $c_{p,n} = (\e_{X \in B_p^n}(X_1^2))^{1/2}$. That is, the uniform distribution in $B_p^n/c_{p,n}$ is isotropic.

As we do not state a full analysis of any particular ICA routine, we do not state explicit approximation guarantees.

\begin{algorithm}
\caption{Reduction from Problem 1 to ICA}\label{alg:simplexreduction}

\begin{algorithmic}
\State Input: A uniformly random sample $p(1), \dotsc, p(t)$ from an $n$-dimensional simplex $S$.
\State Output: Vectors $\tilde v(1), \dotsc, \tilde v(n+1)$ such that their convex hull is close to $S$.
\vspace{.2em}
\hrule
\vspace{.2em}
\State Embed the sample in $\RR^{n+1}$: Let $p'(i) = (p(i), 1)$.

\State For every $i = 1, \dotsc, t$, generate a random scalar $T(i)$ distributed as $\mathrm{Gamma}(n+1, 1)$. Let $q(i) = p'(i) T(i)$.
\State Invoke $ICA(q(1), \dotsc, q(t))$ to obtain a approximately separating matrix $\tilde M$.
\State Compute the inverse of $\tilde M$ and multiply every column by the sign of its last entry to get a matrix $\tilde A$.
\State Remove the last row of $\tilde A$ and return the columns of the resulting matrix as $\tilde v(1), \dotsc, \tilde v(n+1)$.
\end{algorithmic}
\end{algorithm}
Algorithm \ref{alg:simplexreduction} works as follows: Let $X$ be an $(n+1)$-dimensional random vector with iid coordinates distributed as $\expdist{1}$. Let $V$ be the matrix having columns $(v(i),1)$ for $i=1, \dotsc, n+1$. Let $Y$ be random according to the distribution that results from scaling in the algorithm. Theorem \ref{thm:simplexscaling} implies that $Y$ and $VX$ have the same distribution. Also, $X-\ones$ is isotropic and $Y$ and $V(X-\ones) + V \ones$ have the same distribution. Thus, the discussion about ICA earlier in this section gives that the only separating matrices $M$ are such that $M V = D P$ where $P$ is a permutation matrix and $D$ is a diagonal matrix with diagonal entries in $\{-1, 1\}$. That is, $V P^T = M^{-1} D$. As the last row of $V$ is all ones, the sign change step in Algorithm \ref{alg:simplexreduction} undoes the effect of $D$ and recovers the correct orientation.
 
\begin{algorithm}
\caption{Reduction from Problem 2 to ICA}\label{alg:lpreduction}
\begin{algorithmic}
\State Input: A uniformly random sample $p(1), \dotsc, p(t)$ from $A(B_p^n)$ for a known parameter $p \in [1, \infty)$, where $A:\RR^n \to \RR^n$ is an unknown invertible linear transformation.
\State Output: A matrix $\tilde A$ such that $\tilde A B_p^n$ is close to $A(B_p^n)$.
\vspace{.2em}
\hrule
\vspace{.2em}

\State For every $i = 1, \dotsc, t$, generate a random scalar $T(i)$ distributed as $\mathrm{Gamma}((n/p)+1,1)$. Let $q(i) = p(i)  T(i)^{1/p}$.
\State Invoke $ICA(q(1), \dotsc, q(t))$ to obtain an approximately separating matrix $\tilde M$. 
\State Output $\tilde A = c_{p,n}^{-1} \tilde M^{-1}$.
\end{algorithmic}
\end{algorithm}
%
%

Similarly, Algorithm \ref{alg:lpreduction} works as follows: Let $X$ be a random vector with iid coordinates, each with density proportional to $\exp(-\abs{t}^p)$.
Let $Y$ be random according to the distribution that results from scaling in the algorithm.
Theorem \ref{thm:lpscaling} implies that $Y$ and $AX$ have the same distribution.
Also, $X/c_{p,n}$ is isotropic and we have $Y$ and $Ac_{p,n}(X/c_{p,n})$ have the same distribution.
Thus, the discussion about ICA earlier in this section gives that the only separating matrices $M$ are such that $M A c_{p,n} = D P$ where $P$ is a permutation matrix and $D$ is a diagonal matrix with diagonal entries in $\{-1, 1\}$.
That is, $A P^{T} D^{-1} = c_{p,n}^{-1} M^{-1}$.
The fact that $B_p^n$ is symmetric with respect to coordinate permutations and sign changes implies that $A P^{T} D^{-1} B_p^n = A B_p^n$ and is the same as $c_{p,n}^{-1} M^{-1}$.
When $p \neq 2$, the assumptions in the discussion above about ICA are satisfied and Algorithm \ref{alg:lpreduction} is correct.
When $p = 2$, the distribution of the scaled sample is Gaussian and this introduces ambiguity with respect to rotations in the definition of $M$, but this ambiguity is no problem as it is counteracted by the fact that the $\ell_2$ ball is symmetric with respect to rotations.



\chapter{Learning Gaussian Mixture Models}\label{ch:gmm}

The question of  recovering a probability distribution from a finite set of samples is one of the most fundamental questions of statistical inference. While classically such problems have  been considered in 
low dimension, more recently inference in high dimension has drawn significant attention in statistics and computer science literature. 

In particular, an active line of investigation in theoretical computer science has dealt with the question of learning a Gaussian Mixture Model in high dimension. This line of work was started in~\cite{dasgupta99} where  the first algorithm to recover parameters using a number of samples polynomial in the dimension was presented. The method relied on random projections to a low dimensional space and required certain separation conditions for the means of the Gaussians.  Significant work was done in order to weaken the separation conditions and to generalize the  result (see e.g.,~\cite{dasgupta00,arora01,vempala02, achlioptas05,feldman06}). 
Much of this work has polynomial sample and time complexity but requires strong separation conditions on the Gaussian components. A completion of the attempts to weaken the separation conditions was achieved 
in~\cite{BelkinFOCS10} and~\cite{moitra10}, where it was shown that arbitrarily small separation was sufficient for learning a general mixture with a fixed number of components in polynomial time. Moreover, a one-dimensional example given in~\cite{moitra10} showed  that  an exponential dependence on the number of components was unavoidable unless strong separation requirements were imposed.  Thus the question of polynomial learnability appeared to be settled.  It is worth noting that  while quite different in many aspects, all of these  papers used a general scheme similar to that in the original work~\cite{dasgupta00} by reducing  high-dimensional inference to a small number of low-dimensional problems through appropriate projections.

However,  a surprising result was recently proved in~\cite{HsuK13}. The authors showed that a 
mixture of $d$ Gaussians
 in dimension $d$ 
could be learned using a polynomial number of samples, assuming a non-degeneracy condition 
on the configuration of the means.  The result in~\cite{HsuK13} is inherently high-dimensional as that condition is 
never satisfied when  the means belong to  a lower-dimensional space. 
Thus the problem of learning a mixture gets progressively  computationally easier as the dimension increases, a ``blessing of dimensionality!" It is important to note that this was quite different from much of the previous work, which had primarily used projections to lower-dimension spaces. 

Still, there remained a large gap between the worst case impossibility of efficiently learning more than a fixed number of Gaussians  in low dimension and the situation when the number of components is equal to the dimension. Moreover, it was not completely clear whether the underlying problem was  genuinely easier in high dimension or our algorithms in low dimension were suboptimal. The one-dimensional example  in~\cite{moitra10} cannot answer this question as it is a specific worst-case scenario, which can be potentially ruled out by some genericity condition. 

We take a step to eliminate this gap by showing that even very large mixtures of Gaussians can be polynomially learned. More precisely, we show that a mixture of $m$ Gaussians with equal known covariance can be polynomially learned as long as $m$ is bounded from above by a polynomial of the dimension $n$ and a  certain more complex non-degeneracy condition for the means is satisfied. We show that if $n$ is high enough, these non-degeneracy conditions are  generic in the smoothed complexity sense. Thus for any fixed $d$, $O(n^d)$ generic Gaussians can be polynomially learned in dimension $n$.

Further, we prove that no such condition can exist in low dimension.  A  measure of non-degeneracy must be monotone in the sense that adding Gaussian components must make the condition number worse. 
However, we show that for $k^2$ points uniformly sampled from $[0,1]$ there are (with high probability) two mixtures of    unit Gaussians with means on non-intersecting subsets of these points, whose $L^1$ distance is $O^*(e^{-{k}})$ and which are thus not polynomially identifiable. More generally, in dimension $n$ the distance becomes  $O^*(\exp({-\sqrt[n]{k}}))$.  That is, the conditioning improves as the dimension increases, which is consistent with our algorithmic results.


To summarize, our contributions are as follows:

\noindent (1)
We show that for any $q$, a mixture of $n^q$ Gaussians in dimension $n$ can be learned in time and  number of samples polynomial in $n$ and a certain ``condition number" $\sigma$. 
For sufficiently high dimension, this results in an algorithm polynomial from the smoothed analysis point of view (Theorem \ref{thm:gmm-correctness}). To do that we provide smoothed analysis of the condition number using certain results from~\cite{RudelsonVershynin} and anti-concentration inequalities.  
    The main technical ingredient of the algorithm  is a new ``Poissonization" technique to reduce Gaussian mixture estimation to a problem of recovering a linear map of a product distribution known as underdetermined Independent Component Analysis (ICA). 
We combine this with the recent work on efficient algorithms for underdetermined ICA from~\cite{GVX} to obtain the necessary bounds.

\noindent (2)
We show that in low dimension polynomial identifiability fails in a certain generic sense (see Theorem~\ref{thm:low-dim-identifiability}).  Thus the efficiency of our main algorithm is truly a consequence of the ``blessing of dimensionality" and no comparable algorithm exists in low dimension. The analysis is based on  results from approximation theory and Reproducing Kernel Hilbert Spaces.

Moreover, we combine the approximation theory results with the Poissonization-based  technique to show how to embed difficult instances of low-dimensional Gaussian mixtures into the ICA setting, thus establishing exponential information-theoretic lower bounds for underdetermined ICA in low dimension. To the best of our knowledge, this is the first such result  in the literature.


We discuss our main contributions more formally now.
The notion of Khatri--Rao power $A^{\odot d}$ of a matrix $A$ is defined in Section~\ref{sec:gmm-preliminaries}.
\begin{theorem}[Learning a GMM with Known Identical Covariance] \label{thm:gmm-correctness}
Suppose $\means \geq \dim$ and
let $\epsilon, \delta > 0$.
Let $w_1 \Norm(\mu_1, \Sigma) + \dotsb + w_m \Norm(\mu_m, \Sigma)$ be an $\dim$-dimensional GMM, i.e.
$\mu_i \in \R^n$, $w_i > 0$, and $\Sigma \in \R^{n \times n}$. 
Let $B$ be the $\dim \times \means$ matrix whose $i^{th}$ column is $\mu_i / \norm{\mu_i}$.
If there exists $d \in \N$ so that $\sigma_m\left(B^{\odot d} \right) > 0$, then Algorithm \ref{alg:reduction} recovers each $\mu_i$ to within $\epsilon$ accuracy with probability $1 - \delta$.
Its sample and time complexity are at most
\begin{align*}
\poly\left(m^{d^2},
      \sigma^{d^2}, u^{d^2}, w^{d^2},
      d^{d^2}, r^{d^2}, 1/\epsilon, 1/\delta, 1/b, \log^{d^2}\bigl({1}/{(b \eps \delta)}\bigr) \right)
\end{align*}
where 
$w \geq \max_{i}(w_i)/\min_{i}(w_i)$,
$u \geq \max_i \norm{\mu_i}$,
$r \geq \big(\max_i\norm{\mu_i} + 1)/(\min_i\norm{\mu_i})\big)$,
$0 < b \leq \sigma_m(B^{\odot d})$ are bounds provided to the algorithm, and
$\sigma = \sqrt{\lambda_{\max}(\Sigma)}$.
\end{theorem}

We note that the requirement that $m \geq n$ is due to the invocation of Theorem 1.3 of \cite{GVX}; it should not be difficult, however, to adapt the algorithm to use a method similar to that of \cite{HsuK13} to handle the case where $m < n$.


Given that the means have been estimated, the weights can be recovered using the tensor structure of higher order cumulants (see Section~\ref{sec:gmm-preliminaries} for the definition of cumulants). 
This is shown in Appendix~\ref{sec:WeightRecovery}.

We show that $\sigma_{\min}(A^{\odot d})$ is large in the smoothed analysis sense, namely, if we 
start with a base matrix $A$ and perturb each entry randomly to get $\tilde A$, then
$\sigma_{\min}(\tilde A^{\odot d})$ is likely to be large:

\begin{theorem}\label{thm:smoothed-sigma-min-intro}
For $n > 1$, let $M \in \R^{n \times \binom{n}{2}}$ be an arbitrary matrix. Let 
$N \in \R^{n \times \binom{n}{2}}$ 
be a randomly sampled matrix with each entry iid from $\Normal(0, \sigma^2)$, for 
$\sigma > 0$. Then, for some absolute constant $C$,
\begin{vstd}
\begin{align*}
\prob{\sigma_{\min}((M+N)^{\odot 2}) \leq {\sigma^2}/{n^7}} \leq {2C}/{n}.
\end{align*}
\end{vstd}
\end{theorem}

We point out the simultaneous and independent work of \cite{Bhaskara2013}, where the authors prove learnability results related to our Theorems~\ref{thm:gmm-correctness} and \ref{thm:smoothed-sigma-min-intro}. We now provide a comparison. The results in 
\cite{Bhaskara2013}, which are based on tensor decompositions,  are stronger in that they can learn mixtures of axis-aligned Gaussians (with
non-identical covariance matrices) without requiring to know the covariance matrices in advance. Their results hold
under a smoothed analysis setting similar to ours. To learn a mixture of roughly $n^{\ell/2}$ Gaussians up to an 
accuracy of $\epsilon$ their algorithm has running time and sample complexity 
$\mathrm{poly}_\ell(n, 1/\epsilon, 1/\rho)$ and succeeds with probability at least $1- \exp(-Cn^{1/3^\ell})$, 
where the means are perturbed by adding an $n$-dimensional Gaussian from $\Normal(0, I_n \rho^2/n )$. 
On the one hand, the success probability of their algorithm is much better (as a function of $n$, exponentially close to $1$ as opposed to polynomially close to $1$, as in our result). 
On the other hand, this comes at a high price in terms of the running time and sample complexity: The polynomial $\mathrm{poly}_\ell(n, 1/\epsilon, 1/\rho)$ above has degree
exponential in $\ell$, unlike the degree of our bound which is polynomial in $\ell$.
Thus, in this respect, the two results can be regarded as incomparable points on an error vs running time 
(and sample complexity) trade-off curve. 
Our result is based on a reduction from learning GMMs to ICA  which could be of  independent interest given that both 
problems are extensively studied in somewhat disjoint communities.
Moreover, our analysis can be used in the reverse direction to obtain hardness results for ICA.

The technique of Poissonization is, to the best of our knowledge, new in the GMM setting, though we note that it has been previously applied in computational learning. For instance, in \cite{ValiantV13}, it has been used for the estimation of properties of discrete probability distributions such as the support size and the entropy.

Finally, in Section~\ref{identifiability_low_dimension} we show that in low dimension the situation is very different from the high-dimensional generic efficiency given by Theorems \ref{thm:gmm-correctness} and \ref{thm:smoothed-sigma-min-intro}: The problem is generically hard. More precisely, we show:
\begin{theorem}\label{thm:low-dim-identifiability}
Let $X$ be a set of $4k^2$ points uniformly sampled from $[0,1]^n$. Then with high probability there exist two mixtures with equal number of unit Gaussians $p$, $q$ centered on disjoint subsets of $X$, such that, for some $C>0$,
\begin{vstd}
\[
\|p-q\|_{L^1(\R^n)} < e^{-C \left({k}/{\log k}\right)^{{1}/{n}} }.
\]
\end{vstd}
\end{theorem}

Here we would like to note that the assumption  that $X$ is random is convenient as it provides a natural model for {\it genericity}, guarantees (with high probability) small {\it fill}  (Section~\ref{identifiability_low_dimension})  and also ensures that the means of the Gaussian mixture components of $p$ and $q$ are not too close. In particular, it is not difficult to verify that with high probability any pair of the random means  are at least $1 / \poly(k)$ separated. However the randomness assumption  is not essential. In fact, the above theorem will hold for an arbitrary set of points with sufficiently small fill (see Section~\ref{identifiability_low_dimension}).


Combining the above lower bound with our reduction provides a similar lower bound for ICA; 
see a discussion on the connection with ICA below.
Our lower bound gives an information-theoretic barrier. This is in contrast to conjectured computational 
barriers that arise in related settings based on the noisy parity problem (see \cite{HsuK13} for pointers).
The only previous information-theoretic lower bound for learning GMMs we are aware of is due to \cite{moitra10}
and holds for two specially designed one-dimensional mixtures. 

\section{Preliminaries}
\label{sec:gmm-preliminaries}
The singular values of a 
matrix $A \in \R^{m \times n}$ will be ordered in the decreasing order: $\sigma_1 \geq
\sigma_2 \geq \dotsb \geq \sigma_{\min(m, n)}$. By $\sigma_{\min}(A)$ we mean $\sigma_{\min(m,n)}$. 

For a real-valued random variable $X$, the \emph{cumulants} of
$X$ are certain polynomials in the moments of $X$. For $j \geq 1$, the $j$th
cumulant is denoted $\cum_j(X)$. Denoting $\m_j := \E{X^j}$, we have, for
example: $\cum_1(X) = \m_1,\ \cum_2(X) = \m_2 - \m_1^2,$ and $\cum_3(X) =
\m_3 - 3 \m_2\m_1 + 2\m_1^3$. In general, cumulants can be defined 
as certain coefficients of a Taylor expansion of the logarithm of the moment generating function of $X$:
$
\log(\mathbb{E}_X(e^{tX})) = \sum_{j=1}^{\infty}\cum_j(X) \frac{t^j}{j!} 
$. 
The first two cumulants are the same as the expectation and the
variance, resp. Cumulants have the property that for two independent
random variables $X, Y$ we have $\cum_j(X+Y) = \cum_j(X) + \cum_j(Y)$ (assuming
that the first $j$ moments exist for both $X$ and $Y$).  Cumulants are degree-$j$ homogeneous, i.e.~if $\alpha \in \R$ and $X$ is a random variable, then
$\cum_j(\alpha X) = \alpha^j\cum_j(X)$.  The third and higher
cumulants of the Gaussian distribution are 0.

\subsection{Gaussian Mixture Model.}


For $i = 1, 2, \dots, m$, define Gaussian random vectors
$\eta_i \in \R^n$ with distribution $\eta_i \sim \Normal(\mu_i, \Sigma_i)$ where
$\mu_i \in \R^{n}$ and $\Sigma_i\in \R^{n\times n}$.  Let
$h$ be an integer-valued random variable which takes on value $i \in [m]$ with probability 
$w_i > 0$, henceforth called weights. (Hence $\sum_{i=1}^m w_i = 1$.)  Then, the random vector drawn as $Z = \eta_h$ is said
to be a Gaussian Mixture Model (GMM) 
$w_1 \Normal(\mu_1, \Sigma_1) + \ldots + w_m \Normal(\mu_m, \Sigma_m)$. 
The sampling of $Z$ can be interpreted as first picking
one of the components $i \in [m]$ according to the weights, and then sampling a Gaussian vector from 
component $i$.  We will be primarily interested in the mixture of
identical Gaussians of known covariance.  In particular, there exists known $\Sigma \in \R^{n\times n}$ such that $\Sigma_i = \Sigma$ for each $i$.
Letting $\eta \sim \Normal(0, \Sigma)$, and denoting by $\mathbf{e}_h$ the 
random variable which takes on the $i$\textsuperscript{th} canonical vector
$\mathbf{e}_i$ with probability $w_i$, we can write the GMM model as follows:
\begin{equation} \label{eq:GMM_Model}
  Z = [\mu_1 | \mu_2 | \cdots | \mu_m] \mathbf{e}_h + \eta \ .
\end{equation}
In this formulation, $\mathbf{e}_h$ acts as a selector of a Gaussian mean.  Conditioning on $h=i$, we have $Z \sim \Normal(\mu_i, \Sigma)$, which is consistent with the GMM model.

Given samples from the GMM, the goal is to recover the unknown parameters of the GMM, namely the means $\mu_1, \dots, \mu_m$ and the weights $w_1, \dots, w_m$.

\subsection{Underdetermined ICA.}
In the basic formulation of ICA, the observed random variable $X \in \R^n$ is drawn according to the model $X = AS$, where $S \in \R^m$ is a latent random vector whose components $S_i$ are independent random variables, and $A \in \R^{n \times m}$ is an unknown \textit{mixing matrix}.
The probability distributions of the $S_i$ are unknown except that they are not Gaussian. 
The ICA problem is to
recover $A$ to the extent possible. The underdetermined ICA problem corresponds the case $m \geq n$.
We cannot hope to recover $A$ fully because if we flip the sign of the $i$\textsuperscript{th} column of $A$, or scale this column by some
nonzero factor, then the resulting mixing matrix with an appropriately scaled $S_i$ will again generate the same distribution on $X$ as before.  
There is an additional ambiguity that arises from not having an ordering on the coordinates $S_i$:
If $P$ is a permutation matrix, then $PS$ gives a new random vector with independent reordered coordinates, $AP^T$ gives a new mixing matrix with reordered columns, and $X = AP^TPS$ provides the same samples as $X = AS$ since $P^T$ is the inverse of $P$.
As $AP^T$ is a permutation of the columns of $A$, this ambiguity implies that we cannot recover the order of the columns of $A$.
However, it turns out that under certain genericity requirements, we can recover $A$ up to these necessary ambiguities, that is to say we can recover the directions (up to sign) of the columns of $A$, even in the underdetermined setting. 

It will be important for us to work with an ICA model where there is Gaussian noise
in the data: 
$X = AS + \eta$,
where $\eta \sim \Normal(0, \Sigma)$ is an additive Gaussian noise independent of $S$, and the covariance of $\eta$ given by $\Sigma \in \R^{n \times n}$ is in general unknown and not necessarily spherical. 
We will refer to this model as the noisy ICA model.%

We define the flattening operation $\vec{\cdot}$ from a tensor to a vector in the natural way.  Namely, when $T \in \R^{n^\ell}$ is a tensor, then $\vec{T}_{\delta(i_1, \dotsc, i_\ell)} = T_{i_1, \dotsc, i_\ell}$ where $\delta(i_1, \dotsc, i_\ell) = 1 + \sum_{j=1}^\ell n^{\ell-j} (i_j - 1)$ is a bijection with indices $i_j$ running from $1$ to $n$.
Roughly speaking, each index is being converted into a digit in a base $n$ number up to the final offset by 1.
This is the same flattening that occurs to go from a tensor outer product of vectors to the Kronecker product of vectors.

The ICA algorithm from \cite{GVX} to which we will be reducing learning a GMM relies on the shared tensor structure of the derivatives of the second characteristic function and the higher order multi-variate cumulants. This tensor structure motivates the following form of the Khatri-Rao product:
\begin{definition}
  Given matrices $A\in \R^{n_1 \times m}, B \in \R^{n_2\times m}$, a column-wise Khatri-Rao product is defined by $A \odot B := [\vec{A_1 \otimes B_1} | \cdots | \vec{A_m \otimes B_m}]$, where $A_i$ is the $i$\textsuperscript{th}
column of $A$, $B_i$ is the $i$\textsuperscript{th} column of $B$, $\otimes$ denotes the Kronecker product and 
$\vec{A_1 \otimes B_1}$ is flattening of the tensor $A_1 \otimes B_1$ into a vector.
  The related Khatri-Rao power is defined by $A^{\odot \ell} = A \odot \cdots \odot A$ ($\ell$ times).
\end{definition}
This form of the Khatri-Rao product arises when performing a change of coordinates under the ICA model using either higher order cumulants or higher order derivative tensors of the second characteristic function.

\subsection{ICA Results.}
\vnote{We should have a brief discussion here on the typical sizes of $d$ and $k$.}
Theorem \ref{thm:UICA_noisy} (Theorem~\ref{thm:UICA_noisy}, from \cite{GVX}) allows us to recover $A$ up to the necessary ambiguities in the noisy ICA setting. 
The theorem establishes guarantees for an algorithm from \cite{GVX} for noisy underdetermined ICA, \textbf{UnderdeterminedICA}. This algorithm takes as input a tensor order parameter $d$, number of signals $m$, access to samples according to the noisy underdetermined ICA model with unknown noise, accuracy parameter $\epsilon$, confidence parameter $\delta$, bounds on moments and cumulants $M$ and $\Delta$, a bound on the conditioning parameter $\sigma_m$, and a bound on the cumulant order $k$. It returns approximations to the columns of $A$ up to sign and permutation. 


\paragraph{ICA Results.}
The following theorem (from \cite{GVX}) allows us to recover $A$ up to the necessary ambiguities in the noisy ICA setting. 
The theorem establishes guarantees for an algorithm from \cite{GVX} for noisy underdetermined ICA, \textbf{UnderdeterminedICA}. This algorithm takes as input a tensor order parameter $d$, number of signals $m$, access to samples according to the noisy underdetermined ICA model with unknown noise, accuracy parameter $\epsilon$, confidence parameter $\delta$, bounds on moments and cumulants $M$ and $\Delta$, a bound on the conditioning parameter $\sigma_m$, and moment order $k$. It returns approximations to the columns of $A$ up to sign and permutation. 

\begin{theorem}[\cite{GVX}]\label{thm:UICA_noisy}
Let a random vector $x \in \R^n$ be given by an underdetermined ICA model with unknown Gaussian noise $x= As + \eta$ where $A \in \R^{n \times m}$, with unit norm columns, and the covariance matrix $\Sigma \in \R^{n \times n}$
are unknown. Let $d \in 2 \N$ be such that $\sigma_m(A^{\odot d/2}) > 0$. Let $k > d$ be such that
 for each $s_i$, there is a $k_i$ satisfying $d < k_i \le k$ and $\abs{\cum_{k_i}(s_i)} \ge
  \Delta$, and $\E{ \abs{s_i}^k} \le M$.  
Moreover, suppose that the noise also satisfies the same moment condition: $\E{\abs{\angles{u, \eta_i}}^k} \le M$ for any unit
vector $u \in \R^n$ (this is satisfied if we have $k! \sigma^k \leq M$ where $\sigma^2$ is the maximum eigenvalue of 
$\Sigma$).  
Then algorithm \textbf{UnderdeterminedICA} returns a set of $n$-dimensional vectors $(\tilde A_i)_{i=1}^m$ so that for some permutation $\pi$ of $[m]$ and signs $\alpha_i \in \{-1, 1\}$ we have $\norm{\alpha_i \tilde A_{\pi(i)} - A_i} \leq \eps$ for all $i \in [m]$.
Its sample and time complexity are $\text{poly}\left( n^{k} , m^{k^2}, M^k, 1/ \Delta^k, 1/\sigma_m(A^{\odot d/2})^k, 1/\eps, 1/\delta \right)$.
\end{theorem}

\section{Learning GMM means using underdetermined ICA: The basic idea}
\label{sec:gmm-reduction}
In this section we give an informal outline of the proof of our main result, namely learning the means of 
the components in GMMs via reduction to the underdetermined ICA problem.
Our reduction will be discussed in two parts. The first part gives the main idea of the reduction and will demonstrate how to recover the means $\mu_i$ up to their norms and signs, i.e. we will get $\pm \mu_i/\norm{\mu_i}$. We will then present the reduction in full. It combines the basic reduction with some preprocessing
of the data to recover the $\mu_i$'s themselves.
The reduction relies on some well-known properties of the Poisson distribution stated in the 
lemma below; its proof can be found in Appendix~\ref{sec:Poisson-lemmas}. 


\begin{lemma} \label{lem:Poisson-independence}
Fix a positive integer $k$, and let $p_i \geq 0$ be such that $p_1 + \dotsb + p_k =1$. If $X \sim \Pois(\lambda)$ and
$(Y_1, \ldots, Y_k)|_{X = x} \sim \Multinom(x; p_1, \ldots, p_k)$ then $Y_i \sim \Pois(p_i \lambda)$ for all $i$ and
$Y_1, \ldots, Y_k$ are mutually independent.
\end{lemma}

\paragraph{Basic Reduction: The main idea.}
Recall the GMM from equation \eqref{eq:GMM_Model} is given by 
$Z = [\mu_1 | \cdots | \mu_m] \mathbf{e}_h + \eta$. 
Henceforth, we will set $A = [\mu_1 | \cdots | \mu_m]$.  
We can write the GMM in the form $Z = A\mathbf{e}_h + \eta$, which is similar in form to the noisy ICA model, except that $\mathbf{e}_h$ does not have independent coordinates.
We now describe how a single sample of an approximate noisy ICA problem is generated.


The reduction involves two internal parameters $\lambda$ and $\tau$ that we will set later.
We generate a Poisson random variable $R \sim \Poisson(\lambda)$, 
and we run the following experiment $R$ times: At the $i$\textsuperscript{th} step, generate
sample $Z_i$ from the GMM.
Output the sum of the outcomes of these experiments: $Y = Z_1 + \cdots + Z_R$. 

Let $S_i$ be the random variable denoting the number of times samples from the $i$\textsuperscript{th} component were chosen in the above experiment.  Thus $S_1 + \cdots + S_m = R$.
Note that $S_1, \dots, S_m$ are not observable although we know their sum. By 
Lemma~\ref{lem:Poisson-independence}, each $S_i$ has distribution $\Poisson(w_i \lambda)$, 
and the random variables $S_i$ are mutually independent.
Let $S := (S_1, \dots, S_m)^T$.

For a non-negative integer $t$, we define $\eta(t) := \sum_{i=1}^t \eta_i$ where the $\eta_i$ 
are iid according to $\eta_i \sim \Normal(0, \Sigma)$. In this definition $t$ can be a r.v., in 
which case the $\eta_i$ are sampled independent of $t$. 
Using $\deq$ to indicate that two random variables have the same distribution, then 
\begin{equation} \label{eq:sum_model}
  Y \deq AS + \eta(R) \ .
\end{equation}

If there were no Gaussian noise in the GMM (i.e.{}~if we were sampling from a discrete set of points) then the model in \eqref{eq:sum_model} becomes $Y = AS$, which is the ICA model without noise,
and so we could recover $A$ up to necessary ambiguities.
However, the model $Y \deq AS + \eta(R)$ fails to satisfy even the assumptions of the noisy ICA model, both because $\eta(R)$ is not independent of $S$ and because $\eta(R)$ is not distributed as a Gaussian random vector.

%

As the covariance of the additive Gaussian noise is known, we may add additional noise to the samples of $Y$ to obtain a good
approximation of the noisy ICA model.
Parameter $\tau$, the second parameter of the reduction, is chosen so that with high probability we have $R \leq \tau$.
Conditioning on the event $R \leq \tau$ we draw $X$ according to the rule
\begin{align*}
X = Y + \eta(\tau-R) \deq AS + \eta(R) + \eta(\tau-R),
\end{align*}
\lnote{same as before, not equal}
where $\eta(R)$, $\eta(\tau - R)$, and $S$ are drawn independently conditioned on $R$.
Then, conditioned on $R \leq \tau$, we have $X \deq AS + \eta(\tau)$.

Note that we have only created an approximation to the ICA model.
In particular, restricting $\sum_{i=1}^m S_i = R \leq \tau$ can be accomplished using rejection sampling, but the coordinate random variables $S_1, \dots, S_m$ would no longer be independent.
We have two models of interest: (1) $X \deq AS + \eta(\tau)$, a noisy ICA model with no restriction on $R = \sum_{i=1}^m S_i$, and (2) $X \deq (AS + \eta(\tau))|_{R \leq \tau}$ the restricted model.

We are unable to produce samples from the first model, but it meets the assumptions of the noisy ICA problem.
Pretending we have samples from model (1), we can apply Theorem \ref{thm:UICA_noisy} to recover the Gaussian means up to sign and scaling.
On the other hand, we can produce samples from model (2), and depending on the choice of $\tau$, the statistical distance between models (1) and (2) can be made arbitrarily close to zero.
It will be demonstrated that given an appropriate choice of $\tau$, running \textbf{UnderdeterminedICA} on samples from model (2) is equivalent to running \textbf{UnderdeterminedICA} on samples from model (1) with high probability, allowing for recovery of the Gaussian mean directions $\pm \mu_i/\norm{\mu_i}$ up to some error.


\paragraph{Full reduction.} To be able to recover the $\mu_i$ without sign or scaling ambiguities, we add an extra coordinate to the GMM as follows. The new
means $\mu_i'$ are $\mu_i$ with an additional coordinate whose value is $1$ for all $i$, i.e.~%
$\mu_i' := \left(\begin{array}{c} \mu_i \\ 1  \end{array}\right)$.
Moreover, this coordinate has no noise. In
other words, each Gaussian component now has an $(n+1)\times (n+1)$ covariance matrix  $\Sigma' := \left( \begin{array}{cc} \Sigma & 0 \\ 0 & 0 \end{array} \right)$.
It is easy to construct samples from this new GMM given samples from the original: If the original samples were
$u_1, u_2 \ldots$, then the new samples are $u'_1, u'_2 \ldots$ where
$u'_i  := \left(\begin{array}{c} u_i \\ 1  \end{array}\right)$.  
The reduction proceeds similarly to the above on the new inputs.

Unlike before, we will define the ICA mixing matrix to be 
$A' := \left[ \frac{\mu_1'}{\norm{\mu_1'}} | \cdots | \frac{\mu_m'}{\norm{\mu_m'}}\right]$ such that it has unit norm columns.
The role of matrix $A$ in the basic reduction will now be played by $A'$.
Since we are normalizing the columns of $A'$, we have to
scale the ICA signal $S$ obtained in the basic reduction to compensate for this: Define $S'_i := \norm{\mu'_i} S_i$. Thus, the ICA models obtained in the full reduction are:
\begin{align}
X' &= A' S' + \eta'(\tau),  \label{eq:full-ideal-model} \\
X' &= (A' S' + \eta'(\tau)), |_{R \leq \tau} \label{eq:full-approximate-model}
\end{align} 
where we define $\eta'(\tau) = \left(\begin{array}{c} \eta(\tau) \\ 0  \end{array}\right)$.
As before, we have an ideal noisy ICA model \eqref{eq:full-ideal-model} from which we cannot sample, and an approximate noisy ICA model \eqref{eq:full-approximate-model} which can be made arbitrarly close to 
\eqref{eq:full-ideal-model} in statistical distance by choosing $\tau$ appropriately.  
With appropriate application of Theorem~\ref{thm:UICA_noisy} to these models, we can recover estimates (up to sign) $\{\tilde A_1', \dotsc, \tilde A_m'\}$ of the columns of $A'$.
%
%
%
%

By construction, the last coordinate of $\tilde A_i'$ now
tells us both the signs and magnitudes of the $\mu_i$: Let $\tilde{A}'_i(1:n) \in \R^n$ be the vector consisting of the
first $n$ coordinates of $\tilde{A}'_i$, and let $\tilde{A}'_i(n+1)$ be the last coordinate of $\tilde{A}'_i$. Then
$$\mu_i = \frac{A_i'(1:n)}{A_i'(n+1)} \approx \frac{\tilde{A}'_i(1:n)}{\tilde{A}'_i(n+1)}, $$
with the sign indeterminancy canceling in the division.

\floatname{algorithm}{Subroutine}
\begin{algorithm} 
\caption{Single sample reduction from GMM to approximate ICA} \label{sub:independentsamples}
\begin{algorithmic}[1]
\Require Covariance parameter $\Sigma$,
access to samples from a mixture of $\means$ identical Gaussians in $\R^{\dim}$ with variance $\Sigma$,
Poisson threshold $\tau$,
Poisson parameter $\lambda$,
\Ensure $Y$ (a sample from model (\ref{A'_ICAModel_actual})).
\vspace{.2em}
\hrule
\vspace{.2em}
  \State Generate $R$ according to $\Poisson(\lambda)$.
  \If{$R > \tau$}
    \Return failure.
  \EndIf
  \State Let $Y = 0$.
  \For{$j = 1$ to $R$}
    \State Get a sample $Z_j$ from the GMM.
    \State Let $Z'_j = (Z_j, 1)$ to embed the sample in $\R^{\dim+1}$
    \State $Y = Y + Z_j$.
  \EndFor
  \State Let $\Sigma' = \begin{pmatrix} \Sigma & 0 \\ 0 & 0 \end{pmatrix}$ (add a row and column of all zeros)
  \State Generate $\eta$ according to $\Norm(0,(\tau - R) \Sigma')$.
  \State $Y = Y + \eta$.
\Return $Y$.
\end{algorithmic}
\end{algorithm}
\floatname{algorithm}{Algorithm}
\begin{algorithm}

\caption{Use ICA to learn the means of a GMM} \label{alg:reduction}
\begin{algorithmic}[1]
\Require Covariance matrix $\Sigma$, number of components $m$,
upper bound on tensor order parameter $d$,
access to samples from a mixture of $\means$ identical,
spherical Gaussians in $\R^{\dim}$ with covariance $\Sigma$,
confidence parameter $\delta$,
accuracy parameter $\epsilon$,
upper bound $w \geq \max_{i}(w_i) / \min_{i}(w_i)$,
upper bound on the norm of the mixture means $u$,
$r \geq (\max_i\norm{\mu_i} + 1)/(\min_i\norm{\mu_i})$, and
lower bound $b$ so $0 < b \leq \sigma_m(A^{\odot d/2})$.
\Ensure $\{\tilde{\mu}_1, \tilde{\mu}_2, \dots, \tilde{\mu}_\means\} \subseteq \R^\dim$ (approximations to the means of the GMM).
\vspace{.2em}
\hrule
\vspace{.2em}
\State Let $\delta_2 = \delta_1 = \delta/2$.
\State Let $\sigma = \sup_{v \in S^{n-1}}\sqrt{\Var(v^T\eta(1))}$, for $\eta(1) \sim \Norm(0, \Sigma)$.
\State Let $\lambda = m$ be the parameter to be used to generate the Poisson random variable in Subroutine \ref{sub:independentsamples}.
\State Let $\tau = 4\big(\log(1/\delta_2) + \log(q(\Theta))\big)\max\left((e\lambda)^2, 4Cd^2\right)$
(the threshold used to add noise in the samples from Subroutine \ref{sub:independentsamples}, $C$ is a universal constant, and $q(\Theta)$ is a polynomial defined as (\ref{eq:q-theta}) in the proof of Theorem \ref{thm:gmm-correctness}).
\State Let $\epsilon^* = \epsilon \left(\sqrt{1 + u^2} + {2 (1 + u^2)}\right)^{-1}$.
\State Let $M = \max \left((\tau \sigma)^{d+1}, (w/(\sqrt{1+u^2})^{d+1}\right) (d+1)^{d+1}$.
\State Let $k = d + 1$.
\State Let $\Delta$ = $w$.
\State Invoke \textbf{UnderdeterminedICA} with access to Subroutine \ref{sub:independentsamples}, parameters $\delta_1, \epsilon^*$, $\Delta$, $M$, and $d$ to obtain $\tilde{A'}$ (whose columns approximate the normalized means up to sign and permutation).  If any calls to Subroutine \ref{sub:independentsamples} result in failure, the algorithm will halt completely.
\State Divide each column of $\tilde{A'}$ by the value of its last entry.
\State Remove the last row of $\tilde{A'}$ to obtain $\tilde{B}$.
\Return the columns of $\tilde{B}$ as $\{\tilde{\mu}_1, \tilde{\mu}_2, \dots, \tilde{\mu}_\means\}$.
\end{algorithmic}
\end{algorithm}


\section{Correctness of the Algorithm and Reduction}

Subroutine \ref{sub:independentsamples} is used to capture the main process of the reduction:
Let $\Sigma$ be the covariance matrix of the GMM, $\lambda$ be an integer chosen as input, and a threshold value $\tau$ also computed elsewhere and provided as input.
Let $R \sim \Poisson(\lambda)$.
If $R$ is larger $\tau$, the subroutine returns a failure notice and the calling algorithm halts immediately.
A requirement, then, should be that the threshold is chosen so that the chance of failure is very small; in our case, $\tau$ is chosen so that the chance of failure is half of the confidence parameter given to Algorithm \ref{alg:reduction}.
The algorithm then goes through the process described in the full reduction: sampling from the GMM, lifting the sample by appending a 1, then adding a lifted gaussian noise so that the total noise has distribution $\Norm(0, \tau \Sigma)$.
The resulting sample is from the model given by (\ref{eq:full-approximate-model}).

Algorithm \ref{alg:reduction} works as follows: it takes as input the parameters of the GMM (covariance matrix, number of means), tensor order (as required by \textbf{UnderdeterminedICA}), error parameters, and bounds on certain properties of the weights and means.
The algorithm then calculates various internal paremeters: a bound on directional covariances, error parameters to be split between the ``Poissonization'' process and the call to \textbf{UnderdeterminedICA}, the threshold parameter $\tau$ and Poisson parameter $\lambda$ to be used in Subroutine \ref{sub:independentsamples}, and values explicitly needed by the proof of \textbf{UnderdeterminedICA} in \cite{GVX}.
Other internal values are given by a constanct $C$ and a polynomial $q(\Theta)$; these are determined by the proof of Theorem \ref{thm:gmm-correctness}.
Essentially, $C$ is a constant so that one can cleanly compute a value of $\tau$ that will involve a polynomial, $q(\Theta)$, of all the other parameters.
The algorithm then calls \textbf{UnderdeterminedICA}, but instead of giving samples from the GMM, it allows access to Subroutine \ref{sub:independentsamples}.
It is then up to \textbf{UnderdeterminedICA} to generate samples as needed (bounded by the polynomial in Theorem \ref{thm:gmm-correctness}).
In the case that Subroutine \ref{sub:independentsamples} returns a failure, the entire algorithm process halts, and returns nothing.
If no failure occurs, the matrix returned by \textbf{UnderdeterminedICA} will be the matrix of normalized means embedded in $\R^{\dim+1}$, and the algorithm de-normalizes, removes the last row, and then has approximations to the means of the GMM.

The bounds are used instead of actual values to allow flexibility -- in the context under which the algorithm is invoked -- on what the algorithm needs to succeed.
However, the closer the bounds are to the actual values, the more efficient the algorithm will be.

We show that the reduction results in a model that allows one to use the ICA algorithm \textbf{UnderdeterminedICA} presented in \cite{GVX} (see Section \ref{sec:preliminaries}, Theorem \ref{thm:UICA_noisy}).  That is, we can choose the parameters, $\tau$ and $\lambda$, for the Poissonization and threshold that yield a model which will satisfy the assumptions of Theorem \ref{thm:UICA_noisy}. This also involves translating the bounds on cumulants and higher moments of our Poisson signals and added noise to obtain feasible bounds for Theorem \ref{thm:UICA_noisy}.

\paragraph{Sketch of correctness argument.}
The proof of correctness of Algorithm \ref{alg:reduction} has two main parts.
In the first part (section \ref{subsec:IdealErrorAnalysis}), we analyze the sample complexity of recovering the Gaussian means using \textbf{UnderdeterminedICA} when samples are taken from the ideal noisy ICA model \eqref{eq:full-ideal-model}.
In the second part, we note that we do not have access to the ideal model \eqref{eq:full-ideal-model}, and that we can only sample from the approximate noisy ICA model \eqref{eq:full-approximate-model} using the full reduction.
Choosing $\tau$ appropriately, we use total variation distance to argue that with high probability, running \textbf{UnderdeterminedICA} with samples from the approximate noisy ICA model will produce equally valid results as running \textbf{UnderdeterminedICA} with samples from the ideal noisy ICA model.
The total variation distance bound is explored in section \ref{subsec:TotalVarDist}.

These ideas are combined in section \ref{subsec:CorrectnessProof} to prove the correctness of Algorithm \ref{alg:reduction}.
One additional technicality arises from the implementation of Algorithm \ref{alg:reduction}.
Samples can be drawn from the noisy ICA model $X' = (AS' + \eta'(\tau))|_{R \leq \tau}$ using rejection sampling on $R$.
In order to guarantee Algorithm \ref{alg:reduction} executes in polynomial time, when a sample of $R$ needs to be rejected, Algorithm \ref{alg:reduction} terminates in explicit failure.
To complete the proof, we argue that with high probability, Algorithm \ref{alg:reduction} does not explicitly fail.

%


\paragraph{Error Analysis of the Ideal Noisy ICA Model}
\label{subsec:IdealErrorAnalysis}

The proposed full reduction from Section \ref{sec:gmm-reduction} provides us with two models.
The first is a noisy ICA model from which we cannot sample:
\begin{align}
\text{(Ideal ICA)} \qquad  X' &= A'S' + \eta'(\tau) \ . \label{A'_ICAModel_ideal}
\end{align}
The second is a model that fails to satisfy the assumption that $S'$ has independent coordinates, but it is a model from which we can sample:
\begin{align}
\text{(Approximate ICA)} \qquad  X' &= (A'S' + \eta'(\tau))|_{R \leq \tau} \ . \label{A'_ICAModel_actual}
\end{align}
Both models rely on the choice of two parameters, $\lambda$ and $\tau$.
The dependence on $\tau$ is explicit in the models.  
The dependence on $\lambda$ can be summarized in the unrestricted model as $S_i = \frac 1 {\norm{\mu_i'}} S_i' \sim \Poisson(w_i \lambda)$ independently of each other, and $R = \sum_{i=1}^m S_i \sim \Poisson(\lambda)$.

The probability of choosing $R > \tau$ will be
seen to be exponentially small in $\tau$.  For this reason, running \textbf{UnderdeterminedICA}
with polynomially many samples from model
\eqref{A'_ICAModel_ideal} will with high probability be equivalent to
running the ICA Algorithm with samples from model
\eqref{A'_ICAModel_actual}.  This notion will be made
precise later using total variation distance.


For the remainder of this subsection, we proceed as if samples are drawn from
the ideal noisy ICA model \eqref{A'_ICAModel_ideal}.  Thus, to recover the columns of $A'$, it suffices to run \textbf{UnderdeterminedICA} on samples of $X'$.
Theorem \ref{thm:UICA_noisy} can be used for this analysis so long as we can
obtain the necessary bounds on the cumulants of $S'$, moments of $S'$, and the moments of $\eta'(\tau)$.
We define $w_{\min} := \min_{i} w_i$ and $w_{\max} := \max_{i} w_i$.
Then, the cumulants of $S'$ are bounded by the following lemma:

\begin{lemma} \label{lem:cumS_i'}
  Given $\ell \in \Z^+$, $\cum_\ell(S_i') \geq w_i \lambda$ for each $S_i'$.
  In particular, then $\cum_\ell(S_i') \geq w_{\min} \lambda$.
\end{lemma}
\begin{proof}
  By construction, $S'_i = \norm{\mu_i'}S_i$.
  By the homogeneity property of univariate cumulants,
  \[
    \cum_\ell(S'_i) = \cum_\ell(\norm{\mu_i'}S_i) 
    	= \norm{\mu_i'}^\ell \cum_\ell(S_i)
  \]
  As $\mu_i'(n+1) = 1$, $\norm{\mu_i'} \geq 1$. \lnote{Aha! This is one of the steps that would kill invariance under scaling of the algorithm.}
  The cumulants of the Poisson distribution are given in Lemma \ref{lem:PoissonCumulants}.  It follows that $\cum_{\ell}(S_i') \geq \cum_{\ell}(S_i) = w_i \lambda$.
\end{proof}

The bounds on the moments of $S'_i$ for each $i$ can be computed using the following lemma:
\begin{lemma} \label{lem:expectation_S_i'}
For $\ell \in \Z^+$, we have $\E{S_i'^\ell} \leq (\norm{\mu_i'}w_i\lambda)^\ell \ell^\ell$.
\end{lemma}
\begin{proof}
  Let $Y$ denote a random variable drawn from $\Poisson(\alpha)$.
  It is known (see \cite{Riordan}) that
  \[
    \E{Y^\ell} = \sum_{i=1}^\ell \alpha^i \stirling{\ell}{i}
  \]
  where $\stirling{\ell}{i}$ denotes Stirling number of the second kind.
  Using Lemma \ref{lem:StirlingBound_n}, it follows that
  \[
    \E{Y^\ell} \leq \sum_{i=1}^\ell \alpha^{i} \ell^{\ell-1}
      \leq \ell \alpha^{\ell} \ell^{\ell-1} 
      = \alpha^\ell\ell^\ell. 
  \]
  Since $S_i' = \mu_i'S_i$ where $S_i \sim \Poisson(\lambda w_i)$, it follows that $\E{S_i'^\ell} = \norm{\mu_i'}^\ell\E{S_i^\ell} \leq \norm{\mu_i'}^\ell (w_i \lambda_i)^\ell \ell^\ell$.
\end{proof}

The absolute moments of Gaussian random variables are well known.  For
completeness, the bounds are provided in Lemma \ref{lem:GaussianMoments} of Appendix~\ref{sec:GaussMoments}.

Defining $\sigma = \sup_{v \in S^{n-1}}\sqrt{\Var(v^T\eta'(1))}$; vectors $\mu_{\max}' = \argmax_{i}\norm{\mu_i'}$, $\mu_{\min}' = \argmin_{i}\norm{\mu_i'}$, and similarly $\mu_{\max}$ and $\mu_{\min}$ for later; and choosing $\lambda = m$, we can now show a polynomial bound for the error in recovering the columns of $A'$
using \textbf{UnderdeterminedICA}.
\begin{theorem}[ICA specialized to the ideal case] \label{thm:ICA_ideal_lifted}
  Suppose that samples of $X'$ are taken from the 
  unrestricted ICA model \eqref{A'_ICAModel_ideal} choosing parameter $\lambda = m$ and $\tau$ a constant.
  Suppose that \textbf{UnderdeterminedICA} is run using these samples.
  Suppose $\sigma_m(A'^{\odot d/2}) > 0$.  Fix $\epsilon \in (0, 1/2)$ and $\delta \in (0, 1/2)$. 
  Then with probability $1-\delta$, when the number of samples $N$ is:
   \begin{align}
    N & \geq \poly\left(n^d, m^{d^2}, 
      (\tau\sigma)^{d^2}, \norm{\mu'_{\max}}^{d^2}, \left({w_{\max}}/{w_{\min}}\right)^{d^2},
      d^{d^2}, 
      1/\sigma_m(A'^{\odot d/2})^d, 1/\epsilon, 1/\delta \right) \label{eq:poly_bound_lifted_recovery}
  \end{align}
  the columns of $A'$ are recovered within error $\epsilon$ up to their signs.
  That is, denoting the
  columns returned from \textbf{UnderdeterminedICA} by $\tilde A_1', \dotsc, \tilde A_m'$, 
  there exists $\alpha_1, \cdots, \alpha_m,  \in \{-1, +1\}$ and a permutation $p$ of $[m]$ such that
    $\norm{A'_i - \alpha_i\tilde A'_{p(i)}} < \epsilon$
  for each $i$\lnote{prime/no prime?}.
\end{theorem}
\begin{proof}
   Obtaining the sample bound is an exercise of rewriting the parameters associated with the model $X' = A'S' + \eta'(\tau)$ in a way which can be used by Theorem \ref{thm:UICA_noisy}.
  In what follows, where new parameters are introduced without being described, they will correspond to parameters of the same name defined in and used by the statement of Theorem \ref{thm:UICA_noisy}.
  
  Parameter $d$ is fixed.  We must choose $k_1, \dotsc, k_m$ and $k$ such that $d < k_i \leq k$ and $\cum_{k_i}(S_i')$ is bounded away from $0$.
  It suffices to choose $k_1 = \dotsb = k_m = k = d+1$.
  By Lemma \ref{lem:cumS_i'}, $\cum_{d+1}(S_i') \geq w_{\min} \lambda = w_{\min} m$ for each $i$.  As 
  $w_{\max} \geq \frac 1 m \sum_{i=1}^m w_i = \frac 1 m$, we have that $\cum_{d+1}(S_i') \geq \frac {w_{\min}}{w_{\max}}$ for each $i$, giving a somewhat more natural condition number.  
  In the notation of Theorem \ref{thm:UICA_noisy}, we have a constant
  \begin{equation} \label{eq:DeltaChoice}
    \Delta = \frac{w_{\min}}{w_{\max}}
  \end{equation}
  such that $\cum_{d+1}(S_i') \geq \Delta$ for each $i$.
  
  Now we consider the upper bound $M$ on the absolute moments of both 
  $\eta'(\tau)$ and on $S_i'$.  As the Poisson distribution takes on non-negative
  values, it follows that $S_i' = \norm{\mu_i'}S_i$ takes on non-negative values.
  Thus, the moments and absolute moments of $S_i'$ coincide.
  Using Lemma \ref{lem:expectation_S_i'}, we have that $\E{\abs{S_i'}^{d+1}}
  = \E{(S_i')^{d+1}} \leq (\norm{\mu_i'}w_i\lambda)^{d+1} (d+1)^{d+1}$.  Thus,
  for $M$ to bound the $(d+1)$\textsuperscript{th} moment of $S_i'$, it suffices
  that $M \geq (\norm{\mu_{\max}'}w_{\max} \lambda)^{d+1} (d+1)^{d+1}$.
  Noting that 
  \[
    w_{\max} \lambda = w_{\max} m = \frac{w_{\max}}{1/m} \leq \frac{w_{\max}}{w_{\min}}
  \]
  it suffices that $M \geq (\norm{\mu_{\max}'}\frac{w_{\max}}{w_{\min}})^{d+1} (d+1)^{d+1}$,
  giving a more natural condition number.
  
  Now we bound the absolute moments of the Gaussian distribution.  
  As $d \in 2\N$, it follows that $d+1$ is odd.
  Given a unit vector $u \in \R^n$, it follows from Lemma \ref{lem:GaussianMoments} that 
  \begin{align*}
    \E{\abs{\langle u, \eta'(\tau) \rangle}^{d+1}}
      &= \Var(\langle u, \eta'(\tau)\rangle)^{\frac {(d+1)} 2} 2^{d/2}\left(d / 2\right) ! \frac 1 {\sqrt \pi} \\ &= \tau^{d+1}\Var(\langle u, \eta'(1) \rangle)^{\frac{(d+1)}2} 2^{d/2}\left(d / 2\right) ! \frac 1 {\sqrt \pi} \ .
    \end{align*}
  $\sigma$ gives a clear upper bound for $\Var(\langle u, \eta'(1) \rangle)^{1/2}$, and $(d+1)^{d+1}$ 
  gives a clear upper bound to $\frac 1 {\sqrt \pi} 2^{d/2}(d / 2)!$. 
  As such, it suffices that $M \geq (\tau \sigma)^{d+1}(d+1)^{d+1}$ in order to
  guarantee that $M \geq \E{\abs{\langle u, \eta'(\tau) \rangle}^{d+1}}$.
  Using the obtained bounds for $M$ from the Poisson and Normal variables,
  it suffices that $M$ be taken such that 
  \begin{equation} \label{eq:MChoice}
        M \geq \max \left((\tau \sigma)^{d+1}, (\norm{\mu_{\max}'} \frac{w_{\max}}{w_{\min}})^{d+1}\right) (d+1)^{d+1}
  \end{equation}
  to guarantee that $M$ bounds all required order $d+1$ absolute moments.
  
  We can now apply Theorem \ref{thm:UICA_noisy}, using the parameter values $k = d+1$,
  $\Delta$ from \eqref{eq:DeltaChoice}, and $M$ from \eqref{eq:MChoice}.
  Then with probability $1 - \delta$,
  \begin{align}
    N &\geq \poly\left(n^{2d+1}, m^{d^2}, 
      (\tau \sigma)^{d^2}, \norm{\mu_{\max}'}^{d^2}, ({w_{\max}}/{w_{\min}})^{d^2},
      (d+1)^{d^2}, \right. \notag \\
      & \quad \quad \quad \quad \left.
      1/\sigma_m(A'^{\odot d/2})^{d+1}, 1/\epsilon, 1/\delta \right) \label{eq:N_poly_with_mu_prime}
  \end{align} 
  samples suffice to recover up to sign the columns of $A'$ within $\epsilon$ accuracy.  
  More precisely, letting 
  $\tilde A_1', \dotsc, \tilde A_m'$ give the columns produced by 
   \textbf{UnderdeterminedICA}, then there exists parameters $\alpha_1, \dotsc, \alpha_m$ such that $\alpha_i \in \{-1, +1\}$ captures the sign 
  indeterminacy, and a permutation $p$ on $[m]$ such that $\norm{A_i' - \tilde A_{p(i)}'} < \epsilon$ for each $i$.
  
  The poly bound in \eqref{eq:N_poly_with_mu_prime} is equivalent to the poly bound in \eqref{eq:poly_bound_lifted_recovery}.
\end{proof}

Theorem \ref{thm:ICA_ideal_lifted} allows us to recover the columns of $A'$ up
to sign.  However, what we really want to recover are the means of the original
Gaussian mixture model, which are the columns of $A$.  Recalling the
correspondence between $A'$ and $A$ laid out in section \ref{sec:gmm-reduction},
the Gaussian means $\mu_1, \dotsc, \mu_m$ which form the columns of $A$ are
related to the columns $\mu_1', \dotsc, \mu_m'$ of $A'$ by the rule 
$\mu_i = \mu_i'(1:n) / \mu_i'(n+1)$.  Using this rule, we can construct estimate
the Gaussian means from the estimates of the columns of $A'$.  By propogating
the errors from Theorem \ref{thm:ICA_ideal_lifted}, we arrive at the following
result:

\begin{theorem}[Recovery of Gaussian means in Ideal Case] \label{thm:MeanRecoveryIdeal}
  Suppose that \textbf{UnderdeterminedICA} is run using samples of $X'$ from the
  ideal noisy ICA model \eqref{A'_ICAModel_ideal} choosing parameters $\lambda = m$ and $\tau$ a constant.
  Define $B\in \R^{n \times m}$ such that $B_i = A_i / \norm{A_i}$.  
  Suppose further that $\sigma_m(B^{\odot d/2}) > 0$.  Let 
  $\tilde A_1', \cdots, \tilde A_m'$ be the returned estimates of the columns of $A'$ (from model \eqref{A'_ICAModel_ideal}) by \textbf{UnderdeterminedICA}.
  Let $\tilde \mu_i = \tilde A_i'(1:n) / \tilde A_i'(n+1)$ for each $i$.
  Fix error parameters $\epsilon \in (0, 1/2)$ and $\delta \in (0, 1/2)$.
  When at least
  \begin{align}
    N & \geq \\ & \poly\left(n^{d}, m^{d^2}, 
      (\tau \sigma)^{d^2}, \norm{\mu_{\max}}^{d^2}, \left(\frac{w_{\max}}{w_{\min}}\right)^{d^2},
      d^{d^2}, 
      \left(\frac{\norm{\mu_{\max}} + 1}{\norm{\mu_{\min}}}\right)^{d^2}, \frac 1 {\sigma_m(B^{\odot d/2})^{d}}, \frac 1 \epsilon, \frac 1 \delta \right)
    \label{eq:poly_bound_ideal_means}
  \end{align}
  samples are used, then with probability $1 - \delta$ there exists a permutation
  $p$ of $[m]$ such that $\norm{\tilde \mu_{p(i)} - \mu_i} < \epsilon$ for each $i$.
\end{theorem}
\begin{proof}
  Let $\epsilon^* > 0$ (to be chosen later) give a desired bound on the 
  errors of the columns of $A'$.  Then, from Theorem \ref{thm:ICA_ideal_lifted},
  using
  \begin{align}
    N &\geq \poly\left(n^{d}, m^{d^2}, 
      (\tau\sigma)^{d^2}, \norm{\mu_{\max}'}^{d^2}, \left({w_{\max}}/{w_{\min}}\right)^{d^2},
      d^{d^2}, 
      1/\sigma_m(A'^{\odot d/2})^{d}, 1/\epsilon^*, 1/\delta \right) \label{eq:sample_bound_for_err_prop0}
  \end{align}
  samples suffices with probability $1 - \delta$ to produce column estimates $\tilde A_1', \dotsc, \tilde A_m'$ such
  that for an unknown permutation $p$ and signs $\alpha_1, \dotsc, \alpha_m$,
  $\alpha_{p(1)}\tilde A_{p(1)}', \dotsc, \alpha_{p(m)}\tilde A_{p(m)}'$ give
  $\epsilon^*$-close estimates of the columns $A_1', \dotsc, A_m'$ 
  respectively of $A'$.
  In order to avoid notational clutter, we will assume without
  loss of generality that $p$ is the identity map, and hence that 
  $\norm{\alpha_i\tilde A_i' - \alpha A_i'} < \epsilon^*$ holds.
  
  This proof proceeds in two steps.  
  First, we replace the dependencies in \eqref{eq:sample_bound_for_err_prop0} on parameters from the lifted GMM model generated by the full reduction with dependencies based on the GMM model we are trying to learn.
  Then, we propagate the error from recovering the columns $\tilde A_i'$ to that of recovering $\tilde \mu_i$.
  
  \paragraph{Step 1:  GMM Dependency Replacements.}
  In the following two claims, we consider alternative lower bounds for $N$ for recovering column estimators $\tilde A_1', \dotsc, \tilde A_m'$ which are $\epsilon^*$-close up to sign to the columns of $A'$.
  In particular, so long as we use at least as many samples of $X'$ as in \eqref{eq:sample_bound_for_err_prop0} when calling \textbf{UnderdeterminedICA}, then $A'$ will be recovered with the desired precision with probability $1 - \delta$.
\begin{claim*}
The $\poly(\norm{\mu_{\max}'}^{d^2}, d^{d^2})$ dependence in \eqref{eq:sample_bound_for_err_prop0} can be replaced by a \\ $\poly(\norm{\mu_{\max}}^{d^2}, \allowbreak d^{d^2})$ dependence.
 \end{claim*}
\begin{proof}[Proof of Claim.] \renewcommand{\qedsymbol}{$\blacktriangle$}
    By construction, $\mu_{\max}' = \left( \begin{array}{c} \mu_{\max} \\ 1 \end{array} \right)$.
    By the triangle inequality,
    \[
      \norm{\mu_{\max}'}^{d^2} \leq (\norm{\mu_{\max}} + 1)^{d^2}
    \]
    where $(\norm{\mu_{\max}} + 1)^{d^2}$ is a polynomial $q$ of $\norm{\mu_{\max}}$ with coefficients bounded by $(d^2)^{d^2} = d^{2d^2} = \allowbreak \poly(d^{d^2})$.  
   The maximal power of $\norm{\mu_{\max}}$ in $q(\norm{\mu_{\max}})$ is $d^{d^2}$. It follows that $q(\norm{\mu_{\max}}) = \poly(\norm{\mu_{\max}}^{d^2}, d^{d^2})$. 
\end{proof}
\begin{claim*} \renewcommand{\qedsymbol}{}
The $\poly(1/\sigma_m(A'^{\odot d/2})^d)$ in \eqref{eq:sample_bound_for_err_prop0} can be replaced by a $\poly((\frac{\norm{\mu_{\max}} + 1}{\norm{\mu_{\min}}})^{d^2}, \ \allowbreak 1/\sigma_m(B^{\odot d/2})^{d})$ dependence.
\end{claim*}
\begin{proof}[Proof of Claim]\renewcommand{\qedsymbol}{$\blacktriangle$}    
    First define $\underbar A'$ to be the unnormalized version of $A'$.  
    That is, $\underbar A'_i := \mu_i'$.
    Then, $\underbar A' = A' \diag{\norm{\mu_1'}, \dotsc, \norm{\mu_m'}}$ implies $\underbar A'^{\odot d/2} = A'^{\odot d/2} \diag{ \norm{\mu_1'}^{d/2}, \dotsc \norm{\mu_m'}^{d/2}}$.
    Thus, $\sigma_m(\underbar A'^{\odot d/2}) \leq \sigma_m(A'^{\odot d/2})\norm{\mu_{\max}'}^{d/2}$.
    
    Next, we note that $\underbar A' = \left( \begin{array}{c} A \\ \mathbf 1 \end{array}\right)$ where $\mathbf 1$ is an all ones row vector.
    It follows that the rows of $A^{\odot d/2}$ are a strict subset of the rows of $\underbar A'^{\odot d/2}$.  Thus,
    \[
      \sigma_m(A^{\odot d/2}) 
      	= \inf_{\norm{u} = 1} \norm{A^{\odot d/2}u}
      	\leq \inf_{\norm{u} = 1} \norm{\underbar A'^{\odot d/2}u}
      	= \sigma_m(\underbar A'^{\odot d/2}) \ .
    \]
    
    Finally, $B = A \diag{\left ( \frac 1 {\norm{\mu_1}}, \dotsc, \frac 1 {\norm{\mu_m}}\right )}$ and $B^{\odot d/2} = A^{\odot d/2} \diag{\left ( \frac 1 {\norm{\mu_1}^{d/2}}, \dotsc, \frac 1 {\norm{\mu_m}^{d/2}}\right )}$.
    It follows that $\sigma_m(B^{\odot d/2}) \leq \sigma_m(A^{\odot d/2})\frac 1 {\norm{\mu_{\min}}^{d/2}}$.
    Chaining together inequalities yields:  
    \begin{align*}
      \sigma_m(B^{\odot d/2}) &\leq \frac{\norm{\mu_{\max}'}^{d/2}}{\norm{\mu_{\min}}^{d/2}} \sigma_m(A'^{\odot d/2}) & \text{or alternatively} & & \\ & & \frac{\norm{\mu_{\max}'}^{d/2}}{\norm{\mu_{\min}}^{d/2}} \cdot \frac 1 {\sigma_m(B^{\odot d/2})} &\geq \frac 1 {\sigma_m(A'^{\odot d/2})} \ .
    \end{align*}
    As $\mu_{\max}' = (\mu_{\max}^T\ 1)^T$, the triangle inequality implies $\norm{\mu'_{\max}} \leq \norm{\mu_{\max}} + 1$.
    As we require the dependency of at least $N > \poly((1/\sigma_m(A'^{\odot d/2}))^d)$ samples, it suffices to have the replacement dependency of $N > \poly((\frac{\norm{\mu_{\max}} + 1}{\norm{\mu_{\min}}})^{\frac d 2 \cdot d}(1/\sigma_m(B^{\odot d/2})^d) = \allowbreak \poly((\frac{\norm{\mu_{\max}} + 1}{\norm{\mu_{\min}}})^{d^2}(1/\sigma_m(B^{\odot d/2})^d)$ samples. 
\end{proof}
  
  Thus, it is sufficient to call \textbf{UnderdeterminedICA} with
  \begin{align}
    N &\geq \\ & \poly\left(n^{d}, m^{d^2}, 
      (\tau\sigma)^{d^2}, \norm{\mu_{\max}}^{d^2}, \Big(\frac {w_{\max}} {w_{\min}}\Big)^{d^2},
      d^{d^2}, 
      \left(\frac {{\norm{\mu_{\max}} + 1}}{\norm{\mu_{\min}}}\right)^{d^2}, \frac 1{\sigma_m(B^{\odot d/2})^{d}}, \frac 1 {\epsilon^*}, \frac 1 \delta \right) \label{eq:sample_bound_for_err_prop}
  \end{align}
  samples to achieve the desired $\epsilon^*$ accuracy on the returned estimates of the columns of $A'$ with probability $1-\delta$.

  \paragraph{Step 2:  Error propagation.}
  
    What remains to be
  shown is that an appropriate choice of $\epsilon^*$ enforces 
  $\norm{\mu_i - \tilde \mu_i} < \epsilon$ by propagating the error.
  
  Recall that $A_i' = \left(\begin{array}{c}\mu_i \\ 1\end{array}\right) \cdot \norm{\left(\begin{array}{c}\mu_i \\ 1\end{array}\right)}^{-1}$, making
  $A_i'(n+1) = \frac 1 {\sqrt{1 + \norm{\mu_i}^2}}$.  Thus,
  \begin{align} 
    A_i'(n+1) & \geq \frac 1 {\sqrt{1 + \norm{\mu_{\max}}^2}}
         \ . \label{eq:ineq1}
  \end{align}
  
  We have that:
\begin{align*}
  \norm{\mu_i - \tilde \mu_i} 
  &= \norm{\frac{A_i'(1:n)}{A_i'(n+1)} - \frac{\tilde A_i'(1:n)}{\tilde A_i'(n+1)}} \\
  &= \norm{\frac{A_i'(1:n)}{A_i'(n+1)} - \frac{\alpha_i\tilde A_i'(1:n)}{A_i'(n+1)} + \frac{\alpha_i\tilde A_i'(1:n)}{A_i'(n+1)} - \frac{\alpha_i\tilde A_i'(1:n)}{\alpha_i\tilde A_i'(n+1)}} \\
  & \leq \frac{\norm{A_i'(1:n)-\alpha_i\tilde A_i'(1:n)}}{\abs{A_i'(n+1)}} +
\frac{\norm{\tilde A_i'(1:n)}\abs{\alpha_i\tilde A_i'(n+1)-A_i'(n+1)}}{\abs{A_i'(n+1)\alpha_i\tilde A_i'(n+1)}} \\
  & \leq \epsilon^* \sqrt{1 + \norm{\mu_{\max}}^2} + \frac {\abs{\alpha_i\tilde A_i'(n+1)-A_i'(n+1)}}{\abs{A_i'(n+1)}[\abs{A_i'(n+1)} - \abs{\alpha_i\tilde A_i'(n+1) - A_i'(n+1)}]}
\end{align*}
which follows in part by applying \eqref{eq:ineq1} for the left summand and noting that $\tilde A_i'$ is a unit vector for the right summand, giving the bound $\norm{\tilde A_i'(1:n)} \leq 1$.  Continuing with the restriction that $\epsilon^* < \frac 1 2 \frac 1 {\sqrt{1 + \norm{\mu_{\max}}^2}}$,
\begin{align*}
  \norm{\mu_i - \tilde\mu_i} 
  & \leq \epsilon^* \sqrt{1 + \norm{\mu_{\max}}^2} + \frac {\epsilon^*\sqrt{1 + \norm{\mu_{\max}}^2}}{\left[\frac 1 {\sqrt{1 + \norm{\mu_{\max}}^2}} - \epsilon^*\right]} \\
  & \leq \epsilon^* \left(\sqrt{1 + \norm{\mu_{\max}}^2} + 2 (1 + \norm{\mu_{\max}}^2)\right).
\end{align*}

Then, in order to guarantee that $\norm{\mu_i - \tilde\mu_i} < \epsilon$, it suffices to choose $\epsilon^*$ such that 
\[
  \epsilon^*\left(\sqrt{1 + \norm{\mu_{\max}}^2} + {2 (1 + \norm{\mu_{\max}}^2)}\right) \leq \epsilon, 
\]
which occurs when
\begin{equation} \label{eq:eps*choice}
  \epsilon^* \leq \frac {\epsilon}{\left(\sqrt{1 + \norm{\mu_{\max}}^2} + {2 (1 + \norm{\mu_{\max}}^2)}\right)}  \ .
\end{equation}
As $\epsilon < \frac 1 2$, the restriction $\epsilon^* < \frac 1 2 \sqrt{1 + \norm{\mu_{\max}^2}}$ holds automatically for the choice of $\epsilon^*$ in \eqref{eq:eps*choice}. 
The sample bound from \eqref{eq:sample_bound_for_err_prop} contains the dependency $N > \poly(\frac 1 {\epsilon^*}, \norm{\mu_{\max}}^{d^2})$.  Propagating the error gives a replacement dependency of \[N > \poly\left(\frac 1 \epsilon, \sqrt{1 + \norm{\mu_{\max}}^2}, \norm{\mu_{\max}}^{d^2}\right) = \poly( \frac 1 \epsilon, \norm{\mu_{\max}}^{d^2})\] as $d$ is non-negative.  This propagated dependency is reflected in \eqref{eq:poly_bound_ideal_means}.
\end{proof}


\paragraph{Distance of the Sampled Model to the Ideal Model}
\label{subsec:TotalVarDist}

An important part of the reduction is that the coordinates of $S$ are mutually independent. Without the threshold $\tau$, this is true (c.f. Lemma \ref{lem:Poisson-independence}).
However, without the threshold, one cannot know how to add more noise so that the total noise on each sample is iid.
We show that we can choose the threshold $\tau$ large enough that the samples still come from a distribution with arbitrarily small total variation distance to the one with truly independent coordinates.

\begin{lemma} \label{lem:poisson-threshold}
Fix $\delta > 0$. Let $S \sim \Poisson(\lambda)$ for $\lambda \geq \ln \delta$.  Let $b = e \lambda$, If $\tau > e \lambda$, $\tau \geq 1$, and $\tau \geq \ln(1/\delta) - \lambda$, then $\prob{S > \tau} < \delta$.
\end{lemma}

\begin{proof}
By the Chernoff bound (See Theorem A.1.15 in \cite{alon2004probabilistic}),
\[
  \prob{S > \lambda(1+\epsilon)} \leq \left( e^\epsilon (1+\epsilon)^{-(1+\epsilon)} \right)^\lambda.
\]
For any $\tau > \lambda$, letting $\epsilon = \tau/\lambda - 1$, we get $$\prob{S > \tau} \leq \frac{e^{-\lambda} (e\lambda)^{\tau}}{\tau^\tau}.$$
To get $\prob{S > \tau} < \delta$, it suffices that $\tau - \tau \log_b \tau \leq \log_b(\delta e^\lambda)$. Note that $$\tau(1 - \log_b \tau) = \tau - \tau \log_b \tau =
\log_b \big(b^{\tau} (1/\tau)^{\tau}\big).$$ If $\tau - \tau \log_b \tau \leq \log_b \left( \delta e^{\lambda} \right)$, then we have
\begin{align*}
\log_b \big(b^{\tau} (1/\tau)^{\tau}\big) &\leq \log_b(\delta e^{\lambda})
\end{align*}
which then implies it suffices that
\begin{align*}
\frac{b^{\tau}}{\tau^{\tau}} = \frac{(e\lambda)^\tau}{\tau^\tau} \leq \lambda^\tau / \tau^\tau \leq (1/e)^\tau &\leq \delta e^{\lambda}
\end{align*}
which holds for $\tau \geq \ln \left( \frac{1}{\delta e^\lambda} \right) = \ln(1/\delta) - \lambda$, giving the desired result.

\end{proof}

\begin{lemma} \label{lem:union-bound}
Let $N, \delta > 0$, $N \in \N$, and $T_1, T_2, \dots, T_N$ be iid with distribution $\Poisson(\lambda)$.
If $\tau \geq \ln (N/\delta) - \lambda$ then $$\prob{ \bigcup_{i} \left\{ T_i > \tau \right\}} < \delta.$$
\end{lemma}

\begin{proof}
By Lemma \ref{lem:poisson-threshold} $\tau \geq \ln (N/\delta) - \lambda$ implies $\prob{T_i > \tau} < \delta/N$ for every $i$.
The union bound gives us the desired result.
\end{proof}

We have now that if we choose our threshold $\tau$ large enough, our samples can be statistically close (See Appendix \ref{app:total-variation}) to ones that would come from the truly independent distribution.
This claim is made formal now:

\begin{lemma} \label{lem:total-variation}
Fix $\delta > 0$.
Let $\tau > 0$.
Let $F$ be a poisson distribution with parameter $\lambda$ and have corresponding density $\density_F$.
Let $G$ be a discrete distribution with density $\density_G(x) = \density_F(x) / F(\tau)$ when $0 \leq x \leq \tau$ and 0 otherwise.
Then $\d_{TV}(F,G) = 1 - F(\tau)$.
\end{lemma}

\begin{proof}
Since we are working with discrete distributions, we can write $$d_{TV}(F,G) = \frac{1}{2} \sum_{i=0}^{\infty}| f(i) - g(i) |.$$
Then we can compute
\begin{align*}
\d_{TV}(F,G) &= \frac{|F(\tau) - 1|}{2F(\tau)} \sum_{i=0}^{\tau} \density_F(i) + \frac{1}{2} \sum_{i=\tau+1}^{\infty} \density_F(i) = \frac{|F(\tau) - 1|}{2} + \frac{1 - F(\tau)}{2} = 1 - F(\tau).
\end{align*}
\end{proof}

\paragraph{Proof of Theorem \ref{thm:gmm-correctness}}
\label{subsec:CorrectnessProof}

We now show that after the reduction is applied, we can use the ICA routine given in \cite{GVX} to learn the GMM.
Instead of requiring exact values of each parameter, we simply require a bound on each.
The algorithm remains polynomial on those bounds, and hence polynomial on the true values.

\begin{proof}
The algorithm is provided parameters:
Covariance matrix $\Sigma$,
upper bound on tensor order parameter $d$,
access to samples from a mixture of $\means$ identical,
spherical Gaussians in $\R^{\dim}$ with covariance $\Sigma$,
confidence parameter $\delta$,
accuracy parameter $\epsilon$,
upper bound $w \geq \max_{i}(w_i) / \min_{i}(w_i)$,
upper bound on the norm of the mixture means $u$,
lower bound $v$ so $0 < b \leq \sigma_m(A^{\odot d/2})$, and
$r \geq \big(\max_i\norm{\mu_i} + 1)/(\min_i\norm{\mu_i})\big)$.

The algorithm then needs to fix the number of samples $N$, sampling threshold $\tau$, Poisson parameter $\lambda$, and two new errors $\delta_1$ and $\delta_2$ so that $\delta_1 + \delta_2 \leq \delta$.
For simplicity, we will take $\delta_1 = \delta_2 = \delta/2$. Then fix $\sigma = \sup_{v \in S^{n-1}}\sqrt{\Var(v^T\eta(1))}$ for $\eta(1) \sim \Norm(0, \Sigma)$.
Recall that $B$ is the matrix whose $i$th column is $\mu_i/\norm{\mu_i}$. Let $A'$ be the matrix whose $i$th column is $(\mu_i, 1) / \norm{(\mu_i, 1)}$.

\item \paragraph{Step 1}
Assume that after drawing samples from Subroutine \ref{sub:independentsamples}, the signals $S_i$ are mutually independent (as in the ``ideal'' model given by (\ref{A'_ICAModel_ideal})) and the mean matrix $B$ satisfies $\sigma_m(B^{\odot d/2}) \geq b > 0$.
Then by Theorem \ref{thm:MeanRecoveryIdeal}, with probability of error $\delta_1$, the call to \textbf{UnderdeterminedICA} in Algorithm \ref{alg:reduction} recovers the columns of $B$ to within $\epsilon$ and up to a permutation using $N$ samples where
\begin{align*} N & \geq p\left(\tau^{d^2}, \Theta \right) = \poly\left(n^{d}, m^{d^2},
      (\tau \sigma)^{d^2}, u^{d^2}, w^{d^2},
      d^{d^2}, r^{d^2}, 1/b^{d}, 1/\epsilon, 1/\delta_1 \right)
\end{align*}
where $p(\tau^{d^2}, \Theta)$ is the bound on $N$ promised by Theorem \ref{thm:MeanRecoveryIdeal} and $\Theta$ is all its arguments except the dependence in $\tau$.
So then we have that with at least $N$ samples in this ``ideal'' case, we can recover approximations to the true means in $\R^\dim$ up to a permutation and within $\epsilon$ distance.

\item \paragraph{Step 2}
We need to show that after getting $N$ samples from the reduction, the resulting distribution is still close in total variation to the independent one.
We will choose a new $\delta' = \delta_2/(2N)$.
Let $R \sim \Pois(\lambda)$.
Given $\delta'$, Lemma \ref{lem:total-variation} shows that for $\tau \geq \ln(1/\delta') - \lambda$, with probability $1-\delta'$, $R \leq \tau$.

Take $N$ iid random variables $X_1, X_2, \dots, X_N$ with distribution $F = \Poisson(\lambda)$ and density $\density_F$.
Let $G$ be a distribution with density function $\density_G(x) = (\density_F(x) \ind_{0 \leq x \leq \tau})/F(\tau)$.
Let $Y_1, Y_2, \dots, Y_N$ be iid random variables with distribution $G$.
Denote the joint distribution of the $X_i$'s by $F'$ \details{with density $\density_{F'}$,} and the joint distribution of the $Y_i$'s as $G'$ \details{with density $\density_{G'}$}.
By the union bound and the fact that total variation distance satisfies the triangle inequality, $$\d_{TV}(F', G') \leq \sum_{i=1}^{N} \d_{TV}(F,G) = N d_{TV}(F,G).$$
Then for our choice of $\tau$, by Lemma \ref{lem:poisson-threshold} and Lemma \ref{lem:total-variation}, we have $$\d_{TV}(F',G') \leq N d_{TV}(F, G) = N \prob{X_1 > \tau} \leq N\delta' = \delta_2/2. $$

By the same union bound argument, the probability that the algorithm fails (when $R > \tau$) is at most $\delta_2/2$, since it has to draw $N$ samples.
So with high probability, the algorithm does not fail; otherwise, it still does not take more than polynomial time, and will terminate instead of returning a false result.

\item \paragraph{Step 3}
We know that $N$ is at least a polynomial which can be written in terms of the dependence on $\tau$ as $p(\tau^{d^2}, \Theta)$.
This means there will be a power of $\tau$ which dominates all of the $\tau$ factors in $p$, and in particular, will be $\tau^{Cd^2}$ for some $C$.
It then suffices to choose $C$ so that $p\left(\tau^{d^2}, \Theta \right) \leq \tau^{Cd^2} q(\Theta) \leq N$, where
\begin{align} q(\Theta) & = \poly\left(n^{d}, m^{d^2},
      \sigma^{d^2}, u^{d^2}, w^{d^2},
      d^{d^2}, r^{d^2}, 1/b^{d}, 1/\epsilon, 1/\delta_1\right) \label{eq:q-theta}.
\end{align}
Then, with the proper choice of $\tau$ (to be specified shortly), from step 2 we have
$$p\left(\tau^{d^2}, \Theta \right) \leq \tau^{Cd^2}q(\Theta) \leq N = \frac{\delta_2}{\delta'} \leq \frac{\delta_2 \tau^\tau e^\lambda}{(e\lambda)^\tau} = \frac{\delta \tau^\tau e^\lambda}{2(e\lambda)^\tau}.$$
Since $\lambda \geq 1$ it suffices to choose $\tau$ so that
\begin{equation} \label{eq:tau-sample-bound}
\frac{2}{\delta}q(\Theta)\tau^{Cd^2} \leq \frac{\tau^\tau}{\tau^{Cd^2}(e\lambda)^\tau}.
\end{equation}
Finally, we claim that
$$
\tau = 4\big(\log(2/\delta) + \log(q(\Theta))\big)\max\left((e\lambda)^2, 4Cd^2\right) = O\left((\lambda^2 + d^2)\log \frac{q(\Theta)}{\delta}\right)
$$
is enough for the desired bound on the sample size. Observe that $4(\log(2/\delta) + \log(q(\Theta))) \geq 1$.

An useful fact is that for general $x, a, b \geq 1$, $x \geq \max(2a, b^2)$ satisfies $x^a \leq x^x/b^x$.
This captures the essence of our situation nicely.
Letting $e\lambda$ play the role of $b$, $Cd^2$ play the role of $a$ and $x$ play the role of $\tau$, to satisfy (\ref{eq:tau-sample-bound}), it suffices that
\begin{align*}
\frac{2}{\delta}q(\Theta) &\leq \frac{\tau^{\tau/2} \tau^{\tau/4} \tau^{\tau/4}}{\tau^{Cd^2}(e\lambda)^2}.
\end{align*}
We can see that $\tau^{\tau/2} \geq (e\lambda)^2$ and $\tau^{\tau/4} \geq \tau^{Cd^2}$ by construction.
But we also get $\tau/4 \geq \log(2/\delta) + \log{q(\Theta)}$ which implies $\tau^{\tau/4} \geq e^{\tau/4} \geq \frac{2}{\delta} q(\Theta)$.
Thus for our choice of $\tau$, which also preserves the requirement in Step 2, there is a corresponding set of choices for $N$, where the required sample size remains polynomial as
\begin{align*}
& \poly\left(n^{d}, m^{d^2},
      (\tau \sigma)^{d^2}, u^{d^2}, w^{d^2},
      d^{d^2}, r^{d^2}, 1/b^{d}, 1/\epsilon, 1/\delta \right)
\end{align*}
where we used the bound $q(\Theta) \leq (n^d m^{d^2} \sigma^{d^2} u^{d^2} w^{d^2} (d+1)^{d^2} r^{d^2} / b^{d} \delta_1 \epsilon)^{O(1)}$.
\end{proof}

\section{Smoothed Analysis} \label{sec:smoothed}

In this section, we prove that in high enough dimension, a Gaussian Mixture will satisfy the non-degeneracy conditions discussed above.
We start with a base matrix $M \in \R^{n \times \binom{n}{2}}$ and add a perturbation
matrix $N \in \R^{n \times \binom{n}{2}}$ with each entry coming iid from $\Normal(0, \sigma^2)$ for some 
$\sigma > 0$. (We restrict the discussion to the second power for simplicity; extension to higher power is
straightforward.) As in \cite{GVX}, it will be convenient to work with the multilinear part of the Khatri--Rao 
product: For a column vector $A_k \in \R^n$ define $A_k^{\ominus 2} \in \R^{\binom{n}{2}}$, a subvector of 
$A_k^{\odot 2} \in \R^{n^2}$, given by $(A_k^{\ominus 2})_{ij} := (A_k)_i (A_k)_j$ for $1 \leq i < j \leq n$. Then 
for a matrix $A = [A_1, \ldots, A_m]$ we have $A^{\ominus 2} := [A_1^{\ominus 2}, \ldots, A_m^{\ominus 2}]$. 

\begin{theorem} \label{thm:smoothed-sigma-min-multilinear}
With the above notation, for any base matrix $M$ with dimensions as above, we have, for some absolute constant $C$,
\begin{vstd}
\begin{align*}
\prob{\sigma_{\min}((M+N)^{\ominus 2}) \leq \frac{\sigma^2}{n^7}} \leq \frac{2C}{n}.
\end{align*}
\end{vstd}
\end{theorem}

Theorem~\ref{thm:smoothed-sigma-min-intro} follows by noting
that $\sigma_{\min}(A^{\odot 2}) \geq \sigma_{\min}(A^{\ominus 2})$. 

\begin{proof}
In the following, for a vector space $V$ (over the reals) $\dist(v, V')$ denotes the distance between vector
$v \in V$ and subspace $V' \subseteq V$; more precisely, $\dist(v, V') := \min_{v' \in V'} \norm{v-v'}_2$.
We will use a lower bound on $\sigma_{\min}(A)$, found in Appendix \ref{subsec:rudelson-vershynin}.

With probability $1$, the columns of the matrix $(M+N)^{\ominus 2}$ are linearly 
independent. This can be proved along the lines of a similar result in \cite{GVX}. 
Fix $k \in {\binom{n}{2}}$ and let $u \in \R^{{\binom{n}{2}}}$ be a 
unit vector orthogonal to the subspace spanned by the columns of $(M+N)^{\ominus 2}$ other than column $k$. Vector
$u$ is well-defined with probability $1$. Then the distance of the $k$'th column $C_k$ from the span of the 
rest of the columns is given by 
\begin{align}
u^T C_k &= u^T (M_k+N_k)^{\ominus 2} \nonumber = \sum_{1 \leq i < j \leq n} u_{ij} (M_{ik}+N_{ik})(M_{jk}+N_{jk}) \nonumber \\
&=  \sum_{1 \leq i < j \leq n} u_{ij} M_{ik}M_{jk} + \sum_{1 \leq i < j \leq n} u_{ij}M_{ik}N_{jk} 
\\ & \quad + \sum_{1 \leq i < j \leq n} u_{ij}N_{ik}M_{jk} + \sum_{1 \leq i < j \leq n} u_{ij}N_{ik}N_{jk} \nonumber \\
&=: P(N_{1k}, \ldots, N_{nk}). \label{eqn:polynomial}
\end{align}

Now note that this is a quadratic polynomial in the random variables $N_{ik}$. We will apply the anticoncentration
inequality of \processifversion{vstd}{Carbery--Wright~}\cite{CarberyWright} to this polynomial to conclude that the distance between 
the $k$'th column of $(M+N)^{\ominus 2}$ and the span of the rest of the columns is unlikely to be very small (see Appendix \ref{subsec:carbery-wright} for the precise result).

Using $\norm{u}_2 = 1$, the variance of our polynomial in \eqref{eqn:polynomial} becomes 
\begin{align*}
\Vr{P(N_{1k}, \ldots, N_{nk})} 
&= \\ & \sigma^2 \biggl(\sum_j \Bigl(\sum_{i:i < j} u_{ij}M_{ik}\Bigr)^2 + \sum_i \Bigl(\sum_{j: i < j} u_{ij}M_{jk}\Bigr)^2 \biggr) 
+ \sigma^4 \sum_{i < j} u_{ij}^2  \geq \sigma^4.
\end{align*}
In our application, random variables $N_{ik}$ for $i \in [n]$ are not standard
Gaussians but are iid Gaussian with variance $\sigma^2$, and our polynomial does not have unit variance. 
After adjusting for these differences using 
the estimate on the variance of $P$ above, Lemma~\ref{lem:CarberyWright} gives 
$
\prob{\abs{P(N_{1k}, \ldots, N_{nk})-t} \leq \epsilon} \leq 
2C \sqrt{{\epsilon}/{\sigma^2}} = 2C \sqrt{\epsilon}/\sigma$.
Therefore, $\prob{\exists k \mathrel{:} \dist(C_k, C_{-k}) \leq \epsilon} 
\leq {\binom{n}{2}} 2C  \sqrt{\epsilon}/\sigma$ by the union bound over the choice of $k$.

Now choosing $\epsilon = \sigma^2/n^6$, Lemma~\ref{lem:RudelsonVershynin} gives
$\prob{\sigma_{\min}((M+N)^{\ominus 2}) \leq {\sigma^2}/{n^7}} \leq {2C}/{n}$. 
\end{proof}

We note that while the above discussion is restricted to Gaussian perturbation, the same technique would work
for a much larger class of perturbations. To this end, we would require a version of the Carbery-Wright
anticoncentration inequality which is applicable in more general situations. We omit such generalizations here.


\subsubsection{The curse of low dimensionality for Gaussian mixtures}
\label{identifiability_low_dimension}
\lnote{domain of gaussian kernel, H. shouldn't the kernel be in Rn instead of R?}
In this section we prove that for small $n$ there is a large class of superpolynomially close mixtures in $\R^n$ with fixed variance.  This goes beyond the specific example of exponential closeness  given in~\cite{moitra10} as we demonstrate that such mixtures are ubiquitous as long as there is no lower bound on the separation between the components.

Specifically, let $S$ be a cube $[0,1]^n \subset \R^n$. 
We will show that for any two sets of $k$ points $X$ and $Y$ in $S$, with fill $h$ (we say that $X$ has fill $h$, if there is a point of $X$ within distance $h$ of any point of $S$), there exist two mixtures $p,q$ with means on non-overlapping subsets of $X\cup Y$, which are (nearly) exponentially close in $1/h$ in the $L^1(\R^n)$ norm. 
Note that the fill of a sample from the uniform distribution on the cube can be bounded (with high probability) by $O(\frac{\log k}{k^{1/n}})$. 

We start by defining some of the key objects.  Let $K(x,z)$ be the unit Gaussian kernel.
Let ${\cal K}$ be the integral operator corresponding to the convolution with a unit Gaussian: ${\cal K}g(z)= \int_\R K(x,z)g(x) dx$.
Let $X$ be any subset of $k$ points in $[0,1]^n$.
Let $K_X$ be the kernel matrix corresponding to $X$, $(K_X)_{ij} = K(x_i,x_j)$. 
It is known to be positive definite.
The \emph{interpolant} is defined as $f_{X,k}(x) = \sum w_i K(x_i,x)$, where the coefficients $w_i$ are chosen so that  $(\forall i) f_{X,k}(x_i)=f(x_i)$. 
It is easy to see that such interpolant exists and is unique, obtained by solving a linear system involving $K_X$.

We will  need some properties of the Reproducing Kernel Hilbert Space $H$ corresponding to the kernel $K$.
In particular, we need the bound $\|f\|_\infty \le \|f\|_H$ and  the reproducing property, $\langle f(\cdot), K(x,\cdot)\rangle_H = f(x), \forall f \in H$. 

\begin{lemma} 
Let $g$ be any positive function with $L_2$ norm $1$ supported on $[0,1]^n$ and let  $f={\cal K}g$.
If fill $X$ is $h$, there exists $A>0$ such that
$$
\| f - f_{X,k}\|_{L^\infty(\R^n)}  < \exp (A \frac{\log h}{h})
$$ 
\end{lemma}
\begin{proof}
From~\cite{rieger2010sampling}, Corollary 5.1 (taking $\lambda=0$)\lnote{is it really 5.1? or tmh 6.1? it seems to be already in the book for lambda = 0} we have that for some $A>0$ and $h$ sufficiently small\lnote{what's the use of the l infinity estimate here?}
$$
\|f - f_{X,k}\|_{L^\infty([0,1]^n)} < \exp (A \frac{\log h}{h})
$$
and 
$$
\|f - f_{X,k}\|_{L^2([0,1]^n)} < \exp (A \frac{\log h}{h})
$$
Note that the norm is on $[0,1]$ while we need to control the norm on $\R$.\lnote{$^n$?}

That, however, does not allow us to directly control the function behavior outside of the interval\lnote{box, you mean?}. To do that we need a bound on the RKHS norm of $f - f_{X,k}$. 
We first observe for any $x_i \in X$, $f(x_i) - f_{X,k}(x_i) = 0$. Thus, from the reproducing property of RKHS, $\langle f - f_{X,k}, f_{X,k}\rangle_H = 0$. Using properties of RKHS with respect to the operator ${\cal K}$ (see, e.g., Proposition 10.28 of ~\cite{wendland2005scattered})\lnote{you really mean $L^2(X)$ here?}
\begin{align*}
\|f - f_{X,k}\|_H^2 
&= \langle f - f_{X,k}, f - f_{X,k}\rangle_H \\
&= \langle f - f_{X,k}, f \rangle_H \\
&= \langle f - f_{X,k}, {\cal K}g\rangle_H \\
&= \langle f - f_{X,k}, g \rangle_{L_2([0,1]^n)} \\
&\le  \|f - f_{X,k}\|_{L^2(X)} \|g\|_{L^2(X)} < \exp (A \frac{\log h}{h}) 
\end{align*}
Thus\lnote{}
$$
\|f - f_{X,k}\|_{L^\infty(\R^n)} \le \|f - f_{X,k}\|_H < \exp (A \frac{\log h}{h}) 
$$
\end{proof}

\begin{theorem}\label{thm:differentmixtures}
Let   $X$ and $Y$ be any two subsets of $[0,1]^n$ with fill $h$. 
Then there exist two Gaussian mixtures $p$ and $q$ (with positive coefficients summing to one, but not necessarily the same number of components), which are centered on
two non-intersecting\lnote{do you mean disjoint?} subsets of $X\cup Y$ and such that 
$$
\|p-q\|_{L^1(\R^n)} < \exp(B \frac{\log h}{h})
$$
for some constant $B>0$.
 
\end{theorem}
\begin{proof}
To simplify the notation we assume that $n=1$.  The general case follows verbatim, except that the interval 
of integration, $[-1/h,1/h]$, and its complement need to be replaced by the sphere of radius $1/h$ and its complement respectively.

Let  $f_{X,k}$ and $f_{Y,k}$ be the interpolants  as above\lnote{for what g or f?}. 
We see that $\|f_{X,k} - f_{Y,k}\|_{L^\infty(\R)} < 2\exp (A \frac{\log h}{h}) $. $f_{X,k}$ and $f_{Y,k}$ are both linear combinations of Gaussians possibly with negative coefficients and so is $f_{X,k} -  f_{Y,k}$ .  By collecting positive and negative coefficients  we write  $f_{X,k} -  f_{Y,k} = p_1 - p_2$, where, $p_1$ and $p_2$ are mixtures with positive coefficients only.

Put $p_1 = \sum_{i \in S_1} \alpha_i K(x_i, x)$, $p_2 = \sum_{i \in S_2} \beta_i K(x_i, x)$, where $S_1$ and $S_2$ are non-overlapping subsets of $X \cup Y$. Now we need to ensure that the coefficients can be normalized to sum to $1$.

Let $\alpha = \sum \alpha_i$, $\beta = \sum \beta_i$. By integrating over the interval $[0,1]$, and since $f$ is strictly positive on the interval,  it is easy to see that $\alpha >C$ (and by the same token $\beta>C$), where $C$ is some universal constant\lnote{???}. We have\lnote{is there a constant missing in the first equality? or the kernel is normalized somehow?}
 $$
|\alpha - \beta| = \left | \int_{\R} p_1(x) - p_2(x) dx \right| \le \|p_1 - p_2\|_{L^1} 
$$
$$
\|p_1 - p_2\|_{L^1}  \le \int_{[-1/h,1/h]} \|f_{X,k} - f_{Y,k}\|_{L^\infty(\R)} dx + (\alpha + \beta) \int_{x \notin [-1/h,1/h]} K(0,x-1)dx
$$

Noticing that the first summand is bounded by $\frac{2}{h} \exp (A \frac{\log h}{h})$ and the integral in the second summand is smaller than $e^{-1/h}$ for $h$ sufficiently small, 
it follows immediately, that $|1 - \frac{\beta}{\alpha}| < \exp(A' \frac{\log h}{h})$ \lnote{I don't think you get the log here (because of the second term)}for some $A'$ and $h$ sufficiently small.

Hence,  we have
$$
\left \|\frac{1}{\alpha} p_1 - \frac{1}{\beta}p_2 \right\|_{L^{1}} \le 
\frac{1}{C}\left \|\frac{\beta}{\alpha} p_1 -  p_2 \right\|_{L^{1}} \le \frac{1}{C}\left |1 - \frac{\beta}{\alpha}\right| \|p_1\|_{L^1}+
\frac{1}{C}\left \| p_1 -  p_2 \right\|_{L^{1}} 
$$

Collecting exponential inequalities completes the proof.

\end{proof}

\begin{proof}[Theorem \ref{thm:low-dim-identifiability}]
For convenience we will use a set of $4k^2$ points instead of $k^2$. Clearly it does not affect the exponential rate. 
By a simple covering set argument (cutting the cube into $m^n$ cubes with size $1/m$) and basic probability (the coupon collector's problem), we see that the fill $h$ of $2 n m^n \log m$ points is at most $O(\sqrt{n}/m)$ with probability $1-o(1)$\details{Coupon collector problem bound used (pg 59 Motwani and Raghavan)}. 
Hence, given $k$ points, we have $h = O(\sqrt{n}(\frac{\log k}{k})^{1/n})$. 
We see that with a smaller probability (but still close to $1$ for large $k$), we can sample $k$ points $4k$ times and still have the same fill on each group of $k$\lnote{union bound? what if $k$ depends on $n$?}.

Pairing the sets of $k$ points into $2k$ pairs of sets arbitrarily and applying Theorem \ref{thm:differentmixtures} (to $k+k$ points) we obtain $2k$ pairs of exponentially close mixtures with at most $2k$ components each. 
If one of the pairs has the same number of components, we are done. If not, 
by the pigeon-hole principle for at least two pairs of mixtures $p_1 \approx q_1$ and $p_2 \approx q_2$ the differences of the number of components (an integer  number between $0$ and $2k-2$) must coincide.  
Assume without loss of generality that $p_1$ has no more components than $q_1$ and $p_2$ has no more components than $q_2$.Taking $p = \frac{1}{2}(p_1 + q_2)$ and $q= \frac{1}{2}(p_2 + q_1)$ completes the proof.
\end{proof}


\section{Recovery of Gaussian Weights} \label{sec:WeightRecovery}

\paragraph{Multivariate cumulant tensors and their properties.}
Our technique for the recovery of the Gaussian weights relies on the tensor properties of multivariate cumulants that have been used in the ICA literature.

Given a random vector $Y \in \R^n$, the moment generating function of $Y$ is defined as $M_Y(t) := \mathbb{E}_Y(\exp(t^TY))$.
The \textit{cumulant generating function} is the logarithm of the moment generating function: $g_{\ds Y}(t) := \log(\mathbb{E}_{\ds Y}(\exp(t^T Y))$.


Similarly to the univariate case, multivariate cumulants are defined using the coefficients of the Taylor expansion of the cumulant generating function.
We use both $\cumtns{Y}{j_1, \dots, j_\ell}$ and $\cum(Y_{j_1}, \dots, Y_{j_\ell})$ to denote the order-$\ell$ cross cumulant between the random variables $Y_{j_1}, Y_{j_2}, \dots, Y_{j_\ell}$.
Then, the cross-cumulants $\cumtns{Y}{j_1, \dots, j_\ell}$ 
are given by the formula $\cumtns{Y}{j_1, \dots, j_\ell} = \frac{\partial}{\partial t_{j_1}} \cdots \frac{\partial}{\partial t_{j_\ell}} g_{\ds Y}(t) \big|_{t=0}$.
When unindexed, $\cumtns{Y}{}$ will denote the full order-$\ell$ tensor containing all cross-cumulants, with the order of the tensor being made clear by context.
In the special case where $j_1 = \cdots = j_\ell = j$, we obtain the order-$\ell$ univariate cumulant $\cum_{\ell}(Y_j) = \cumtns{Y}{j, \dots, j}$ ($j$ repeated $\ell$ times) previously defined.
We will use some well known properties of multivariate cumulants, found in Appendix~\ref{sec:cumulant-properties}.

The most theoretically justified ICA algorithms have relied on the tensor structure of multivariate cumulants, including
the early, popular practical algorithm JADE \cite{CardosoS93}.
In the fully determined ICA setting in which the number source signals does not exceed the ambient dimension, the papers \cite{AroraGMS12} and \cite{Belkin2012} demonstrate that ICA with additive Gaussian noise can be solved in polynomial time and using polynomial samples.  The tensor structure of the cumulants was (to the best of our knowledge) first exploited in \cite{FOOBI} and later in \cite{BIOME} to solve underdetermined ICA.
Finally, \cite{GVX} provides an algorithm with rigorous polynomial time and sampling bounds \vnote{same exponential caveats as our proposed technique in this paper} for underdetermined ICA in the presence of Gaussian noise.

\paragraph{Weight recovery (main idea).}
Under the basic ICA reduction (see section~\ref{sec:gmm-reduction}) using the Poisson distribution with parameter $\lambda$, we have that $X = AS + \eta$ is observed such that $A = [\mu_1 | \cdots | \mu_m]$ and $S_i \sim \Poisson (w_i \lambda)$.
As $A$ has already been recovered, what remains to be recovered are the weights $w_1, \cdots, w_m$.
These can be recovered using the tensor structure of higher order cumulants.
The critical relationship is captured by the following Lemma:
\begin{lemma} \label{lem:cum_COV}
  Suppose that $X = AS + \eta$ gives a noisy ICA model.
  When $\cumtns{X}{}$ is of order $\ell > 2$, then $ \vec{\cumtns{X}{}} = A^{\odot \ell} (\cum_\ell(S_1), \dotsc, \cum_\ell(S_m))^T$.
\end{lemma}
\begin{proof}
  It is easily seen that the Gaussian component has no effect on the cumulant:
  \begin{equation*}
    \cumtns{X}{} = \cumtns{AS + \eta}{} = \cumtns{AS}{}  + \cumtns{\eta}{}
      = \cumtns{AS}{}
  \end{equation*}
  Then, we expand $\cumtns X {}$:
  \begin{align*}
    \cumtns{X}{i_1, \cdots, i_\ell} &= \cumtns{AS}{i_1, \cdots, i_\ell} =\cum((AS)_{i_1}, \cdots, (AS)_{i_\ell}) \\
    &= \cum\left(\sum_{j_1 = 1}^m A_{i_1j_1}S_{j_1}, \cdots, \sum_{j_\ell = 1}^m A_{i_\ell j_\ell}S_{j_\ell}\right) \\
    &= \sum_{j_1, \cdots, j_\ell \in [m]} \left(\prod_{k=1}^\ell A_{i_kj_k}\right) \cum(S_{j_1}, \cdots, S_{j_\ell}) & \text{by multilinearity}
  \end{align*}
  But, by independence, $\cum(S_{j_1}, \cdots, S_{j_m}) = 0$ whenever $j_1 = j_2 = \cdots = j_\ell$ fails to hold.
  Thus, 
  \begin{align*}
    \cumtns{X}{i_1, \cdots, i_\ell} &= \sum_{j=1}^m \left(\prod_{k=1}^\ell A_{i_kj}\right) \cum_\ell(S_j) = \sum_{j=1}^m \big((A_j)^{\otimes \ell}\big)_{i_1, \cdots, i_\ell} \cum_\ell(S_j)
  \end{align*}
  Flattening yields: $\vec{\cumtns{X}{}} = A^{\odot \ell}(\cum_\ell(S_1), \cdots, \cum_\ell(S_m))^T$.
\end{proof}



In particular, we have
that $S_i \sim \Poisson(w_i \lambda)$ with $w_i$ the probability of sampling from the $i$\textsuperscript{th} Gaussian.  Given knowledge of $A$ and the cumulants
of the Poisson distribution, we can recover the Gaussian weights.
  \begin{theorem}
    Suppose that $X = AS + \eta(\tau)$ is the unrestricted noisy ICA model from the basic reduction (see section~\ref{sec:gmm-reduction}).
    Let $\ell > 2$ be such that $A^{\odot \ell}$ has linearly independent columns, 
and let $(A^{\odot \ell})^\dagger$ be its Moore-Penrose pseudoinverse.  
    Let $\cumtns{X}{}$ be of order $\ell$.
    Then $\frac 1 \lambda (A^{\odot \ell})^\dagger \vec{\cumtns{X}{}}$ is the vector of mixing weights $(w_1, \ldots, w_m)^T$ of the Gaussian mixture model.
  \end{theorem}
  \begin{proof}
  From Lemma \ref{lem:PoissonCumulants}, $\cum_{\ell}(S_i) = \lambda w_i$.  
  Then Lemma~\ref{lem:cum_COV} implies that $\vec{\cumtns{X}{}} = \lambda A^{\odot \ell} (w_1, \dotsc, w_m)^T$.
  Multiplying on the left by $\frac 1 \lambda (A^{\odot \ell})^\dagger$ gives the result.
  \end{proof}


\section{Addendum}

\subsection{Properties of Cumulants}\label{sec:cumulant-properties}

The following properties of multivariate cumulants are well known and are largely inherited from the definition of the cumulant generating function:
\begin{itemize}
  \item (Symmetry) Let $\sigma$ give a permutation of $k$ indices.  
  Then, $\cumtns Y {i_1, \cdots, i_\ell} = \cumtns Y {\sigma(i_1), \cdots, \sigma(i_\ell)}$.
  \item (Multilinearity of coordinate random variables)
  Given constants $\alpha_1, \cdots, \alpha_\ell$, then
  \[
    \cum(\alpha_1 Y_{i_1}, \cdots, \alpha_\ell Y_{i_\ell}) = 
      \left(\prod_{i=1}^\ell \alpha_i\right)\cum(Y_{i_1}, \cdots, Y_{i_\ell}) \ .
  \]
  Also, given a scalar random variable $Z$, then
  \[
    \cum(Y_{i_1}+Z, Y_{i_2}, \cdots, Y_{i_\ell}) = 
      \cum(Y_{i_1}, Y_{i_2}, \cdots, Y_{i_\ell}) + \cum(Z, Y_{i_2}, \cdots, Y_{i_\ell})
  \]
  with symmetry implying the additive multilinear property for all other coordinates.
  \item (Independence)
  If there exists $i_j, i_k$ such that $Y_{i_j}$ and $Y_{i_k}$ are independent
  random variables, then the cross-cumulant $\cumtns Y {i_1, \cdots, i_\ell} = 0$.
  Combined with multilinearity, it follows that when there are two independent random vectors $Y$ and $Z$, then $\cumtns {Y+Z} {} = \cumtns Y {} + \cumtns Z {}$.
  \item (Vanishing Gaussians)
  When $\ell \geq 3$, then for the Gaussian random variable $\eta$, $\cumtns \eta {} = 0$.
\end{itemize}

\subsection{Rudelson-Vershynin subspace bound}\label{subsec:rudelson-vershynin}

\begin{lemma}[\processifversion{vstd}{Rudelson--Vershynin~}\cite{RudelsonVershynin}] \label{lem:RudelsonVershynin}
If $A \in \R^{n \times m}$ has columns $C_1, \ldots, C_m$, then denoting $C_{-i} = \spn{C_j: j \neq i}$, we have 

\begin{align*}
\frac{1}{\sqrt{m}} \min_{i \in [m]}\dist(C_i, C_{-i}) \leq \sigma_{\min}(A),
\end{align*}

where as usual $\sigma_{\min}(A) = \sigma_{\min(m,n)}(A)$.
\end{lemma}

\subsection{Carbery-Wright anticoncentration}\label{subsec:carbery-wright}

The version of the anticoncentration inequality we use is explicitly given in \cite{MOS} which
in turn follows immediately from \cite{CarberyWright}:

\begin{lemma}[\cite{MOS}]\label{lem:CarberyWright}
Let $Q(x_1, \ldots, x_n)$ be a multilinear polynomial of degree $d$. 
Suppose that $\Vr{Q} = 1$ when $x_i \sim \Normal(0,1)$ for all $i$. Then there exists an absolute constant $C$
such that for $t \in \R$ and $\epsilon > 0$,
\begin{align*}
\Pr_{(x_1, \ldots, x_n) \sim \Normal(0,I_n)}(\abs{Q(x_1, \ldots, x_n)-t} \leq \epsilon) \leq C d \epsilon^{1/d}. 
\end{align*}
\end{lemma}
\subsection{Lemmas on the Poisson Distribution} \label{sec:Poisson-lemmas}
The following lemmas are well-known; see, e.g., \cite{bookDasgupta}. We provide proofs for 
completeness.
\begin{lemma} \label{lem:appendix-Poisson-basic}
If $X \sim \Pois(\lambda)$ and $Y |_{X=x} \sim \Bin(x,p)$ then $Y \sim \Pois(p \lambda)$.
\end{lemma}

\begin{proof}
\begin{align*}
\prob{Y=y}
&= \sum_{x: x \geq y}^{\infty} \prob{Y = y \suchthat X = x} \prob{X = x} \\
&= \sum_{x: x \geq y}^{\infty} {x \choose y} p^y (1-p)^{x-y} \frac{\lambda^x e^{-\lambda}}{x!} \\
&= p^y e^{-\lambda} \sum_{x: x \geq y}^{\infty} \frac{\lambda^x}{x!} {x \choose y} (1-p)^{x-y} \\
&= \frac{(p \lambda)^y e^{-\lambda}}{y!} \sum_{x: x \geq y}^{\infty} \frac{(\lambda(1-p))^{x-y}}{(x-y)!} \\
&= \frac{(p \lambda)^y e^{-\lambda}}{y!} e^{(1-p)\lambda} \\
&= \frac{(p \lambda)^y e^{-p \lambda}}{y!}.
\end{align*}
\end{proof}

\begin{lemma} \label{lem:appendix-poisson-independence}
Fix a positive integer $k$, and let $p_i \geq 0$ be such that $p_1 + \dotsb + p_k =1$. If $X \sim \Pois(\lambda)$ and
$(Y_1, \ldots, Y_k)|_{X = x} \sim \Multinom(x; p_1, \ldots, p_k)$ then $Y_i \sim \Pois(p_i \lambda)$ for all $i$ and
$Y_1, \ldots, Y_k$ are mutually independent.
\end{lemma}

\begin{proof}
  The first part of the lemma (i.e., $Y_i \sim \Pois(p_i \lambda)$ for all $i$) follows from Lemma~\ref{lem:appendix-Poisson-basic}.
For the second part, let's prove it for the binomial case ($k=2$); the general case is similar.
\begin{align*}
\prob{Y_1 = y_1, Y_2 = y_2} &= \prob{Y_1 = y_1, Y_2 = y_2 \suchthat X = y_1+y_2} \prob{X = y_1+y_2} \\
&= {y_1+y_2 \choose y_1} p^{y_1}(1-p)^{y_2} \cdot \frac{\lambda^{y_1+y_2} e^{-\lambda}}{(y_1+y_2)!} \\
&= \frac{(p\lambda)^{y_1} e^{-p\lambda}}{y_1!} \cdot \frac{((1-p)\lambda)^{y_2} e^{-(1-p)\lambda}}{y_2!} \\
&= \prob{Y_1 = y_1} \cdot \prob{Y_2 = y_2}.
\end{align*}
\end{proof}

\subsection{Bounds on Stirling Numbers of the Second Kind}
The following bound comes from \cite{Stirling} Theorem 3.
\begin{lemma} \label{lem:StirlingBound_nr}
  If $n \geq 2$ and $1 \leq r \leq n-1$ are integers, then $\stirling{n}{r} \leq \frac 1 2 {n \choose r} r^{n-r}$.
\end{lemma}

From this, we can derive a somewhat looser bound on the Stirling numbers of the second kind which does not depend on $r$:
\begin{lemma} \label{lem:StirlingBound_n}
  If $n, r \in \Z^+$ such that $r \leq n$, then $\stirling{n}{r} \leq n^{n-1}$.
\end{lemma}
\begin{proof}
  The Stirling number $\stirling{n}{k}$ of the second kind gives a count of the number of ways of splitting a set of $n$ labeled objects into $k$ unlabeled subsets.
  In the case where $r = n$, then $\stirling{n}{r} = 1$  
  As $n \geq 1$, it is clear that for these choices of $n$ and $r$, $\stirling{n}{r} \leq n^{n-1}$.
  By the restriction $1 \leq r \leq n$, when $n=1$, then $n=r$ giving that $\stirling{n}{r} = 1$.  As such, the only remaining cases to consider are when $n \geq 2$ and $1 \leq r \leq n-1$, the cases where Lemma \ref{lem:StirlingBound_nr} applies.
  
  When $n \geq 2$ and $1 \leq r \leq n-1$, then
  \begin{align*}
\stirling{n}{r} &\leq \frac 1 2 {n \choose r}r^{n-r} = \frac 1 2 \frac {n!}{r!(n-r)!} r^{n-r}
      \leq \frac 1 2 n^{r}r^{n-r-1} < \frac 1 2 n^r n^{n-r-1}
       = \frac 1 2 n^{n-1} \ ,
  \end{align*}
  which is slightly stronger than the desired upper bound.
\end{proof}

\subsection{Values of Higher Order Statistics}\label{sec:GaussMoments} \label{app:HOS}

In this appendix, we gather together some of the explicit values for higher order
statistics of the Poisson and Normal distributions required for the analysis
of our reduction from learning a Gaussian Mixture Model to learning an ICA model
from samples.

\begin{lemma}[Cumulants of the Poisson distribution] \label{lem:PoissonCumulants}
  Let $X \sim \Poisson(\lambda) $.  
  Then, $\cum_{\ell}(X) = \lambda$ for every positive integer $\ell$.
\end{lemma}
\begin{proof}
  The moment generating function of the Poisson distribution is given by $M(t) = \exp(\lambda (e^{t} - 1))$.
  The cumulant generating function is thus $g(t) = \log(M(t)) = \lambda(e^{t} - 1)$.
  The $\ell$\textsuperscript{th} derivative $(\ell \geq 1)$ is given by $g^{(\ell)}(t) = \lambda e^{t}$.
  
  By definition, $\cum_{\ell}(X) = g^{(\ell)}(0) = \lambda$.
\end{proof}

\begin{lemma}[Absolute moments of the Gaussian distribution] \label{lem:GaussianMoments}
For a Gaussian random variable $\eta \sim N(0, \sigma^2)$\anonnote{xxx}, the absolute moments of $\eta$ are given by:
\[
\E{\abs{\eta}^\ell} =  
\begin{cases}
    	\sigma^\ell \frac{\ell!}{2^{\frac \ell 2}(\frac \ell 2)!} & \text{if $\ell$ is even} \\
    	\sigma^\ell 2^{\frac \ell 2}(\frac{\ell-1}2)! \frac 1 {\sqrt \pi } & \text{if $\ell$ is odd}.
\end{cases}
\]
\end{lemma}
The case that $\ell$ is even in Lemma \ref{lem:GaussianMoments} is well known, and can be found for instance in \cite[Section 3.4]{Kendall94}. For general $\ell$, it is known (see \cite{Winkelbauer2012}) that 
  \[
    \E{\abs{\eta}^{\ell}} = \sigma^\ell 2^{\frac \ell 2} \Gamma\left(\frac {\ell+1} 2 \right) \frac 1 {\sqrt \pi }  \ .
  \]
  When $\ell$ is odd, $\frac{\ell + 1} 2$ is an integer, allowing the Gamma function to simplify to a factorial:  $\Gamma\left(\frac {\ell + 1} 2\right) = \left(\frac{\ell - 1}{2}\right)!$.
  This gives the case where $\ell$ is odd in Lemma \ref{lem:GaussianMoments}.

\subsection{Total Variation Distance} \label{app:total-variation}

Total variation is a type of statistical distance metric between probability distributions.
In words, the total variation between two measures is the largest difference between the measures on a single event.
Clearly, this distance is bounded above by 1.

For probability measures $F$ and $G$ on a sample space $\Omega$ with sigma-algebra $\Sigma$, the total variation is denoted and defined as:
$$\d_{TV}(F,G) := \sup_{A \in \Sigma} |F(A) - G(A)|. $$

Equivalently, when $F$ and $G$ are distribution functions having densities $f$ and $g$, respectively, $$ \d_{TV}(F,G) = \frac{1}{2} \int_{\Omega} |f - g| d\mu $$
where $\mu$ is an arbitrary positive measure for which $F$ and $G$ are absolutely continuous.

More specifically, when $F$ and $G$ are discrete distributions with known densities, we can write $$ \d_{TV}(F,G) = \frac{1}{2} \sum_{k=0}^{\infty} |f(k) - g(k)| $$
where we choose $\mu$ that simply assigns unit measure to each atom of $\Omega$ (in this case, absolute continuity is trivial since $\mu(A) = 0$ only when $A$ is empty and thus $F(A)$ must also be 0). For more discussion, one can see Definition 15.3 in \cite{nielsen1997introduction} and Sect. 11.6 in \cite{royden1988real}.


%
%

\appendix

%
%

\bibliographystyle{plain}
\bibliography{IEEEabrv,ICA_bibliography,GMM_bibliography,proposal}

\end{document}